\documentclass[11pt,a4paper]{article}
\usepackage[margin=0.75in]{geometry}
\usepackage[utf8]{inputenc}
\usepackage{cmap}
\usepackage[T1]{fontenc}
\usepackage{lmodern}
\usepackage[square, numbers]{natbib}     %[square, numbers]
\usepackage[colorlinks=true]{hyperref}
\usepackage{amssymb, amsthm, amsmath, amsfonts}
\usepackage{authblk, dsfont, mathrsfs, color, bm, enumitem, mathtools}
\usepackage{subfigure, graphicx, transparent}
\usepackage{dirtytalk}
\graphicspath{{figures/}}
\usepackage[dvipsnames]{xcolor}
\usepackage{comment}
\hypersetup{colorlinks, linkcolor={red!80!black}, citecolor={blue!80!black}, urlcolor={ForestGreen}}
\theoremstyle{definition}
\newtheorem{theorem}{Theorem}[section]

\newtheorem{assumption}[theorem]{Assumption}
\newtheorem{proposition}[theorem]{Proposition}
\newtheorem{corollary}[theorem]{Corollary}
\newtheorem{lemma}[theorem]{Lemma}
\newtheorem{remark}[theorem]{Remark}
\newtheorem{example}[theorem]{Example}
\newtheorem{fact}[theorem]{Fact}

\DeclareMathOperator*{\argmin}{arg\,min}

\DeclareUnicodeCharacter{211D}{\ensuremath{\mathbb R}}

\DeclarePairedDelimiterX{\inner}[2]{\langle}{\rangle}{#1, #2}

\title{Wasserstein Distributionally Robust Estimation in High Dimensions: Performance Analysis and Optimal Hyperparameter Tuning}
\author[1]{Liviu Aolaritei}
\author[2]{Soroosh Shafiee}
\author[1]{Florian Dörfler}

\affil[1]{Automatic Control Lab, ETH Z\"urich, Switzerland \protect\\ \texttt{\{aliviu,dorfler\}@ethz.ch}}
\affil[2]{Operations Research and Information Engineering, Cornell University, USA \texttt{shafiee@cornell.edu}}

\date{}

\begin{document}
\maketitle

\begin{abstract}
Distributionally robust optimization (DRO) has become a powerful framework for estimation under uncertainty, offering strong out-of-sample performance and principled regularization. In this paper, we propose a DRO-based method for linear regression and address a central question: how to optimally choose the robustness radius, which controls the trade-off between robustness and accuracy. Focusing on high-dimensional settings where the dimension and the number of samples are both large and comparable in size, we employ tools from high-dimensional asymptotic statistics to precisely characterize the estimation error of the resulting estimator. Remarkably, this error can be recovered by solving a simple convex-concave optimization problem involving only four scalar variables. This characterization enables efficient selection of the radius that minimizes the estimation error. In doing so, it achieves the same effect as cross-validation, but at a fraction of the computational cost. Numerical experiments confirm that our theoretical predictions closely match empirical performance and that the optimal radius selected through our method aligns with that chosen by cross-validation, highlighting both the accuracy and the practical benefits of our approach.
\end{abstract}

\section{Introduction}
\label{sec:intro}

We consider the problem of estimating an unknown parameter $\theta_0 \in \Theta \subseteq \mathbb R^d$ which describes the relationship between two random variables $X$ and $Y$ through the noisy linear model $Y = \theta_0^\top X + Z$. Here, $(X,Y)$ is distributed according to an unknown probability distribution $\mathbb P$ that is only indirectly observable through a set of $n$ independent training samples $\{(x_i,y_i)\}_{i=1}^n$, and $Z$ represents the measurement noise distributed according to an unknown distribution $\mathbb P_Z$. A popular solution to this problem is then to approximate $\mathbb P$ using the $n$ i.i.d.\ samples as $\mathbb P \approx \widehat{\mathbb P}_n = \frac{1}{n} \sum_{i=1}^{n} \delta_{(x_i,y_i)}$, leading to an estimate of the unknown $\theta_0$ via the following empirical risk minimization (ERM) problem
\begin{align*}
     \min_{\theta \in \Theta}\; \mathbb E_{(X,Y)\sim\widehat{\mathbb P}_n} \left[ L(Y-\theta^\top X) \right] = \min_{\theta \in \Theta}\;  \frac{1}{n}\sum_{i=1}^{n}L(y_i - \theta^\top x_i),
\end{align*}
for some loss function $L:\mathbb R \to \mathbb R_{>0}$. However, $\widehat {\mathbb P}_n$ invariably differs from~$\mathbb P$, and optimizing in view of~$\widehat {\mathbb P}_n$ instead of~$\mathbb P$ may lead to estimators that display a poor performance on test data. For example, it is well known that standard machine learning models trained in view of the empirical distribution $\widehat {\mathbb P}_n$ can be easily fooled by adversarial examples, that is, test samples subject to seemingly negligible noise that cause these models to make a wrong prediction. Even worse, the decision problem at hand could suffer from a \emph{distribution shift}, that is, the training data may originate from a distribution other than~$\mathbb P$, under which estimators are evaluated.

A promising strategy to mitigate these issues would be to minimize the worst-case expected loss with respect to all distributions in some neighborhood of~$\widehat{\mathbb P}_n$. There is ample evidence that, for many natural choices of the neighborhood of~$\widehat{\mathbb P}_n$, this distributionally robust approach leads to tractable optimization models and provides a simple means to derive powerful generalization bounds \citep{esfahani2018data, lam2019recovering, blanchet2019confidence, van2021data, duchi2020learning, gao2020finite, duchi2021statistics,shafieezadeh2023new}. Specialized distributionally robust approaches may even enable generalization in the face of domain shifts \citep{farnia2016minimax, volpi2018generalizing,  lee2018minimax, duchi2020learning} or may make the training of deep neural networks more resilient against adversarial attacks~\citep{sinha2018certifying, wang2019convergence, tu2019theoretical, kwon2020principled}. 

Motivated by this, in this paper we propose a distributionally robust approach to estimate the unknown parameter $\theta_0$, which comes in the form of a minimax optimization problem where the maximization robustifies against perturbations (in the probability space) around the empirical distribution $\widehat{\mathbb P}_n$. These perturbations are captured through an ambiguity set $\mathbb B_{\varepsilon}(\widehat{\mathbb P}_n)$, that is, an $\varepsilon$-neighborhood of $\widehat{\mathbb P}_n$ with respect to an optimal transport discrepancy on the probability space $\mathcal P(\mathbb R^{d+1})$. Then, the \textit{distributionally robust estimation} (DRE) problem of interest becomes 
\begin{align}
\label{eq:dro}
\min_{\theta \in \Theta}\; \sup_{\mathbb Q \in \mathbb B_{\varepsilon}(\widehat{\mathbb P}_n)}\; \mathbb E_{(X,Y) \sim \mathbb Q} \left[ L(Y-\theta^\top X) \right].
\end{align}

Problem~\eqref{eq:dro} can be viewed as a zero-sum game between a statistician, who chooses the estimator~$\theta$, and a fictitious adversary, often envisioned as `nature', who chooses the distribution~$\mathbb Q$ of~$(X,Y)$.
Throughout the paper, we denote an optimal solution of \eqref{eq:dro} by $\hat{\theta}_{\mathrm{DRE}}$. From now on, we define nature's feasible set $\mathbb B_{\varepsilon}(\widehat{\mathbb P}_n) = \{ \mathbb Q \in \mathcal P(\mathbb R^{d+1}): W_p ( \mathbb Q , \widehat{\mathbb P}_n ) \leq \varepsilon \}$ using the optimal transport discrepancy
\begin{equation}
\label{eq:optimal:transport}
W_p \left( \mathbb Q, \widehat{\mathbb P}_n \right) = \inf_{\gamma \in \Gamma(\mathbb Q, \widehat{\mathbb P}_n)}\; \mathbb E_{((X_1,Y_1),(X_2,Y_2))\sim\gamma} \left[\|X_1-X_2\|^p + \infty |Y_1-Y_2|\right],
\end{equation}
where $\|\cdot\|$ represents the standard Euclidean norm, and $\Gamma(\mathbb Q, \widehat{\mathbb P}_n)$ represents the set of all joint probability distributions of~$(X_1,Y_1)$ and~$(X_2,Y_2)$ with marginals~$\mathbb Q$ and~$\widehat{\mathbb P}_n$, respectively. In particular, we will concentrate on the cases $p=1,2$, which are the most important in applications. To avoid confusion, in what follows we will refer to $W_1$ as the \emph{type-$1$ Wasserstein} discrepancy and to $W_2$ as the \emph{type-$2$ Wasserstein} discrepancy. Notice that the underlying transportation cost $\|x_1-x_2\|^p + \infty |y_1-y_2|$ on $\mathbb R^{d+1}$ assigns infinite cost when $y_1 \neq y_2$, restricting the distributional robustness to the input distribution, as will be later shown in Lemma~\ref{lemma:properties:ambiguity}. This covers the non-stationary environment known as \emph{covariate shift}, which has been recently attracting significant attention in machine learning \citep{ben2006analysis, sugiyama2007direct}. This transportation cost has been shown to recover some very popular regularized estimators such as square-root LASSO (using $W_2$), as well as regularized logistic regression and support vector machines (using $W_1$) \citep{shafieezadeh2019regularization, blanchet2019robust}. However, for general loss functions, the distributional robustness introduces only an implicit regularization, without having, in general, an equivalent explicit regularized formulation; see \citep{shafieezadeh2023new} for more details on the implicit regularization effect and when it reduces to explicit regularization. 

In the past few years, considerable effort has been put into understanding the statistical guarantees and the regularization effect of problem \eqref{eq:dro}. We summarize the major developments in what follows.
\begin{itemize}
    \item \textbf{Statistical guarantees}. Based on the decay rate of the ambiguity radius $\varepsilon$ (with the number of training samples $n$), many guarantees have been established. First, if $\varepsilon$ is chosen in the order of $n^{-p/\max\{2,d\}}$\footnote{The term $p$ in all the decrease rates appears since we did not employ the $1/p$ exponent in the definition of the optimal transport discrepancy \eqref{eq:optimal:transport}.}, the concentration inequality proposed in \cite[Theorem~2]{Fournier2015} shows that the underlying data-generating distribution $\mathbb P$ is contained in the ambiguity set $\mathbb B_{\varepsilon}(\widehat{\mathbb P}_n)$ with high probability. This has been exploited in \cite{esfahani2018data} to conclude that, for such choice of radius, for any $\theta$, the Wasserstein robust loss is an upper bound on the true loss with high probability. Secondly, through an asymptotic analysis, the works \cite{blanchet2019confidence, blanchet2016sample} show that if $\varepsilon$ is chosen in the order of $n^{-p/2}$, then the Wasserstein ball contains at least one distribution (not necessarily equal to $\mathbb P$) for which there exists an optimal solution which is equal to $\argmin_{\theta \in \Theta} \mathbb E_{(X,Y)\sim \mathbb P} \left[ L(Y-\theta^\top X) \right]$, with high probability. This result is valid asymptotically, namely, as $n \to \infty$, while keeping $d$ fixed. Thirdly, \cite{gao2020finite} shows that a non-asymptotic $n^{-p/2}$ rate for $\varepsilon$ guarantees that, up to a high-order residual, the true loss is upper bounded by the Wasserstein robust loss, uniformly for all $\theta$. Similarly, the analysis carried out in this paper will be in the setting where $\varepsilon$ satisfies the $n^{-p/2}$ decrease rate, i.e., $\varepsilon = \varepsilon_0/(n^{p/2})$, for some constant $\varepsilon_0$. Finally, statistical guarantees have also been proposed in \citep{sinha2018certifying,lee2018minimax,najafi2019robustness} for the setting where the ambiguity radius remains fixed, i.e., it does not vary with the sample size.

    \item \textbf{Variation regularization effect}. For type-$1$ Wasserstein ambiguity sets and convex loss functions, \cite{esfahani2018data} shows that the distributional robustness is equivalent to a Lipschitz norm regularization of the empirical loss. More generally, for the type-$p$ Wasserstein case, \citep{gao2017wasserstein} shows that the inner supremum in the distributionally robust optimization problem \eqref{eq:dro} is closely related (and asymptotically equivalent, as $\varepsilon \to 0$, at the rate $n^{-p/2}$) to a variation regularized estimation problem, where the variation boils down to the Lipschitz norm of the loss, for $p = 1$, or to the empirical gradient norm of the loss, for $p>1$. Particular cases of this general result have also been recovered in \cite{shafieezadeh2015distributionally,shafieezadeh2019regularization,blanchet2019robust}. Moreover, a connection with higher-order variation regularization has been established in \cite{bartl2020robust,shafieezadeh2023new}. From a practical perspective, Lipschitz and gradient regularization have proven themselves very efficient in a variety of applications in machine learning, such as  adversarial robustness of neural networks \citep{finlay2018lipschitz, hein2017formal, goodfellow2014explaining, jakubovitz2018improving, lyu2015unified, ororbia2017unifying, varga2017gradient, finlay2021scaleable}, and stability of generative adversarial networks \citep{roth2017stabilizing, nagarajan2017gradient, gulrajani2017improved}, to name a few. Moreover, high-order variation regularization is used in image recovery applications, to address the shortcomings of the total variation regularization \citep{lefkimmiatis2013hessian}.
\end{itemize}

Different from the existing literature, the focus of this paper is on studying the performance of the distributionally robust estimator $\hat{\theta}_{\mathrm{DRE}}$ in the \emph{high-dimensional regime}, where $d$ is of the same order as (or possibly even larger than) $n$. Driven by the wide empirical success of deep neural networks, the high-dimensional regime has gained tremendous attention in the recent years \citep{belkin2018understand, allen2019learning, mei2022generalization, arora2019fine, belkin2019reconciling}. In this regime, classical statistical asymptotic theory \citep{huber2011robust}, where the sample size $n$ is taken to infinity, while the dimension $d$ is kept fixed, often fails to provide useful predictions, and standard methods generally break down \citep{wainwright2019high}. For instance, it is known that for $d$ fixed or $d/n \to 0$ fast enough, least squares is optimal for Gaussian errors, whereas least absolute deviations is optimal for double-exponential errors. The work \cite{bean2013optimal} shows that this is no longer true in the high-dimensional regime, with the answer depending, in general, on the under/over-parametrization parameter $\rho$, as well as the form of the error distribution. Another example is shown in the work \cite{sur2019modern} where, in the context of logistic regression, the authors show that the maximum-likelihood estimate is biased, and that its variability is far greater than classically estimated. Such observations, as well as many others, have sparked a lot of interest in the modern area of high-dimensional statistics, leading to numerous works trying to unravel and solve these surprising phenomena. Many of these works have focused on linear models, i.e., $y_i = \theta_0^\top x_i + z_i$, for $i=1,\ldots,n$, under the assumption of isotropic Gaussian features $x_i \sim \mathcal N(0,d^{-1} \text{I}_d)$, and have studied different aspects of the problem in the proportional asymptotic regime $d,n \to \infty$, with $d/n \to \rho$; see \citep{donoho2016high, el2018impact, thrampoulidis2018precise} and the references therein. 

This setting has been shown to be both theoretically rich and practically powerful in capturing many interesting phenomena observed in more complex models, as detailed in what follows. A direct connection between linear models and neural networks has been established in many works \citep{jacot2018neural, du2018gradient, allen2019convergence, chizat2019lazy, hastie2022surprises,bartlett2021deep}. One important example is constituted by the \emph{double descent} behavior of neural networks \citep{belkin2019reconciling}, which has been theoretically analyzed through the lenses of linear regression \citep{advani2020high, hastie2022surprises, belkin2020two, bartlett2020benign, muthukumar2020harmless, mei2022generalization}. To the best of our knowledge, the idea of using Gaussian features can be traced back to the work \cite{donoho2009counting} which studied the phase-transition of $\ell_1$-minimization in the \emph{compressed sensing} problem. Remarkably, compared to other assumptions on the features, which allow the characterization of the error performance only up to loose constants \citep{candes2006near}, the isotropic Gaussian features assumption allowed the authors to obtain an asymptotically precise upper bound on the minimum number of measurements which guarantee, with probability $1$, the exact recovery of a structured signal from noiseless linear measurements. Interestingly, for many problems, the i.i.d.\ Gaussian ensembles are known to enjoy a \emph{universality} property, extending the validity of the results to broader classes of probability ensembles \citep{bayati2015universality,oymak2018universality,thrampoulidis2018precise, donoho2009observed,abbasi2019universality,panahi2017universal,hu2020universality,montanari2022universality}.

In this paper, we undertake this setting and focus on studying the performance of the distributionally robust estimator $\hat{\theta}_{\mathrm{DRE}}$ in the high-dimensional asymptotics, where $d, n \to \infty$, with $d/n \to \rho \in (0, \infty)$. Practically, $\rho$ encodes the under-parametrization or the over-parametrization of the problem, corresponding to $\rho \in (0, 1]$ and $\rho \in (1, \infty)$, respectively. The structural constraints on $\theta_0$ are encoded in our analysis by assuming that $\theta_0$ is sampled from a probability distribution $\mathbb P_{\theta_0}$. This modeling choice reflects a common practice in high-dimensional statistics, allowing us to analyze a broad class of structured problems and capture typical-case performance rather than a single fixed instance. We quantify the performance of $\hat{\theta}_{\mathrm{DRE}}$ through the normalized squared error $\|\hat{\theta}_{\mathrm{DRE}} - \theta_0\|^2/d$. Notice that the normalization $d^{-1}$ is necessary when $d \to \infty$, since an arbitrarily small error in each entry of the vector $\hat{\theta}_{\mathrm{DRE}} - \theta_0$ will result in an infinite $\|\hat{\theta}_{\mathrm{DRE}} - \theta_0\|$.

In the spirit of the works described above, the analysis throughout this paper will be done under the following assumptions, which are relatively standard in the high-dimensional statistics literature.
\begin{assumption}[High-dimensional asymptotics] ~
\label{assump:0}
\begin{enumerate}[label=(\roman*)]
    \item \label{assump:features} \textit{Isotropic Gaussian features}: the covariate vectors $x_i$, $i\in \{1,\ldots,n\}$, are i.i.d.\ $\mathcal N(0,d^{-1}\text{I}_d)$.
    \item \label{assump:independence} The true parameter $\theta_0$, the measurement noise $Z$, and the feature vectors $x_i$, $i\in \{1,\ldots,n\}$ are independent random variables.
    \item \label{assump:d:n:rho} The dimension of the problem and the number of measurements go to infinity at a fixed ratio, i.e., $d,n \to \infty$, with $d/n \to \rho \in (0, \infty)$.
\end{enumerate} 
\end{assumption}

As highlighted previously, many results derived for Gaussian ensembles have been rigorously proven and/or empirically observed to enjoy a \textit{universality} property, in the sense that they continue to hold true for a broad family of probability ensembles. In Section~\ref{sec:numerical}, we will validate numerically that the results proposed in this paper also demonstrate such universality. As a consequence, Assumption~\ref{assump:0}\ref{assump:features} introduces a convenient theoretical limitation, which has been shown in many cases to not affect the practical validity of the results in more general scenarios. Moreover, our distributionally robust optimization formulation \eqref{eq:dro} naturally captures all the covariate probability distributions which are ``close'' to $\mathcal N(0,d^{-1}\text{I}_d)$, including more general probability ensembles, as well as features whose entries are correlated. Finally, Assumption~\ref{assump:0}\ref{assump:d:n:rho} is the standard setting in high-dimensional statistics (compared to the classical setting, where $n \to \infty$, and $d/n \to 0$).

These assumptions allow us to use the \textit{Convex Gaussian Minimax Theorem} (CGMT), which is an extension of a classical Gaussian comparison inequality, that under convexity assumptions is shown in \cite{thrampoulidis2015regularized} to become tight. The CGMT will be one of the main technical ingredients in a sequence of steps, at the end of which we will show that the normalized squared error $\|\hat{\theta}_{\mathrm{DRE}} - \theta_0\|^2/d$ converges in probability to a non-trivial deterministic limit (despite both $\hat{\theta}_{\mathrm{DRE}}$ and $\theta_0$ being stochastic), which we recover as the solution of a convex-concave optimization problem which, surprisingly, involves at most four scalar variables. Both the under/over-parametrization parameter $\rho$ and the ambiguity radius $\varepsilon$ appear explicitly in the objective function of this optimization, while the loss function $L$ and the noise distribution $\mathbb P_Z$ will appear implicitly through an \textit{expected Moreau envelope}. Our results build upon the framework proposed in \citep{thrampoulidis2018precise}, which focused on convex-regularized estimation problems. In this paper, we extend their methodology and results to Wasserstein distributionally robust estimation problems.

\begin{figure}
    \centering
    \includegraphics[width=0.49\linewidth]{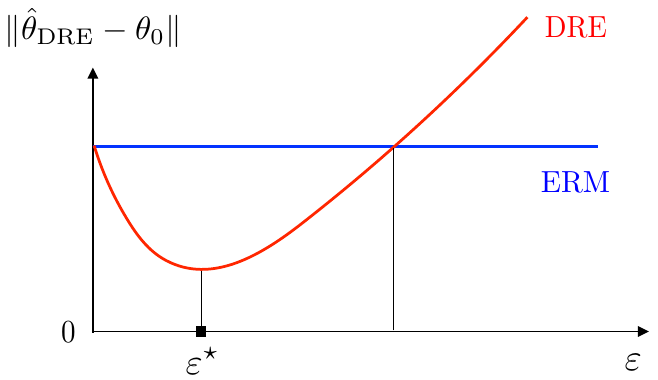}
    \caption{The impact of the ambiguity radius $\varepsilon$ on the estimation error $\|\hat{\theta}_{\mathrm{DRE}} - \theta_0\|$. This illustrative behavior aligns with the numerical results presented in Figures~\ref{figure:high:W1:W2} and \ref{figure:high:universality}. A central goal of this work is to determine, in a computationally efficient manner, the value $\varepsilon^\star$ that minimizes the estimation error in the high-dimensional regime.}
    \label{figue:DREvsERM}
\end{figure}

\paragraph{Contributions.} The main contributions of this paper can be summarized as follows.
\begin{itemize}
    \item[1)] \textbf{Asymptotic error for type-$1$ Wasserstein DRE.}\; In the high-dimensional proportional regime, where both the ambient dimension and the number of samples go to infinity at a proportional rate which encodes the under/over-parametrization of the problem, i.e., $d, n \to \infty$, with $d/n \to \rho \in (0, \infty)$, we show that the normalized squared error $\|\hat{\theta}_{\mathrm{DRE}} - \theta_0\|^2/d$ converges in probability to a deterministic value which can be recovered from the optimal solution of a convex-concave minimax optimization problem that involves four scalar variables.

    \item[2)] \textbf{Asymptotic error for type-$2$ Wasserstein DRE.}\; In the high-dimensional proportional regime, the normalized squared error $\|\hat{\theta}_{\mathrm{DRE}} - \theta_0\|^2/d$ converges in probability to a deterministic value which can be recovered from the optimal solutions of two convex-concave minimax optimization problems that involve at most four scalar variables. Interestingly, the analysis reveals that in order to precisely recover this error, an upper bound on the over-parametrization of the problem $d/n$ should hold. We then adopt a distributionally regularized relaxation of the DRE problem, and show that the normalized squared error can be recovered  from the minimizer of only one (simpler) convex-concave minimax problem, which involves only three scalar variables.

    \item[3)] \textbf{Hyperparameter tuning in high-dimensional DRE problems.}\; We show how the precise quantification of the squared error opens the path to accurate comparisons between different ambiguity radii $\varepsilon$, and to providing an answer to the following fundamental questions: \textit{\say{What is the minimum estimation error achievable by $\hat{\theta}_{\mathrm{DRE}}$?}}, \textit{\say{How to optimally choose $\varepsilon$ in order to achieve the minimum estimation error (see Figure~\ref{figue:DREvsERM})?}}, \textit{``How does the under/over-parametrization  parameter $\rho$ affect the estimation error?''}.
\end{itemize}

Aside from the contribution to the distributionally robust optimization community, we hope that this paper can bring the DRE formulation \eqref{eq:dro} to the attention of the high-dimensional statistics community, whose focus has been so far centered primarily around convex-regularized estimation problems. In particular, this paper complements and contributes to the rapidly expanding research stream which studies the \emph{precise high-dimensional asymptotics} of estimators, and which has so far produced many results in the context of phase transitions in $\ell_1$-minimization \citep{donoho2011noise,stojnic2009various,stojnic2013framework} and convex optimization problems \citep{amelunxen2014living}, LASSO \citep{bayati2011lasso,miolane2021distribution} and generalized LASSO \citep{oymak2013squared,oymak2016sharp,thrampoulidis2015lasso,hu2019asymptotics}, least squares \citep{hastie2022surprises} and convex-regularized least squares \citep{celentano2022fundamental}, estimation \citep{donoho2016high,lei2018asymptotics} and convex-regularized estimation \citep{karoui2013asymptotic, el2018impact,taheri2021fundamental,thrampoulidis2018precise}, generalized linear estimation \citep{gerbelot2020asymptotic,emami2020generalization}, binary models \citep{kini2020analytic,taheri2020sharp,deng2019model,montanari2019generalization,mignacco2020role}, specifically SVM \citep{huang2017asymptotic,kammoun2021precise} and logistic regression \citep{candes2020phase,mai2019large,salehi2019impact,salehi2020performance}, to name a few.

\paragraph{Paper organization.} The paper is organized as follows. In Section~$2$ we introduce the asymptotic CGMT, and discuss how it will be employed in this paper. Then, in Section~$3$ we study the distributionally robust estimation problem, and prove that it can be brought into the form required by the CGMT. Subsequently, in Section~$4$ we present the main results of this paper, which quantify the distributionally robust estimation error in high-dimensions. Finally, in Section $5$, we numerically validate the theoretical findings, and in Section $6$ we discuss some exciting research directions enabled by the results in this paper. All proofs are deferred to the Appendix.

\paragraph{Notation.} Throughout the paper, $\|\cdot\|$ will denote the standard Euclidean norm. For a convex function $f : \mathbb{R} \to \mathbb{R}$, we denote by $\partial f(x)$ the subdifferential of $f$ at $x$. If $f$ is differentiable at $x$, we write $f'(x)$ for its derivative. More generally, we define $f_+'(x) := \sup_{\xi \in \partial f(x)} |\xi|$ to represent the maximum (absolute) subgradient at $x$. Moreover, we denote by $e_f:\mathbb R \times \mathbb R_{> 0} \to \mathbb R$ its Moreau envelope, defined as $e_f (x,\tau) = \min_{\xi \in \mathbb R} |x-\xi|^2/(2 \tau) + f(\xi)$, and by $\text{Lip}(f)$ its Lipschitz constant. The expectation of a random variable $\zeta \sim \mathbb P$ are denoted by $\mathbb E_{\zeta\sim\mathbb P}[\zeta]$, and the probability of an event is denoted by $\text{Pr}(\cdot)$. Given two probability distributions $\mathbb P$, $\mathbb Q$, we denote by $\mathbb P \otimes \mathbb Q$ their product distribution. Given a measurable map $f$, we denote by $f_\# \mathbb{P}$ the pushforward of $\mathbb{P}$ via $f$, defined by $(f_\# \mathbb{P})(\mathcal{A}) = \mathbb{P}(f^{-1}(\mathcal{A}))$ for all measurable sets $\mathcal{A}$. We denote projection maps by $\pi$. Specifically, given the joint distribution $\gamma$ of $(X,Y)$ on $\mathbb R^d \times \mathbb R$, the pushforwards $(\pi_X)_\# \gamma$ and $(\pi_Y)_\# \gamma$ denote the marginal distributions of $\gamma$ on $\mathbb R^d$ and $\mathbb R$, respectively. Finally, we use the notation $X_n \overset{P}{\to} c$ to denote that the sequence of random variables $\{X_n\}_{n \in \mathbb N}$ converges in probability to the constant $c$, i.e., for any $\epsilon > 0$, $\lim_{n \to \infty} \text{Pr}(|X_n-c|>\epsilon) = 0$. In this case we say that $|X_n-c|\leq\epsilon$ holds with probability approaching $1$ (henceforth abbreviated as w.p.a.\ $1$).

%-------------------------------------------------------------------------------------------------
%-------------------------------------------------------------------------------------------------
%-------------------------------------------------------------------------------------------------
\section{Convex Gaussian Minimax Theorem}
\label{sec:cgmt}

The key technical instrument employed in the high-dimensional analysis, where both the number of measurements and the dimension of the problem go to infinity, is the asymptotic version of the Convex Gaussian Minimax Theorem (CGMT), which we will present in Fact~\ref{prop:cgmt}. A self-contained exposition of this result is presented in what follows.

Consider the general problem of analyzing the following \textit{primary optimization} (PO) problem
\begin{align}
\label{eq:cgmt:po}
    \Phi(A) := \min_{w \in \mathcal S_w} \; \max_{u \in \mathcal S_u} \; u^\top A w + \psi(w,u),
\end{align}
where $A$ is a random matrix, with entries i.i.d.\ standard normal $\mathcal N(0,1)$, $\psi:\mathbb R^d \times \mathbb R^n \to \mathbb R$ is a continuous function, and $\mathcal S_w \in \mathbb R^d$, $\mathcal S_u \in \mathbb R^n$ are compact sets. We denote by $w_\Phi(A)$ an optimal solution of problem~\eqref{eq:cgmt:po}.

Analyzing the (PO) problem is in general very challenging, due to the bilinear term which includes the random matrix $A$. The CGMT associates to the (PO) problem a simpler \textit{auxiliary optimization} (AO) problem, from which we can infer properties of the optimal cost and solution of the original (PO) problem. The (AO) problem is a \say{decoupled} version of the (PO) problem, where the bilinear term $u^\top A w$ is replaced by the two terms $\|w\| g^\top u$ and $\|u\| h^\top w$, as follows
\begin{align}
\label{eq:cgmt:ao}
    \phi(g,h) := \min_{w \in \mathcal S_w} \; \max_{u \in \mathcal S_u} \; \|w\| g^\top u + \|u\| h^\top w + \psi(w,u),
\end{align}
where $g$ and $h$ are random vectors with entries i.i.d.\ standard normal $\mathcal N(0,1)$, and $\|\cdot\|$ is the standard Euclidean norm. We denote by $w_\phi(g,h)$ an optimal solution of problem~\eqref{eq:cgmt:ao}. Then, \cite{thrampoulidis2015regularized} shows that an extension of Gordon's Gaussian minimax theorem (GMT, \cite{gordon1988milman}) guarantees that the following inequality between the optimal values of the (PO) and (AO) problems holds
\begin{align}
\label{eq:cgmt:00}
    \text{Pr}(\Phi(A) < c) \leq 2\, \text{Pr}(\phi(g,h) \leq c),
\end{align}
for all $c \in \mathbb R$. Specifically, inequality~\eqref{eq:cgmt:00} says that whenever $\text{Pr}(\phi(g,h) \leq c)$ is close to zero, and therefore $c$ is a high probability lower bound on $\phi(g,h)$, we also have that $c$ is a high probability lower bound on $\Phi(A)$.

Moreover, \cite{thrampoulidis2015regularized} shows that under additional convexity assumptions, namely if $\psi$ is also convex-concave on $\mathcal S_w \times \mathcal S_u$, with $\mathcal S_w \in \mathbb R^d$ and $\mathcal S_u \in \mathbb R^n$ convex compact sets, whenever $c$ is a high probability upper bound on $\phi(g,h)$, we also have that $c$ is a high probability upper bound on $\Phi(A)$, i.e.,
\begin{align}
\label{eq:cgmt:01}
    \text{Pr}(\Phi(A) > c) \leq 2\, \text{Pr}(\phi(g,h) \geq c).
\end{align}
Combining inequalities \eqref{eq:cgmt:00} and \eqref{eq:cgmt:01} leads to the CGMT, which states that under the additional convexity assumptions, the (AO) problem tightly bounds the optimal value of the (PO) problem, in the sense that for all $\mu \in \mathbb R$ and $t > 0$,
\begin{align}
\label{eq:cgmt}
    \text{Pr}(|\Phi(A) - \mu|>t) \leq 2\, \text{Pr}(|\phi(g,h) - \mu| \geq t).
\end{align}
The concentration inequality \eqref{eq:cgmt} shows that if the optimal value $\phi(g,h)$ concentrates around $\mu$ with high probability, and therefore $\text{Pr}(|\phi(g,h) - \mu| \geq t)$ is small, then $\Phi(A)$ concentrates around $\mu$ with high probability. In particular, inequality~\eqref{eq:cgmt} will be essential in the study of the asymptotic (as both $d$ and $n$ go to infinity) properties of the optimal solution $w_\Phi(A)$ of the (PO) problem, based solely on the optimal value $\phi(g,h)$ of the (AO) problem.

In the rest of this section, we focus on explaining how the concentration inequality \eqref{eq:cgmt} will be used in this paper. First of all, we are interested in quantifying the performance of the distributionally robust estimator through the normalized squared error $\|\hat{\theta}_{\mathrm{DRE}} - \theta_0\|^2/d$, therefore we introduce the following change of variables
\begin{align}
\label{eq:theta:to:w}
    w := \frac{\theta-\theta_0}{\sqrt{d}}.
\end{align}
We use the notation $\hat{w}_{\textrm{DRE}}$ to denote the value of $w$ when the estimator $\theta = \hat{\theta}_{\mathrm{DRE}}$ is used in \eqref{eq:theta:to:w}. Now, in order to prove that the norm of $\hat{w}_{\textrm{DRE}}$ (which is random by nature) converges (in probability) to some deterministic value $\alpha_\star$, we will show that w.p.a.\ $1$ as $d,n \to \infty$, we have that for all $\eta > 0$,
\begin{align}
\label{eq:S}
    \hat{w}_{\textrm{DRE}} \in \mathcal S_\eta := \{w \in \mathcal S_w: \; |\|w\|-\alpha_\star|<\eta\}.
\end{align}
 
In Section~\ref{sec:linear}, we will show that the value of $\alpha_\star$, which quantifies the performance of the distributionally robust estimator, can be recovered from the optimal solution of a convex-concave minimax problem which involves at most four scalars. While showing this, the CGMT will play a major role, as explained in what follows. 

We will start by showing that problem \eqref{eq:dro} can be rewritten as a (PO) problem, in the form \eqref{eq:cgmt:po}, with matrix $A$ constructed using the measurement vectors $x_i$. Then, following the reasoning presented thus far, we will formulate the (AO) problem associated to it. Assume now that the set $\mathcal S_u$ takes the form $\mathcal S_u = \{u \in \mathbb R^n:\, \|u\|\leq K_\beta\}$. In Section~\ref{sec:linear}, we will see that this is precisely the case here. Now, in order to arrive at the scalar convex-concave minimax problem, whose optimal solution is $\alpha_\star$, we will need to work with the following slight variation of the (AO) problem \eqref{eq:cgmt:ao},
\begin{align}
\label{eq:cgmt:ao:1}
    \phi'(g,h) := \max_{0 \leq \beta \leq K_\beta}\; \min_{w \in \mathcal S_w} \; \max_{\|u\|\leq \beta} \; \|w\| g^\top u + \|u\| h^\top w + \psi(w,u).
\end{align}

Notice that the two problems \eqref{eq:cgmt:ao} and \eqref{eq:cgmt:ao:1} can be obtained easily from one-another by exchanging the minimization over $w$ and the maximization over $\beta$. However, despite the small difference, the two problems are generally not equivalent. This is due to the fact that, different from the (PO) problem  \eqref{eq:cgmt:po}, the objective function of the (AO) problem \eqref{eq:cgmt:ao} is not convex-concave in $(w,u)$. Indeed, due to the random vectors $g$ and $h$, the terms $\|w\|g^\top u$ and $\|u\|h^\top w$ could be, for example, concave in $w$ and convex in $u$ if $g^\top u$ is negative and $h^\top w$ is positive, respectively, prohibiting the use of standard tools (such as Sion's minimax principle) to exchange the order of the minimization and maximization in \eqref{eq:cgmt:ao:1}. 

Nonetheless, the tight relation between the (PO) problem \eqref{eq:cgmt:po} and the (AO) problem \eqref{eq:cgmt:ao} shown by the CGMT, together with the convexity of the (PO) problem, can be used to asymptotically (as both $d$ and $n$ go to infinity) infer properties of the optimal solution $w_\Phi(A)$ of the (PO) problem, based solely on the optimal value $\phi(g,h)$ of the modified (AO) problem \eqref{eq:cgmt:ao:1}. The following result is instrumental in the high-dimensional error analysis presented in Section~\ref{sec:linear}. %Although it is not new, having previously appeared in a slightly different form in \cite[Lemma~7]{thrampoulidis2018precise}, in order to keep the paper self-contained, a proof is presented in the appendix.

\begin{fact}[Asymptotic CGMT]
\label{prop:cgmt}
Let $\mathcal S$ be an arbitrary open subset of $\mathcal S_w$ and $\mathcal S^c = \mathcal S_w/\mathcal S$. We denote by $\phi'_{\mathcal S^c}(g,h)$ the optimal value of the modified (AO) problem \eqref{eq:cgmt:ao:1} when the minimization over $w$ is restricted as $w \in \mathcal S^c$. If there exist constants $\overline{\phi'}$ and $\overline{\phi'}_{\mathcal S^c}$, with $\overline{\phi'} < \overline{\phi'}_{\mathcal S^c}$, such that $\phi'(g,h) \overset{P}{\to} \overline{\phi'}$ and $\phi'_{\mathcal S^c}(g,h) \overset{P}{\to} \overline{\phi'}_{\mathcal S^c}$ as  $d,n \to \infty$, then
\begin{align}
\label{eq:asympt:cgmt}
    \lim_{d,n \to \infty} \text{Pr}(w_\Phi(A) \in \mathcal S) = 1.
\end{align}
\end{fact}

The core intuition behind Fact~\ref{prop:cgmt} is as follows. Although the (AO) problem~\eqref{eq:cgmt:ao} and its modified version~\eqref{eq:cgmt:ao:1} are not necessarily equivalent, depending on the realizations of $g$ and $h$, Fact~\ref{prop:cgmt} demonstrates that strong duality holds with high probability over the randomness in $g$ and $h$ in high dimensions. Consequently, the modified (AO) problem~\eqref{eq:cgmt:ao:1} can be used as an effective surrogate for the original (AO) problem~\eqref{eq:cgmt:ao} in analyzing the high-dimensional behavior of the (PO) problem~\eqref{eq:cgmt:po}. The asymptotic result in~\eqref{eq:asympt:cgmt} then follows from inequalities~\eqref{eq:cgmt:00} and~\eqref{eq:cgmt:01}, with $\phi(g,h)$ effectively replaced by $\phi'(g,h)$. The proof of this result can be found in \cite[Corollary~1]{thrampoulidis2018precise}. As the proof is illuminating and to keep the paper self-contained, we provide a sketch of its proof in Appendix~\ref{appendix:additional:proofs:2}.

Armed with Fact~\ref{prop:cgmt}, we can conclude the desired result ($\|\hat{w}_{\textrm{DRE}}\| \overset{P}{\to} \alpha_\star$) with $w_\Phi(A) = \hat{w}_{\textrm{DRE}}$ and $\mathcal S = \mathcal S_\eta$ (for all $\eta > 0$). This will necessitate studying the limits $\phi'(g,h) \overset{P}{\to} \overline{\phi'}$ and $\phi'_{\mathcal S^c}(g,h) \overset{P}{\to} \overline{\phi'}_{\mathcal S^c}$ for the modified (AO) problem, required by Fact~\ref{prop:cgmt}, which will be the main technical challenge.

\section{Wasserstein Distributionally Robust Estimation}
\label{sec:dro:M-est}

In this section we study the DRE problem~\eqref{eq:dro}, and prove that it can be brought into the form of a (PO) problem, as required by the CGMT framework. While the distributional robustness induced by $W_1$ corresponds to a Lipschitz regularization, making it possible to rely on results from the high-dimensional analysis of regularized estimators literature, the distributional robustness induced by $W_2$ does not generally have a regularized formulation. Nonetheless, we will present a strong duality result which will enable its representation as a (PO) problem and, with it, the high-dimensional analysis based on the asymptotic CGMT result presented in Fact~\ref{prop:cgmt}.
%In this direction, we will provide a full characterization of the distributions inside the ambiguity set $\mathbb B_\varepsilon(\widehat{\mathbb P}_n)$, prove that (under some assumptions) the optimal value of \eqref{eq:dro} is finite, and we will present a novel dual formulation of \eqref{eq:dro}, which will be in the form needed for the high-dimensional analysis using the Convex Gaussian Minimax framework. 

At different points in the paper, we will make use of the following assumptions on the loss function $L$ and the baseline distribution $\mathbb P$. While reading the assumptions, recall that the two scenarios of interest in this paper are $p=1$ and $p=2$.

\begin{assumption}[DRE regularity assumptions] ~
\label{assump:dro}
\begin{enumerate}[label=(\roman*)]
%    \item \label{assump:theta0} The true parameter $\theta_0$ satisfies $\|\theta_0\|^2/d \overset{P}{\to} k$, for some $k>0$.
	\item \label{assump:ell1} The loss function $L$ is proper, continuous, and convex.
	\item \label{assump:ell2} The loss function $L$ satisfies the growth rate $L(u) \leq C(1 + |u|^p)$, for some $C>0$.
	\item \label{assump:pstar} The distribution $\mathbb P$ has finite $p^\text{th}$ moment, i.e., $\mathbb E_{\mathbb P}\left[ \|(X,Y)\|^p \right]  < \infty$.
	\item \label{assump:ell:diff} For $p=2$, the loss function $L$ is $M$-smooth, for some $M > 0$, i.e., it is differentiable and has a Lipschitz continuous gradient, with Lipschitz constant $M$.	
	%\item \label{assumpt:radius} The ambiguity radius is of the form $\varepsilon = \varepsilon_0/n$, for some $\varepsilon_0$ satisfying $ \varepsilon_0 < n\, d^{-1} M^{-2} R_{\theta}^{-2}\underline{L}$.
\end{enumerate} 
\end{assumption}

These are essential and common assumptions in distributionally robust optimization. Specifically, Assumption~\ref{assump:dro}\ref{assump:ell1} ensures that the inner maximization in \eqref{eq:dro} admits a strong dual, which we will present in Fact~\ref{thm:nonconvexduality}. Moreover, Assumptions~\ref{assump:dro}\ref{assump:ell2}-\ref{assump:pstar} will ensure that the optimal value of \eqref{eq:dro} is finite (see Lemma~\ref{lemma:finiteoptimalvalue:1}). Finally, Assumption~\ref{assump:dro}\ref{assump:ell:diff} will be used for the type-$2$ Wasserstein case, in Lemma~\ref{lemma:convexity}, to show that the dual formulation mentioned above is convex-concave minimax problem.
%Finally, Assumption~\ref{assump:dro}\ref{assumpt:radius} states that the ambiguity radius decreases inversely proportional to the number of measurements $n$, and will be of central importance in both Lemma~\ref{lemma:convexity} and Lemma~\ref{lemma:growthrates}, as well as in the proof of Theorem~\ref{thm:estimationerror:2}.

\begin{lemma}[Finite optimal value of DRE problem]
\label{lemma:finiteoptimalvalue:1}
Let Assumptions~\ref{assump:dro}\ref{assump:ell2}-\ref{assump:pstar} be satisfied. Then the optimal value of \eqref{eq:dro} is finite.
\end{lemma}

Recall from Section~\ref{sec:cgmt} and Fact~\ref{prop:cgmt} that, after rewritting the DRE problem \eqref{eq:dro} as a (PO) problem, and subsequently as an (AO) problem, we can study the optimal solution of the (PO) problem based solely on the optimal value of the (AO) problem. In this stream of reasoning, having a finite optimal value is crucial. The proof of this lemma can be easily recovered from \cite[Proposition~2.5]{shafieezadeh2023new}, and is therefore omitted. Before proceeding to the main results of this section, we present the following lemma, which sheds light on the distributional robustness introduced by the transportation cost in the definition of $\mathbb B_\varepsilon(\widehat{\mathbb P}_n)$.

\begin{lemma}[Ambiguity set description]
\label{lemma:properties:ambiguity}
Let $\mathbb Q$ be a distribution inside the ambiguity set $\mathbb B_\varepsilon(\widehat{\mathbb P}_n)$. Moreover, given $\mathbb Q_1, \mathbb Q_2 \in \mathcal P(\mathbb R^d)$, let $d_p ( \mathbb Q_1, \mathbb Q_2 ) = \inf_{\gamma \in \Gamma(\mathbb Q_1, \mathbb Q_2)} \mathbb E_{(X_1,X_2)\sim\gamma} \left[\|X_1-X_2\|^p\right]$. Then the marginals of $\mathbb Q$ with respect  to $X$ and $Y$ satisfy
\begin{align*}
   (\pi_{Y})_\# \mathbb Q = \frac{1}{n}\sum_{i=1}^{n}\delta_{y_i} \quad\quad \text{and} \quad\quad d_p \left((\pi_{X})_\# \mathbb Q, \frac{1}{n}\sum_{i=1}^{n}\delta_{x_i}\right) \leq \varepsilon.
\end{align*}
\end{lemma}

Lemma~\ref{lemma:properties:ambiguity} shows that the ambiguity set $\mathbb B_\varepsilon(\widehat{\mathbb P}_n)$ contains distributions $\mathbb Q$ whose marginal distribution of $Y$ is equal to the empirical distribution $\frac{1}{n}\sum_{i=1}^{n}\delta_{y_i}$ and whose marginal distribution of $X$ can be both discrete or continuous, and whose support is not necessarily concentrated on the samples $\{x_i\}_{i=1}^n$. 

% \inf_{\gamma \in \Gamma((\pi_{X})_\# \mathbb Q,\, n^{-1}\sum_{i=1}^{n}\delta_{x_i}}\; \mathbb E_{\gamma} [\|X_1-X_2\|^p]

% A proof of Lemma~\ref{lemma:properties:ambiguity} is offered in Appendix~\ref{appendix:additional:proofs:3}. 
In what follows, we will consider separately the two cases of distributional robustness induced by $W_1$ and $W_2$. Moreover, for $W_2$, we will also consider a distributionally regularized alternative to \eqref{eq:dro}, which will be introduced and motivated in Section~\ref{sec:distrib:reg}. 

\subsection{Type-1 Wasserstein DRE}
\label{sec:distr:rob:p=1}

We first concentrate on the type-$1$ Wasserstein case and recall the well-known fact that, for the type-$1$ Wasserstein discrepancy, the distributional robustness in problem \eqref{eq:dro} has an explicit Lipschitz regularization effect. For details, we refer to \cite{shafieezadeh2019regularization} for the case when the loss function is convex, and to \cite{gao2017wasserstein} for nonconvex Lipschitz loss functions.

\begin{fact}[Regularized estimation for $p=1$]
\label{prop:nonconvexduality:1}
Let Assumptions~\ref{assump:dro}\ref{assump:ell1}-\ref{assump:ell2} be satisfied for $p=1$. Then the optimal value of \eqref{eq:dro} is equal to the optimal value of the following regularized estimation problem
\begin{align}
\label{eq:dual:p=1}
    \min_{\theta \in \Theta}\; \frac{1}{n} \sum_{i=1}^{n}\, L(y_i - \theta^\top x_i)  + \varepsilon \text{Lip}(L)\|\theta\|.
\end{align}
\end{fact}

The regularized estimation formulation of problem \eqref{eq:dro} presented in Fact~\ref{prop:nonconvexduality:1} will constitute the starting point of the high-dimensional error analysis presented in Section~\ref{subsec:W1}. For illustration purposes, we conclude this section with the following example, which shows (informally) that problem \eqref{eq:dro} with $W_1$ can be rewritten as a (PO) problem.

\begin{example}[DRE $\to$ (PO)] 
\label{ex:cgmt:reg}
Consider the DRE problem \eqref{eq:dro}, with type-$1$ Wasserstein discrepancy, which has the explicit regularized formulation \eqref{eq:dual:p=1}. Using the fact that $y_i = \theta_0^\top x_i + z_i$, and introducing the change of variables $w = (\theta-\theta_0)/\sqrt{d}$, we obtain
\begin{align*}
    \min_{w \in \mathcal W}\; \frac{1}{n} \sum_{i=1}^{n}\, L(z_i - \sqrt{d}w^\top x_i)  + \varepsilon\, \text{lip}(L)\|\sqrt{d}w + \theta_0\|.
\end{align*}
with $\mathcal W$ the feasible set of $w$ obtained from $\Theta$ and the change of variable. Now, introducing the convex conjugate of $L$, we obtain
\begin{align*}
    \min_{w \in \mathcal W}\; \frac{1}{n} \sum_{i=1}^{n}\, \left(\sup_{u_i \in \mathbb R} u_i(z_i - \sqrt{d}w^\top x_i) - L^*(u_i)\right)  + \varepsilon\, \text{lip}(L)\|\sqrt{d}w + \theta_0\|,
\end{align*}
which can be rewritten as
\begin{align*}
%\label{eq:dre->po}
    \min_{w \in \mathcal W}\; \sup_{u \in \mathbb R^n}\; - \frac{1}{n} u^\top (\sqrt{d} A) w + \frac{1}{n} u^\top z - \frac{1}{n}\sum_{i=1}^{n}\,L^*(u_i) + \varepsilon\, \text{lip}(L)\|\sqrt{d}w + \theta_0\|.
\end{align*}

Assume now that we can restrict this optimization to $w \in \mathcal S_w$ and $u \in \mathcal S_u$, for some convex and compact sets $\mathcal S_w, \mathcal S_u$, without changing its optimal value. This step is formally justified in Section 4, as part of the proofs of our main theoretical results. Then, by defining $\psi(w,u) =  n^{-1} u^\top z - n^{-1}\sum_{i=1}^{n}\,L^*(u_i) + \varepsilon\, \text{lip}(L)\|\sqrt{d}w + \theta_0\|$, which can be easily seen to be convex-concave in $(w,u)$, and noticing that $\sqrt{d}A$ has entries i.i.d.\ standard normal (since the feature vectors $x_i$ are i.i.d.\ $\mathcal N(0,d^{-1}\text{I}_d)$ by Assumption~\ref{assump:0}), we see that problem \eqref{eq:dro} can be rewritten in the form of the (PO) problem.\qed
\end{example}

\subsection{Type-2 Wasserstein DRE}
\label{sec:distr:rob:p=2}

We now move to the type-$2$ Wasserstein case and proceed by presenting the minimax dual reformulation of problem \eqref{eq:dro}. Similar to Fact~\ref{prop:nonconvexduality:1}, this formulation will represent the starting point of the high-dimensional analysis conducted in Section~\ref{subsec:W2}.

\begin{fact}[Minimax reformulation for $p=2$]
\label{thm:nonconvexduality}
Let Assumption~\ref{assump:dro}\ref{assump:ell1} be satisfied for $p=2$. Then the optimal value of \eqref{eq:dro} is equal to the optimal value of the following minimax problem
\begin{align}
\label{eq:dro:dual:uni}
    \inf_{\theta \in \Theta, \lambda \geq 0}\; \sup_{u \in \mathbb R^n} \; \lambda \varepsilon + \frac{1}{n} \sum_{i=1}^{n}\; u_i(y_i - \theta^\top x_i) + \frac{u_i^2}{4 \lambda}\|\theta\|^2 - L^*(u_i).
\end{align}
\end{fact}

The proof of this fact is a simple application of \cite[Theorem~3.8(i)]{shafieezadeh2023new} and is therefore omitted. In the rest of Section~\ref{sec:distr:rob:p=2}, we will concentrate on the dual formulation \eqref{eq:dro:dual:uni} and study the conditions under which \eqref{eq:dro:dual:uni} is concave in $u$, and the norm of its optimal solution $u_\star$ is in the order of $\sqrt{n}$. In particular, the $\sqrt{n}\,$-growth rate of $u_\star$ will allow us to guarantee that, after the rescaling $u \to u\sqrt{n}$, the variable $u$ can be restricted without loss of generality to a convex compact set $\mathcal S_u$. Recall from Section~\ref{sec:cgmt} that the concavity in $u$, together with the convexity and compactness of $\mathcal S_u$ are essential in the high-dimensional error analysis using the CGMT framework. It turns out that both properties hold as long as the ambiguity radius $\varepsilon$ satisfies a mild upper-bound, which will be presented in Lemmas~\ref{lemma:convexity} and \ref{lemma:growthrates}. These results build upon the work \cite{blanchet2018optimal}, and require the following assumption on the set $\Theta$.

\begin{assumption}[Bounded parameter set]
\label{assump:Theta} 
The set $\Theta$ is defined as $\Theta := \{\theta \in \mathbb R^d:\, \|\theta\| \leq R_{\theta} \sqrt{d}\}$.
\end{assumption}

Assumption~\ref{assump:Theta} is essential for the application of the CGMT framework in the high-dimensional error analysis. Specifically, it will be used first in the proof of Lemma~\ref{lemma:growthrates}, while studying the growth rates of the optimal solutions. 

Before proceeding, we need to introduce the following concepts used in Lemmas~\ref{lemma:convexity} and \ref{lemma:growthrates}. Let $\underline{L}_n$ and $\overline{L}_n$ denote two constants satisfying
\begin{align}
    \label{eq:int:ell:bounds}
    0<\underline{L}_n \leq \mathbb E_{\widehat{\mathbb P}_n}\left[L'(Y - \theta^\top X)^2 \right] \leq \overline{L}_n,
\end{align}
for all $\theta \in \Theta$. Recall from Assumption~\ref{assump:dro}\ref{assump:ell:diff} that $L$ is differentiable, and therefore $L'$ is well-defined. Notice that, due to the upper bound $L(u) \leq C(1 + |u|^2)$ from Assumption~\ref{assump:ell2}, the value $\overline{L}_n$ is finite. 
%Moreover, it continues to remain finite in the infinite data limit, due to the finite second moment of $\mathbb P$ from Assumption~\ref{assump:pstar}. %In that case, we denote by $\underline{L}$ and $\overline{L}$ the lower and upper bounds, i.e.,
%\begin{align}
%    \label{eq:int:ell:bounds}
%    \underline{L} \leq \mathbb E_{\mathbb P}\left[L'(Y - \theta^\top X)^2 \right] \leq \overline{L}.
%\end{align}
Moreover, excluding the trivial and practically not relevant case where the loss function $L$ is $\widehat{\mathbb P}_n$-a.s. constant, resulting in $\underline{L_n} = 0$, in general $\underline{L_n}$ is strictly positive. %Similar reasoning holds for $\underline{L}$.

In the following lemma, we show that an upper bound on the radius $\varepsilon$ of the ambiguity set guarantees that \eqref{eq:dro:dual:uni} is a convex-concave minimax problem.

\begin{lemma}[Concavity in $u$ of \eqref{eq:dro:dual:uni}]
\label{lemma:convexity}
Let Assumptions~\ref{assump:dro}\ref{assump:ell1} and \ref{assump:ell:diff}, and Assumption~\ref{assump:Theta} be satisfied. Moreover, let $\varepsilon = \varepsilon_0/n$, for some constant $\varepsilon_0$ satisfying $\varepsilon_0 \leq \rho^{-1} M^{-2} R_{\theta}^{-2}\underline{L}_n$. Then, problem \eqref{eq:dro:dual:uni} is equivalent to the following convex-concave minimax problem 
\begin{align*}
    \inf_{\theta \in \Theta, \lambda \in \Lambda}\; \sup_{u \in \mathbb R^n} \; \lambda \varepsilon + \frac{1}{n} \sum_{i=1}^{n}\; u_i(y_i - \theta^\top x_i) + \frac{u_i^2}{4 \lambda}\|\theta\|^2 - L^*(u_i),
\end{align*}
with $\Lambda = \{\lambda \geq 0:\; \lambda \geq M R_\theta \sqrt{d} \|\theta\|/2 \}$.
\end{lemma}

Moreover, in the next lemma we show that the same upper bound on $\varepsilon$ guarantees that the norm of the optimal solution $u_\star$ in \eqref{eq:dro:dual:uni} can be upper bounded by $K_u \sqrt{n}$, for some constant $K_u>0$. Recall that the CGMT requires $u$ to live in a compact set. This will be achieved in Theorem~\ref{thm:estimationerror:2} using this upper bound and the rescaling $u \to u \sqrt{n}$.

\begin{lemma}[Growth rate of $u_\star$ in \eqref{eq:dro:dual:uni}]
\label{lemma:growthrates}
Let Assumption~\ref{assump:dro}\ref{assump:ell1} and Assumption~\ref{assump:Theta} be satisfied. Moreover, let $\varepsilon = \varepsilon_0/n$, for some constant $\varepsilon_0$ satisfying $ \varepsilon_0 \leq \rho^{-1} M^{-2} R_{\theta}^{-2}\underline{L}_n$. Then the optimal $u_\star$ in \eqref{eq:dro:dual:uni} satisfies $\|u_\star\|\leq K_u \sqrt{n}$, for some constant $K_u>0$. 
\end{lemma}

\subsection{Type-2 Wasserstein Distributional Regularization}
\label{sec:distrib:reg}

The results presented in Section~\ref{sec:distr:rob:p=2} will be used in Section~\ref{subsec:W2} to analyze the estimation error in high-dimensions. As will be shown in the proof of Theorem~\ref{thm:estimationerror:2}, one of the main challenges in the analysis will be posed by the minimization over $\lambda$ which appears in the dual reformulation \eqref{eq:dro:dual:uni}. This challenge will be one of the main reasons for the very intricated proof. To address this, in Section~\ref{subsec:regW2} we will show that by considering problem
\begin{align}
\label{eq:regdro}
    \min_{\theta \in \Theta}\; \sup_{\mathbb Q \in \mathcal P(\mathbb R^{d+1})}\; \mathbb E_{\mathbb Q} \left[ L(Y-\theta^\top X) \right] - \lambda W_p(\mathbb Q,\widehat{\mathbb P}_n),
\end{align}
as an alternative to \eqref{eq:dro}, a much simpler proof, together and a stronger result can be obtained. We will refer to problem \eqref{eq:regdro} as distributionally \emph{regularized} estimation, and we will denote its optimal solution by $\hat{\theta}_{\mathrm{DReg}}$. In what follows, we will mirror the results obtained in Section~\ref{sec:distr:rob:p=2}, i.e., the minimax reformulation, the concavity in $u$, and the growth rate of $u_\star$, for problem \eqref{eq:regdro}. We start by presenting a strong dual reformulation of problem \eqref{eq:regdro}.

\begin{fact}[Minimax reformulation of \eqref{eq:regdro}]
\label{lemma:nonconvexduality:3}
Let Assumption~\ref{assump:dro}\ref{assump:ell1} be satisfied for $p=2$. Then the optimal value of \eqref{eq:regdro} is equal to the optimal value of the following minimax problem
\begin{align}
\label{eq:dro:dual:uni:3}
    \min_{\theta \in \Theta}\; \sup_{u \in \mathbb R^n} \; \frac{1}{n} \sum_{i=1}^{n}\; u_i(y_i - \theta^\top x_i) + \frac{u_i^2}{4 \lambda}\|\theta\|^2 - L^*(u_i).
\end{align}
\end{fact}

Similar to Fact~\ref{thm:nonconvexduality}, the proof of this fact is a simple application of \cite[Theorem~3.8(i)]{shafieezadeh2023new} and is therefore omitted. Next, we will show that when the regularization parameter $\lambda$ satisfies a certain lower bound, the dual formulation \eqref{eq:dro:dual:uni:3} is a convex-concave minimax problem.

\begin{lemma}[Concavity in $u$ of \eqref{eq:dro:dual:uni:3}]
\label{lemma:convexity:3}
Let Assumptions~\ref{assump:dro}\ref{assump:ell1} and \ref{assump:ell:diff}, and Assumption~\ref{assump:Theta} be satisfied. Moreover, let $\lambda = d \lambda_0$, for some constant $\lambda_0$ satisfying $\lambda_0 \geq M R_{\theta}^{2}/2$. Then, the objective function in \eqref{eq:dro:dual:uni:3} is concave in $u$.
\end{lemma}

Before proceeding, we would like to highlight the connection between the upper bound on $\varepsilon$ in Lemma~\ref{lemma:growthrates} and the lower bound on $\lambda$ in Lemma~\ref{lemma:growthrates:3}. In the proof of Lemma~\ref{lemma:growthrates} we have shown that the upper bound $ \varepsilon \leq (\rho^{-1} M^{-2} R_{\theta}^{-2}\underline{L}_n)/n$ guarantees that $ \lambda \geq {M R_\theta \sqrt{d} \|\theta\|}/{2}$. Therefore, the lower bound $\lambda \geq M R_{\theta}^{2} d/2$ corresponds to replacing $\|\theta\|$ by its upper bound $R_\theta \sqrt{d}$.

Finally, in the following lemma we show that when $\lambda$ satisfies the same lower bound, the optimal solution $u_\star$ is in the order of $\sqrt{n}$, allowing us to apply the CGMT after the re-scaling $u \to u\sqrt{n}$.

\begin{lemma}[Growth rate of $u_\star$ in \eqref{eq:dro:dual:uni:3}]
\label{lemma:growthrates:3}
Let Assumption~\ref{assump:dro}\ref{assump:ell1} and Assumption~\ref{assump:Theta} be satisfied. Moreover, let $\lambda = d \lambda_0$, for some constant $\lambda_0$ satisfying $ \lambda_0 > M R_{\theta}^{2}/2$. Then the optimal $u_\star$ in \eqref{eq:dro:dual:uni:3} satisfies $\|u_\star\|\leq K_\beta \sqrt{n}$, for some constant $K_\beta>0$. 
\end{lemma}

%-------------------------------------------------------------------------------------------------
%-------------------------------------------------------------------------------------------------
%-------------------------------------------------------------------------------------------------
\section{High-Dimensional Error Analysis}
\label{sec:linear}

In this section, we present the main results of this paper, which quantify the performance of both the distributionally robust estimator $\hat{\theta}_{\mathrm{DRE}}$ and the distributionally regularized estimator $\hat{\theta}_{\mathrm{DReg}}$ in the high-dimensional regime. 
%To do so, we will quantify the performance of the estimator $\hat{\theta}_{\mathrm{DRE}}$ through the normalized squared error $\|\hat{\theta}_{\mathrm{DRE}}-\theta_0\|^2/d$. A major goal of this section is to uncover the role that the various factors, such as the choice of the loss function $L$, the ambiguity radius $\varepsilon$, the noise $Z$, and, very importantly, the under/over-parametrization $\rho = d/n$ of the problem, play in the estimator's performance.
Building on the results of the previous section, namely the dual formulations presented in Facts~\ref{prop:nonconvexduality:1}, \ref{thm:nonconvexduality}, and \ref{lemma:nonconvexduality:3}, as well as on the asymptotic CGMT presented in Fact~\ref{prop:cgmt}, we will show that the quantities $\|\hat{\theta}_{\mathrm{DRE}}-\theta_0\|^2/d$ and $\|\hat{\theta}_{\mathrm{DReg}}-\theta_0\|^2/d$ converge in probability to two deterministic values which can be recovered from the solutions of two convex-concave minimax optimization problems that involve at most four scalar variables. %Interestingly, the parameters appear in the objective function of this minimax problem either explicitly ($\varepsilon$, $\rho$) or implicitly ($L$, $Z$), opening the door towards an understanding of the influence that these parameters have on the performance of the estimation problem. 
The analysis carried out in what follows requires the following assumptions on the loss function $L$ and the true parameter $\theta_0$.%, and the noise $Z$.

\begin{assumption}[Signal strength] ~
\label{assump:cgmt}
\begin{enumerate}[label=(\roman*)]
    \item \label{assump:ell:0} The loss function $L$ is real-valued and satisfies $L(0) = \min_{u \in \mathbb R} L(u) = 0$.
    \item \label{assump:theta0} The true parameter $\theta_0$ satisfies $\|\theta_0\|^2/d \overset{P}{\to} \sigma_{\theta_0}^2$, with $\sigma_{\theta_0}>0$.
\end{enumerate} 
\end{assumption}

Assumption~\ref{assump:cgmt}\ref{assump:ell:0} is natural in estimation problems, while Assumption~\ref{assump:cgmt}\ref{assump:theta0} will be essential while studying the converge in probability of the quantities $\|\hat{\theta}_{\mathrm{DRE}}-\theta_0\|^2/d$ and $\|\hat{\theta}_{\mathrm{DReg}}-\theta_0\|^2/d$ in Theorems~\ref{thm:estimationerror:1}, \ref{thm:estimationerror:2} and \ref{thm:estimationerror:3}.

%-------------------------------------------------------------------------------------------------
\subsection{Type-1 Wasserstein DRE}
\label{subsec:W1}

We start by focusing on the type-$1$ Wasserstein case. Recall from Fact~\ref{prop:nonconvexduality:1} that, in this case, the DRE problem is equivalent to the regularized estimation problem \eqref{eq:dual:p=1}. This regularized formulation turns out to be very convenient, since it allows us to rely on the results in \cite{thrampoulidis2018precise} for the high-dimensional error analysis. Making use of these results requires the following additional mild assumptions on the loss function $L$ and the noise $Z$.
\begin{assumption}[Regularity assumptions]~
\label{assump:thrampoulidis}
\begin{enumerate}[label=(\roman*)]
    \item[(i)] Either there exists $u \in \mathbb R$ at which $L$ is not differentiable, or there exists an interval $\mathcal I \subset \mathbb R$ where $L$ is differentiable with a strictly increasing derivative.
    \item[(ii)] The absolutely continuous part of the noise distribution $\mathbb P_Z$ has a continuous Radon-Nikodym derivative.
    \item[(iii)] The noise $Z$ has nonzero variance.
\end{enumerate}
\end{assumption}

Before proceeding, we need to introduce some notation, which will be later used in the statement of Theorem~\ref{thm:estimationerror:1}. We first define the \textit{expected Moreau envelope} $\mathcal L : \mathbb R \times \mathbb R \to \mathbb R$ associated to $L$ as
\begin{align}
    \label{eq:L}
    \mathcal L(c,\tau) := \mathbb E_{\mathcal N(0,1) \otimes \mathbb P_Z}\left[ e_L(cG+Z;\tau) \right].
\end{align}
with $G$ a standard normal random variable $\mathcal N(0,1)$. Notice that the expectation is with respect to the joint distribution $\mathcal N(0,1) \otimes \mathbb P_Z$ of the two independent random variables $G$ and $Z$.

The effects of the loss function $L$ and the noise $Z$ on the error $\|\hat{\theta}_{\mathrm{DRE}}-\theta_0\|^2/d$ will only appear implicitly. In Theorem~\ref{thm:estimationerror:1} we will see that the function $\mathcal L$ quantifies their contribution. Moreover, we introduce the function $\mathcal G : \mathbb R \times \mathbb R \to \mathbb R$, defined as follows
\begin{align}
    \label{eq:R}
    \mathcal G(c,\tau) := \begin{cases} \sqrt{(c^2 + \sigma_{\theta_0}^2)/\rho}  - {\tau}/{(2 \rho)} - {\sigma_{\theta_0}}/{\sqrt{\rho}} \quad\quad &\mbox{if }\; \sqrt{\rho}\sqrt{c^2 + \sigma_{\theta_0}^2} > \tau, \\  \left( c^2 + \sigma_{\theta_0}^2 \right)/(2 \tau)- {\sigma_{\theta_0}}/{\sqrt{\rho}} 
 \quad\quad &\mbox{if }\; \sqrt{\rho}\sqrt{c^2 + \sigma_{\theta_0}^2} \leq \tau. \end{cases}
\end{align}

The effects of the ambiguity radius $\varepsilon$ and the under/over-parametrization parameter $\rho$ on the error $\|\hat{\theta}_{\mathrm{DRE}}-\theta_0\|^2/d$ will appear explicitly. In Theorem~\ref{thm:estimationerror:1} we will see that the function $\mathcal G$, as well as other explicit terms, quantifies their contribution. We are now ready to state the main result on the high-dimensional error analysis for $p=1$.

\begin{theorem}[Performance of $\hat{\theta}_{\mathrm{DRE}}$ for $p=1$]
\label{thm:estimationerror:1}
Let Assumptions~\ref{assump:0}, \ref{assump:dro}, \ref{assump:cgmt} and \ref{assump:thrampoulidis} be satisfied. Moreover, let $\Theta = \mathbb R^d$ and $\varepsilon = \varepsilon_0/\sqrt{n}$, for some constant $\varepsilon_0$. Then, as $d,n \to \infty$, it holds in probability that
\begin{align}
    \label{eq:est:1:error}
    \lim_{d/n \to \rho} \frac{\|\hat{\theta}_{\mathrm{DRE}}-\theta_0\|^2}{d} = \alpha_\star^2,
\end{align}
where $\alpha_\star$ is the unique solution\footnote{Provided that the minimizer of \eqref{eq:est:1:minimax} over $\alpha$ is bounded.} of the convex-concave minimax scalar problem
\begin{align}
    \label{eq:est:1:minimax}
    \min_{\alpha \geq 0,\, \tau_1 > 0}\; \max_{\beta \geq 0,\, \tau_2>0}\; \frac{\beta \tau_1}{2} - \frac{\alpha \tau_2}{2} - \frac{\alpha \beta^2}{2 \tau_2} + \frac{1}{\rho} \,\mathcal L \left(\alpha, \frac{\tau_1}{\beta}\right) + \varepsilon_0 \text{Lip}(L)\, \mathcal G \left(\frac{\alpha \beta}{\tau_2}, \frac{\alpha \varepsilon_0 \text{Lip}(L)}{\tau_2}\right),
\end{align}
with $\mathcal L(\cdot,\cdot)$ and $\mathcal G(\cdot,\cdot)$ defined in \eqref{eq:L} and \eqref{eq:R}, respectively.
\end{theorem}

The Moreau envelope $e_{L}(c,\tau)$ can be computed in closed-form for many important classes of estimators. One such example, which we will recall in Section~\ref{sec:numerical}, is the LAD estimator.

\begin{example}[Moreau envelope for the LAD estimator]
\label{example:Moreau:LAD}
Consider the least absolute deviations (LAD) estimator, where the loss function is defined as $L(u) = |u|$. 
\begin{comment}
To recover the expression of the Moreau envelope $e_{L}(c,\tau)$, we use the following well-known relationship between the Moreau envelopes of $L$ and its convex conjugate $L^*$,
\begin{align*}
    e_L\left(c,\tau\right) + e_{L^*}\left(\frac{c}{\tau},\frac{1}{\tau}\right) = \frac{c^2}{2\tau}.
\end{align*}
Now, since 
\begin{align*}
    L^*(u) = \sup_{v \in \mathbb R} uv- |v| = \begin{cases} 0 & \quad\quad \text{if} \quad\quad  |u| \leq 1\\ \infty & \quad\quad \text{otherwise},  \end{cases}
\end{align*}
the Moreau envelope of $L^*$ becomes
\begin{align*}
    e_{L^*}\left(\frac{c}{\tau},\frac{1}{\tau}\right) = \min_{u \in \mathbb R}\; \frac{\tau}{2}\left( \frac{c}{\tau} - u \right)^2 + L^*(u) &= \min_{|u| \leq 1}\; \frac{\tau}{2}\left( \frac{c}{\tau} - u \right)^2 \\ &= \begin{cases} 0 & \quad\quad \text{if} \quad\quad  |c| \leq \tau \\ {\tau}\left( {c}/{\tau} - \text{Sign}(c) \right)^2/{2} & \quad\quad \text{otherwise},  \end{cases}
\end{align*}
from which we conclude that
\end{comment}
Then, it can be easily shown that its Moreau envelope is equal to
\begin{align*}
    e_L\left(c,\tau\right) = \begin{cases} {c^2}/{(2\tau)}  & \quad\quad \text{if} \quad\quad  |c| \leq \tau \\ |c| - {\tau}/{2} & \quad\quad \text{otherwise}.  \end{cases}
\end{align*}\qed
\end{example}

\subsection{Type-2 Wasserstein DRE}
\label{subsec:W2}

We now proceed to the type-$2$ Wasserstein case where, different from the $W_1$ case, the distributional robustness does not generally have an equivalent regularized formulation, and therefore the results in \cite{thrampoulidis2018precise} are not applicable anymore. However, building on the analysis and results presented in Section~\ref{sec:distr:rob:p=2}, we will show that the CGMT framework can be applied to study the high-dimensional performance of the distributionally robust estimator $\hat{\theta}_{\mathrm{DRE}}$. This will be shown in Theorem~\ref{thm:estimationerror:2}.

Similarly to Section~\ref{subsec:W1}, we need to introduce some notation, which will be later used in the statement of Theorem~\ref{thm:estimationerror:2}. In the same spirit as equation \eqref{eq:L}, we define the \textit{expected Moreau envelope} $\mathcal F : \mathbb R \times \mathbb R \to \mathbb R$ associated to a convex function $f$ as
\begin{align}
    \label{eq:L_*}
    \mathcal F(c,\tau) := \mathbb E_{\mathcal N(0,1) \otimes \mathbb P_Z}\left[ e_{f}(cG+Z;\tau)\right],
\end{align}
where, as before, $G$ is a standard normal random variable $\mathcal N(0,1)$, and the expectation is with respect to the joint distribution $\mathcal N(0,1) \otimes \mathbb P_Z$ of the two independent random variables $G$ and $Z$. As in the case of $W_1$, the effects of the loss function $L$ and the noise $Z$ on the error $\|\hat{\theta}_{\mathrm{DRE}}-\theta_0\|^2/d$ will only appear implicitly, this time through the expected Moreau envelope of a convex function $f$ which can be derived from the loss function $L$.% In the proof of Theorem~\ref{thm:estimationerror:2} we will show that the quantity inside the expectation in \eqref{eq:L_*} is absolutely integrable for the function $f$ of interest, and therefore the expected Moreau envelope is well-defined.

Secondly, we introduce the following two functions used in the statement of Theorem~\ref{thm:estimationerror:2},
\begin{align*}
    \mathcal O_1(\alpha,\tau_1,\tau_2,\beta) := \frac{\beta \tau_1}{2} + \frac{\varepsilon_0 \beta \tau_2}{2} - \frac{\beta^2}{2M} + \mathcal F(\alpha,\tau_1/\beta) - \alpha \beta \sqrt{\rho} \sqrt{\frac{\rho \varepsilon_0}{\tau_2^2}\sigma_{\theta_0}^2 + 1} + \frac{ \sqrt{\varepsilon_0}\beta \rho}{2\tau_2} \left( \sigma_{\theta_0}^2 + \alpha^2 \right),
\end{align*}
and
\begin{align*}
    \mathcal O_2(\alpha,\tau_1,\beta) := \inf_{\tau_2 > 0}\; \frac{\beta \tau_1}{2} + \frac{p\tau_2}{2}+\frac{\beta^2 \tau_2}{2q} -\frac{\beta^2}{2M} &+ \mathcal F(\alpha,\tau_1/\beta) - \alpha\sqrt{\rho} \sqrt{\left( p + \frac{\beta^2}{q} \right)^2 \frac{\rho \sigma_{\theta_0}^2}{\tau_2^2} + \beta^2}\\ &+ \rho \left( p + \frac{\beta^2}{q} \right) \frac{\sigma_{\theta_0}^2+\alpha^2}{2 \tau_2},
\end{align*}
with constants $p := (\varepsilon_0\sqrt{\rho}M R_\theta)/2$ and $q := 2 \sqrt{\rho} M R_\theta$, and $f$ defined below, in the statement of Theorem~\ref{thm:estimationerror:2}. Recall that $M$ is the parameter of the smoothness of $L$, defined in Assumption~\ref{assump:dro}\ref{assump:ell:diff}, and $R_\theta$ is the bound on the norm of $\theta/\sqrt{d}$, defined in Assumption~\ref{assump:Theta}. We are now ready to state the main result on the high-dimensional error analysis.

\begin{theorem}[Performance of $\hat{\theta}_{\mathrm{DRE}}$ for $p=2$]
\label{thm:estimationerror:2}
Let Assumptions~\ref{assump:0}, \ref{assump:dro}, \ref{assump:Theta}, and \ref{assump:cgmt} be satisfied, and suppose that the loss function $L$ can be written as the Moreau envelope with parameter $M$, i.e., $L(\cdot) = e_f(\cdot, 1/M)$, of a convex function $f$ which satisfies
\begin{align}
\label{eq:est:2:assump:f}
    \mathbb E_{\mathcal N(0,1) \otimes \mathbb P_Z}\left[ f(\alpha G + Z)\right] < \infty,
\end{align}
for all $\alpha \leq \sigma_{\theta_0}$. Moreover, let $R_\theta \geq 2 \sigma_{\theta_0}$ and $\varepsilon = \varepsilon_0/n$, for some constant $\varepsilon_0$ satisfying $\varepsilon_0 \leq \rho^{-1} M^{-2} R_{\theta}^{-2}\underline{L}_n$, and consider the following convex-concave minimax optimization problems
\begin{align}
    \label{eq:est:2:minimax}
    \mathcal V_1 := \inf_{\substack{0 \leq \alpha \leq \sigma_{\theta_0} \\ \tau_1,\tau_2>0}}\; \sup_{\beta > B}\; \mathcal O_1(\alpha,\tau_1,\tau_2,\beta), \quad\quad \text{and} \quad\quad
    \mathcal V_2 := \inf_{\substack{0 \leq \alpha \leq \sigma_{\theta_0}\\ \tau_1>0}}\; \max_{0 \leq \beta \leq B} \; \mathcal O_2(\alpha,\tau_1,\beta),
\end{align}
with $\mathcal O_1(\alpha,\tau_1,\tau_2,\beta)$ and $\mathcal O_2(\alpha,\tau_1,\beta)$ defined above, and $B = \sqrt{\varepsilon_0 \rho} M R_{\theta}$. If the optimal solutions $\alpha_{\star,1}$ and $\alpha_{\star,2}$ attaining $\mathcal V_1$ and $\mathcal V_2$, respectively, are strictly smaller than $\sigma_{\theta_0}$, then, as $d,n \to \infty$, it holds in probability that
\begin{align}
    \label{eq:est:2:error}
    \lim_{d/n \to \rho} \frac{\|\hat{\theta}_{\mathrm{DRE}}-\theta_0\|^2}{d} \begin{cases} = \alpha_{\star,1}^2 & \quad\quad \text{if} \quad\quad \mathcal V_1 > \mathcal V_2 \\ = \alpha_{\star,2}^2 & \quad\quad \text{if} \quad\quad \mathcal V_2 > \mathcal V_1 \\ \leq \max\{\alpha_{\star,1}^2,\alpha_{\star,2}^2\} & \quad\quad \text{if} \quad\quad \mathcal V_1 = \mathcal V_2. \end{cases}
\end{align}
\end{theorem}

\begin{remark}[Error-to-signal ratio]
Theorem~\ref{thm:estimationerror:2} requires the optimal solutions of the two optimization problems to be strictly smaller than $\sigma_{\theta_0}$, i.e., $\alpha_{\star,1}, \alpha_{\star,2} < \sigma_{\theta_0}$. This condition is equivalent to imposing the following relative error bound 
\begin{align*}
    \frac{\|\hat{\theta}_{\mathrm{DRE}} -\theta_0\|}{\|\theta_0\|} < 1,
\end{align*}
which is practically desirable in any estimation problem. \qed
\end{remark}

\begin{remark}[Convexity of $\mathcal O_2(\alpha,\tau_1,\beta)$]
In the proof of Theorem~\ref{thm:estimationerror:2} we show that the objective function appearing in the definition of $\mathcal O_2(\alpha,\tau_1,\beta)$ is convex in $\tau_2$. Consequently, $\mathcal O_2(\alpha,\tau_1,\beta)$ can be easily computed for any $(\alpha,\tau_1,\beta)$.  \qed
\end{remark}

Many loss functions which have at-most-quadratic growth rate (as required by Assumption~\ref{assump:dro}\ref{assump:ell2}) can be written as the Moreau envelope with parameter $M$ of some convex function $f$ which satisfies the integrability condition $\mathbb E_{\mathcal N(0,1) \otimes \mathbb P_Z}\left[ f(\alpha G + Z)\right] < \infty$. In the following example we show that the Huber loss satisfies this requirement.

\begin{example}[Huber loss]
Consider the Huber loss function with parameter $\delta$, defined as
\begin{align}
\label{eq:huber:loss}
    L(u) = \begin{cases} {u^2}/{2}  & \quad\quad \text{if} \quad\quad  |u| \leq \delta \\ \delta|u| - {\delta^2}/{2} & \quad\quad \text{otherwise}.  \end{cases}
\end{align}
It can be immediately checked that $M=1$ in this case. It can be easily shown that its convex conjugate is equal to
\begin{align*}
    L^*(u) = \begin{cases} {u^2}/{2}  & \quad\quad \text{if} \quad\quad  |u| \leq \delta \\ \infty & \quad\quad \text{otherwise}.  \end{cases}
\end{align*}
In the proof of Theorem~\ref{thm:estimationerror:2} we have seen that $f$ can be recovered from $L^*$ as follows
\begin{align*}
    f(u) = (L^*(\cdot) - 1/2(\cdot)^2)^*(u) &= \sup_{v \in \mathbb R}\; uv - \begin{cases} 0  & \quad\quad \text{if} \quad\quad  |v| \leq \delta \\ \infty & \quad\quad \text{otherwise}  \end{cases} \\ &= \delta |u|.
\end{align*}
Finally, notice that 
\begin{align*}
    \mathbb E_{\mathcal N(0,1) \otimes \mathbb P_Z}\left[ \delta|\alpha G + Z|\right] < \infty
\end{align*}
for all $\alpha \in \mathbb R$ is a consequence of the finite first moment of $G$ and $Z$ (Assumption~\ref{assump:dro}\ref{assump:pstar} implies that $Z$ has a finite second moment, and therefore also a finite first moment). \qed
\end{example}

There are also loss functions which cannot be written in such way. One important example is the squared loss, as explained in the following remark.

\begin{remark}[Squared loss I]
\label{remark:squared:loss:no:f}
Consider the squared loss function $L(u)  = u^2$. It can be easily checked that $L^*(v) = v^2/4$,  and $M=2$. Therefore, $f^* = L^*(\cdot) - (\cdot)^2/4 = 0$, and consequently,
\begin{align*}
    f(u) = \sup_{v \in \mathbb R}\; uv = \begin{cases} 0  & \quad\quad \text{if} \quad\quad  u = 0 \\ \infty & \quad\quad \text{otherwise},  \end{cases}
\end{align*}
from which $\mathbb E_{\mathcal N(0,1) \otimes \mathbb P_Z}\left[ f(\alpha G + Z)\right] = \infty$ follows. \qed
\end{remark}

However, for squared losses this requirement is not necessary, and Theorem~\ref{thm:estimationerror:2} continues to hold if the expected Moreau envelope $\mathcal F(\alpha,\tau_1/\beta)$ is replaced by the term $\mathbb E_{\mathcal N(0,1) \otimes \mathbb P_Z}\left[\frac{\beta}{2 \tau_1} (\alpha G+Z)^2\right]$. This is explained in the next remark. 

\begin{remark}[Squared loss II]
\label{remark:thm2:squared:loss}
Recall from the proof of Theorem~\ref{thm:estimationerror:2} that $f$ was introduced as a consequence of the decomposition $L^*(u) = u^2/(2M) + f^*(u)$ (which, in turn, was needed to preserve the concavity in $u$ of the objective function, in order to apply the CGMT). Whenever $L(u)  = M u^2/2$, we have that $L^*(u) = u^2/(2M)$, and therefore $f^*(u) = 0$ for all $u \in \mathbb R$. In this case, there is no need to introduce the function $f$, and the result of Theorem~\ref{thm:estimationerror:2} continues to hold by simply replacing the expected Moreau envelope $\mathcal F(\alpha,\tau_1/\beta) = \mathbb E_{\mathcal N(0,1) \otimes \mathbb P_Z}\left[ e_{f}(cG+Z;\tau)\right]$ with the term
\begin{align*}
    \mathbb E_{\mathcal N(0,1) \otimes \mathbb P_Z}\left[\frac{\beta}{2 \tau_1} (\alpha G+Z)^2\right].
\end{align*}
Furthermore, the minimization over $\tau_1$ can be now solved in closed-form as
\begin{align*}
    \inf_{\tau_1>0}\; \frac{\beta \tau_1}{2} + \frac{\beta}{2 \tau_1} \mathbb E_{\mathcal N(0,1) \otimes \mathbb P_Z}\left[ (\alpha G+Z)^2\right] = \beta \sqrt{\mathbb E_{\mathcal N(0,1) \otimes \mathbb P_Z}\left[ (\alpha G+Z)^2\right]} = \beta \sqrt{\alpha^2 + \sigma_Z^2},
\end{align*}
where $\sigma_Z^2$  denotes the second moment of $Z$, i.e., $\sigma_{Z}^2 := \mathbb E_{\mathbb P_Z} [Z^2]$, reducing the number of scalar variables from four to three in the two minimax problems \eqref{eq:est:2:minimax}. \qed
\end{remark}

Alternatively, it can be shown that for the particular case of squared losses some of the assumptions (such as $\varepsilon_0 \leq \rho^{-1} M^{-2} R_{\theta}^{-2}\underline{L}_n$) can be dropped, and that the estimation error can be obtained from the solution of only one optimization problem. The reason for this is that when $L(u) = u^2$, the distributional robustness has the following equivalent regularized formulation \cite{blanchet2019robust},
\begin{align}
\label{eq:reg:for:squared:loss}
    \min_{\theta \in \Theta}\; \sup_{\mathbb Q \in \mathbb B_{\varepsilon}(\widehat{\mathbb P}_n)}\; \mathbb E_{\mathbb Q} \left[ (Y-\theta^\top X)^2 \right] = \left(\min_{\theta \in \Theta}\; \sqrt{\mathbb E_{\widehat{\mathbb P}_n}\left[ (Y-\theta^\top X)^2 \right]} + \sqrt{\varepsilon}\|\theta\| \right)^2.
\end{align}
As a consequence, the estimation error can be recovered from Theorem~$1$ in \cite{thrampoulidis2018precise}, as summarized in the following proposition. Before doing so, we need to introduce some notation. We define the function $\mathcal E:\mathbb R \times \mathbb R \to \mathbb R$,
\begin{align*}
    \mathcal E(c,\tau) := \begin{cases} \sqrt{(c^2 + \sigma_{Z}^2)}  - {\tau}/{2} - {\sigma_{Z}}/{\sqrt{\rho}} \quad\quad &\mbox{if }\; \sqrt{c^2 + \sigma_{Z}^2} > \tau \\  \left( c^2 + \sigma_{Z}^2 \right)/(2 \tau)- {\sigma_{Z}}
 \quad\quad &\mbox{if }\; \sqrt{c^2 + \sigma_{Z}^2} \leq \tau, \end{cases}
\end{align*}
which is used in the statement of the following proposition.

\begin{proposition}[Performance of $\hat{\theta}_{\mathrm{DRE}}$ for $L(\cdot) = (\cdot)^2$]
\label{prop:est:error:squared:loss}
Let Assumptions~\ref{assump:0}, \ref{assump:dro}, and \ref{assump:cgmt} be satisfied for the loss function $L(u) = u^2$. Moreover, let $\Theta = \mathbb R^d$, $\sigma_{Z}^2 := \mathbb E_{\mathbb P_Z} [Z^2]$, and $\varepsilon = \varepsilon_0/n$, for some constant $\varepsilon_0$. Then, as $d,n \to \infty$, it holds in probability that
\begin{align*}
    \lim_{d/n \to \rho} \frac{\|\hat{\theta}_{\mathrm{DRE}}-\theta_0\|^2}{d} = \alpha_\star^2,
\end{align*}
provided that $\alpha_\star$ is the unique minimizer of the following convex-concave minimax scalar problem
\begin{align*}
    \min_{\alpha \geq 0,\, \tau_1 > 0}\; \max_{\beta \geq 0,\, \tau_2>0}\; \frac{\beta \tau_1}{2} - \frac{\alpha \tau_2}{2} - \frac{\alpha \beta^2}{2 \tau_2} + \frac{1}{\rho} \,\mathcal E \left(\alpha, \frac{\tau_1}{\beta}\right) + \sqrt{\varepsilon_0} \, \mathcal G \left(\frac{\alpha \beta}{\tau_2}, \frac{\alpha \sqrt{\varepsilon_0}}{\tau_2}\right),
\end{align*}
with $\mathcal E$ defined above, and $\mathcal G$ defined in \eqref{eq:R}.
\end{proposition}

\begin{remark}[Uniqueness of $\alpha_\star$]
Notice that, different from Theorem~\ref{thm:estimationerror:2}, the uniqueness of the optimizer $\alpha_\star$ is not guaranteed in Proposition~\ref{prop:est:error:squared:loss}. When $\alpha_\star$ is not unique, and Proposition~\ref{prop:est:error:squared:loss} cannot be used to recover the estimation error, the adaptation of Theorem~\ref{thm:estimationerror:2} for squared losses, proposed in Remark~\ref{remark:thm2:squared:loss}, can be used to recover the estimation error, provided that the minimizers of $\mathcal V_1$ and $\mathcal V_2$ in $\eqref{eq:est:2:minimax}$ are strictly smaller than $\sigma_{\theta_0}$. \qed
\end{remark}

%---------------------------------------------------------------------------------------------

\subsection{Type-2 Wasserstein Distributional Regularization}
\label{subsec:regW2}

In Theorem~\ref{thm:estimationerror:2} we have shown that the estimation error in the distributionally robust setting with $W_2$ can be computed from the optimal solution of two scalar convex-concave minimax problems. Nonetheless, the result has an inconvenient limitation when compared to the results obtained for $p=1$: the estimation error can be found only when the minimizers of $\mathcal V_1$ and $\mathcal V_2$ in $\eqref{eq:est:2:minimax}$ are strictly smaller than $\sigma_{\theta_0}$. In what follows, we will show that by considering the distributionally regularized problem~\eqref{eq:regdro} this assumption can be dropped. Additionally, the estimation error can be recovered from the minimizer of only one (simpler) convex-concave minimax problem, which involves only three scalar variables. This is shown in the following theorem.

\begin{theorem}[Performance of $\hat{\theta}_{\mathrm{DReg}}$ for $p=2$]
\label{thm:estimationerror:3}
Let Assumptions~\ref{assump:0}, \ref{assump:dro}, \ref{assump:Theta}, and \ref{assump:cgmt} be satisfied, and suppose that the loss function $L$ can be written as the Moreau envelope of a convex function $f$ which satisfies
\begin{align}
\label{eq:est:3:assump:f}
    \mathbb E_{\mathcal N(0,1) \otimes \mathbb P_Z}\left[ f(\alpha G + Z)\right] < \infty,
\end{align}
for all $\alpha \geq 0$. Moreover, let $R_\theta \geq \sigma_{\theta_0}$ and $\lambda = d \lambda_0$, for some constant $\lambda_0$ satisfying $\lambda_0 > M R_{\theta}^{2}/2$. 
Then, as $d,n \to \infty$, it holds in probability that
\begin{align}
    \label{eq:est:3:error}
    \lim_{d/n \to \rho} \frac{\|\hat{\theta}_{\mathrm{DReg}}-\theta_0\|^2}{d} \begin{cases} = \alpha_{\star}^2 & \quad\quad \text{if} \quad\quad \mathcal \alpha_{\star} \in \left[0, R_\theta - \sigma_{\theta_0}\right] \\ \leq \alpha_{\star}^2 & \quad\quad \text{if} \quad\quad \alpha_{\star} \in \left(R_\theta - \sigma_{\theta_0}, R_\theta + \sigma_{\theta_0} \right], \end{cases}
\end{align}
where $\alpha_\star$ is the unique solution of the following convex-concave minimax scalar problem
\begin{align}
    \label{eq:est:3:minimax}
    \min_{\substack{0 \leq \alpha \leq R_{\theta} + \sigma_{\theta_0}\\ \tau_1>0}}\; \max_{\beta \geq 0}\; \frac{\beta \tau_1}{2}-\frac{\beta^2}{2M} + \mathcal F(\alpha,\frac{\tau_1}{\beta}) + \frac{\beta^2}{4 \lambda_0}\left( \sigma_{\theta_0}^2 + \alpha^2 \right) - \alpha \beta \sqrt{\rho + \frac{\beta^2 \sigma_{\theta_0}^2}{4 \lambda_0^2}}.
\end{align}
\end{theorem}

\begin{remark}[Relationship between $\varepsilon_0$ and $\lambda_0$]
As stated in Section~\ref{sec:distrib:reg}, there is a close connection between the upper bound on $\varepsilon_0$ in Theorem~\ref{thm:estimationerror:2} and the lower bound on $\lambda_0$ in Theorem~\ref{thm:estimationerror:3}. In Lemma~\ref{lemma:growthrates} we have shown that the upper bound $ \varepsilon \leq (\rho^{-1} M^{-2} R_{\theta}^{-2}\underline{L}_n)/n$ guarantees that $ \lambda \geq {M R_\theta \sqrt{d} \|\theta\|}/{2}$. Therefore, the lower bound $\lambda \geq M R_{\theta}^{2} d/2$ corresponds to replacing $\|\theta\|$ by its upper bound $R_\theta \sqrt{d}$.
\end{remark}

Recall from Remark~\ref{remark:squared:loss:no:f} that the squared loss $L(u) = M u^2/2$ cannot be written as the Moreau envelope with parameter $M$ of some convex function $f$ which satisfies the integrability condition $\mathbb E_{\mathcal N(0,1) \otimes \mathbb P_Z}\left[ f(\alpha G + Z)\right] < \infty$. Nonetheless, in the following corollary we show how Theorem~\ref{thm:estimationerror:3} can be easily adapted to account for this case as well, resulting in a simpler convex-concave minimax problem, which involves only two scalar variables.

\begin{corollary}[Performance of $\hat{\theta}_{\mathrm{DReg}}$ for $L(\cdot) = (\cdot)^2$]
\label{cor:regdro:squared:loss}
Let Assumptions~\ref{assump:0}, \ref{assump:dro}, \ref{assump:Theta}, and \ref{assump:cgmt} be satisfied for the loss function $L(u) = u^2$. Moreover, let $R_\theta \geq \sigma_{\theta_0}$, $\sigma_{Z}^2 := \mathbb E_{\mathbb P_Z} [Z^2]$, and $\lambda = d \lambda_0$, for some constant $\lambda_0$ satisfying $\lambda_0 > R_{\theta}^{2}$. 
Then, as $d,n \to \infty$, it holds in probability that
\begin{align*}
    \lim_{d/n \to \rho} \frac{\|\hat{\theta}_{\mathrm{DReg}}-\theta_0\|^2}{d} \begin{cases} = \alpha_{\star}^2 & \quad\quad \text{if} \quad\quad \mathcal \alpha_{\star} \in \left[0, R_\theta - \sigma_{\theta_0}\right] \\ \leq \alpha_{\star}^2 & \quad\quad \text{if} \quad\quad \alpha_{\star} \in \left(R_\theta - \sigma_{\theta_0}, R_\theta + \sigma_{\theta_0} \right], \end{cases}
\end{align*}
where $\alpha_\star$ is the unique solution of the following convex-concave minimax scalar problem
\begin{align}
\label{eq:DReg_quadratic}
    \min_{0 \leq \alpha \leq R_{\theta} + \sigma_{\theta_0}}\; \max_{\beta \geq 0}\; \beta \sqrt{\alpha^2 + \sigma_Z^2}-\frac{\beta^2}{2M} + \frac{\beta^2}{4 \lambda_0}\left( \sigma_{\theta_0}^2 + \alpha^2 \right) - \alpha \beta \sqrt{\rho + \frac{\beta^2 \sigma_{\theta_0}^2}{4 \lambda_0^2}}.
\end{align}
\end{corollary}

We conclude the theoretical developments in this paper with two remarks: one discussing how the signal and noise variances required by our framework can be estimated in practice, and the other outlining a possible extension of our results to non-isotropic feature distributions.

\begin{remark}[Estimation of $\sigma_{\theta_0}^2$ and $\sigma_Z^2$]
The asymptotic error $\alpha_\star^2$ for both $\hat{\theta}_{\mathrm{DRE}}$ and $\hat{\theta}_{\mathrm{DReg}}$ depends on the signal strength $\sigma_{\theta_0}^2$ and the noise variance $\sigma_Z^2$, which in practice need to be estimated. While our theory does not require the noise $Z$ to be Gaussian, most high-dimensional estimators for $\sigma_Z^2$ (and for $\sigma_{\theta_0}^2$) are developed under the additional assumption that $Z \sim \mathcal{N}(0, \sigma_Z^2)$. Under this assumption, consistent and asymptotically normal estimators can be constructed using method-of-moments or maximum likelihood techniques, as proposed in \cite{dicker2014variance,dicker2016maximum}. In addition, the EigenPrism procedure of \cite{janson2017eigenprism} provides finite-sample confidence intervals for $\sigma_{\theta_0}^2$, allowing practitioners to assess the robustness of the selected radius $\varepsilon$ to estimation uncertainty.
\end{remark}

\begin{remark}[Non-isotropic case]
\label{remark:non-isotropic}
Our theoretical results assume isotropic Gaussian covariates $X \sim \mathcal{N}(0, d^{-1} \mathrm{I}d)$, which simplifies the geometry of the estimation problem. However, the analysis can be extended to the more general setting where $X \sim \mathcal{N}(0, d^{-1} \Sigma)$ for a positive definite covariance matrix $\Sigma$. In this case, a whitening transformation of the features can recover isotropy, assuming $\Sigma$ is known or can be consistently estimated. Importantly, both the estimation error and the minimax objective must then be redefined with respect to the $\Sigma$-induced norm $\| \cdot \|_\Sigma^2 = \langle \cdot, \Sigma \cdot \rangle$, and the signal strength becomes $\sigma_{\theta_0}^2 = d^{-1} \theta_0^\top \Sigma \theta_0$. We view a systematic treatment of this extension as an important direction for future work.
\end{remark}

%-------------------------------------------------------------------------------------------------
%-------------------------------------------------------------------------------------------------
%-------------------------------------------------------------------------------------------------

\section{Numerical Experiments}
\label{sec:numerical}

In this section we numerically validate the theoretical results presented so far in the context of hyperparameter tuning for regression problems in high dimensions. Specifically, we consider the problem of choosing the radius $\varepsilon$ which minimizes the normalized squared estimation error in the high-dimensional regime, where both the dimension of the problem $d$ and the number of samples $n$ are very large. In such regimes, numerical techniques such as cross-validation are computationally very demanding or even prohibitive. Instead, Theorems~\ref{thm:estimationerror:1}, \ref{thm:estimationerror:2}, and \ref{thm:estimationerror:3} provide an easy-to-compute relationship between the normalized squared estimation error and the values of $\varepsilon$ and $\rho$, which appear explicitly in the objective functions of \eqref{eq:est:1:minimax} and \eqref{eq:est:2:minimax}. Although these results were recovered in the asymptotic regime and for isotropic Gaussian features, in what follows we will show numerically that they are valid for broader classes of probability ensembles, already when $d,n \approx 500$. We conclude the section by comparing our approach with standard cross-validation for radius tuning, illustrating its practical advantages.

\subsection{Radius Tuning in High Dimensions}
\label{subsec:radius:tuning}

We focus on the problem of estimating an unknown parameter $\theta_0$ in $d$ dimensions from $n$ noisy linear measurements of the form $y_i = \theta_0 ^\top x_i + z_i$. Throughout this entire section, we consider the case where the true parameter $\theta_0$ is sparse, with roughly $10\%$ of its entries being non-zero. In our experiments, we encode such sparsity in a probabilistic fashion, by sampling the entries of $\theta_0$ according to the distribution $\mathbb P_{\theta_0} = 0.1\, \mathcal N(0,10) + 0.9\, \delta_0$, where $\delta_0$ represents the Dirac delta distribution at $0$. From the definition of $\mathbb P_{\theta_0}$ we automatically recover, by the weak law of large numbers, that $\sigma_{\theta_0}^2 = 1$ (recall the definition of $\sigma_{\theta_0}^2$ from Assumption~\ref{assump:cgmt}). Moreover, we assume that the noise values $z_i$ are i.i.d.\ according to the distribution $\mathbb P_Z = \mathcal N(0,0.1)$, resulting in $\sigma_Z^2 = 0.1$. 

\begin{figure}
    \center
	\subfigure[Type-$1$ Wasserstein DRE]{\includegraphics[width=0.49\columnwidth]{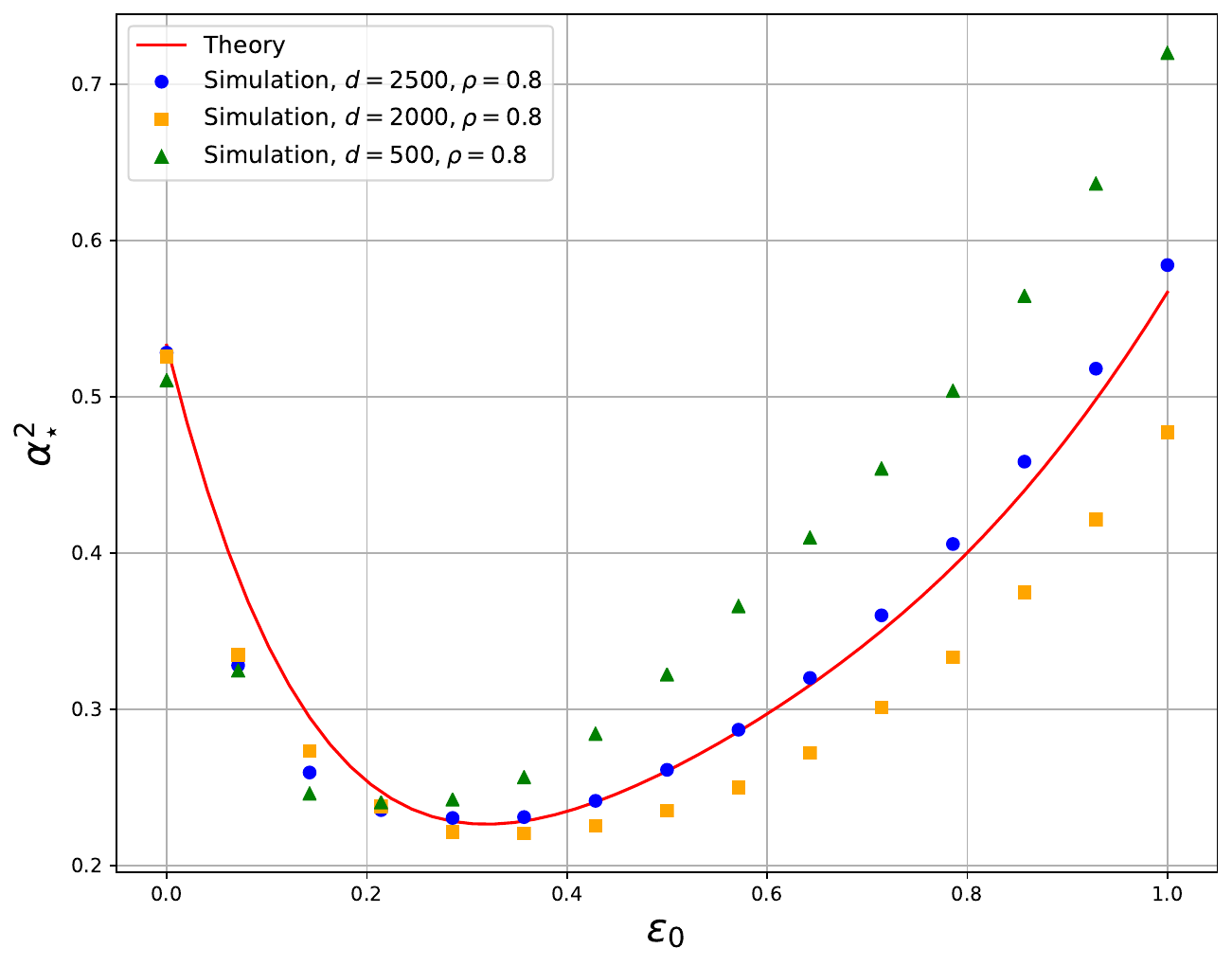}} ~
	\subfigure[Type-$2$ Wasserstein DRE]{\includegraphics[width=0.49\columnwidth]{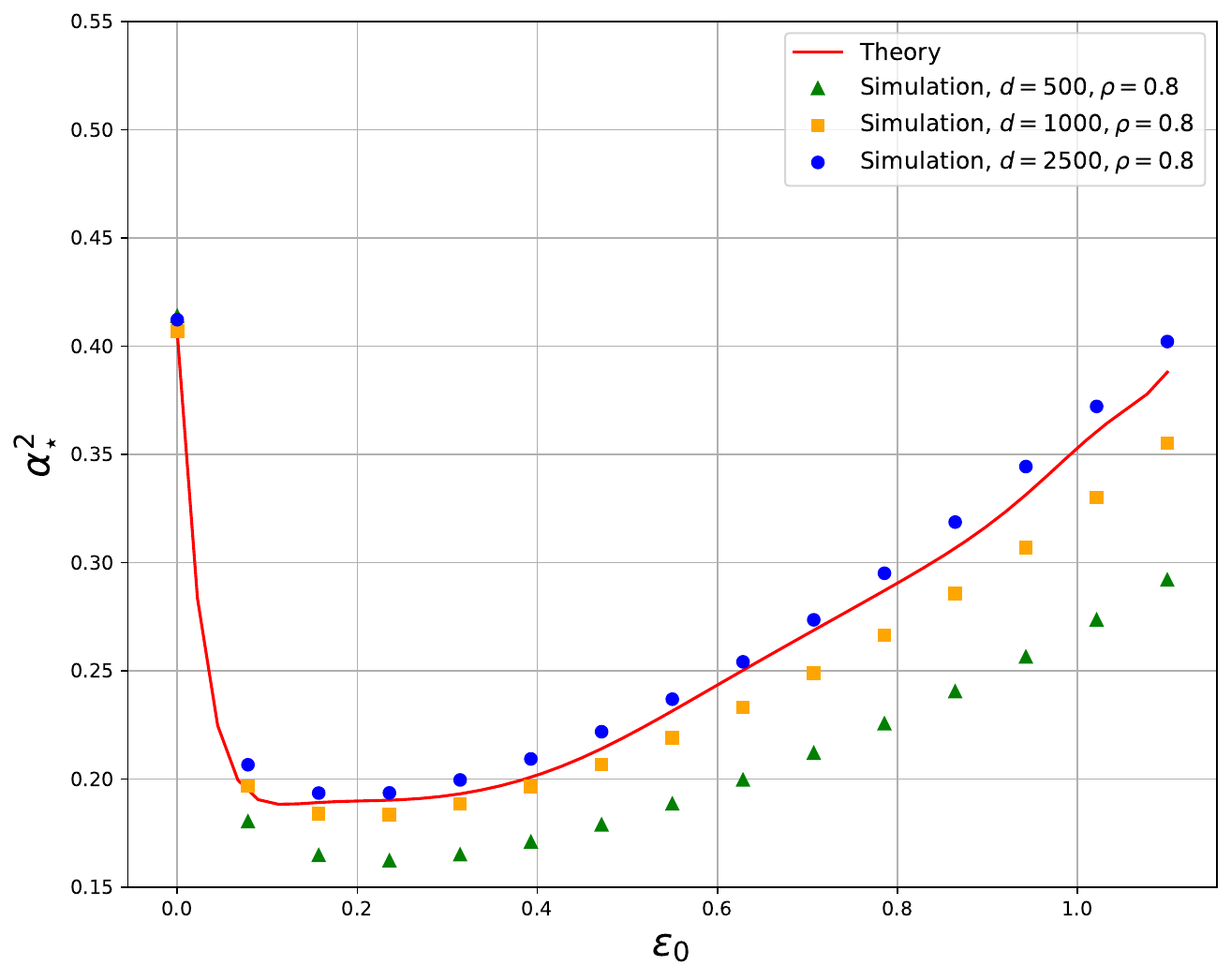}}
    \caption{Theory vs simulation for different choices of $d$ and $\rho = d/n = 0.8$.}
    \label{figure:high:W1:W2}
\end{figure}

We start by considering the type-$1$ Wasserstein DRE problem with loss function $L(\cdot) = |\cdot|$. Notice that Assumption~\ref{assump:thrampoulidis}\ref{assump:thrampoulidis} is trivially satisfied. Following Theorem~\ref{thm:estimationerror:1}, the asymptotic normalized squared error $\lim_{d \to \infty} {\|\hat{\theta}_{\mathrm{DRE}}-\theta_0\|^2}/{d}$ is equal to $\alpha_\star^2$, where $\alpha_\star$ is the solution of the convex-concave minimax scalar problem \eqref{eq:est:1:minimax}. Since the Moreau envelope is equal to the Huber loss (see Example~\ref{example:Moreau:LAD}), we have that the expected Moreau envelope can be easily recovered in closed-form for our choice of $\mathbb P_Z$. Finally, we consider the case where $\rho = d/n = 0.8$. In Figure~\ref{figure:high:W1:W2}(a) we show in red the asymptotic normalized squared error $\alpha_\star^2$ as a function of the radius $\varepsilon_0$ (recall that $\varepsilon = \varepsilon_0 / \sqrt{n}$). We now investigate to what extent the asymptotic error $\alpha_\star^2$ is informative in realistic scenarios, where the dimension $d$ might be large but nonetheless finite. For this, we show on the same plot the normalized squared error averaged over $5$ independent realizations for dimensions $d = 500, 2000$, and $2500$, while keeping fixed the value of $\rho$ at $0.8$. It is important to notice that, although there is a slight discrepancy between the asymptotic and the three non-asymptotic error curves, the asymptotic curve captures the true behavior (i.e., the minimum normalized squared error is roughly the same for all of them). In turn, this allows for a computationally very efficient way to tune the radius $\varepsilon_0$ using the developed asymptotic tools for cases where the dimension is as low as $500$. 

We now proceed to the type-$2$ Wasserstein DRE problem with loss function $L(\cdot) = (\cdot)^2$. In this case, the asymptotic normalized squared error $\alpha_\star^2$ can be recovered from Proposition~\ref{prop:est:error:squared:loss} under less restrictive assumptions than Theorem~\ref{thm:estimationerror:2}. In Figure~\ref{figure:high:W1:W2}(b) we plot in red the value of $\alpha_\star^2$ as a function of the radius $\varepsilon_0$. Moreover, we plot the normalized squared error averaged over $5$ independent realizations for dimensions $d = 500, 1000$, and $2500$, while keeping fixed the value of $\rho$ at $0.8$. Similarly to the type-$1$ Wasserstein case, the asymptotic error curve captures very well the true behavior of the error even when the dimension is as low as $500$. 

We now study the effect of the under/over-parametrization $\rho = d/n$ on the normalized squared estimation error. We consider again the type-$2$ Wasserstein DRE problem described above for the two values $\rho = 0.8$ (under-parametrized DRE problem) and $\rho = 1.2$ (over-parametrized DRE problem). In Figure~\ref{figure:high:universality}(a) we plot both the asymptotic and the non-asymptotic (for $d=2500$) error curves, which show that a smaller estimation error can be obtained in the case $\rho = 0.8$, i.e., where the number of samples is larger than the problem dimension. Moreover, we explore numerically the validity of our theoretical results beyond the assumption that the features are Gaussian. For this, we consider the cases where the entries of the vectors $x_i$ are i.i.d. Bernoulli or Poisson random variables with the same mean and variance as in Assumption~\ref{assump:0}\ref{assump:features}. We consider the type-$2$ Wasserstein DRE problem described above, and we plot the results in Figure~\ref{figure:high:universality}(b), where it can be seen that the theoretical results allow us to find the right tuning of the radius $\varepsilon_0$ even in cases where the features fail to be Gaussian.

\begin{figure}
    \center
	\subfigure[Under-parametrization vs over-parametrization]{\includegraphics[width=0.49\columnwidth]{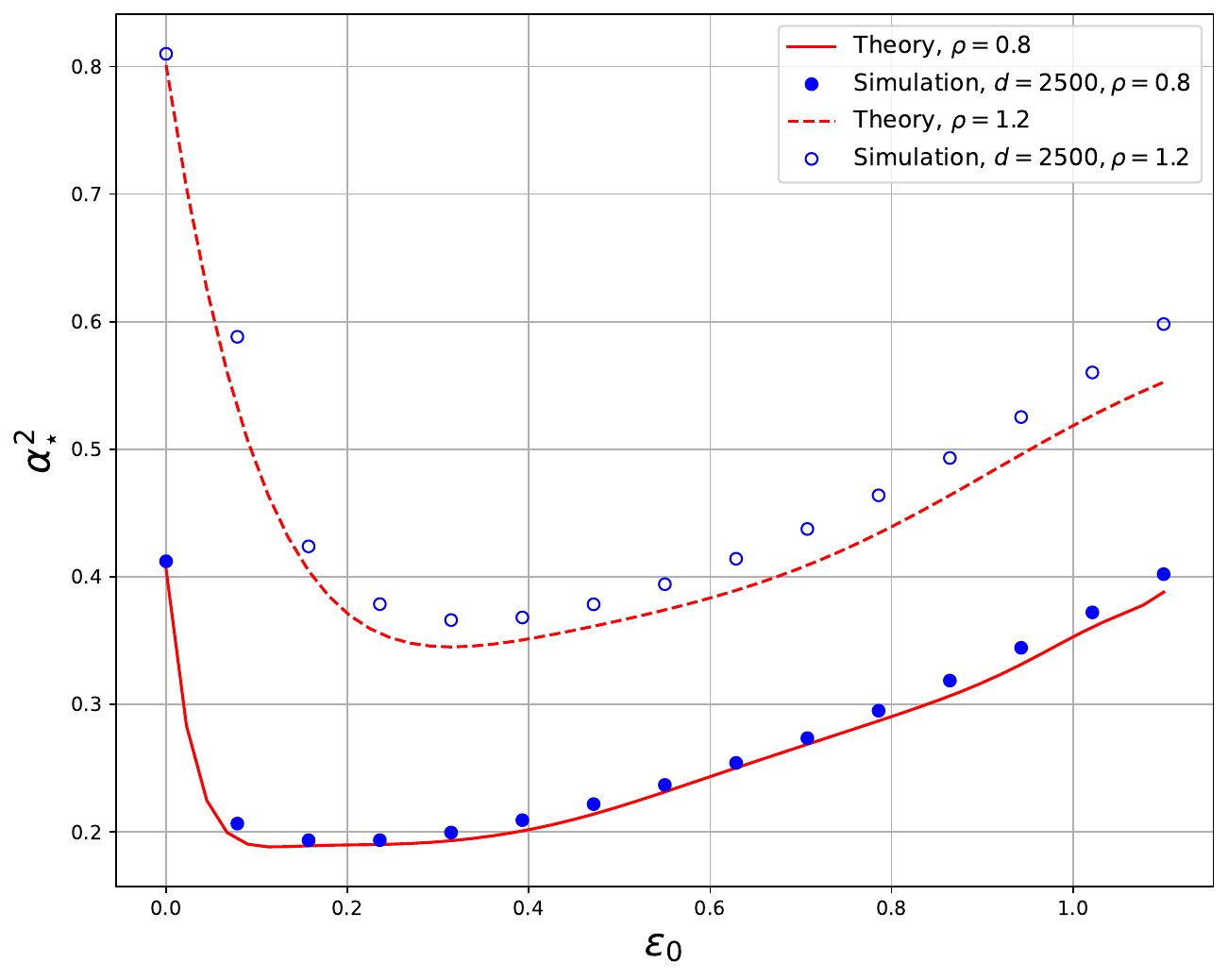}} ~
	\subfigure[Universality]{\includegraphics[width=0.49\columnwidth]{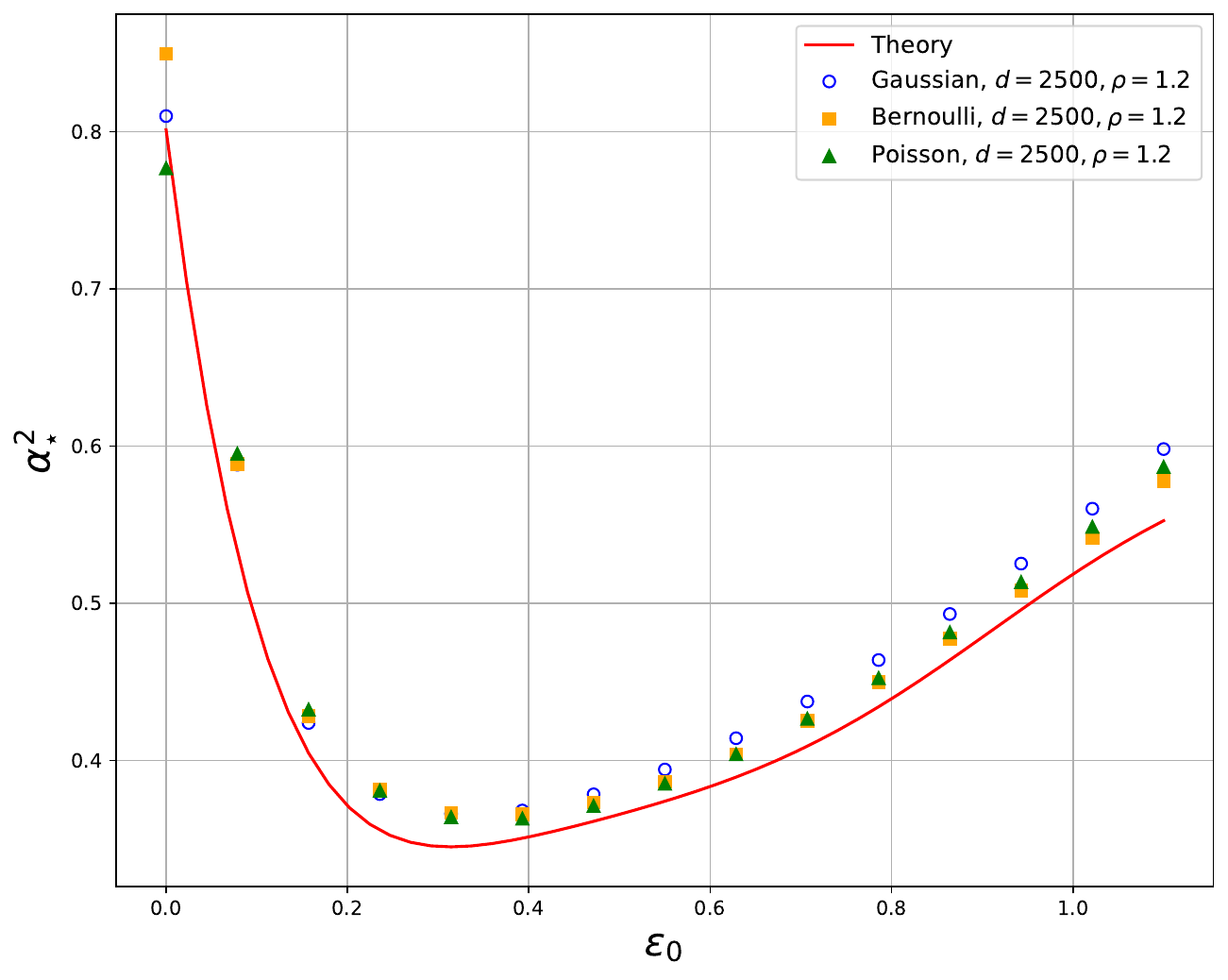}}
    \caption{(a) The effect of $\rho = d/n$ on the estimation error; (b) the validity of the results extends to broader classes of probability ensembles.}
    \label{figure:high:universality}
\end{figure}

Finally, we numerically validate the type-2 Wasserstein distributionally regularized estimator presented in Theorem~\ref{thm:estimationerror:3}. For simplicity, we focus on the case where $L(\cdot)=(\cdot)^2$, so that the asymptotic error can be computed using the scalar minimax optimization problem in Corollary~\ref{cor:regdro:squared:loss}. We consider
the exact same scenario as in the previous experiments, with $\sigma_{\theta_0}^2 = 1$ and $\sigma_Z^2 = 0.1$. However, differently than above, the radius $\varepsilon_0$ is replaced by the regularization parameter $\lambda_0$. Since the theoretical guarantees only hold for $\lambda_0 \geq R_{\theta}^2$, with $R_{\theta}$ defined in Assumption~\ref{assump:Theta}, we are incentivized to pick the value of $R_{\theta}$ as close as possible to the norm of the true parameter $\theta_0$. We thus pick $R_{\theta} = \sigma_{\theta_0}$, acknowledging that this value will result in an optimal solution $\alpha_\star$ to the scalar optimization problem \eqref{eq:DReg_quadratic} that is only an upper bound on the true asymptotic normalized squared error. We illustrate this in Figure~\ref{fig:W2_DReg}. Notice that, in this case, picking a value of $R_\theta$ larger than $\sigma_{\theta_0}$ will result in a lower bound on $\lambda_0$ which is much larger than $1$. This, in turn, will cause the asymptotic curve to miss the value of $\lambda_0$ which attains the minimum error. As in the previous cases, it is important to highlight that, although we only obtain an upper bound on the true estimation error in high dimensions, the asymptotic curve captures the true behavior of the error even for dimensions as low as $d=500$.

\subsection{Comparison with Cross-Validation}
\label{subsec:cross:val}

\begin{figure}
    \centering
    \includegraphics[width=0.49\linewidth]{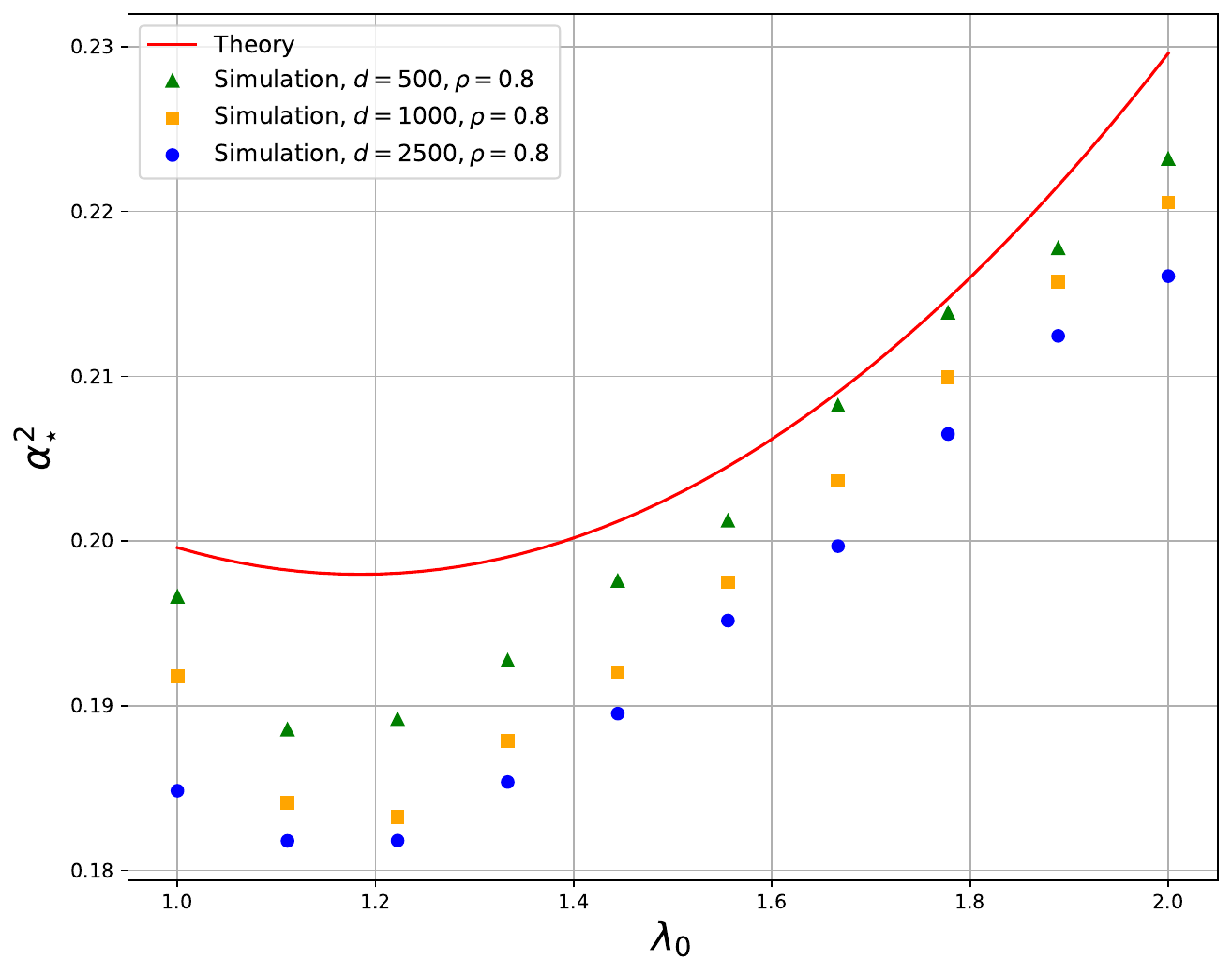}
    \caption{Type-$2$ Wasserstein distributional regularization.}
    \label{fig:W2_DReg}
\end{figure}

We conclude this section by comparing our theory-driven approach to hyperparameter tuning via the standard cross-validation procedure commonly used in practice. While our method selects the robustness radius $\varepsilon_0$ by minimizing the asymptotic estimation error derived in Theorems~\ref{thm:estimationerror:1}, \ref{thm:estimationerror:2}, and \ref{thm:estimationerror:3}, cross-validation selects $\varepsilon_0$ by minimizing the empirical loss on test data. Although both strategies aim to identify an appropriate level of regularization, their objectives are fundamentally different: in estimation problems, the goal is to recover the true signal $\theta_0$ as accurately as possible, as measured by the squared error $\|\hat{\theta}_{\mathrm{DRE}} - \theta_0\|_2^2$, not merely to minimize empirical loss. Since $\theta_0$ is typically unknown, cross-validation is used in practice as a surrogate. As we demonstrate below, our asymptotic formula not only targets the correct estimation objective, but also yields tuning decisions that closely match those obtained via cross-validation, at a small fraction of the computational cost. This distinction becomes particularly relevant in high-dimensional settings, where cross-validation may be prohibitively expensive.

Figure~\ref{fig:crossval} reports the results of this comparison for both the type-$1$ and type-$2$ Wasserstein DRE estimators. We consider the same setting as in the previous experiments: the dimension is $d = 2500$, the underparametrization ratio is $\rho = d/n = 0.8$, and the distributions of $\theta_0$ and $Z$ are as described in Section~\ref{subsec:radius:tuning}. For each value of $\varepsilon_0$, we compute the out-of-sample loss via $5$-fold cross-validation (green curve). Specifically, we solve the DRE problem using $n = d/\rho$ samples, obtain $\hat{\theta}_{\mathrm{DRE}}$, and evaluate its performance on $1000$ test points, repeating the procedure five times and averaging the results. We also plot the theoretical asymptotic estimation error $\alpha_\star^2$ (red curve), and the empirical estimation error $\|\hat{\theta}_{\mathrm{DRE}} - \theta_0\|_2^2$ (blue curve), computed as in Section~\ref{subsec:radius:tuning}. Each plot features two y-axes: the left for estimation error and the right for the cross-validated loss.

We observe that the empirical loss and the estimation error exhibit qualitatively similar behavior across different values of $\varepsilon_0$, and are minimized at nearly the same value of the radius. This finding suggests that empirical loss minimization can indeed serve as a reasonable proxy for estimation error in this setting. However, it also reinforces the strength of our theoretical analysis: our method recovers a near-optimal value of $\varepsilon_0$ directly from the asymptotic formula, without the need to repeatedly solve high-dimensional DRE problems as required by cross-validation. In fact, our approach only requires solving a four-dimensional scalar optimization problem, making it substantially more efficient.

\begin{figure}
    \centering
    \subfigure[Type-$1$ Wasserstein DRE]{\includegraphics[width=0.49\linewidth]{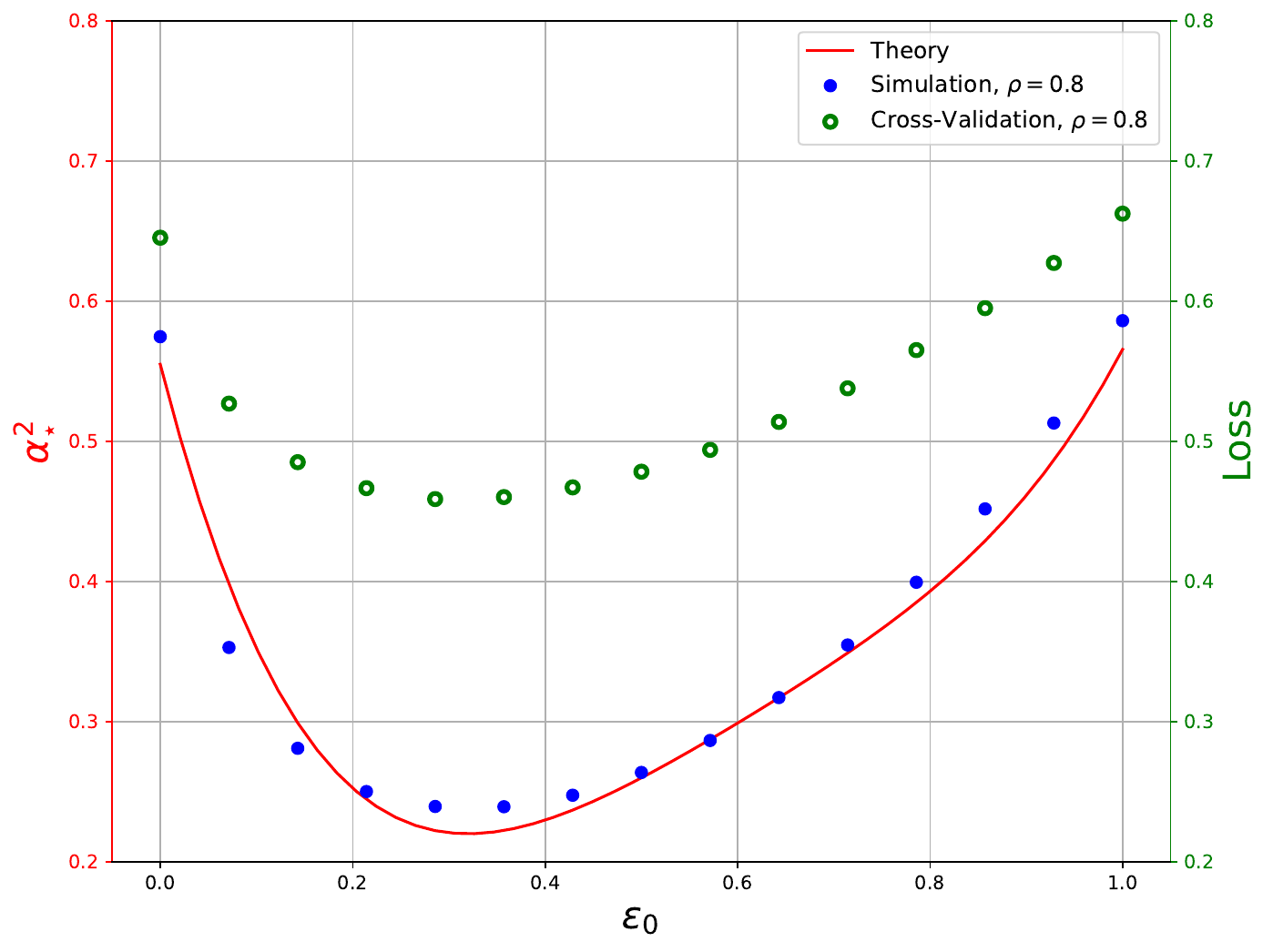}} \hfill
    \subfigure[Type-$2$ Wasserstein DRE]{\includegraphics[width=0.49\linewidth]{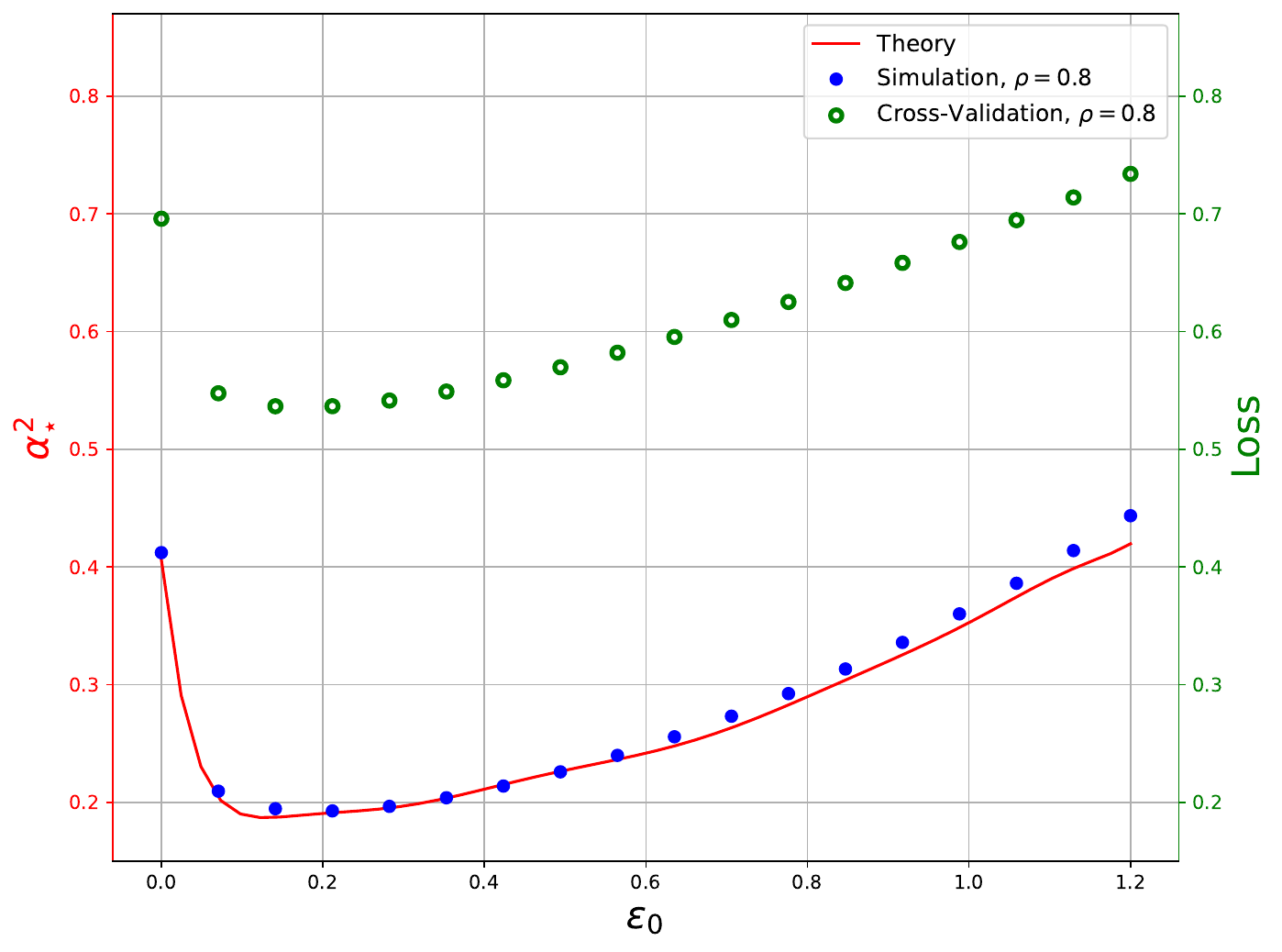}}
    \caption{Comparison between theory-based and cross-validation-based selection of $\varepsilon$ in Wasserstein-1 and Wasserstein-2 DRE.}
    \label{fig:crossval}
\end{figure}

Finally, we emphasize that our plots also highlight the benefits of robustness over empirical risk minimization (ERM). In all cases, the choice $\varepsilon_0 = 0$ corresponds to ERM, which serves as a natural baseline. As seen in Figure~\ref{fig:crossval}, a moderate level of robustness (i.e., $\varepsilon_0 > 0$) significantly improves the estimation error, further confirming the practical value of the proposed approach.
%-------------------------------------------------------------------------------------------------
%-------------------------------------------------------------------------------------------------
%-------------------------------------------------------------------------------------------------

\section{Conclusion and Future Work}
\label{sec:conclusion}

In this paper, we have addressed the problem of quantifying the performance of Wasserstein distributionally robust estimators in high dimensions. To do so, we have adopted a modern high-dimensional statistical viewpoint, and we have shown that in the high-dimensional proportional regime, and under the assumption of isotropic Gaussian features, the squared error of such estimators can be easily recovered from the solution of a convex-concave minimax problem which involves at most four scalar variables. We envision numerous interesting directions in which the results of this paper can be employed, extended, or serve as starting point for future work:
\begin{itemize}
    \item[1.] \emph{Beyond Gaussian features.} The analysis presented in this paper
    relies on the assumption that the measurement matrix $A$ has entries i.i.d.\ $\mathcal N(0,1/d)$. One future direction is to extend these results to the case where $A$ is an isotropically random orthogonal matrix, i.e., a matrix sampled uniformly at random from the manifold of row-orthogonal matrices satisfying $A A^\top = I$. This class of measurement matrices is known to be practically relevant due to the fact that their condition number is $1$, and they do not amplify the noise. In \cite{thrampoulidis2015isotropically} the authors extended the CGMT framework to recover the squared error of LASSO, unde the assumptions that $A$ is isotropically random orthogonal and $Z$ is Gaussian. Similar ideas could be employed to study the performance of Wasserstein distributionally robust estimators isotropically random orthogonal measurement matrices. 

    \item[2.] \emph{Universality.} An important open question is whether our analysis extends to broader classes of feature distributions. Recent progress provides promising tools in this direction: \cite{han2023universality} establish CGMT-based universality results for convex-regularized regression under Gaussian design, offering the most direct connection to our setting. While type-1 Wasserstein DRE reduces to a convex regularized form and may fall within their framework, general type-2 Wasserstein DRE—where our main contributions lie—requires distinct analytical techniques. Nevertheless, dual reformulations of DRE problems may offer a promising pathway. Further insights may also be drawn from recent extensions of CGMT to high-dimensional logistic regression under dependence \cite{mallory2025universality}, and from more general frameworks for universality in empirical risk minimization beyond CGMT, such as \cite{montanari2022universality} and the references therein.
    
    \item[3.] \emph{Fundamental limits and optimal tuning of $L$ and $\varepsilon$.} The results presented in this paper open the path towards answering the following two fundamental question: \textit{\say{What is the minimum estimation error $\alpha_{min}$ achievable by $\hat\theta_{DRO}$ as a function of $\rho$?}} and \textit{\say{How to optimally choose $L$ and $\varepsilon$ in order to achieve $\alpha_{min}$?}}. So far, these and similar type questions have been addressed in the context of unregularized estimators \cite{bean2013optimal,donoho2016high,advani2016statistical}, ridge-regularized estimation \cite{taheri2021fundamental}, and convex-regularized least-squares for linear models over structured (e.g., sparse, low-rank) signals \cite{celentano2022fundamental}. In particular, the methodology proposed in \cite{taheri2021fundamental} builds upon the CGMT framework, and could be employed to provide an answer to the aforementioned questions in the context of Wasserstein distributionally robust estimators.
    
    \item[4.] \emph{Binary models.} The analysis and results presented in this paper could be extended to binary observation models, following similar ideas as in \cite{salehi2019impact, taheri2020sharp}. In particular, the problem \emph{fundamental limits and optimal tuning of $L$ and $\varepsilon$} could also be studied for these models, similarly to what was done in \cite{taheri2020sharp} for Empirical Risk Minimation estimators.
    
    \item[5.] \emph{f-Divergence distributionally robust estimation.} The reasoning presented in this paper could be employed to recover the error of other distributionally robust estimation formulations, where the ambiguity set is constructed using a different statistical distance on the probability space. In particular, the case where the statistical distance is chosen to be an f-Divergence has recently received a lot of attention in Operations Research and Machine Learning, due to its computational tractability \citep{ben2013robust,shapiro2017distributionally,namkoong2016stochastic}, variance regularization effect \citep{lam2016robust,lam2019recovering,duchi2019variance}, and optimality in terms of out-of-sample disappointment \citep{van2021data,sutter2020general,bennouna2021learning}. Recovering the error of f-Divergence distributionally robust estimators would enable a direct comparison between the performances of these two very popular distributionally robust formulations. 
    
    \item[6.] \emph{Fundamental trade-offs in distributional adversarial training.} Most machine learning models are known to be highly vulnerable to small adversarial perturbations to their inputs \citep{szegedy2013intriguing}. For example, in the context of image classification, even small perturbations of the image, which are imperceptible to a human, can lead to incorrect classification by these models \citep{goodfellow2014explaining, moosavi2016deepfool}. To solve this issue, many works have proposed \emph{adversarial training} as a principled and effective technique to improve robustness of machine learning models against adversarial perturbations (see \citep{silva2020opportunities} for a survey). Nonetheless, while adversarial training has been successful at improving the accuracy of the trained model on adversarially perturbed inputs (\emph{robust accuracy}), often this success comes at the cost of decreasing the accuracy on natural unperturbed inputs (\emph{standard accuracy}). Understanding the tradeoff between the standard and robust accuracies has been the subject of many recent works, which considered the adversarial setting with $\ell_p$-norm-bounded perturbations \citep{madry2017towards,schmidt2018adversarially,tsipras2018robustness,raghunathan2019adversarial,zhang2019theoretically,javanmard2020precisetradeoffs,javanmard2020precisestatistical,dobriban2020provable,taheri2020asymptotic,dong2020adversarial,min2021curious}. However, the $\ell_p$-norm can be a poor adversarial setting when dealing, for example, with images: two similar images for human perception are not necessarily close under the $\ell_p$-norm, and a slight translation or rotation of an image may change the $\ell_p$-distance drastically. Interestingly, these limitations are generally not present when the adversarial attacks are modeled using the Wasserstein distance, which has been shown to effectively capture the geometric information in the image space (see \citep{wong2019wasserstein,wu2020stronger} and the references therein). Consequently, the Wasserstein adversarial model, often referred to as \emph{distribution shifts}, has recently received considerable attention  \citep{staib2017distributionally,sinha2018certifying,pydi2020adversarial,mehrabi2021fundamental,li2021internal}. Despite this, an understanding of the fundamental tradeoff between the standard and robust accuracies in the \emph{Wasserstein adversarial training} is still lacking. The results presented in this paper could serve as a starting point for this problem. Specifically, we believe that this problem could be solved by employing the same reasoning as in \cite{javanmard2020precisetradeoffs}, which characterized this fundamental tradeoff for squared losses and $\ell_2$-norm-bounded adversarial attacks by: (i) analyzing the Pareto optimal points of the two dimensional region consisting of all the achievable (standard risk, robust risk) pairs; (ii) characterizing precisely the algorithmic standard and robust risks achieved by the $\ell_2$-norm adversarially trained estimator (this has been done in the high-dimensional regime, where $d,n \to \infty$ with $d/n \to \rho$, using the CGMT framework); (iii) showing that, as $\rho$ decreases, the algorithmic tradeoff curve approaches the Pareto-optimal tradeoff curve. In the context of Wasserstein adversarial training, the Pareto-optimal tradeoff curve from step (i) has been characterized in \cite{mehrabi2021fundamental} for the particular case of type-$2$ Wasserstein discrepancy and squared losses, where the distributional robustness simplifies to the norm-regularized formulation presented in equation \eqref{eq:reg:for:squared:loss}.
\end{itemize}

\subsection*{Statements and Declarations}
The authors declare that they have no competing interests.

\bibliographystyle{abbrvnat} 
\bibliography{bibfile.bib}

\begin{thebibliography}{135}
\providecommand{\natexlab}[1]{#1}
\providecommand{\url}[1]{\texttt{#1}}
\expandafter\ifx\csname urlstyle\endcsname\relax
  \providecommand{\doi}[1]{doi: #1}\else
  \providecommand{\doi}{doi: \begingroup \urlstyle{rm}\Url}\fi

\bibitem[Abbasi et~al.(2019)Abbasi, Salehi, and Hassibi]{abbasi2019universality}
E.~Abbasi, F.~Salehi, and B.~Hassibi.
\newblock Universality in learning from linear measurements.
\newblock \emph{Advances in Neural Information Processing Systems}, 32, 2019.

\bibitem[Advani and Ganguli(2016)]{advani2016statistical}
M.~Advani and S.~Ganguli.
\newblock Statistical mechanics of optimal convex inference in high dimensions.
\newblock \emph{Physical Review X}, 6\penalty0 (3):\penalty0 031034, 2016.

\bibitem[Advani et~al.(2020)Advani, Saxe, and Sompolinsky]{advani2020high}
M.~S. Advani, A.~M. Saxe, and H.~Sompolinsky.
\newblock High-dimensional dynamics of generalization error in neural networks.
\newblock \emph{Neural Networks}, 132:\penalty0 428--446, 2020.

\bibitem[Allen-Zhu et~al.(2019{\natexlab{a}})Allen-Zhu, Li, and Liang]{allen2019learning}
Z.~Allen-Zhu, Y.~Li, and Y.~Liang.
\newblock Learning and generalization in overparameterized neural networks, going beyond two layers.
\newblock \emph{Advances in neural information processing systems}, 32, 2019{\natexlab{a}}.

\bibitem[Allen-Zhu et~al.(2019{\natexlab{b}})Allen-Zhu, Li, and Song]{allen2019convergence}
Z.~Allen-Zhu, Y.~Li, and Z.~Song.
\newblock A convergence theory for deep learning via over-parameterization.
\newblock In \emph{International Conference on Machine Learning}, pages 242--252. PMLR, 2019{\natexlab{b}}.

\bibitem[Amelunxen et~al.(2014)Amelunxen, Lotz, McCoy, and Tropp]{amelunxen2014living}
D.~Amelunxen, M.~Lotz, M.~B. McCoy, and J.~A. Tropp.
\newblock Living on the edge: Phase transitions in convex programs with random data.
\newblock \emph{Information and Inference: A Journal of the IMA}, 3\penalty0 (3):\penalty0 224--294, 2014.

\bibitem[Arora et~al.(2019)Arora, Du, Hu, Li, and Wang]{arora2019fine}
S.~Arora, S.~Du, W.~Hu, Z.~Li, and R.~Wang.
\newblock Fine-grained analysis of optimization and generalization for overparameterized two-layer neural networks.
\newblock In \emph{International Conference on Machine Learning}, pages 322--332. PMLR, 2019.

\bibitem[Bartl et~al.(2020)Bartl, Drapeau, Obloj, and Wiesel]{bartl2020robust}
D.~Bartl, S.~Drapeau, J.~Obloj, and J.~Wiesel.
\newblock Robust uncertainty sensitivity analysis.
\newblock \emph{arXiv preprint arXiv:2006.12022}, 2020.

\bibitem[Bartlett et~al.(2020)Bartlett, Long, Lugosi, and Tsigler]{bartlett2020benign}
P.~L. Bartlett, P.~M. Long, G.~Lugosi, and A.~Tsigler.
\newblock Benign overfitting in linear regression.
\newblock \emph{Proceedings of the National Academy of Sciences}, 117\penalty0 (48):\penalty0 30063--30070, 2020.

\bibitem[Bartlett et~al.(2021)Bartlett, Montanari, and Rakhlin]{bartlett2021deep}
P.~L. Bartlett, A.~Montanari, and A.~Rakhlin.
\newblock Deep learning: a statistical viewpoint.
\newblock \emph{Acta numerica}, 30:\penalty0 87--201, 2021.

\bibitem[Bayati and Montanari(2011)]{bayati2011lasso}
M.~Bayati and A.~Montanari.
\newblock The {LASSO} risk for gaussian matrices.
\newblock \emph{IEEE Transactions on Information Theory}, 58\penalty0 (4):\penalty0 1997--2017, 2011.

\bibitem[Bayati et~al.(2015)Bayati, Lelarge, and Montanari]{bayati2015universality}
M.~Bayati, M.~Lelarge, and A.~Montanari.
\newblock Universality in polytope phase transitions and message passing algorithms.
\newblock \emph{The Annals of Applied Probability}, 25\penalty0 (2):\penalty0 753--822, 2015.

\bibitem[Bean et~al.(2013)Bean, Bickel, El~Karoui, and Yu]{bean2013optimal}
D.~Bean, P.~J. Bickel, N.~El~Karoui, and B.~Yu.
\newblock Optimal {M}-estimation in high-dimensional regression.
\newblock \emph{Proceedings of the National Academy of Sciences}, 110\penalty0 (36):\penalty0 14563--14568, 2013.

\bibitem[Belkin et~al.(2018)Belkin, Ma, and Mandal]{belkin2018understand}
M.~Belkin, S.~Ma, and S.~Mandal.
\newblock To understand deep learning we need to understand kernel learning.
\newblock In \emph{International Conference on Machine Learning}, pages 541--549. PMLR, 2018.

\bibitem[Belkin et~al.(2019)Belkin, Hsu, Ma, and Mandal]{belkin2019reconciling}
M.~Belkin, D.~Hsu, S.~Ma, and S.~Mandal.
\newblock Reconciling modern machine-learning practice and the classical bias--variance trade-off.
\newblock \emph{Proceedings of the National Academy of Sciences}, 116\penalty0 (32):\penalty0 15849--15854, 2019.

\bibitem[Belkin et~al.(2020)Belkin, Hsu, and Xu]{belkin2020two}
M.~Belkin, D.~Hsu, and J.~Xu.
\newblock Two models of double descent for weak features.
\newblock \emph{SIAM Journal on Mathematics of Data Science}, 2\penalty0 (4):\penalty0 1167--1180, 2020.

\bibitem[Ben-David et~al.(2006)Ben-David, Blitzer, Crammer, and Pereira]{ben2006analysis}
S.~Ben-David, J.~Blitzer, K.~Crammer, and F.~Pereira.
\newblock Analysis of representations for domain adaptation.
\newblock \emph{Advances in neural information processing systems}, 19, 2006.

\bibitem[Ben-Tal et~al.(2013)Ben-Tal, Den~Hertog, De~Waegenaere, Melenberg, and Rennen]{ben2013robust}
A.~Ben-Tal, D.~Den~Hertog, A.~De~Waegenaere, B.~Melenberg, and G.~Rennen.
\newblock Robust solutions of optimization problems affected by uncertain probabilities.
\newblock \emph{Management Science}, 59\penalty0 (2):\penalty0 341--357, 2013.

\bibitem[Bennouna and Van~Parys(2021)]{bennouna2021learning}
M.~Bennouna and B.~P. Van~Parys.
\newblock Learning and decision-making with data: Optimal formulations and phase transitions.
\newblock \emph{arXiv preprint arXiv:2109.06911}, 2021.

\bibitem[Blanchet and Kang(2016)]{blanchet2016sample}
J.~Blanchet and Y.~Kang.
\newblock Sample out-of-sample inference based on {W}asserstein distance.
\newblock \emph{arXiv:1605.01340}, 2016.

\bibitem[Blanchet et~al.(2018)Blanchet, Murthy, and Zhang]{blanchet2018optimal}
J.~Blanchet, K.~Murthy, and F.~Zhang.
\newblock Optimal transport based distributionally robust optimization: {S}tructural properties and iterative schemes.
\newblock \emph{arXiv:1810.02403}, 2018.

\bibitem[Blanchet et~al.(2019{\natexlab{a}})Blanchet, Kang, and Murthy]{blanchet2019robust}
J.~Blanchet, Y.~Kang, and K.~Murthy.
\newblock Robust {W}asserstein profile inference and applications to machine learning.
\newblock \emph{Journal of Applied Probability}, 56\penalty0 (3):\penalty0 830--857, 2019{\natexlab{a}}.

\bibitem[Blanchet et~al.(2019{\natexlab{b}})Blanchet, Murthy, and Si]{blanchet2019confidence}
J.~Blanchet, K.~Murthy, and N.~Si.
\newblock Confidence regions in {W}asserstein distributionally robust estimation.
\newblock \emph{arXiv:1906.01614}, 2019{\natexlab{b}}.

\bibitem[Cand{\`e}s and Sur(2020)]{candes2020phase}
E.~J. Cand{\`e}s and P.~Sur.
\newblock The phase transition for the existence of the maximum likelihood estimate in high-dimensional logistic regression.
\newblock \emph{The Annals of Statistics}, 48\penalty0 (1):\penalty0 27--42, 2020.

\bibitem[Candes and Tao(2006)]{candes2006near}
E.~J. Candes and T.~Tao.
\newblock Near-optimal signal recovery from random projections: Universal encoding strategies?
\newblock \emph{IEEE transactions on information theory}, 52\penalty0 (12):\penalty0 5406--5425, 2006.

\bibitem[Celentano and Montanari(2022)]{celentano2022fundamental}
M.~Celentano and A.~Montanari.
\newblock Fundamental barriers to high-dimensional regression with convex penalties.
\newblock \emph{The Annals of Statistics}, 50\penalty0 (1):\penalty0 170--196, 2022.

\bibitem[Chizat et~al.(2019)Chizat, Oyallon, and Bach]{chizat2019lazy}
L.~Chizat, E.~Oyallon, and F.~Bach.
\newblock On lazy training in differentiable programming.
\newblock \emph{Advances in Neural Information Processing Systems}, 32, 2019.

\bibitem[Deng et~al.(2019)Deng, Kammoun, and Thrampoulidis]{deng2019model}
Z.~Deng, A.~Kammoun, and C.~Thrampoulidis.
\newblock A model of double descent for high-dimensional binary linear classification.
\newblock \emph{arXiv preprint arXiv:1911.05822}, 2019.

\bibitem[Dicker(2014)]{dicker2014variance}
L.~H. Dicker.
\newblock Variance estimation in high-dimensional linear models.
\newblock \emph{Biometrika}, 101\penalty0 (2):\penalty0 269--284, 2014.

\bibitem[Dicker and Erdogdu(2016)]{dicker2016maximum}
L.~H. Dicker and M.~A. Erdogdu.
\newblock Maximum likelihood for variance estimation in high-dimensional linear models.
\newblock In \emph{Artificial Intelligence and Statistics}, pages 159--167. PMLR, 2016.

\bibitem[Dobriban et~al.(2020)Dobriban, Hassani, Hong, and Robey]{dobriban2020provable}
E.~Dobriban, H.~Hassani, D.~Hong, and A.~Robey.
\newblock Provable tradeoffs in adversarially robust classification.
\newblock \emph{arXiv preprint arXiv:2006.05161}, 2020.

\bibitem[Dong et~al.(2020)Dong, Deng, Pang, Zhu, and Su]{dong2020adversarial}
Y.~Dong, Z.~Deng, T.~Pang, J.~Zhu, and H.~Su.
\newblock Adversarial distributional training for robust deep learning.
\newblock \emph{Advances in Neural Information Processing Systems}, 33:\penalty0 8270--8283, 2020.

\bibitem[Donoho and Montanari(2016)]{donoho2016high}
D.~Donoho and A.~Montanari.
\newblock High dimensional robust {M}-estimation: Asymptotic variance via approximate message passing.
\newblock \emph{Probability Theory and Related Fields}, 166\penalty0 (3):\penalty0 935--969, 2016.

\bibitem[Donoho and Tanner(2009{\natexlab{a}})]{donoho2009counting}
D.~Donoho and J.~Tanner.
\newblock Counting faces of randomly projected polytopes when the projection radically lowers dimension.
\newblock \emph{Journal of the American Mathematical Society}, 22\penalty0 (1):\penalty0 1--53, 2009{\natexlab{a}}.

\bibitem[Donoho and Tanner(2009{\natexlab{b}})]{donoho2009observed}
D.~Donoho and J.~Tanner.
\newblock Observed universality of phase transitions in high-dimensional geometry, with implications for modern data analysis and signal processing.
\newblock \emph{Philosophical Transactions of the Royal Society A: Mathematical, Physical and Engineering Sciences}, 367\penalty0 (1906):\penalty0 4273--4293, 2009{\natexlab{b}}.

\bibitem[Donoho et~al.(2011)Donoho, Maleki, and Montanari]{donoho2011noise}
D.~L. Donoho, A.~Maleki, and A.~Montanari.
\newblock The noise-sensitivity phase transition in compressed sensing.
\newblock \emph{IEEE Transactions on Information Theory}, 57\penalty0 (10):\penalty0 6920--6941, 2011.

\bibitem[Du et~al.(2018)Du, Zhai, Poczos, and Singh]{du2018gradient}
S.~S. Du, X.~Zhai, B.~Poczos, and A.~Singh.
\newblock Gradient descent provably optimizes over-parameterized neural networks.
\newblock \emph{arXiv preprint arXiv:1810.02054}, 2018.

\bibitem[Duchi and Namkoong(2019)]{duchi2019variance}
J.~Duchi and H.~Namkoong.
\newblock Variance-based regularization with convex objectives.
\newblock \emph{Journal of Machine Learning Research}, 20\penalty0 (1):\penalty0 2450--2504, 2019.

\bibitem[Duchi and Namkoong(2020)]{duchi2020learning}
J.~Duchi and H.~Namkoong.
\newblock Learning models with uniform performance via distributionally robust optimization.
\newblock \emph{Annals of Statistics (Forthcoming)}, 2020.

\bibitem[Duchi et~al.(2021)Duchi, Glynn, and Namkoong]{duchi2021statistics}
J.~C. Duchi, P.~W. Glynn, and H.~Namkoong.
\newblock Statistics of robust optimization: {A} generalized empirical likelihood approach.
\newblock \emph{Mathematics of Operations Research (Forthcoming)}, 2021.

\bibitem[Durrett(2019)]{durrett2019probability}
R.~Durrett.
\newblock \emph{Probability: theory and examples}, volume~49.
\newblock Cambridge university press, 2019.

\bibitem[El~Karoui(2018)]{el2018impact}
N.~El~Karoui.
\newblock On the impact of predictor geometry on the performance on high-dimensional ridge-regularized generalized robust regression estimators.
\newblock \emph{Probability Theory and Related Fields}, 170\penalty0 (1):\penalty0 95--175, 2018.

\bibitem[Emami et~al.(2020)Emami, Sahraee-Ardakan, Pandit, Rangan, and Fletcher]{emami2020generalization}
M.~Emami, M.~Sahraee-Ardakan, P.~Pandit, S.~Rangan, and A.~Fletcher.
\newblock Generalization error of generalized linear models in high dimensions.
\newblock In \emph{International Conference on Machine Learning}, pages 2892--2901. PMLR, 2020.

\bibitem[Farnia and Tse(2016)]{farnia2016minimax}
F.~Farnia and D.~Tse.
\newblock A minimax approach to supervised learning.
\newblock In \emph{Advances in Neural Information Processing Systems}, pages 4240--4248, 2016.

\bibitem[Finlay and Oberman(2021)]{finlay2021scaleable}
C.~Finlay and A.~M. Oberman.
\newblock Scaleable input gradient regularization for adversarial robustness.
\newblock \emph{Machine Learning with Applications}, 3\penalty0 (100017), 2021.

\bibitem[Finlay et~al.(2018)Finlay, Calder, Abbasi, and Oberman]{finlay2018lipschitz}
C.~Finlay, J.~Calder, B.~Abbasi, and A.~Oberman.
\newblock Lipschitz regularized deep neural networks generalize and are adversarially robust.
\newblock \emph{arXiv preprint arXiv:1808.09540}, 2018.

\bibitem[Fournier and Guillin(2015)]{Fournier2015}
N.~Fournier and A.~Guillin.
\newblock On the rate of convergence in {W}asserstein distance of the empirical measure.
\newblock \emph{Probability Theory and Related Fields}, 162\penalty0 (3-4):\penalty0 707--738, 8 2015.

\bibitem[Gao(2020)]{gao2020finite}
R.~Gao.
\newblock Finite-sample guarantees for {W}asserstein distributionally robust optimization: {B}reaking the curse of dimensionality.
\newblock \emph{arXiv:2009.04382}, 2020.

\bibitem[Gao et~al.(2020)Gao, Chen, and Kleywegt]{gao2017wasserstein}
R.~Gao, X.~Chen, and A.~J. Kleywegt.
\newblock Wasserstein distributionally robust optimization and variation regularization.
\newblock \emph{arXiv:1712.06050v3}, 2020.

\bibitem[Gerbelot et~al.(2020)Gerbelot, Abbara, and Krzakala]{gerbelot2020asymptotic}
C.~Gerbelot, A.~Abbara, and F.~Krzakala.
\newblock Asymptotic errors for teacher-student convex generalized linear models (or: How to prove kabashima's replica formula).
\newblock \emph{arXiv preprint arXiv:2006.06581}, 2020.

\bibitem[Goodfellow et~al.(2014)Goodfellow, Shlens, and Szegedy]{goodfellow2014explaining}
I.~J. Goodfellow, J.~Shlens, and C.~Szegedy.
\newblock Explaining and harnessing adversarial examples.
\newblock \emph{arXiv preprint arXiv:1412.6572}, 2014.

\bibitem[Gordon(1988)]{gordon1988milman}
Y.~Gordon.
\newblock On {M}ilman's inequality and random subspaces which escape through a mesh in ℝ$^n$.
\newblock In \emph{Geometric aspects of functional analysis}, pages 84--106. Springer, 1988.

\bibitem[Gulrajani et~al.(2017)Gulrajani, Ahmed, Arjovsky, Dumoulin, and Courville]{gulrajani2017improved}
I.~Gulrajani, F.~Ahmed, M.~Arjovsky, V.~Dumoulin, and A.~Courville.
\newblock Improved training of {W}asserstein {GAN}s.
\newblock In \emph{Advances in Neural Information Processing Systems}, pages 5769--5779, 2017.

\bibitem[Han and Shen(2023)]{han2023universality}
Q.~Han and Y.~Shen.
\newblock Universality of regularized regression estimators in high dimensions.
\newblock \emph{The Annals of Statistics}, 51\penalty0 (4):\penalty0 1799--1823, 2023.

\bibitem[Hastie et~al.(2022)Hastie, Montanari, Rosset, and Tibshirani]{hastie2022surprises}
T.~Hastie, A.~Montanari, S.~Rosset, and R.~J. Tibshirani.
\newblock Surprises in high-dimensional ridgeless least squares interpolation.
\newblock \emph{The Annals of Statistics}, 50\penalty0 (2):\penalty0 949--986, 2022.

\bibitem[Hein and Andriushchenko(2017)]{hein2017formal}
M.~Hein and M.~Andriushchenko.
\newblock Formal guarantees on the robustness of a classifier against adversarial manipulation.
\newblock In \emph{Advances in Neural Information Processing Systems}, pages 2263--2273, 2017.

\bibitem[Hu and Lu(2019)]{hu2019asymptotics}
H.~Hu and Y.~M. Lu.
\newblock Asymptotics and optimal designs of slope for sparse linear regression.
\newblock In \emph{2019 IEEE International Symposium on Information Theory (ISIT)}, pages 375--379. IEEE, 2019.

\bibitem[Hu and Lu(2020)]{hu2020universality}
H.~Hu and Y.~M. Lu.
\newblock Universality laws for high-dimensional learning with random features.
\newblock \emph{arXiv preprint arXiv:2009.07669}, 2020.

\bibitem[Huang(2017)]{huang2017asymptotic}
H.~Huang.
\newblock Asymptotic behavior of support vector machine for spiked population model.
\newblock \emph{The Journal of Machine Learning Research}, 18\penalty0 (1):\penalty0 1472--1492, 2017.

\bibitem[Huber(2011)]{huber2011robust}
P.~J. Huber.
\newblock Robust statistics.
\newblock In \emph{International encyclopedia of statistical science}, pages 1248--1251. Springer, 2011.

\bibitem[Jacot et~al.(2018)Jacot, Gabriel, and Hongler]{jacot2018neural}
A.~Jacot, F.~Gabriel, and C.~Hongler.
\newblock Neural tangent kernel: Convergence and generalization in neural networks.
\newblock \emph{Advances in neural information processing systems}, 31, 2018.

\bibitem[Jakubovitz and Giryes(2018)]{jakubovitz2018improving}
D.~Jakubovitz and R.~Giryes.
\newblock Improving {DNN} robustness to adversarial attacks using {J}acobian regularization.
\newblock In \emph{European Conference on Computer Vision}, pages 514--529, 2018.

\bibitem[Janson et~al.(2017)Janson, Barber, and Candes]{janson2017eigenprism}
L.~Janson, R.~F. Barber, and E.~Candes.
\newblock Eigenprism: inference for high dimensional signal-to-noise ratios.
\newblock \emph{Journal of the Royal Statistical Society Series B: Statistical Methodology}, 79\penalty0 (4):\penalty0 1037--1065, 2017.

\bibitem[Javanmard and Soltanolkotabi(2020)]{javanmard2020precisestatistical}
A.~Javanmard and M.~Soltanolkotabi.
\newblock Precise statistical analysis of classification accuracies for adversarial training.
\newblock \emph{arXiv preprint arXiv:2010.11213}, 2020.

\bibitem[Javanmard et~al.(2020)Javanmard, Soltanolkotabi, and Hassani]{javanmard2020precisetradeoffs}
A.~Javanmard, M.~Soltanolkotabi, and H.~Hassani.
\newblock Precise tradeoffs in adversarial training for linear regression.
\newblock In \emph{Conference on Learning Theory}, pages 2034--2078. PMLR, 2020.

\bibitem[Kammoun and AlouiniFellow(2021)]{kammoun2021precise}
A.~Kammoun and M.-S. AlouiniFellow.
\newblock On the precise error analysis of support vector machines.
\newblock \emph{IEEE Open Journal of Signal Processing}, 2:\penalty0 99--118, 2021.

\bibitem[Karoui(2013)]{karoui2013asymptotic}
N.~E. Karoui.
\newblock Asymptotic behavior of unregularized and ridge-regularized high-dimensional robust regression estimators: rigorous results.
\newblock \emph{arXiv preprint arXiv:1311.2445}, 2013.

\bibitem[Kini and Thrampoulidis(2020)]{kini2020analytic}
G.~R. Kini and C.~Thrampoulidis.
\newblock Analytic study of double descent in binary classification: The impact of loss.
\newblock In \emph{2020 IEEE International Symposium on Information Theory (ISIT)}, pages 2527--2532. IEEE, 2020.

\bibitem[Kwon et~al.(2020)Kwon, Kim, Won, and Paik]{kwon2020principled}
Y.~Kwon, W.~Kim, J.-H. Won, and M.~C. Paik.
\newblock Principled learning method for {W}asserstein distributionally robust optimization with local perturbations.
\newblock In \emph{International Conference on Machine Learning}, pages 5567--5576, 2020.

\bibitem[Lam(2016)]{lam2016robust}
H.~Lam.
\newblock Robust sensitivity analysis for stochastic systems.
\newblock \emph{Mathematics of Operations Research}, 41\penalty0 (4):\penalty0 1248--1275, 2016.

\bibitem[Lam(2019)]{lam2019recovering}
H.~Lam.
\newblock Recovering best statistical guarantees via the empirical divergence-based distributionally robust optimization.
\newblock \emph{Operations Research}, 67\penalty0 (4):\penalty0 1090--1105, 2019.

\bibitem[Lee and Raginsky(2018)]{lee2018minimax}
J.~Lee and M.~Raginsky.
\newblock Minimax statistical learning with {W}asserstein distances.
\newblock In \emph{Advances in Neural Information Processing Systems}, pages 2687--2696, 2018.

\bibitem[Lefkimmiatis et~al.(2013)Lefkimmiatis, Ward, and Unser]{lefkimmiatis2013hessian}
S.~Lefkimmiatis, J.~P. Ward, and M.~Unser.
\newblock {Hessian Schatten-norm regularization for linear inverse problems}.
\newblock \emph{IEEE transactions on image processing}, 22\penalty0 (5):\penalty0 1873--1888, 2013.

\bibitem[Lei et~al.(2018)Lei, Bickel, and El~Karoui]{lei2018asymptotics}
L.~Lei, P.~J. Bickel, and N.~El~Karoui.
\newblock Asymptotics for high dimensional regression {M}-estimates: fixed design results.
\newblock \emph{Probability Theory and Related Fields}, 172\penalty0 (3):\penalty0 983--1079, 2018.

\bibitem[Li et~al.(2021)Li, Cao, Zhang, Xu, Chen, and Tan]{li2021internal}
J.~Li, J.~Cao, S.~Zhang, Y.~Xu, J.~Chen, and M.~Tan.
\newblock Internal {W}asserstein distance for adversarial attack and defense.
\newblock \emph{arXiv preprint arXiv:2103.07598}, 2021.

\bibitem[Lyu et~al.(2015)Lyu, Huang, and Liang]{lyu2015unified}
C.~Lyu, K.~Huang, and H.-N. Liang.
\newblock A unified gradient regularization family for adversarial examples.
\newblock In \emph{International Conference on Data Mining}, pages 301--309. IEEE, 2015.

\bibitem[Madry et~al.(2017)Madry, Makelov, Schmidt, Tsipras, and Vladu]{madry2017towards}
A.~Madry, A.~Makelov, L.~Schmidt, D.~Tsipras, and A.~Vladu.
\newblock Towards deep learning models resistant to adversarial attacks.
\newblock \emph{arXiv preprint arXiv:1706.06083}, 2017.

\bibitem[Mai et~al.(2019)Mai, Liao, and Couillet]{mai2019large}
X.~Mai, Z.~Liao, and R.~Couillet.
\newblock A large scale analysis of logistic regression: Asymptotic performance and new insights.
\newblock In \emph{ICASSP 2019-2019 IEEE International Conference on Acoustics, Speech and Signal Processing (ICASSP)}, pages 3357--3361. IEEE, 2019.

\bibitem[Mallory et~al.(2025)Mallory, Huang, and Austern]{mallory2025universality}
M.~E. Mallory, K.~H. Huang, and M.~Austern.
\newblock Universality of high-dimensional logistic regression and a novel cgmt under dependence with applications to data augmentation.
\newblock \emph{arXiv preprint arXiv:2502.15752}, 2025.

\bibitem[Mehrabi et~al.(2021)Mehrabi, Javanmard, Rossi, Rao, and Mai]{mehrabi2021fundamental}
M.~Mehrabi, A.~Javanmard, R.~A. Rossi, A.~Rao, and T.~Mai.
\newblock Fundamental tradeoffs in distributionally adversarial training.
\newblock In \emph{International Conference on Machine Learning}, pages 7544--7554. PMLR, 2021.

\bibitem[Mei and Montanari(2022)]{mei2022generalization}
S.~Mei and A.~Montanari.
\newblock The generalization error of random features regression: Precise asymptotics and the double descent curve.
\newblock \emph{Communications on Pure and Applied Mathematics}, 75\penalty0 (4):\penalty0 667--766, 2022.

\bibitem[Mignacco et~al.(2020)Mignacco, Krzakala, Lu, Urbani, and Zdeborova]{mignacco2020role}
F.~Mignacco, F.~Krzakala, Y.~Lu, P.~Urbani, and L.~Zdeborova.
\newblock The role of regularization in classification of high-dimensional noisy gaussian mixture.
\newblock In \emph{International Conference on Machine Learning}, pages 6874--6883. PMLR, 2020.

\bibitem[Min et~al.(2021)Min, Chen, and Karbasi]{min2021curious}
Y.~Min, L.~Chen, and A.~Karbasi.
\newblock The curious case of adversarially robust models: More data can help, double descend, or hurt generalization.
\newblock In \emph{Uncertainty in Artificial Intelligence}, pages 129--139. PMLR, 2021.

\bibitem[Miolane and Montanari(2021)]{miolane2021distribution}
L.~Miolane and A.~Montanari.
\newblock The distribution of the {LASSO}: Uniform control over sparse balls and adaptive parameter tuning.
\newblock \emph{The Annals of Statistics}, 49\penalty0 (4):\penalty0 2313--2335, 2021.

\bibitem[Mohajerin~Esfahani and Kuhn(2018)]{esfahani2018data}
P.~Mohajerin~Esfahani and D.~Kuhn.
\newblock {Data-driven distributionally robust optimization using the Wasserstein metric: Performance guarantees and tractable reformulations}.
\newblock \emph{Mathematical Programming}, 171\penalty0 (1-2):\penalty0 115--166, 2018.

\bibitem[Montanari and Saeed(2022)]{montanari2022universality}
A.~Montanari and B.~Saeed.
\newblock Universality of empirical risk minimization.
\newblock \emph{arXiv preprint arXiv:2202.08832}, 2022.

\bibitem[Montanari et~al.(2019)Montanari, Ruan, Sohn, and Yan]{montanari2019generalization}
A.~Montanari, F.~Ruan, Y.~Sohn, and J.~Yan.
\newblock The generalization error of max-margin linear classifiers: High-dimensional asymptotics in the overparametrized regime.
\newblock \emph{arXiv preprint arXiv:1911.01544}, 2019.

\bibitem[Moosavi-Dezfooli et~al.(2016)Moosavi-Dezfooli, Fawzi, and Frossard]{moosavi2016deepfool}
S.-M. Moosavi-Dezfooli, A.~Fawzi, and P.~Frossard.
\newblock Deep{F}ool: a simple and accurate method to fool deep neural networks.
\newblock In \emph{Proceedings of the IEEE conference on computer vision and pattern recognition}, pages 2574--2582, 2016.

\bibitem[Muthukumar et~al.(2020)Muthukumar, Vodrahalli, Subramanian, and Sahai]{muthukumar2020harmless}
V.~Muthukumar, K.~Vodrahalli, V.~Subramanian, and A.~Sahai.
\newblock Harmless interpolation of noisy data in regression.
\newblock \emph{IEEE Journal on Selected Areas in Information Theory}, 1\penalty0 (1):\penalty0 67--83, 2020.

\bibitem[Nagarajan and Kolter(2017)]{nagarajan2017gradient}
V.~Nagarajan and J.~Z. Kolter.
\newblock {Gradient descent GAN optimization is locally stable}.
\newblock In \emph{Advances in Neural Information Processing Systems}, pages 5591--5600, 2017.

\bibitem[Najafi et~al.(2019)Najafi, Maeda, Koyama, and Miyato]{najafi2019robustness}
A.~Najafi, S.-i. Maeda, M.~Koyama, and T.~Miyato.
\newblock Robustness to adversarial perturbations in learning from incomplete data.
\newblock \emph{Advances in Neural Information Processing Systems}, 32, 2019.

\bibitem[Namkoong and Duchi(2016)]{namkoong2016stochastic}
H.~Namkoong and J.~C. Duchi.
\newblock Stochastic gradient methods for distributionally robust optimization with f-divergences.
\newblock \emph{Advances in neural information processing systems}, 29, 2016.

\bibitem[Ororbia~II et~al.(2017)Ororbia~II, Kifer, and Giles]{ororbia2017unifying}
A.~G. Ororbia~II, D.~Kifer, and C.~L. Giles.
\newblock Unifying adversarial training algorithms with data gradient regularization.
\newblock \emph{Neural Computation}, 29\penalty0 (4):\penalty0 867--887, 2017.

\bibitem[Oymak and Hassibi(2016)]{oymak2016sharp}
S.~Oymak and B.~Hassibi.
\newblock Sharp {MSE} bounds for proximal denoising.
\newblock \emph{Foundations of Computational Mathematics}, 16\penalty0 (4):\penalty0 965--1029, 2016.

\bibitem[Oymak and Tropp(2018)]{oymak2018universality}
S.~Oymak and J.~A. Tropp.
\newblock Universality laws for randomized dimension reduction, with applications.
\newblock \emph{Information and Inference: A Journal of the IMA}, 7\penalty0 (3):\penalty0 337--446, 2018.

\bibitem[Oymak et~al.(2013)Oymak, Thrampoulidis, and Hassibi]{oymak2013squared}
S.~Oymak, C.~Thrampoulidis, and B.~Hassibi.
\newblock The squared-error of generalized {LASSO}: A precise analysis.
\newblock In \emph{2013 51st Annual Allerton Conference on Communication, Control, and Computing (Allerton)}, pages 1002--1009. IEEE, 2013.

\bibitem[Panahi and Hassibi(2017)]{panahi2017universal}
A.~Panahi and B.~Hassibi.
\newblock A universal analysis of large-scale regularized least squares solutions.
\newblock \emph{Advances in Neural Information Processing Systems}, 30, 2017.

\bibitem[Pollard(1991)]{pollard1991asymptotics}
D.~Pollard.
\newblock Asymptotics for least absolute deviation regression estimators.
\newblock \emph{Econometric Theory}, 7\penalty0 (2):\penalty0 186--199, 1991.

\bibitem[Pydi and Jog(2020)]{pydi2020adversarial}
M.~S. Pydi and V.~Jog.
\newblock Adversarial risk via optimal transport and optimal couplings.
\newblock In \emph{International Conference on Machine Learning}, pages 7814--7823. PMLR, 2020.

\bibitem[Raghunathan et~al.(2019)Raghunathan, Xie, Yang, Duchi, and Liang]{raghunathan2019adversarial}
A.~Raghunathan, S.~M. Xie, F.~Yang, J.~C. Duchi, and P.~Liang.
\newblock Adversarial training can hurt generalization.
\newblock \emph{arXiv preprint arXiv:1906.06032}, 2019.

\bibitem[Rockafellar and Wets(2009)]{rockafellar2009variational}
R.~T. Rockafellar and R.~J.-B. Wets.
\newblock \emph{Variational Analysis}, volume 317.
\newblock Springer, 2009.

\bibitem[Roth et~al.(2017)Roth, Lucchi, Nowozin, and Hofmann]{roth2017stabilizing}
K.~Roth, A.~Lucchi, S.~Nowozin, and T.~Hofmann.
\newblock Stabilizing training of generative adversarial networks through regularization.
\newblock In \emph{Advances in Neural Information Processing Systems}, pages 2018--2028, 2017.

\bibitem[Salehi et~al.(2019)Salehi, Abbasi, and Hassibi]{salehi2019impact}
F.~Salehi, E.~Abbasi, and B.~Hassibi.
\newblock The impact of regularization on high-dimensional logistic regression.
\newblock \emph{Advances in Neural Information Processing Systems}, 32, 2019.

\bibitem[Salehi et~al.(2020)Salehi, Abbasi, and Hassibi]{salehi2020performance}
F.~Salehi, E.~Abbasi, and B.~Hassibi.
\newblock The performance analysis of generalized margin maximizers on separable data.
\newblock In \emph{International Conference on Machine Learning}, pages 8417--8426. PMLR, 2020.

\bibitem[Schmidt et~al.(2018)Schmidt, Santurkar, Tsipras, Talwar, and Madry]{schmidt2018adversarially}
L.~Schmidt, S.~Santurkar, D.~Tsipras, K.~Talwar, and A.~Madry.
\newblock Adversarially robust generalization requires more data.
\newblock \emph{Advances in neural information processing systems}, 31, 2018.

\bibitem[Shafieezadeh-Abadeh et~al.(2015)Shafieezadeh-Abadeh, Mohajerin~Esfahani, and Kuhn]{shafieezadeh2015distributionally}
S.~Shafieezadeh-Abadeh, P.~Mohajerin~Esfahani, and D.~Kuhn.
\newblock Distributionally robust logistic regression.
\newblock In \emph{Advances in Neural Information Processing Systems}, pages 1576--1584, 2015.

\bibitem[Shafieezadeh-Abadeh et~al.(2019)Shafieezadeh-Abadeh, Kuhn, and Mohajerin~Esfahani]{shafieezadeh2019regularization}
S.~Shafieezadeh-Abadeh, D.~Kuhn, and P.~Mohajerin~Esfahani.
\newblock Regularization via mass transportation.
\newblock \emph{Journal of Machine Learning Research}, 20\penalty0 (103):\penalty0 1--68, 2019.

\bibitem[Shafieezadeh-Abadeh et~al.(2023)Shafieezadeh-Abadeh, Aolaritei, D{\"o}rfler, and Kuhn]{shafieezadeh2023new}
S.~Shafieezadeh-Abadeh, L.~Aolaritei, F.~D{\"o}rfler, and D.~Kuhn.
\newblock New perspectives on regularization and computation in optimal transport-based distributionally robust optimization.
\newblock \emph{arXiv preprint arXiv:2303.03900}, 2023.

\bibitem[Shapiro(2017)]{shapiro2017distributionally}
A.~Shapiro.
\newblock Distributionally robust stochastic programming.
\newblock \emph{SIAM Journal on Optimization}, 27\penalty0 (4):\penalty0 2258--2275, 2017.

\bibitem[Silva and Najafirad(2020)]{silva2020opportunities}
S.~H. Silva and P.~Najafirad.
\newblock Opportunities and challenges in deep learning adversarial robustness: A survey.
\newblock \emph{arXiv preprint arXiv:2007.00753}, 2020.

\bibitem[Sinha et~al.(2018)Sinha, Namkoong, and Duchi]{sinha2018certifying}
A.~Sinha, H.~Namkoong, and J.~Duchi.
\newblock Certifying some distributional robustness with principled adversarial training.
\newblock In \emph{International Conference on Learning Representations}, 2018.

\bibitem[Staib and Jegelka(2017)]{staib2017distributionally}
M.~Staib and S.~Jegelka.
\newblock Distributionally robust deep learning as a generalization of adversarial training.
\newblock In \emph{NIPS workshop on Machine Learning and Computer Security}, volume~3, page~4, 2017.

\bibitem[Stojnic(2009)]{stojnic2009various}
M.~Stojnic.
\newblock Various thresholds for $\ell_1$-optimization in compressed sensing.
\newblock \emph{arXiv preprint arXiv:0907.3666}, 2009.

\bibitem[Stojnic(2013)]{stojnic2013framework}
M.~Stojnic.
\newblock A framework to characterize performance of {LASSO} algorithms.
\newblock \emph{arXiv preprint arXiv:1303.7291}, 2013.

\bibitem[Sugiyama et~al.(2007)Sugiyama, Nakajima, Kashima, Buenau, and Kawanabe]{sugiyama2007direct}
M.~Sugiyama, S.~Nakajima, H.~Kashima, P.~Buenau, and M.~Kawanabe.
\newblock Direct importance estimation with model selection and its application to covariate shift adaptation.
\newblock \emph{Advances in neural information processing systems}, 20, 2007.

\bibitem[Sur and Cand{\`e}s(2019)]{sur2019modern}
P.~Sur and E.~J. Cand{\`e}s.
\newblock A modern maximum-likelihood theory for high-dimensional logistic regression.
\newblock \emph{Proceedings of the National Academy of Sciences}, 116\penalty0 (29):\penalty0 14516--14525, 2019.

\bibitem[Sutter et~al.(2020)Sutter, Van~Parys, and Kuhn]{sutter2020general}
T.~Sutter, B.~P. Van~Parys, and D.~Kuhn.
\newblock A general framework for optimal data-driven optimization.
\newblock \emph{arXiv preprint arXiv:2010.06606}, 2020.

\bibitem[Szegedy et~al.(2013)Szegedy, Zaremba, Sutskever, Bruna, Erhan, Goodfellow, and Fergus]{szegedy2013intriguing}
C.~Szegedy, W.~Zaremba, I.~Sutskever, J.~Bruna, D.~Erhan, I.~Goodfellow, and R.~Fergus.
\newblock Intriguing properties of neural networks.
\newblock \emph{arXiv preprint arXiv:1312.6199}, 2013.

\bibitem[Taheri et~al.(2020{\natexlab{a}})Taheri, Pedarsani, and Thrampoulidis]{taheri2020asymptotic}
H.~Taheri, R.~Pedarsani, and C.~Thrampoulidis.
\newblock Asymptotic behavior of adversarial training in binary classification.
\newblock \emph{arXiv preprint arXiv:2010.13275}, 2020{\natexlab{a}}.

\bibitem[Taheri et~al.(2020{\natexlab{b}})Taheri, Pedarsani, and Thrampoulidis]{taheri2020sharp}
H.~Taheri, R.~Pedarsani, and C.~Thrampoulidis.
\newblock Sharp asymptotics and optimal performance for inference in binary models.
\newblock In \emph{International Conference on Artificial Intelligence and Statistics}, pages 3739--3749. PMLR, 2020{\natexlab{b}}.

\bibitem[Taheri et~al.(2021)Taheri, Pedarsani, and Thrampoulidis]{taheri2021fundamental}
H.~Taheri, R.~Pedarsani, and C.~Thrampoulidis.
\newblock Fundamental limits of ridge-regularized empirical risk minimization in high dimensions.
\newblock In \emph{International Conference on Artificial Intelligence and Statistics}, pages 2773--2781. PMLR, 2021.

\bibitem[Thrampoulidis and Hassibi(2015)]{thrampoulidis2015isotropically}
C.~Thrampoulidis and B.~Hassibi.
\newblock Isotropically random orthogonal matrices: Performance of {LASSO} and minimum conic singular values.
\newblock In \emph{2015 IEEE International Symposium on Information Theory (ISIT)}, pages 556--560. IEEE, 2015.

\bibitem[Thrampoulidis et~al.(2015{\natexlab{a}})Thrampoulidis, Abbasi, and Hassibi]{thrampoulidis2015lasso}
C.~Thrampoulidis, E.~Abbasi, and B.~Hassibi.
\newblock {LASSO} with non-linear measurements is equivalent to one with linear measurements.
\newblock \emph{Advances in Neural Information Processing Systems}, 28, 2015{\natexlab{a}}.

\bibitem[Thrampoulidis et~al.(2015{\natexlab{b}})Thrampoulidis, Oymak, and Hassibi]{thrampoulidis2015regularized}
C.~Thrampoulidis, S.~Oymak, and B.~Hassibi.
\newblock Regularized linear regression: A precise analysis of the estimation error.
\newblock In \emph{Conference on Learning Theory}, pages 1683--1709. PMLR, 2015{\natexlab{b}}.

\bibitem[Thrampoulidis et~al.(2018)Thrampoulidis, Abbasi, and Hassibi]{thrampoulidis2018precise}
C.~Thrampoulidis, E.~Abbasi, and B.~Hassibi.
\newblock Precise error analysis of regularized {M}-estimators in high dimensions.
\newblock \emph{IEEE Transactions on Information Theory}, 64\penalty0 (8):\penalty0 5592--5628, 2018.

\bibitem[Tsipras et~al.(2018)Tsipras, Santurkar, Engstrom, Turner, and Madry]{tsipras2018robustness}
D.~Tsipras, S.~Santurkar, L.~Engstrom, A.~Turner, and A.~Madry.
\newblock Robustness may be at odds with accuracy.
\newblock \emph{arXiv preprint arXiv:1805.12152}, 2018.

\bibitem[Tu et~al.(2019)Tu, Zhang, and Tao]{tu2019theoretical}
Z.~Tu, J.~Zhang, and D.~Tao.
\newblock Theoretical analysis of adversarial learning: {A} minimax approach.
\newblock In \emph{Advances in Neural Information Processing Systems}, pages 12280--12290, 2019.

\bibitem[Van~Parys et~al.(2021)Van~Parys, Esfahani, and Kuhn]{van2021data}
B.~P. Van~Parys, P.~M. Esfahani, and D.~Kuhn.
\newblock From data to decisions: Distributionally robust optimization is optimal.
\newblock \emph{Management Science}, 67\penalty0 (6):\penalty0 3387--3402, 2021.

\bibitem[Varga et~al.(2017)Varga, Csisz{\'a}rik, and Zombori]{varga2017gradient}
D.~Varga, A.~Csisz{\'a}rik, and Z.~Zombori.
\newblock Gradient regularization improves accuracy of discriminative models.
\newblock \emph{arXiv preprint arXiv:1712.09936}, 2017.

\bibitem[Volpi et~al.(2018)Volpi, Namkoong, Sener, Duchi, Murino, and Savarese]{volpi2018generalizing}
R.~Volpi, H.~Namkoong, O.~Sener, J.~Duchi, V.~Murino, and S.~Savarese.
\newblock Generalizing to unseen domains via adversarial data augmentation.
\newblock In \emph{Advances in Neural Information Processing Systems}, pages 5339--5349, 2018.

\bibitem[Wainwright(2019)]{wainwright2019high}
M.~J. Wainwright.
\newblock \emph{High-dimensional statistics: A non-asymptotic viewpoint}, volume~48.
\newblock Cambridge University Press, 2019.

\bibitem[Wang et~al.(2019)Wang, Ma, Bailey, Yi, Zhou, and Gu]{wang2019convergence}
Y.~Wang, X.~Ma, J.~Bailey, J.~Yi, B.~Zhou, and Q.~Gu.
\newblock On the convergence and robustness of adversarial training.
\newblock In \emph{International Conference on Machine Learning}, pages 6586--6595, 2019.

\bibitem[Wong et~al.(2019)Wong, Schmidt, and Kolter]{wong2019wasserstein}
E.~Wong, F.~Schmidt, and Z.~Kolter.
\newblock Wasserstein adversarial examples via projected sinkhorn iterations.
\newblock In \emph{International Conference on Machine Learning}, pages 6808--6817. PMLR, 2019.

\bibitem[Wu et~al.(2020)Wu, Wang, and Yu]{wu2020stronger}
K.~Wu, A.~Wang, and Y.~Yu.
\newblock Stronger and faster {W}asserstein adversarial attacks.
\newblock In \emph{International Conference on Machine Learning}, pages 10377--10387. PMLR, 2020.

\bibitem[Zhang et~al.(2019)Zhang, Yu, Jiao, Xing, El~Ghaoui, and Jordan]{zhang2019theoretically}
H.~Zhang, Y.~Yu, J.~Jiao, E.~Xing, L.~El~Ghaoui, and M.~Jordan.
\newblock Theoretically principled trade-off between robustness and accuracy.
\newblock In \emph{International conference on machine learning}, pages 7472--7482. PMLR, 2019.

\end{thebibliography}

\appendix

\section{Proofs for Section~\ref{sec:cgmt}}
\label{appendix:additional:proofs:2}

\begin{proof}[Proof of Fact~\ref{prop:cgmt}]
Let $\Phi_{\mathcal S^c}(A)$ be the optimal value of \eqref{eq:cgmt:po} when the minimization over $w$ is restricted as $w \in \mathcal S^c$. Then, in order to prove \eqref{eq:asympt:cgmt}, we will show that $\Phi(A) < \Phi_{\mathcal S^c}(A)$ w.p.a.\ 1 as $n,d \to \infty$. While doing so, the two inequalities \eqref{eq:cgmt:00} and \eqref{eq:cgmt:01} will be instrumental. Specifically, inequality \eqref{eq:cgmt:00} will help us prove that $\Phi_{\mathcal S^c}(A) \geq \overline{\phi'}_{\mathcal S^c} - \eta$ w.p.a.\ 1, for all $\eta > 0$, and inequality \eqref{eq:cgmt:01} will be used to show that $\Phi(A) \leq \overline{\phi'} + \eta$ w.p.a.\ 1, for all $\eta > 0$. Then, from $\overline{\phi'} < \overline{\phi'}_{\mathcal S^c}$, the result easily follows with $\eta < (\overline{\phi'}_{\mathcal S^c} - \overline{\phi'})/3$.

We first concentrate on proving that $\Phi_{\mathcal S^c}(A) \geq \overline{\phi'}_{\mathcal S^c} - \eta$ w.p.a.\ 1. For ease of notation, we denote by $E(w,u)$ the objective function of the modified (AO) problem \eqref{eq:cgmt:ao:1}. First notice that
\begin{align*}
    \phi'_{\mathcal S^c}(g,h) = \max_{0 \leq \beta \leq K_\beta}\; \min_{w \in \mathcal S^c}\; \max_{\|u\|= \beta}\; E(w,u) \leq \min_{w \in \mathcal S^c}\; \max_{u \in \mathcal S_u}\; E(w,u) = \phi_{\mathcal S^c}(g,h),
\end{align*}
and therefore $\phi'_{\mathcal S^c}(g,h) \leq \phi_{\mathcal S^c}(g,h)$, using the minimax inequality. As a consequence, for every $c \in \mathbb R$,
\begin{align*}
    \text{Pr}(\phi_{\mathcal S^c}(g,h) \leq c) \leq \text{Pr}(\phi'_{\mathcal S^c}(g,h) \leq c),
\end{align*}
and therefore, using inequality \eqref{eq:cgmt:00}, we have that 
\begin{align*}
    \text{Pr}(\Phi_{\mathcal S^c}(A) < c) \leq 2\text{Pr}(\phi'_{\mathcal S^c}(g,h) \leq c).
\end{align*}
Using $c = \overline{\phi'}_{\mathcal S^c} - \eta$, we have that $\text{Pr}(\phi'_{\mathcal S^c}(g,h) \leq \overline{\phi'}_{\mathcal S^c} - \eta) \to 0$ as $n,d \to \infty$, and therefore $\Phi_{\mathcal S^c}(A) \geq \overline{\phi'}_{\mathcal S^c} -\eta$ w.p.a.\ 1, for all $\eta > 0$.

We now proceed to proving that $\Phi(A) \leq \overline{\phi'} + \eta$ w.p.a.\ 1. First notice that
\begin{align*}
    \phi'(g,h) = \max_{0 \leq \beta \leq K_\beta}\; \min_{w \in \mathcal S_w}\; \max_{\|u\|= \beta}\; E(w,u) \geq \max_{u \in \mathcal S_u}\; \min_{w \in \mathcal S_w}\;  E(w,u) =: \phi^D(g,h),
\end{align*}
where we have denoted by $\phi^D(g,h)$ the dual problem of $\phi(g,h)$. As a consequence, for every $c \in \mathbb R$,
\begin{align}
\label{proof:lemma:prop:cgmt:0}
    \text{Pr}(\phi^D(g,h) \geq c) \leq \text{Pr}(\phi'(g,h) \geq c).
\end{align}

Now, from \cite{thrampoulidis2015regularized} (see equation (32) in the proof of Theorem~3) we know that $\phi^D(g,h)$ is related to $\Phi(A)$ as follows 
\begin{align}
\label{proof:lemma:prop:cgmt:1}
    \text{Pr}(\Phi(A) > c) \leq 2 \text{Pr}(\phi^D(g,h) \geq c).
\end{align}

Then, from inequalities \eqref{proof:lemma:prop:cgmt:0} and \eqref{proof:lemma:prop:cgmt:1} it follows that
\begin{align*}
    \text{Pr}(\Phi(A) > c) \leq \text{Pr}(\phi'(g,h) \geq c),
\end{align*}
Now, using $c = \overline{\phi'} + \eta$, we have that $\text{Pr}(\phi'(g,h) \geq \overline{\phi'} + \eta) \to 0$ as $n,d \to \infty$, and therefore $\Phi(A) \leq \overline{\phi'}+ \eta$ w.p.a.\ 1, for all $\eta > 0$, which concludes the proof.
\end{proof}

\section{Proofs for Section~\ref{sec:dro:M-est}}
\label{appendix:additional:proofs:3}

\begin{proof}[Proof of Lemma~\ref{lemma:properties:ambiguity}]
Notice that \eqref{eq:optimal:transport} reduces to
\begin{align*}
W_p \left( \mathbb Q, \widehat{\mathbb P}_n \right) &= \inf_{\gamma \in \Gamma(\mathbb Q, \widehat{\mathbb P}_n)}\; \mathbb E_{\gamma} [\|x_1-x_2\|^p + \infty |y_1-y_2|] \\ &= \inf_{\gamma \in \Gamma(\mathbb Q, \widehat{\mathbb P}_n)}\; \mathbb E_{(\pi_{(X_1, X_2)})_\# \gamma} [\|x_1-x_2\|^p] + \mathbb E_{(\pi_{(Y_1, Y_2)})_\# \gamma} [ \infty |y_1-y_2|],
\end{align*}
where $(\pi_{(X_1, X_2)})_\# \gamma$ and $(\pi_{(Y_1, Y_2)})_\# \gamma$ are the marginal distributions of $(X_1,X_2)$ and $(Y_1,Y_2)$, respectively. From this we see that the marginal distributions $(\pi_{Y})_\# \mathbb Q$ and $(\pi_{Y})_\# \widehat{\mathbb P}_n$ need to be equal, otherwise the term $E_{(\pi_{(Y_1, Y_2)})_\# \gamma} [ \infty |y_1-y_2|]$ would become infinity. As a consequence, the ambiguity set $\mathbb B_\varepsilon(\widehat{\mathbb P}_n)$ only contains distributions $\mathbb Q$ whose marginal distribution of $Y$ is $n^{-1}\sum_{i=1}^{n}\delta_{y_i}$. In this case, the term $\mathbb E_{(\pi_{(Y_1, Y_2)})_\# \gamma} [ \infty |y_1-y_2|]$ becomes zero (for the optimal coupling), and from the ambiguity set definition, we see that the marginal distribution of $X$ satisfies
\begin{align*}
\inf_{\tilde\gamma \in \Gamma((\pi_X)_\#\mathbb Q, (\pi_X)_\# \widehat{\mathbb P}_n)}\; \mathbb E_{\tilde\gamma} [\|X_1 - X_2\|^p] \leq \varepsilon,
\end{align*}
for all distributions $\mathbb Q \in \mathbb B_\varepsilon(\widehat{\mathbb P}_n)$. 
%Notice that this is the \emph{standard} type-$p$ Wasserstein distance (defined on $\mathbb R^d$, where the transportation cost is only $c(x_1,x_2) = \|x_1-x_2\|^p$) between the marginals $(\pi_X)_\#\mathbb Q$ and $ (\pi_X)_\# \widehat{\mathbb P}_n$. 
As a consequence, the ambiguity set $\mathbb B_\varepsilon(\widehat{\mathbb P}_n)$ only contains distributions $\mathbb Q$ whose marginal distribution of $X$ is $\varepsilon$ close to $n^{-1}\sum_{i=1}^{n}\delta_{x_i}$, when evaluated with respect to $d_p$ defined in the statement of the lemma. This concludes the proof.
\end{proof}

\begin{proof}[Proof of Lemma~\ref{lemma:convexity}]
Notice first that problem \eqref{eq:dro:dual:uni} is always convex in $(\theta,\lambda)$. This follows from the fact that $\|\theta\|^2/\lambda$ is the perspective function of $\|\theta\|^2$, and therefore it is jointly convex in $(\theta,\lambda)$. In what follows, we will show that assumption $\varepsilon_0 \leq \rho^{-1} M^{-2} R_{\theta}^{-2}\underline{L}_n$ guarantees the concavity in each $u_i$ of the function
\begin{align}
    \label{eq:proof:conc:1}
    u_i(y_i - \theta^\top x_i) + \frac{u_i^2}{4 \lambda}\|\theta\|^2 - L^*(u_i).
\end{align}

Now, from Assumption~\ref{assump:dro}\ref{assump:ell:diff} we know that $L$ is $M$-smooth. As a consequence, its convex conjugate $L^*$ is $1/M$-strongly convex, and therefore can be decomposed as follows
\begin{align*}
    L^*(u_i) = \frac{u_i^2}{2M} + f(u_i),
\end{align*}
for some convex function $f$. Introducing this into \eqref{eq:proof:conc:1}, we obtain
\begin{align*}
    u_i(y_i - \theta^\top x_i) + \frac{u_i^2}{4 \lambda}\|\theta\|^2 - \frac{u_i^2}{2M} - f(u_i),
\end{align*}
which can be seen to be concave in $u_i$ whenever $\lambda$ satisfies 
\begin{align}
    \label{eq:proof:conc:4}
    \lambda \geq \frac{M \|\theta\|^2}{2}.
\end{align}

We will now translate this lower bound on $\lambda$ in the upper bound on the ambiguity radius presented in the statement of the lemma. To do so, we rely on the following result from \cite{blanchet2018optimal} which bounds the optimal solution $\lambda_\star(\theta)$ as a function of $\theta$
\begin{align}
    \label{eq:proof:conc:5}
    \lambda_\star(\theta) \geq \frac{1}{2\sqrt{\varepsilon}}\sqrt{\underline{L}_n} \|\theta\| 
\end{align}
Notice that the lower bound presented in \eqref{eq:proof:conc:5} contains the additional $\sqrt{\varepsilon}$ compared to the one presented in \cite{blanchet2018optimal}. This is because of the scaling by $\sqrt{\varepsilon}$ of $\lambda$ used in \cite{blanchet2018optimal}.

Now, using the bound $\|\theta\| \leq R_{\theta} \sqrt{d}$ from Assumption~\ref{assump:dro}\ref{assump:Theta} in \eqref{eq:proof:conc:4}, and the lower bound \eqref{eq:proof:conc:5}, we can conclude concavity in $u_i$ if
\begin{align*}
    \frac{M R_\theta \sqrt{d} \|\theta\|}{2} \leq \frac{1}{2\sqrt{\varepsilon}} \sqrt{\underline{L}_n}\|\theta\|
\end{align*}
or, equivalently, the ambiguity radius $\varepsilon$ satisfies
\begin{align*}
    \varepsilon \leq \frac{\underline{L}_n}{d M^2 R_\theta^2}
\end{align*}
from which we obtain the upper bound on $\varepsilon_0$ using the fact that $d/n \to \rho$.
\end{proof}

\begin{proof}[Proof of Lemma~\ref{lemma:growthrates}]
We will recover the stated bound on the growth rate of $u_\star$ from the growth rates of $\lambda_\star$ and $\theta$, and the relationship between these three variables. For a bound on the growth rate of $\lambda_\star$, we rely on the following result from \cite{blanchet2018optimal}, which upper bounds $\lambda_\star$ by a function of the norm of $\theta$, as follows
\begin{align}
    \label{eq:proof:growth:1}
    \lambda_\star \leq \left(\frac{1}{2} M R_\theta \sqrt{d} + \sqrt{\overline{L}_n}\right)\|\theta\| = \left(\frac{1}{2} \sqrt{\rho} M R_\theta + \sqrt{\frac{\overline{L}_n}{\sqrt{n}}}\right) \|\theta\| \sqrt{n}.
\end{align}
Notice that, different from the upper bound presented in \cite{blanchet2018optimal}, the multiplication by $\sqrt{\varepsilon}$ is missing from the right-hand side. This is because of the scaling by $\sqrt{\varepsilon}$ of $\lambda$ used in \cite{blanchet2018optimal}.

We will now turn to \eqref{eq:dro:dual:uni} and find the precise relationship between $u_\star$, $\lambda_\star$, and $\theta_\star$. Due to Lemma~\ref{lemma:convexity}, we know that the objective function in \eqref{eq:dro:dual:uni} is concave in $u$. Moreover, it is clear that the objective function is convex in $\theta$ and $\lambda$. These, together with the upper bound \eqref{eq:proof:growth:1}, allow us to rely on Sion's minimax principle to exchange the order of the minimization over $\lambda$ and the maximization over $u$ in \eqref{eq:dro:dual:uni} without changing the optimal value. We thus obtain
\begin{align*}
    \min_{\theta \in \Theta}\; \sup_{u \in \mathbb R^n}\; \inf_{ \lambda \in \Lambda} \; \lambda \varepsilon + \frac{1}{n} \sum_{i=1}^{n}\; u_i(y_i - \theta^\top x_i) + \frac{u_i^2}{4 \lambda}\|\theta\|^2 - L^*(u_i).
\end{align*}

We can now solve the inner minimization over $\lambda$, and recover the optimal solution
\begin{align*}
    \lambda_\star(\theta,u_\star) = \begin{cases} M R_\theta \sqrt{d} \|\theta\|/2 & \quad\quad \text{if} \quad \|\theta\|\|u_\star\|/(2\sqrt{n\varepsilon}) \leq M R_\theta \sqrt{d} \|\theta\|/2 \\ \|\theta\|\|u_\star\|/(2\sqrt{n\varepsilon}) & \quad\quad \text{otherwise}. \end{cases}
\end{align*}
In the first case, when $\|\theta\|\|u_\star\|/(2\sqrt{n\varepsilon}) \leq M R_\theta \sqrt{d} \|\theta\|/2$, we have that
\begin{align*}
    \|u_\star\| \leq \sqrt{\varepsilon_0 d} M R_\theta = \sqrt{\varepsilon_0 \rho} M R_\theta \sqrt{n},
\end{align*}
which gives the desired bound on $\|u_\star\|$. In the second case, the bound on $u_\star$ can be obtained from the upper bound \eqref{eq:proof:growth:1}, as follows
\begin{align*}
    \lambda_\star(\theta,u_\star) = \frac{1}{(2\sqrt{n\varepsilon})} \|\theta\|\|u_\star\| \leq  \left(\frac{1}{2} \sqrt{\rho} M R_\theta + \sqrt{\frac{\overline{L}_n}{n}}\right) \|\theta\| \sqrt{n},
\end{align*}
from which we recover the upper bound on the norm of $u_\star$,
\begin{align*}
    \|u_\star\| \leq 2\sqrt{\varepsilon_0} \left(\frac{1}{2} \sqrt{\rho} M R_\theta + \sqrt{\frac{\overline{L}_n}{n}}\right) \sqrt{n}.
\end{align*}
Finally, the desired result follows by noticing that $\overline{L}_n$ has growth rate at most $n$, due to Assumption~\ref{assump:ell2}.
\end{proof}

\begin{proof}[Proof of Lemma~\ref{lemma:convexity:3}]
We will prove the concavity in each $u_i$ of the objective function in the inner supremum in \eqref{eq:dro:dual:uni:3}, i.e.,
\begin{align}
    \label{eq:proof:conc:3:1}
    \sup_{u_i \in \mathbb R^n} \; u_i(y_i - \theta^\top x_i) + \frac{u_i^2}{4 \lambda}\|\theta\|^2 - L^*(u_i).
\end{align}
From Assumption~\ref{assump:dro}\ref{assump:ell:diff} we know that $L$ is $M$-smooth. As a consequence, its convex conjugate $L^*$ is $1/M$-strongly convex, and therefore can be decomposed as follows
\begin{align*}
    L^*(u_i) = \frac{u_i^2}{2M} + f(u_i),
\end{align*}
for some convex function $f$. Introducing this into \eqref{eq:proof:conc:1}, we obtain
\begin{align*}
    u_i(y_i - \theta^\top x_i) + \frac{u_i^2}{4 \lambda}\|\theta\|^2 - \frac{u_i^2}{2M} - f(u_i),
\end{align*}
which can be seen to be concave in $u_i$ whenever $\lambda$ satisfies 
\begin{align*}
    \lambda \geq \frac{M \|\theta\|^2}{2}.
\end{align*}
However, this follows automatically from the assumption that $\lambda \geq M R_{\theta}^{2} d/2$, since $\|\theta\| \leq  R_\theta \sqrt{d}$.
\end{proof}

\begin{proof}[Proof of Lemma~\ref{lemma:growthrates:3}]
Since this result will be used in the proof of Theorem~\ref{thm:estimationerror:3}, we will use the same formulation as there. Therefore, we introduce the change of variable $w = (\theta-\theta_0)/\sqrt{d}$, and express \eqref{eq:dro:dual:uni:3} in vector form, resulting in
\begin{align}
\label{eq:proof:lemma:growthrates:3:1}
    \min_{w \in \mathcal W}\; \max_{u \in \mathbb R^n} \; - \frac{1}{n} u^\top (\sqrt{d}A)w +\frac{1}{n} u^\top z + \frac{1}{4\lambda n}\|\theta_0 + \sqrt{d}w\|^2\|u\|^2-\frac{1}{n}\sum_{i=1}^{n}L^*(u_i).
\end{align}
Here, $\mathcal W$ is the feasible set of $w$ obtained from $\Theta$ after the change of variable, $A \in \mathbb R^{n \times d}$ is a matrix whose rows are the vectors $x_i$, for $i=1,\ldots,n$, and $z \in \mathbb R^n$ is the measurement noise vector with entries i.i.d.\ distributed according to $\mathbb P_Z$. 

We will prove that the constraint $u \in \mathbb R^n$ can be restricted, without loss of generality, to $\{ u \in \mathbb R^n:\, \|u\|\leq K_\beta \sqrt{n}\}$, for some constant $K_\beta>0$. For this, we will proceed in two steps. We will first show that for any $w \in \mathcal W$, the optimizer $u_{\star,1}$ of
\begin{align}
\label{eq:proof:lemma:growthrates:3:2}
    \max_{u \in \mathbb R^n} \; - \frac{1}{n} u^\top (\sqrt{d}A)w +\frac{1}{n} u^\top z + \frac{1}{4\lambda n}\|\theta_0 + \sqrt{d}w\|^2\|u\|^2-\frac{1}{n}\sum_{i=1}^{n}L^*(u_i)
\end{align}
is upper bounded by the optimizer $u_{\star,2}$ of the following optimization problem
\begin{align}
\label{eq:proof:lemma:growthrates:3:3}
    \max_{u \in \mathbb R^n} \; - \frac{1}{n} u^\top (\sqrt{d}A)w +\frac{1}{n} u^\top z + \frac{1}{4\lambda n}\|\theta_0 + \sqrt{d}w\|^2\|u\|^2- \frac{\|u\|^2}{2 n M}.
\end{align}
Secondly, we will show that the optimizer $u_{\star,2}$ satisfies $\|u_{\star,2}\|\leq K_\beta \sqrt{n}$, for some $K_\beta>0$. 

The first step follows easily by noticing that the objective function in \eqref{eq:proof:lemma:growthrates:3:2} is a difference of convex functions, with the subtracting convex function $\sum_{i=1}^{n}L^*(u_i)/n$ which satisfies the following two properties $\min_{v \in \mathbb R} L^*(v) = L^*(0) =0$, and $L^*(v) \geq v^2/(2M) \geq 0$. In particular, the first property follows from the fact that $\min_{v \in \mathbb R} L(v) = L(0) =0$, and the second property follows from the additional fact that $L^*$ is $1/M$-strongly convex. Consequently, due to these properties, we have that $u_{\star,2} \geq u_{\star_1}$.

We will now focus on the second step. Notice first that the assumptions $\lambda > M R_{\theta}^{2} d/2$ and $\|\theta\|^2 = \|\theta_0 + \sqrt{d}w\|^2 \leq R_{\theta}^2 d$ ensure that the objective function in \eqref{eq:proof:lemma:growthrates:3:3} is concave in $u$. Indeed, can be easily seen from the concavity in $u$ of
\begin{align*}
    \frac{1}{4\lambda}\|\theta_0 + \sqrt{d}w\|^2\|u\|^2- \frac{\|u\|^2}{2M}.
\end{align*}
From the first order optimality conditions in \eqref{eq:proof:lemma:growthrates:3:3}, we have that
\begin{align}
\label{eq:proof:lemma:growthrates:3:4} 
    u_{\star,2} = \left( \frac{\|\theta_0 + \sqrt{d}w\|^2}{2 \lambda} -\frac{1}{M} \right)^{-1} \left( \sqrt{d}A w - z \right).
\end{align}
Now, since $\lambda > M R_\theta^2 d/2 \geq M\|\theta_0 + \sqrt{d}w_\star\|^2/2$, we have that there exists a constant $c_1$ such that
\begin{align*}
    \left(\frac{\|\theta_0 + \sqrt{d}w\|^2}{2 \lambda} - \frac{1}{M}\right)^{-1} < c_1.
\end{align*}
Moreover, we have that there exist constants $c_2, c_3$ such that, w.p.a.\ $1$, $\|A\| \leq c_2$ (since $A$ is a Gaussian matrix with entries i.i.d.\ $\mathcal N(0,1/d)$), and $\|z\| \leq c_3\sqrt{n}$ (due to the weak law of large numbers). Finally, using the fact that $w$ lives in the compact set $\mathcal W$, we have that there exists some constant $K_\beta$ such that $\|u_{\star,2}\|\leq K_\beta \sqrt{n}$. This concludes the proof.
\end{proof}

\section{Proofs for Section~\ref{sec:linear}}
\label{appendix:additional:proofs:4}

\begin{proof}[Proof of Theorem~\ref{thm:estimationerror:1}]
The proof is an application of the main results of \cite{thrampoulidis2018precise} to the particular regularized estimation problem \eqref{eq:dual:p=1}. In what follows, we explain why all the assumptions required by \cite{thrampoulidis2018precise} hold, and why their result simplifies to \eqref{eq:est:1:minimax} in our case. 

First of all, notice that in \eqref{eq:dual:p=1} the loss function is separable, while the regularizer function is not. As a consequence, the minimax scalar problem \eqref{eq:est:1:minimax} is obtained as a combination of Theorem~$1$ (general case) and Theorem~$2$ (separable case) from \cite{thrampoulidis2018precise}. Therefore, we will verify that the loss function and the noise satisfy the assumptions of Theorem~$2$ and that the regularizer function satisfies the assumptions of Theorem~1.

We start by rewriting \eqref{eq:dual:p=1} in the following form
\begin{align}
    \label{eq:proof:thm:estimationerror:1:1}
    \min_{\theta \in \mathbb R^d}\;  \frac{\rho}{d} \left( \sum_{i=1}^{n}\, L(y_i-\theta^\top x_i)  +  \sqrt{n} \varepsilon_0 \text{Lip}(L) \|\theta\|\right).
\end{align}
The normalization factor $1/d$ guarantees that the optimal cost is of constant order. Moreover, the formulation \eqref{eq:dual:p=1} is now in the form required by \cite{thrampoulidis2018precise} (see the proof of Theorem~1 for details), with loss function $\sum_{i=1}^{n}\, L(\theta^\top x_i-y_i)$ and regularizer $\sqrt{n} \varepsilon_0 \text{Lip}(L) \|\theta\|$.

We first show that the loss function and noise satisfy the assumptions of Theorem~$2$ from \cite{thrampoulidis2018precise}:
\begin{enumerate}
    \item The loss function and noise are such that
    \begin{align}
    \label{eq:proof:thm:estimationerror:1:2}
    \mathbb E_{\mathcal N(0,1) \otimes \mathbb P_Z}\left[ L'_+(cG+Z)^2 \right] < \infty.
    \end{align}
    
    \emph{Proof}: From Assumptions~\ref{assump:dro}\ref{assump:ell1}-\ref{assump:ell2} we know that $L$ is continuous, convex and grows linearly at infinity, with slope at most $C$. Therefore, $L'_+$ is bounded and the quantity in \eqref{eq:proof:thm:estimationerror:1:2} is finite. Alternatively, this can also be seen from Assumption~\ref{assump:ell:diff}, which guarantees that $L'$ is continuous, and the fact that $L'$ is at most $C$ at infinity (due to Assumption~\ref{assump:dro}\ref{assump:ell2}). This assumption has been shown in \cite{thrampoulidis2018precise} to guarantee that the quantity $e_L(cG+Z;\tau) - L(Z)$ under the expectation in \eqref{eq:L} is absolutely integrable. 
    
    \item Either $\mathbb E_{\mathbb P_Z}[Z^2]<\infty$, i.e., $Z$ has finite second moment, or $\sup_{u \in \mathbb R} |L(u)|/|u| < \infty$, i.e., $L$ grows at most linearly at infinity.
    
    \emph{Proof}: As discussed above, in our case the latter condition holds. Notice that from Assumption~\ref{assump:dro}\ref{assump:pstar}, which states that the true distribution $\mathbb P$ has finite first moment, $Z$ necessarily has finite first moment, but the second moment might be infinite.
\end{enumerate}

We now show that the objective function from Theorem~$1$ in \cite{thrampoulidis2018precise} reduces to \eqref{eq:est:1:minimax} in our case. For this, we focus on \eqref{eq:proof:thm:estimationerror:1:1}, and consider $\varepsilon_0 \text{Lip}(L)$ as the regularization parameter, and $\sqrt{n}\|\theta\|$ as the regularizer. Consequently, we need to study the convergence in probability of the quantity
\begin{align}
    \label{eq:proof:thm:estimationerror:1:3}
    \frac{1}{d} \left( e_{\sqrt{n}\|\cdot\|}(c h+\theta_0;\tau) - \sqrt{n}\|\theta_0\| \right)
\end{align}
where $h$ a standard Gaussian vector with entries $\mathcal N(0,1)$, and show that its limit is exactly the function $\mathcal G$ defined in \eqref{eq:R}. First notice that, due to Assumption~\ref{assump:cgmt}\ref{assump:theta0}, we have
\begin{align}
    \label{eq:proof:thm:estimationerror:1:4}
    \frac{1}{d} \sqrt{n}\|\theta_0\| \overset{P}{\to} \frac{\sigma_{\theta_0}}{\sqrt{\rho}}.
\end{align}

Now let's concentrate on the first term, which involves the Moreau envelope of $\sqrt{n}\|\cdot\|$. After some basic convex optimization steps, it can be seen that
\begin{align}
    \label{eq:proof:thm:estimationerror:1:5}
    e_{\sqrt{n}\|\cdot\|}(c h+\theta_0;\tau) = \begin{cases} \sqrt{n}\|c h+\theta_0\|-{n\tau}/{2} \quad\quad &\mbox{if }\; \|c h+\theta_0\|>\sqrt{n}\tau, \\ \|c h+\theta_0\|^2/{(2 \tau)}
 \quad\quad &\mbox{if }\; \|c h+\theta_0\|\leq\sqrt{n}\tau. \end{cases}
\end{align}
Finally, the convergence in probability to $\mathcal G(c,\tau)$ follows from 
\begin{align}
    \label{eq:proof:thm:estimationerror:1:6}
    \frac{\|c h+\theta_0\|^2}{d} \overset{P}{\to} c^2 + \sigma_{\theta_0}^2,
\end{align}
using Assumption~\ref{assump:cgmt}\ref{assump:theta0} and Lemma~\ref{lemma:triangular:array}.

We now show that $\mathcal G(c,\tau)$ satisfies the following two assumptions required by Theorem~1 in \cite{thrampoulidis2018precise}:
\begin{enumerate}
    \item $\lim_{\tau \to 0^+}  \mathcal G(\tau,\tau) = 0$.
    
    \emph{Proof}: For $\tau \to 0^+$, we have $\|\tau h+\theta_0\|>\sqrt{n}\tau$, and thus $\mathcal G(\tau,\tau)$ satisfies
    \begin{align*}
        \lim_{\tau \to 0^+} \mathcal G(\tau,\tau) = \lim_{\tau \to 0^+} \frac{\sqrt{\tau^2 + \sigma_{\theta_0}^2}}{\sqrt{\rho}}  - \frac{\tau}{2 \rho} - \frac{\sigma_{\theta_0}}{\sqrt{\rho}} = \frac{\sigma_{\theta_0}}{\sqrt{\rho}} - \frac{\sigma_{\theta_0}}{\sqrt{\rho}}= 0.
    \end{align*}
    
    \item $\lim_{c \to +\infty} (c^2/(2\tau)-\mathcal G(c,\tau)) = +\infty$ for all $\tau>0$.
    
    \emph{Proof}: For $c \to +\infty$, we have $\|c h+\theta_0\|>\sqrt{n}\tau$, and thus $\mathcal G(c,\tau)$ satisfies
    \begin{align*}
        \mathcal G(c,\tau) = \frac{\sqrt{c^2 + \sigma_{\theta_0}^2}}{\sqrt{\rho}}  - \frac{\tau}{2 \rho} - \frac{\sigma_{\theta_0}}{\sqrt{\rho}},
    \end{align*}
    from which the desired limit follows immediately.
\end{enumerate}

Finally, we show that for the particular case of regularized estimators of the form \eqref{eq:dual:p=1}, the definition of the expected Moreau envelope $\mathbb E_{\mathcal N(0,1) \otimes \mathbb P_Z}\left[ e_L(cG+Z;\tau) - L(Z) \right]$ used in \cite{thrampoulidis2018precise} can be simplified to the expected Moreau envelope $\mathcal L(c,\tau)$ defined in \eqref{eq:L}, i.e., the term $-L(Z)$ can be dropped. In \cite{thrampoulidis2018precise}, the term $-L(Z)$ was added to account for the case when the noise $Z$ has unbounded moments. In that case, the sequence 
\begin{align}
\label{eq:conv:moreau:envelope}
    \frac{1}{n} \sum_{i=1}^n e_{L}(\alpha G + Z,\frac{\tau_1}{\beta})
\end{align}
might not converge as $d \to \infty$, while by subtracting $L(Z)$ it is shown to converge. In our case, the convergence of \eqref{eq:conv:moreau:envelope} is automatically guaranteed by Assumption~\ref{assump:dro}\ref{assump:ell2} on the growth rate of the loss function $L$, even though the noise $Z$ might have unbounded second moment. Indeed, its limit (in probability) can be immediately recovered from the weak law of large numbers as
\begin{align}
    \label{eq:conv:moreau:envelope:1}
    \frac{1}{n} \sum_{i=1}^n e_{L}(\alpha G + Z,\frac{\tau_1}{\beta}) \overset{P}{\to} \mathbb E_{\mathcal N(0,1) \otimes \mathbb P_Z}\left[ e_{L}(\alpha G + Z;\frac{\tau_1}{\beta})\right],
\end{align}
which is precisely the expected Moreau envelope. The quantity under the expectation in \eqref{eq:conv:moreau:envelope:1} is absolutely integrable, and therefore the expected Moreau envelope is well defined, as explained in what follows. For every $\beta \geq 0$ and $\tau_1>0$ we have
\begin{align*}
    \left|e_{L}(\alpha G + Z;\frac{\tau_1}{\beta})\right| = \min_{v \in \mathbb R}\; \frac{\beta}{2 \tau_1} \left( \alpha G + Z - v \right)^2 + L(v) &\leq L(\alpha G + Z)
\end{align*}
which is integrable due to the fact that both $G$ and $Z$ have finite first moment. In particular, the fact that $Z$ has finite first moment is a direct consequence of Assumption~\ref{assump:dro}\ref{assump:pstar}.

The result now follows from Theorem~$1$ and Theorem~$2$ in \cite{thrampoulidis2018precise}.
\end{proof}

\begin{proof}[Proof of Theorem~\ref{thm:estimationerror:2}]
Since the proof is long, we divide it into steps, which we briefly explain before jumping into the technical details.

\paragraph{Step 1: We show that problem \eqref{eq:dro} can be equivalently rewritten as a (PO) problem.} The proof builds upon the dual formulation \eqref{eq:dro:dual:uni} presented in Fact~\ref{thm:nonconvexduality}. After introducing the change of variable $w = (\theta-\theta_0)/\sqrt{d}$ and expressing it in vector form, \eqref{eq:dro:dual:uni} becomes
\begin{align}
\label{eq:proof:thm:estimationerror:2:0}
    \min_{w \in \mathcal W, \lambda \geq 0}\; \max_{u \in \mathbb R^n} \; \lambda \varepsilon - \frac{1}{n} u^\top (\sqrt{d}A)w +\frac{1}{n} u^\top z + \frac{1}{4\lambda n}\|\theta_0 + \sqrt{d}w\|^2\|u\|^2-\frac{1}{n}\sum_{i=1}^{n}L^*(u_i)
\end{align}
where $\mathcal W$ is the feasible set of $w$ obtained from $\Theta$ after the change of variable, $A \in \mathbb R^{n \times d}$ denotes the matrix whose rows are the vectors $x_i$, for $i=1,\ldots,n$, and $z \in \mathbb R^n$ is the measurement noise vector with entries i.i.d.\ distributed according to $\mathbb P_Z$. Since we are interested only in the case where $w$ satisfies $\|w\| < \sigma_{\theta_0}$, instead of the constraint set $\mathcal W$, we will work with a set of the form $\mathcal S_w := \{w \in \mathbb R^d: \, \|w\| \leq K_\alpha\}$, for some $K_\alpha < \sigma_{\theta_0}$. This limitation on the magnitude of $w$, which might seem arbitrary and not necessary at this point, will be shown to be fundamental later in the proof. Therefore, from now on we impose such a constraint on $w$ (i.e., $w \in \mathcal S_w$). Notice that if $K_\alpha < \sigma_{\theta_0}$, then the condition $R_\theta \geq 2 \sigma_{\theta_0}$ ensures that $\mathcal S_w \subset \mathcal W$ w.p.a.\ $1$ as $d,n \to \infty$.
%Notice that, due to Assumption~\ref{assump:Theta}, the set $\mathcal W$ is convex and compact. Moreover, following the assumptions~\ref{assump:Theta} and \ref{assump:cgmt}\ref{assump:theta0}, w.p.a.\ 1 as $d,n \to \infty$, we have that $\mathcal W \subset \{w \in \mathbb R^d: \,  \|w\| \leq R_\theta + \sigma_{\theta_0}\}$. 

Due to Lemma~\ref{lemma:convexity}, we know that $\varepsilon_0 \leq \rho^{-1} M^{-2} R_{\theta}^{-2}\underline{L}_n$ guarantees that $\lambda \geq M R_\theta \sqrt{d} \|\theta_0 + \sqrt{d}w\|/2$. As a consequence, we define the set $\Lambda := \{\lambda \in \mathbb R:\, \lambda \geq M R_\theta \sqrt{d} \|\theta_0 + \sqrt{d}w\|/2\}$, and equivalently rewrite \eqref{eq:proof:thm:estimationerror:2:1} as
\begin{align}
\label{eq:proof:thm:estimationerror:2:1}
    \min_{w \in \mathcal S_w, \lambda \in \Lambda}\; \max_{u \in \mathbb R^n} \; \lambda \varepsilon - \frac{1}{n} u^\top (\sqrt{d}A)w +\frac{1}{n} u^\top z + \frac{1}{4\lambda n}\|\theta_0 + \sqrt{d}w\|^2\|u\|^2-\frac{1}{n}\sum_{i=1}^{n}L^*(u_i)
\end{align}
Moreover, from Lemma~\ref{lemma:growthrates} we know that the optimal solution $u_\star$ is in the order of $\sqrt{n}$. As a consequence, by rescaling the variable $u$ as $u \to u \sqrt{n}$, and introducing the convex compact set $\mathcal S_u := \{u \in \mathbb R^n:\, \|u\|\leq K_\beta\}$, for some sufficiently large $K_\beta > 0$, we can equivalently rewrite \eqref{eq:proof:thm:estimationerror:2:1} as
\begin{align}
\label{eq:proof:thm:estimationerror:2:1:1}
    \min_{w \in \mathcal S_w, \lambda \in \Lambda}\; \max_{u \in \mathcal S_u} \; \lambda \varepsilon - \frac{1}{\sqrt{n}} u^\top (\sqrt{d}A)w +\frac{1}{\sqrt{n}} u^\top z + \frac{1}{4\lambda}\|\theta_0 + \sqrt{d}w\|^2\|u\|^2-\frac{1}{n}\sum_{i=1}^{n}L^*(u_i \sqrt{n})
\end{align}

We will now show that the assumption that $L$ can be written as the Moreau envelope with parameter $M$ of a convex function $f$ is natural, and always satisfied for $M$-smooth functions (which is precisely our case, due to Assumption~\ref{assump:dro}\ref{assump:ell:diff}). Since $L$ is $M$-smooth, we have that its convex conjugate $L^*$ is $1/M$-strongly convex, and therefore $L^*$ can be written as 
\begin{align}
    \label{eq:proof:thm:estimationerror:2:2}
    L^*(\cdot) = \frac{1}{2M}(\cdot)^2 + f^*(\cdot),
\end{align}
where $f^*$ is the conjugate of a convex function $f$. Notice that $f$ can be easily obtained (as a function of $L$) from \eqref{eq:proof:thm:estimationerror:2:2} as $f = (L^* - 1/(2M)(\cdot)^2)^*$. Conversely, the loss function $L$ is nothing but the Moreau envelope of $f$, as can be seen from the following
\begin{align*}
    L(\cdot) = \left(\frac{1}{2M}(\cdot)^2 + f^*(\cdot)\right)^*(\cdot) = \left(\frac{M}{2}(\cdot)^2 \star_{\text{inf}} f(\cdot)\right)(\cdot) = e_f\left(\cdot,\frac{1}{M}\right),
\end{align*}
where $\star_{\text{inf}}$ denotes the infimal convolution operation. Since $L$ is continuous and convex (notice that $L$ is trivially proper due to Assumption~\ref{assump:cgmt}\ref{assump:ell:0}), with $\min_{v \in \mathbb R} L(v) = 0$, we have that $L^*$ is lower semicontinuous and convex, and therefore both $f$ and $f^*$ are lower semicontinuous and convex.

Using the decomposition \eqref{eq:proof:thm:estimationerror:2:2}, problem \eqref{eq:proof:thm:estimationerror:2:1:1} can be rewritten as
\begin{align}
\label{eq:proof:thm:estimationerror:2:3}
    \min_{w \in \mathcal S_w, \lambda \in \Lambda}\; \max_{u \in \mathcal S_u} \; \lambda \varepsilon - \frac{1}{\sqrt{n}} u^\top (\sqrt{d}A)w +\frac{1}{\sqrt{n}} u^\top z + \frac{1}{4\lambda}\|\theta_0 + \sqrt{d}w\|^2\|u\|^2- \frac{1}{2 M}\|u\|^2- \frac{1}{n}\sum_{i=1}^{n}f^*(u_i \sqrt{n}).
\end{align}
It is important now to highlight that the expression
\begin{align*}
    \frac{1}{4\lambda}\|\theta_0 + \sqrt{d}w\|^2\|u\|^2- \frac{1}{2M}\|u\|^2
\end{align*}
is concave is $u$. This follows from the constraint $\lambda \in \Lambda$, which ensures that $\lambda \geq M R_\theta \sqrt{d} \|\theta\|/2 \geq M\|\theta\|^2/2 = M\|\theta_0 + \sqrt{d}w\|^2/2$. As will be shown later, the decomposition \eqref{eq:proof:thm:estimationerror:2:2} and the concavity in $u$ of the above-mentioned expression will allow us to use the CGMT.

For ease of notation, we introduce the function $F^*:\mathbb R^n \to \mathbb R$, defined as follows
\begin{align*}
    F^*(v):=\sum_{i=1}^{n}f^*(v_i).
\end{align*}
Since $f,f^*$ are lower semicontinuous and convex, we have that $F^*$ is lower semicontinuous, and convex, and therefore its convex conjugate $F = F^{**}$ can be easily recovered as follows
\begin{align*}
    F(u) = \sup_{v \in \mathbb R^n} v^\top u - F^*(v) = \sum_{i=1}^n\; \sup_{v_i \in \mathbb R}\; v_i u_i - f^*(v_i) = \sum_{i=1}^n f(u_i).
\end{align*}
We now rewrite $F^*(u \sqrt{n})$ using its convex conjugate as follows
\begin{align}
\label{eq:proof:thm:estimationerror:2:4}
\begin{split}
    \min_{\substack{w \in \mathcal S_w \\ \lambda \in \Lambda \\ s \in \mathbb R^n}}\; \max_{u \in \mathcal S_u}\;  \lambda \varepsilon - \frac{1}{\sqrt{n}} u^\top (\sqrt{d}A)w +\frac{1}{\sqrt{n}} u^\top z + \frac{1}{4\lambda}\|\theta_0 + \sqrt{d}w &\|^2 \|u\|^2 -\frac{1}{2 M}\|u\|^2 - \\&-\frac{1}{\sqrt{n}}s^\top u + \frac{1}{n} F(s)
\end{split}
\end{align}
where we have used Sion's minimax principle to exchange the minimization over $s$ with the maximization over $u$. In particular, this last step is possible only because the objective function in \eqref{eq:proof:thm:estimationerror:2:4} is concave in $u$, due to the decomposition \eqref{eq:proof:thm:estimationerror:2:2}. 

Since the objective function in \eqref{eq:proof:thm:estimationerror:2:4} is convex-concave in $(\lambda,u)$ and $\mathcal S_u$ is a convex compact set, we can use Sion's minimax principle to exchange the minimization over $\lambda$ to obtain the equivalent optimization problem
\begin{align}
\label{eq:proof:thm:estimationerror:2:6}
\begin{split}
    \min_{\substack{w \in \mathcal S_w \\ s \in \mathbb R^n}}\; \max_{u \in \mathcal S_u}\;   - \frac{1}{\sqrt{n}} u^\top (\sqrt{d}A)w +\frac{1}{\sqrt{n}} u^\top z &-\frac{1}{2M} \|u\|^2-\frac{1}{\sqrt{n}}s^\top u + \frac{1}{n} F(s) + \\& +\min_{\lambda \in \Lambda}\; \lambda \varepsilon + \frac{1}{4\lambda}\|\theta_0 + \sqrt{d}w\|^2\|u\|^2
\end{split}
\end{align}

Notice that the new objective function in \eqref{eq:proof:thm:estimationerror:2:6} (which contains the minimization over $\lambda$) is convex in $(w,s)$ and concave in $u$. Indeed, the concavity in $u$ follows easily, since the objective function in \eqref{eq:proof:thm:estimationerror:2:6} is the result of minimization of concave functions. Moreover, the convexity in $w$ follows from the fact that the objective function in \eqref{eq:proof:thm:estimationerror:2:4} is jointly convex in $(w,\lambda)$ and partial minimization of jointly convex functions remains convex. Indeed, the joint convexity in $(w,\lambda)$ can be seen by noticing that $1/(4\lambda)\|\theta_0 + \sqrt{d}w\|^2\|u\|^2$ is nothing but a shifted version of the perspective function of $1/4 \|\sqrt{d}w\|^2\|u\|^2$. Finally, the convexity in $s$ can be concluded from the convexity of $F$, and the joint convexity in $(w,s)$ follows easily since $w$ and $s$ are decoupled.

Using Assumption~\ref{assump:0}, we can see that $\sqrt{d}A$ has entries i.i.d.\ $\mathcal N(0,1)$. Moreover, the objective function \eqref{eq:proof:thm:estimationerror:2:6} is convex-concave in $(w,u)$, with $\mathcal S_w, \mathcal S_u$ convex compact sets. As a consequence, problem \eqref{eq:proof:thm:estimationerror:2:6} is a (PO) problem, in the form \eqref{eq:cgmt:po}. 

%------------------------------------------------------------------------------------
\paragraph{Step 2: We apply CGMT and obtain two modified (AO) problems.} 

Since \eqref{eq:proof:thm:estimationerror:2:6} is a (PO) problem, we can associate to it the following (AO) problem
\begin{align}
\label{eq:proof:thm:estimationerror:2:7}
\begin{split}
    \min_{\substack{w \in \mathcal S_w \\ s \in \mathbb R^n}}\; \max_{u \in \mathcal S_u}\;   \frac{1}{\sqrt{n}} \|w\|g^\top u - \frac{1}{\sqrt{n}} \|u\|h^\top w +\frac{1}{\sqrt{n}} u^\top z -\frac{1}{2M}\|u\|^2-\frac{1}{\sqrt{n}}s^\top u + \frac{1}{n} F(s) + \\ + \min_{\lambda \in \Lambda}\; \lambda \varepsilon + \frac{1}{4\lambda}\|\theta_0 + \sqrt{d}w\|^2\|u\|^2
\end{split}
\end{align}
which is tightly related to the (PO) problem in the high-dimensional regime, as explained in Section~\ref{sec:cgmt}. At this point, following the discussion presented in Section~\ref{sec:cgmt}, we consider the following modified (AO) problem,
\begin{align}
\label{eq:proof:thm:estimationerror:2:8}
\begin{split}
    \max_{0 \leq \beta \leq K_\beta}\; \min_{\substack{w \in \mathcal S_w\\ s \in \mathbb R^n}}\; \max_{\|u\|= \beta}\;  \frac{1}{\sqrt{n}} \|w\|g^\top u - \frac{1}{\sqrt{n}} \|u\|h^\top w +\frac{1}{\sqrt{n}} u^\top z -\frac{1}{2M}\|u\|^2-\frac{1}{\sqrt{n}}s^\top u + \frac{1}{n} F(s) + \\ + \min_{\lambda \in \Lambda}\; \lambda \varepsilon + \frac{1}{4\lambda}\|\theta_0 + \sqrt{d}w\|^2\|u\|^2,
\end{split}
\end{align}
where we have separated the maximization over $u \in \mathcal S_u = \{u \in \mathbb R^n:\, \|u\|\leq K_\beta\}$ into the maximization over the magnitude of $u$, i.e., $0 \leq \beta \leq K_\beta$, and the maximization over $\|u\| = \beta$. Recall that, since the (AO) problem is not convex-concave (due to the random vectors $g$ and $h$), the exchange of the minimization over $(w,s)$ and the maximization over $\beta$ is not justified. Therefore the two problems \eqref{eq:proof:thm:estimationerror:2:7} and \eqref{eq:proof:thm:estimationerror:2:8} might not be equivalent. However, as seen in Fact~\ref{prop:cgmt}, the modified (AO) problem \eqref{eq:proof:thm:estimationerror:2:8} can be used to study the optimal solution of the (PO) problem \eqref{eq:proof:thm:estimationerror:2:6} in the asymptotic regime, based solely on the optimal value of the modified (AO) problem.

We now proceed by eliminating the variable $\lambda$ from \eqref{eq:proof:thm:estimationerror:2:8}. The objective function is convex in $\lambda$, and its optimizer $\lambda_\star$ can be easily computed as follows
\begin{align*}
    \lambda_\star(w,u) = \begin{cases} \|\theta_0+\sqrt{d}w\|\|u\|/(2\sqrt{\varepsilon}) & \quad\quad \text{if} \quad \|\theta_0+\sqrt{d}w\|\|u\|/(2\sqrt{\varepsilon}) > M R_\theta \sqrt{d} \|\theta_0+\sqrt{d}w\|/2 \\ M R_\theta \sqrt{d} \|\theta_0+\sqrt{d}w\|/2 & \quad\quad \text{otherwise}, \end{cases}
\end{align*}
or, equivalently,
\begin{align}
\label{eq:proof:thm:estimationerror:2:9}
    \lambda_\star(w,u) = \begin{cases} \sqrt{n} \|\theta_0+\sqrt{d}w\|\|u\|/(2\sqrt{\varepsilon_0}) & \quad\quad \text{if} \quad \|u\| > B := \sqrt{\varepsilon_0 \rho}\, M R_\theta \\ M R_\theta \sqrt{d} \|\theta_0+\sqrt{d}w\|/2 & \quad\quad \text{otherwise}, \end{cases}
\end{align}
where we have used the fact that $\varepsilon = \varepsilon_0/n$.

As a consequence, by introducing the optimal solution $\lambda_\star$ into \eqref{eq:proof:thm:estimationerror:2:8}, we can equivalently rewrite problem \eqref{eq:proof:thm:estimationerror:2:8} as
\begin{align}
\label{eq:proof:thm:estimationerror:2:9:1}
    \max \left\{\mathcal V_{d,1}, \mathcal V_{d,2}\right\}
\end{align}
with
\begin{align}
\label{eq:proof:thm:estimationerror:2:9:2}
\begin{split}
    \mathcal V_{d,1} := \max_{B < \beta \leq K_\beta}\; \min_{\substack{w \in \mathcal S_w\\ s \in \mathbb R^n}}\; \max_{\|u\|= \beta}\;  \frac{1}{\sqrt{n}} \|w\|g^\top u - \frac{1}{\sqrt{n}} \|u\|h^\top w +\frac{1}{\sqrt{n}} u^\top z -\frac{1}{2M}\|u\|^2-\frac{1}{\sqrt{n}}s^\top u  + \\ + \frac{1}{n} F(s) + \frac{\sqrt{\varepsilon_0}}{\sqrt{n}}\|\theta_0 + \sqrt{d}w\|\|u\|,
\end{split}
\end{align}
and
\begin{align}
\label{eq:proof:thm:estimationerror:2:9:3}
\begin{split}
    \mathcal V_{d,2} := \max_{0 \leq \beta \leq B}\; \min_{\substack{w \in \mathcal S_w\\ s \in \mathbb R^n}}\; \max_{\|u\|= \beta}\;  \frac{1}{\sqrt{n}} \|w\|g^\top u - \frac{1}{\sqrt{n}} \|u\|h^\top w +\frac{1}{\sqrt{n}} u^\top z -\frac{1}{2M}\|u\|^2-\frac{1}{\sqrt{n}}s^\top u  + \\ + \frac{1}{n} F(s) + \frac{\|\theta_0 + \sqrt{d}w\|}{\sqrt{n}} \left( p + \frac{\|u\|^2}{q} \right),
\end{split}
\end{align}
where, for ease of notation, we have introduced the two constants $p := (\varepsilon_0\sqrt{\rho}M R_\theta)/2$ and $q := 2 \sqrt{\rho} M R_\theta$.

In what follows, we will first simplify the two problems \eqref{eq:proof:thm:estimationerror:2:9:2} and \eqref{eq:proof:thm:estimationerror:2:9:3}, by reducing them to scalar problems. Then, following the reasoning of Fact~\ref{prop:cgmt}, we will study their asymptotic optimal values for $d,n \to \infty$. We first focus on problem \eqref{eq:proof:thm:estimationerror:2:9:2}, but similar reasoning will be used subsequently for problem \eqref{eq:proof:thm:estimationerror:2:9:3}.

%------------------------------------------------------------------------------------
\paragraph{Step 3.1: We consider the first modified (AO) problem \eqref{eq:proof:thm:estimationerror:2:9:2}, and show that it can be reduced to a scalar problem, which involves only four scalar variables.}

We start by scalarizing problem \eqref{eq:proof:thm:estimationerror:2:9:2} over the magnitude of $u$, leading to the following problem
\begin{align}
\label{eq:proof:thm:estimationerror:2:10}
    \max_{B < \beta \leq K_\beta}\; \min_{\substack{w \in \mathcal S_w\\ s \in \mathbb R^n}}\;   \frac{\beta}{\sqrt{n}}\| \|w\|g + z - s \| - \frac{\beta}{\sqrt{n}} h^\top w + \frac{\sqrt{\varepsilon_0}\beta}{\sqrt{n}}\|\theta_0 + \sqrt{d}w\| -\frac{\beta^2}{2M} + \frac{1}{n} F(s).
\end{align}
We now employ the square-root trick to rewrite the first term in the objective function of \eqref{eq:proof:thm:estimationerror:2:10} as follows
\begin{align*}
    \frac{1}{\sqrt{n}}\|\|w\|g + z - s \| = \inf_{\tau_1>0} \frac{\tau_1}{2} + \frac{1}{2n\tau_1}\|\|w\|g + z - s \|^2.
\end{align*}
Introducing this in \eqref{eq:proof:thm:estimationerror:2:10} and re-organizing the terms gives
\begin{align*}
    \max_{B < \beta \leq K_\beta}\; \min_{\substack{w \in \mathcal S_w\\ s \in \mathbb R^n \\ \tau_1>0}}\;   \frac{\beta \tau_1}{2} -\frac{\beta^2}{2M} + \frac{\beta}{2n\tau_1}\| \|w\|g + z - s \|^2  + \frac{1}{n} F(s) - \frac{\beta}{\sqrt{n}} h^\top w + \frac{\sqrt{\varepsilon_0}\beta}{\sqrt{n}}\|\theta_0 + \sqrt{d}w\|,
\end{align*}
which allows us to introduce the Moreau envelope of $F$, i.e.,
\begin{align}
\label{eq:proof:thm:estimationerror:2:11}
    e_{F}(\alpha g + z,\frac{\tau_1}{\beta}) := \min_{s \in \mathbb R^n}\; \frac{\beta}{2\tau_1}\|\alpha g + z - s \|^2 + F(s),
\end{align}
giving rise to
\begin{align}
\label{eq:proof:thm:estimationerror:2:12}
    \max_{B < \beta \leq K_\beta}\; \min_{\substack{w \in \mathcal S_w \\ \tau_1>0}}\;   \frac{\beta \tau_1}{2}  -\frac{\beta^2}{2M} + \frac{1}{n} e_{F}(\|w\| g + z,\frac{\tau_1}{\beta}) - \frac{\beta}{\sqrt{n}} h^\top w + \frac{\sqrt{\varepsilon_0}\beta}{\sqrt{n}}\|\theta_0 + \sqrt{d}w\|.
\end{align}
We employ the square-root trick a second time and rewrite the last term in the objective function of problem \eqref{eq:proof:thm:estimationerror:2:12} as
\begin{align}
\label{eq:proof:thm:estimationerror:2:13}
    \frac{1}{\sqrt{n}}\|\theta_0 + \sqrt{d}w\| &= \inf_{\tau_2>0} \frac{\tau_2}{2} + \frac{1}{2n\tau_2}\|\theta_0 + \sqrt{d}w\|^2.
\end{align}
leading to
\begin{align}
\label{eq:proof:thm:estimationerror:2:14}
    \max_{B < \beta \leq K_\beta}\; \min_{\substack{w \in \mathcal S_w \\ \tau_1, \tau_2>0}}\; \frac{\beta \tau_1}{2}  -\frac{\beta^2}{2M} + \frac{1}{n} e_{F}(\|w\| g + z,\frac{\tau_1}{\beta}) - \frac{\beta}{\sqrt{n}} h^\top w + \frac{\sqrt{\varepsilon_0}\beta\tau_2}{2} + \frac{\sqrt{\varepsilon_0}\beta}{2n\tau_2}\|\theta_0 + \sqrt{d}w\|^2.
\end{align}

We now proceed with the aim of scalarizing the problem over the magnitude of $w$. Similarly to what was done for $u$, we can separate the constraint $w \in \mathcal S_w$ into the two constraints $0 \leq \alpha \leq K_\alpha$ and $\|w\| = \alpha$, and by expressing $\|\theta_0 + \sqrt{d}w\|^2$ as $\|\theta_0\|^2 + d\|w\|^2 + 2\sqrt{d}\theta_0^\top w$, we can now scalarize the problem over the magnitude of $w$ and obtain
\begin{align}
\label{eq:proof:thm:estimationerror:2:15}
\begin{split}
    \max_{B < \beta \leq K_\beta}\; \inf_{\substack{0 \leq \alpha \leq K_\alpha \\ \tau_1,\tau_2>0}}\; \frac{\beta \tau_1}{2} + \frac{\sqrt{\varepsilon_0}\beta\tau_2}{2}  -\frac{\beta^2}{2M} + \frac{1}{n} e_{F}(\alpha g + z,\frac{\tau_1}{\beta}) &- \frac{\alpha \beta}{\sqrt{n}}\|\frac{\sqrt{\rho \varepsilon_0}}{\tau_2} \theta_0 - h \| + \\&+ \frac{\sqrt{\varepsilon_0}\beta}{2n\tau_2}\left(\|\theta_0\|^2 + d \alpha^2\right).
\end{split}
\end{align}
We denote by $\mathcal O_{d,1}(\alpha,\tau_1,\tau_2,\beta)$ the objective function in problem \eqref{eq:proof:thm:estimationerror:2:15}, parametrized by the dimension $d$, where the index $1$ recalls the fact that we are focusing on the first problem \eqref{eq:proof:thm:estimationerror:2:9:2}.

%------------------------------------------------------------------------------------
\paragraph{Step 3.2: We show that $\mathcal O_{d,1}$ is continuous on its domain, jointly convex in $(\alpha,\tau_1,\tau_2)$, and concave in $\beta$.} 

Let's first concentrate on the continuity property. Notice that we only need to show that the Moreau envelope $e_{F}(\alpha g + z,\tau_1/\beta)$ is continuous, since all the other terms in the objective function are trivially continuous. Since $F$ is lower semicontinuous and convex (recall that $F(u) = \sum_{i=1}^n f(u_i)$), the continuity of the Moreau envelope $e_{F}(\alpha g + z,\tau_1/\beta)$ follows from \cite[Theorem~2.26(b)]{rockafellar2009variational}.

Let's now focus on the joint convexity in $(\alpha,\tau_1,\tau_2)$. Since $\tau_1$ and $\tau_2$ are decoupled, it will be enough to prove the joint convexity in $(\alpha,\tau_1)$ and $(\alpha,\tau_2)$. We first concentrate on the pair $(\alpha,\tau_1)$, and notice that it suffices to prove that the objective function on the right-hand side of \eqref{eq:proof:thm:estimationerror:2:11} is jointly convex in $(\alpha,\tau_1,s)$. Then, after minimizing over $s \in \mathbb R^n$, the Moreau envelope $e_{F}(\alpha g + z,\tau_1/\beta)$ remains jointly convex in $(\alpha,\tau_1)$. But this is certainly the case since $\|\alpha g - s \|^2$ is jointly convex in $(\alpha,s)$, $1/\tau_1\|\alpha g - s \|^2$ is the perspective function of $\|\alpha g - s \|^2$, and therefore jointly convex in $(\alpha,\tau_1,s)$, and the shifted function $1/\tau_1\|\alpha g + z - s \|^2$ remains jointly convex.

We now concentrate on the pair $(\alpha, \tau_2)$. Proving the joint convexity will not be as straightforward, due primarily to the minus sign in front of the fifth term of $\mathcal O_{d,1}$. Moreover, due to the randomness of the objective function (coming from the random vector $h$), the joint convexity will be shown to hold true w.p.a.\ $1$ as $d \to \infty$. However, this is not restrictive for the asymptotic study of interest here. We start by considering the last two terms in the objective function $\mathcal O_{d,1}$ (notice that the linear term $\sqrt{\varepsilon_0}\beta \tau_2/2$ does not affect the convexity). Since $\beta,\varepsilon_0,\rho,\|\theta_0\|$ have no influence on the convexity, and $\beta > 0$, we can simplify the part of the objective function containing the last two terms, which after some manipulations becomes
\begin{align*}
    -\alpha \left\|\frac{1}{\tau_2} \frac{\theta_0}{\|\theta_0\|}-\frac{1}{\sqrt{\rho \varepsilon_0}} \frac{h}{\|\theta_0\|} \right\| + \frac{1}{\tau_2}\left( \frac{\|\theta_0\|^2/d + \alpha^2}{2\|\theta_0\|/\sqrt{d}}\right)
\end{align*}
or, equivalently,
\begin{align}
\label{eq:proof:thm:estimationerror:2:17}
    -\alpha \sqrt{\frac{1}{\tau_2^2} +\frac{\|h\|^2}{\rho \varepsilon_0\|\theta_0\|^2} - \frac{1}{\tau_2}\frac{2 \theta_0^\top h}{\sqrt{\rho \varepsilon_0}\|\theta_0\|^2}} + \frac{1}{\tau_2}\left( \frac{\|\theta_0\|^2/d + \alpha^2}{2\|\theta_0\|/\sqrt{d}}\right).
\end{align}

Due to Assumption~\ref{assump:cgmt}\ref{assump:theta0} we know that $\|\theta_0\|/\sqrt{d} \overset{P}{\to} \sigma_{\theta_0}$. In addition, $\|h\|/\sqrt{d} \overset{P}{\to} 1$ follows from the weak law of large numbers. As a consequence, the term $\|h\|^2/(\rho \varepsilon_0 \|\theta_0\|^2)$ converges in probability to $1/(\rho \varepsilon_0 \sigma_{\theta_0}^2)$. Moreover, from Lemma~\ref{lemma:triangular:array} we know that
\begin{align}
\label{eq:proof:thm:estimationerror:2:18}
    \frac{\theta_0^\top h}{d} \overset{P}{\to} 0,
\end{align}
and therefore the term $2\theta_0^\top h/(\sqrt{\rho \varepsilon_0}\|\theta_0\|^2)$ which multiplies $1/\tau_2$ converges in probability to $0$. 

These facts are employed to prove that the Hessian (with respect to $(\alpha,\tau_2)$) of \eqref{eq:proof:thm:estimationerror:2:17} is positive semidefinite w.p.a.\ $1$ ad $d \to \infty$. After some algebraic manipulations, it can be checked that the diagonal terms of the Hessian are positive for all $\alpha \geq 0,\tau_2>0$ w.p.a.\ $1$ ad $d \to \infty$. %In fact, the diagonal terms will become zero only in the limit of $\tau_2 \to \infty$. However, this case can be excluded (without loss of generality) by noticing that since any $\beta \geq B > 0$,
%\begin{align}
%\label{eq:proof:thm:estimationerror:2:19}
%    \lim_{\tau_2 \to \infty} \frac{\sqrt{\varepsilon_0}\beta\tau_2}{2} - \frac{\alpha \beta}{\sqrt{n}}\| \frac{\sqrt{d \varepsilon_0}}{\sqrt{n}\tau_2} \theta_0 - h \| + \frac{\sqrt{\varepsilon_0}\beta}{2n\tau_2}\left(\|\theta_0\|^2 + d \alpha^2\right) = \lim_{\tau_2 \to \infty} \frac{\sqrt{\varepsilon_0}\beta\tau_2}{2} - \frac{\alpha \beta\|h\|}{\sqrt{n}} = +\infty
%\end{align}
%since $\|h\|/\sqrt{n}$ is bounded with high-probability, for large enough $d$. This follows from $\|h\|/\sqrt{n} \overset{P}{\to} \sqrt{\rho}$, which holds due to the weak law of large numbers.
Moreover, it can be checked that the determinant of the Hessian is nonnegative (and therefore the Hessian is positive semidefinite) w.p.a.\ $1$ as $d \to \infty$ if and only if the following condition holds
\begin{align}
\label{eq:proof:thm:estimationerror:2:20}
    \alpha \leq \frac{\|\theta_0\|}{\sqrt{d}}\sqrt{1+\tau_2^2 \frac{\|h\|^2}{\rho \varepsilon_0\|\theta_0\|^2} - \tau_2 \frac{2 \theta_0^\top h}{\sqrt{\rho \varepsilon_0}\|\theta_0\|^2}}.
\end{align}
From \eqref{eq:proof:thm:estimationerror:2:20}, we recover the following sufficient condition,
\begin{align}
\label{eq:proof:thm:estimationerror:2:21}
    \alpha \leq \sigma_{\theta_0} - \delta,
\end{align}
for some arbitrarily small $\delta>0$, which guarantees that $\mathcal O_{d,1}$ is jointly convex in $(\alpha,\tau_2)$ w.p.a.\ $1$ as $d \to \infty$.  As a consequence, from now on, we will work with $K_\alpha := \sigma_{\theta_0} - \delta$.

Recalling the definition of $\alpha$, this condition is equivalent to imposing the following relative error bound 
\begin{align*}
    \frac{\|\hat{\theta}_{\mathrm{DRE}} -\theta_0\|}{\|\theta_0\|} < 1,
\end{align*}
which is practically desirable in any estimation problem.

We will now prove the concavity in $\beta$ of $\mathcal O_{d,1}$. This follows easily from the concavity of the Moreau envelope $e_F(\alpha g + z,\tau_1/\beta)$ in $\beta$ (as a minimization of affine functions), the concavity of the term $-\beta^2/(2M)$, and the linearity in $\beta$ of all the other terms.

As a consequence of the continuity and convexity-concavity, we can apply Sion's minimax principle to exchange the minimization and the maximization in \eqref{eq:proof:thm:estimationerror:2:15} and obtain the following (equivalent w.p.a.\ $1$) problem
\begin{align}
\label{eq:proof:thm:estimationerror:2:22}
    \inf_{\substack{0 \leq \alpha \leq \sigma_{\theta_0} - \delta \\ \tau_1,\tau_2>0}}\; \max_{B < \beta \leq K_\beta}\; \frac{\beta \tau_1}{2} + \frac{\sqrt{\varepsilon_0}\beta\tau_2}{2}  -\frac{\beta^2}{2M} + \frac{1}{n} e_{F}(\alpha g + z,\frac{\tau_1}{\beta}) - \frac{\alpha \beta}{\sqrt{n}}\|\frac{\sqrt{\rho \varepsilon_0}}{\tau_2} \theta_0 - h \| + \frac{\sqrt{\varepsilon_0}\beta}{2n\tau_2}\left(\|\theta_0\|^2 + d \alpha^2\right).
\end{align}

%------------------------------------------------------------------------------------
\paragraph{Step 3.3: We study the convergence in probability of $\mathcal O_{d,1}$.}

We have now arrived at the scalar formulation \eqref{eq:proof:thm:estimationerror:2:22}, which has the same optimal value as the modified (AO) problem \eqref{eq:proof:thm:estimationerror:2:9:2} (under the additional constraint \eqref{eq:proof:thm:estimationerror:2:21}), w.p.a.\ $1$ as $d \to \infty$. Therefore, following the reasoning presented in Fact~\ref{prop:cgmt}, the next step is to study the convergence in probability of its optimal value. To do so, we start by studying the convergence in probability of its objective function $\mathcal O_{d,1}$. 

Notice first that the last two terms of $\mathcal O_{d,1}$ converge in probability as follows.
\begin{align}
    \label{eq:proof:thm:estimationerror:2:23}    
    \frac{\alpha \beta}{\sqrt{n}}\| \frac{\sqrt{\rho \varepsilon_0}}{\tau_2} \theta_0 - h \| &\overset{P}{\to} \alpha \beta \sqrt{\rho} \sqrt{\frac{\rho \varepsilon_0}{\tau_2^2}\sigma_{\theta_0}^2 + 1}, \\
    \label{eq:proof:thm:estimationerror:2:24} 
    \frac{\sqrt{\varepsilon_0}\beta}{2n\tau_2}\left(\|\theta_0\|^2 + d \alpha^2\right) &\overset{P}{\to} \frac{ \sqrt{\varepsilon_0}\beta \rho}{2\tau_2} \left( \sigma_{\theta_0}^2 + \alpha^2 \right),
\end{align}
where \eqref{eq:proof:thm:estimationerror:2:23} can be recovered using \eqref{eq:proof:thm:estimationerror:2:18} and the continuous mapping theorem (which states that continuous functions preserve limits in probability, when their arguments are sequences of random variables). 

We will now study the convergence in probability of the normalized Moreau envelope $e_{F}(\alpha g + z,{\tau_1}/{\beta})/n$. Since $F(u):=\sum_{i=1}^{n}f(u_i)$, \eqref{eq:proof:thm:estimationerror:2:11} reduces to
\begin{align*}
    \frac{1}{n} e_{F}(\alpha g + z,\frac{\tau_1}{\beta}) = \frac{1}{n} \sum_{i=1}^{n} \;\min_{s_i \in \mathbb R}\; \frac{\beta}{2 \tau_1} (\alpha G + Z - s_i)^2 + f(s_i) = \frac{1}{n} \sum_{i=1}^{n} \; e_{f}(\alpha G + Z;\frac{\tau_1}{\beta}),
\end{align*}
and its limit (in probability) can be immediately recovered from the weak law of large numbers as
\begin{align}
    \label{eq:proof:thm:estimationerror:2:25}
    \frac{1}{n}  e_{F}(\alpha g + z,\frac{\tau_1}{\beta}) \overset{P}{\to} \mathbb E_{\mathcal N(0,1) \otimes \mathbb P_Z}\left[ e_{f}(\alpha G + Z;\frac{\tau_1}{\beta})\right],
\end{align}
which is precisely the expected Moreau envelope $\mathcal F(\alpha,\tau_1/\beta)$ defined in \eqref{eq:L_*}. The quantity under the expectation in \eqref{eq:proof:thm:estimationerror:2:25} is absolutely integrable, and therefore the expected Moreau envelope $\mathcal F$ is well defined, as explained in what follows. For every $\beta > B$ and $\tau_1>0$ we have
\begin{align*}
    \left|e_{f}(\alpha G + Z;\frac{\tau_1}{\beta})\right| = \min_{v \in \mathbb R}\; \frac{\beta}{2 \tau_1} \left( \alpha G + Z - v \right)^2 + f(v) &\leq \frac{\beta}{2 \tau_1} \left( \alpha G + Z \right)^2 + f(0) \\ &= \frac{\beta}{2 \tau_1} (\alpha G + Z)^2
\end{align*}
which is integrable due to the fact that both $G$ and $Z$ have finite second moment. In particular, the fact that $Z$ has finite second moment follows directly from Assumption~\ref{assump:dro}\ref{assump:pstar}. In the first and last equality we have used the fact that $f(0) = \min_{v \in \mathbb R} f(v) = 0$, which can be shown to hold true as follows. Since $L(0) = \min_{v \in \mathbb R} L(v) = 0$ (from Assumption~\ref{assump:cgmt}\ref{assump:ell:0}), we have that $L^*(0) = \min_{v \in \mathbb R} L^*(v) = 0$, and from $f^*(\cdot) = L^*(\cdot) - 1/(2M)(\cdot)^2$, with $L^*$ that is $1/M$-strongly convex, we have that $f^*(0) = \min_{v \in \mathbb R} f^*(v) = 0$, which results in the desired $f(0) = \min_{v \in \mathbb R} f(v) = 0$. Moreover, for $\tau_1 \to 0^+$,  
\begin{align*}
    \lim_{\tau_1 \to 0^+}\; \left|e_{f}(\alpha G + Z;\frac{\tau_1}{\beta})\right| = \lim_{\tau_1 \to 0^+}\; \min_{v \in \mathbb R}\; \frac{\beta}{2 \tau_1} \left( \alpha G + Z - v \right)^2 + f(v) =  f(\alpha G + Z),
\end{align*}
whose expectation is finite by assumption. In particular, the second equality follows from \cite[Theorem~1.25]{rockafellar2009variational}.

Therefore, from \eqref{eq:proof:thm:estimationerror:2:23}-\eqref{eq:proof:thm:estimationerror:2:25}, we have that $\mathcal O_{d,1}(\alpha,\tau_1,\tau_2,\beta)$ converges in probability to the function
\begin{align}
\label{eq:proof:thm:estimationerror:2:26}
    \mathcal O_1(\alpha,\tau_1,\tau_2,\beta) := \frac{\beta \tau_1}{2} + \frac{\varepsilon_0 \beta \tau_2}{2} - \frac{\beta^2}{2M} + \mathcal F(\alpha,\tau_1/\beta) - \alpha \beta \sqrt{\rho} \sqrt{\frac{\rho \varepsilon_0}{\tau_2^2}\sigma_{\theta_0}^2 + 1} + \frac{ \sqrt{\varepsilon_0}\beta \rho}{2\tau_2} \left( \sigma_{\theta_0}^2 + \alpha^2 \right).
\end{align}
Notice that $\mathcal O_1$ is precisely the first objective function of the minimax problem \eqref{eq:est:2:minimax} from the statement of Theorem~\ref{thm:estimationerror:2}.

%------------------------------------------------------------------------------------
\paragraph{Step 3.4: We show that the asymptotic objective function $\mathcal O_{1}$ is jointly convex in $(\alpha,\tau_1,\tau_2)$, jointly strictly convex in $(\alpha,\tau_2)$, and concave in $\beta$.}

Since $\mathcal O_1$ is the pointwise limit (in probability, for each ($\alpha, \tau_1, \tau_2, \beta$)) of the sequence of objective functions $\mathcal O_{d,1}$ which are convex-concave (w.p.a.\ $1$ as $d \to \infty$), and convexity is preserved by pointwise limits (see Lemma~\ref{lemma:convexity:lemma}), we have that $\mathcal O_1$ is jointly convex in $(\alpha,\tau_1,\tau_2)$ and concave in $\beta$. The convexity-concavity of $\mathcal O_1$ can also be checked directly from \eqref{eq:proof:thm:estimationerror:2:26}, following similar lines as the proof of convexity-concavity of $\mathcal O_1$. In fact, in what follows we will use such arguments to prove that $\mathcal O_1$ is jointly strictly convex in $(\alpha,\tau_2)$. This, in turn, will help us later to prove the uniqueness of the optimal solution $\alpha_{\star,1}$, which will be fundamental in the convergence analysis required by Fact~\ref{prop:cgmt}.

For the strict convexity of $\mathcal O_1$ in $(\alpha, \tau_2)$ we can restrict attention only to the last two terms of $\mathcal O_1$. Once again, since $\beta, \rho, \varepsilon_0, \sigma_{\theta_0}$ have no influence on the convexity, we can simplify these terms, which after some manipulations become
\begin{align}
    \label{eq:proof:thm:estimationerror:2:27}
    -\alpha \sqrt{\frac{1}{\tau_2^2}+\frac{1}{\rho \varepsilon_0 \sigma_{\theta_0}^2}} + \frac{1}{\tau_2}\left( \frac{\sigma_{\theta_0}^2 + \alpha^2}{2 \sigma_{\theta_0}} \right)
\end{align}

Notice that \eqref{eq:proof:thm:estimationerror:2:27} is nothing but the limit in probability of \eqref{eq:proof:thm:estimationerror:2:17}, for every $\alpha \geq 0, \tau_2>0$. We will prove the strict convexity in $(\alpha,\tau_2)$ of $\mathcal O_1$ by showing that the Hessian of $\mathcal O_1$ with respect to $(\alpha,\tau_2)$ of \eqref{eq:proof:thm:estimationerror:2:27} is positive definite. After some algebraic manipulations, it can be checked that the diagonal terms of the Hessian are positive for all $\alpha \geq 0, \tau_2>0$. Moreover, it can be checked that the determinant of the Hessian is positive if and only if the following condition holds
\begin{align}
\label{eq:proof:thm:estimationerror:2:28}
    \alpha < \sigma_{\theta_0} \sqrt{1+\tau_2^2 \frac{1}{\rho \varepsilon_0 \sigma_{\theta_0}^2}}.
\end{align}
As a consequence of this inequality and \eqref{eq:proof:thm:estimationerror:2:21}, we have that the constraint $\alpha \leq \sigma_{\theta_0} - \delta$ also ensures the joint strict convexity in $(\alpha,\tau_2)$ of $\mathcal O_1$. We will now show that the strict convexity in $(\alpha,\tau_2)$ of $\mathcal O_1$ guarantees the uniqueness of the optimizer $\alpha_{\star,1}$. With this aim in mind, consider the asymptotic minimax optimization problem
\begin{align}
\label{eq:proof:thm:estimationerror:2:30}
    \inf_{\substack{0 \leq \alpha \leq \sigma_{\theta_0}-\delta \\ \tau_1,\tau_2>0}}\; \max_{B < \beta \leq K_\beta}\; \mathcal O_1(\alpha, \tau_1, \tau_2, \beta).
\end{align}

%------------------------------------------------------------------------------------
\paragraph{Step 3.5: We show that the asymptotic problem \eqref{eq:proof:thm:estimationerror:2:30} has an unique minimizer $\alpha_{\star,1}$.}

In order to prove uniqueness of $\alpha_{\star,1}$, we will show that $\inf_{\substack{\tau_1,\tau_2>0}}\; \max_{B < \beta \leq K_\beta}\; \mathcal O_1(\alpha, \tau_1, \tau_2, \beta)$ is strictly convex in $\alpha$, for $\alpha > 0$. First, we consider the inner maximization in \eqref{eq:proof:thm:estimationerror:2:30}, i.e., $\mathcal O_1^{\beta}(\alpha, \tau_1, \tau_2):= \max_{B < \beta \leq K_\beta}\; \mathcal O_1(\alpha, \tau_1, \tau_2, \beta)$. Since the function $\mathcal O_1(\alpha, \tau_1, \tau_2, \beta)$ is continuous in $\beta$, we can extend it $\beta = B$, by setting $\lim_{\beta \to B} \mathcal O_1(\alpha, \tau_1, \tau_2, \beta) = \mathcal O_1(\alpha, \tau_1, \tau_2, B)$. Therefore, we have $\mathcal O_1^{\beta}(\alpha, \tau_1, \tau_2)= \max_{B \leq \beta \leq K_\beta}\; \mathcal O_1(\alpha, \tau_1, \tau_2, \beta)$. Notice that $\mathcal O_1^{\beta}(\alpha, \tau_1, \tau_2)$ is jointly convex in $(\alpha,\tau_1,\tau_2)$ since we have seen that $\mathcal O_1$ is jointly convex in $(\alpha,\tau_1,\tau_2)$, and pointwise supremum of convex functions is convex. We will now show that $\mathcal O_1^{\beta}(\alpha, \tau_1, \tau_2)$ is jointly strictly convex in $(\alpha,\tau_2)$. From the strict convexity of $\mathcal O_1$ in $(\alpha,\tau_2)$ it follows that for any $\lambda \in (0,1)$, $(\alpha^{(1)},\tau_{2}^{(1)}) \neq (\alpha^{(2)},\tau_{2}^{(2)})$, and $\beta$ we have that
\begin{align*}
    \mathcal O_1(\lambda\alpha^{(1)} + (1-\lambda)\alpha^{(2)},\tau_1,\lambda\tau_{2}^{(1)} + (1-\lambda)\tau_{2}^{(2)},\beta) &< \lambda\mathcal O_1(\alpha^{(1)},\tau_1,\tau_{2}^{(1)} ,\beta) + (1-\lambda) \mathcal O_1(\alpha^{(2)},\tau_1,\tau_{2}^{(2)},\beta)\\ &\leq \lambda \mathcal O_1^{\beta}(\alpha^{(1)},\tau_1,\tau_{2}^{(1)}) + (1-\lambda) \mathcal O_1^{\beta} (\alpha^{(2)},\tau_1,\tau_{2}^{(2)}).
\end{align*}

We can now take the supremum over $\beta \in [B,K_\beta]$ on the left-hand side, and since this will be attained (due to the continuity in $\beta$ of $\mathcal O_1$, and the compactness of $[B,K_\beta]$), we can conclude that
\begin{align*}
    \mathcal O_1^\beta(\lambda\alpha^{(1)} + (1-\lambda)\alpha^{(2)},\tau_1,\lambda\tau_{2}^{(1)} + (1-\lambda)\tau_{2}^{(2)}) < \lambda \mathcal O_1^{\beta}(\alpha^{(1)},\tau_1,\tau_{2}^{(1)}) + (1-\lambda) \mathcal O_1^{\beta} (\alpha^{(2)},\tau_1,\tau_{2}^{(2)}).
\end{align*}
This shows that $\mathcal O_1^{\beta}(\alpha, \tau_1, \tau_2)$ is strictly convex in $(\alpha,\tau_2)$.

Secondly, we consider $\mathcal O_1^{\tau_1,\beta}(\alpha,\tau_2) :=\inf_{\substack{\tau_1>0}}\; \mathcal O_1^{\beta}(\alpha, \tau_1, \tau_2)$. Since $\mathcal O_1^{\beta}(\alpha, \tau_1, \tau_2)$ is jointly convex in $(\alpha,\tau_1,\tau_2)$, and partial minimization of convex functions is convex, we have that $\mathcal O_1^{\tau_1,\beta}(\alpha,\tau_2)$ is convex in $(\alpha,\tau_2)$. Moreover, recall that the function $\mathcal O_1$ is decoupled in $\tau_1$ and $\tau_2$. As a consequence, taking the infimum over $\tau_1$ will not alter the joint strict convexity in $(\alpha,\tau_2)$.

Finally, we consider the function $\mathcal O_1^{\tau_2,\tau_1,\beta}(\alpha) :=\inf_{\substack{\tau_2>0}}\; \mathcal O_1^{\tau_1,\beta}(\alpha, \tau_2)$. Before proving that $\mathcal O_1^{\tau_2,\tau_1,\beta}$ is strictly convex, we first need to show that for any $\alpha\leq \sigma_{\theta_0} - \delta$, $\tau_1 > 0$, and $\beta > B$, the infimum is attained by some $\tau_{2\star}>0$. Indeed, this follows by noticing that
\begin{align*}
    \lim_{\tau_2 \to 0^+}\; \mathcal O_1(\alpha,\tau_1,\tau_2,\beta) &= \lim_{\tau_2 \to 0^+}\; \frac{\varepsilon_0 \beta \tau_2}{2} - \alpha \beta \sqrt{\rho} \sqrt{\frac{\rho \varepsilon_0}{\tau_2^2}\sigma_{\theta_0}^2 + 1} + \frac{ \sqrt{\varepsilon_0}\beta \rho}{2\tau_2} \left( \sigma_{\theta_0}^2 + \alpha^2 \right)\\ &= \lim_{\tau_2 \to 0^+}\; \frac{\sqrt{\varepsilon_0} \beta \rho }{2} \frac{(\sigma_{\theta_0}-\alpha)^2}{\tau_2} = \infty,
\end{align*}
and that
\begin{align}
\label{eq:proof:thm:estimationerror:2:tau2infty}
\begin{split}
    \lim_{\tau_2 \to \infty}\; \mathcal O_1(\alpha,\tau_1,\tau_2,\beta) &= \lim_{\tau_2 \to \infty}\; \frac{\varepsilon_0 \beta \tau_2}{2} - \alpha \beta \sqrt{\rho} \sqrt{\frac{\rho \varepsilon_0}{\tau_2^2}\sigma_{\theta_0}^2 + 1} + \frac{ \sqrt{\varepsilon_0}\beta \rho}{2\tau_2} \left( \sigma_{\theta_0}^2 + \alpha^2 \right)\\ &= \lim_{\tau_2 \to \infty}\; \frac{\varepsilon_0 \beta \tau_2}{2} - \alpha \beta \sqrt{\rho} = \infty.
\end{split}
\end{align}

Consider now $\alpha^{(1)}, \alpha^{(2)}>0$, and choose $\tau_2^{(1)},\tau_2^{(2)}$ satisfying $\mathcal O_1^{\tau_1,\beta}(\alpha^{(1)}, \tau_2^{(1)}) =\mathcal O_1^{\tau_2,\tau_1,\beta}(\alpha)$ and $\mathcal O_1^{\tau_1,\beta}(\alpha^{(2)}, \tau_2^{(2)}) =\mathcal O_1^{\tau_2,\tau_1,\beta}(\alpha)$, respectively. Then, for any $\lambda \in (0,1)$ we have
\begin{align*}
    \mathcal O_1^{\tau_2,\tau_1,\beta}(\lambda\alpha^{(1)}+(1-\lambda)\alpha^{(2)}) &\leq \mathcal O_1^{\tau_1,\beta}(\lambda\alpha^{(1)}+(1-\lambda)\alpha^{(2)}, \lambda\tau_2^{(1)}+(1-\lambda)\tau_2^{(2)})\\ &< \lambda \mathcal O_1^{\tau_1,\beta}(\alpha^{(1)}, \tau_2^{(1)}) + (1-\lambda) \mathcal O_1^{\tau_1,\beta}(\alpha^{(2)}, \tau_2^{(2)}) \\ &= \lambda \mathcal O_1^{\tau_2,\tau_1,\beta}(\alpha^{(1)}) + (1-\lambda) \mathcal O_1^{\tau_2,\tau_1,\beta}(\alpha^{(2)}),
\end{align*}
from which we conclude that $\mathcal O_1^{\tau_2,\tau_1,\beta}(\alpha) = \inf_{\substack{\tau_1,\tau_2>0}}\; \max_{B < \beta \leq K_\beta}\; \mathcal O_1(\alpha, \tau_1, \tau_2, \beta)$ is strictly convex in $\alpha$, for $\alpha>0$. As a consequence, the optimizer $\alpha_{\star,1}$ of \eqref{eq:proof:thm:estimationerror:2:30} is unique.

%------------------------------------------------------------------------------------
\paragraph{Step 3.6: We anticipate how the uniqueness of $\alpha_{\star,1}$, together with Fact~\ref{prop:cgmt}, can be used to conclude the proof.}

The uniqueness of $\alpha_{\star,1}$ is a key element in the convergence analysis required by Fact~\ref{prop:cgmt}, as explained in what follows. We first define, for arbitrary $\eta>0$, the sets $\mathcal S_\eta := \{\alpha \in [0,\sigma_{\theta_0}-\delta]:\, |\alpha-\alpha_{\star,1}|<\eta\}$, with $\alpha_{\star,1}$ the unique solution of \eqref{eq:proof:thm:estimationerror:2:30}, and $\mathcal S_\eta^c := [0,\sigma_{\theta_0}-\delta]\setminus \mathcal S_\eta$. Since $\alpha_{\star,1}$ is the unique solution of \eqref{eq:proof:thm:estimationerror:2:30}, we have that
\begin{align}
\label{eq:proof:thm:estimationerror:2:31}
    \inf_{\substack{0 \leq \alpha \leq \sigma_{\theta_0} - \delta \\ \tau_1,\tau_2>0}}\; \max_{B < \beta \leq K_\beta}\; \mathcal O_1(\alpha, \tau_1, \tau_2, \beta) < \inf_{\substack{\alpha \in \mathcal S_\eta^c \\ \tau_1,\tau_2>0}}\; \max_{B < \beta \leq K_\beta}\; \mathcal O_1(\alpha, \tau_1, \tau_2, \beta)
\end{align}
Now, due to \eqref{eq:proof:thm:estimationerror:2:31}, if we prove that the optimal value of the (scalarized version of the) modified (AO) problem \eqref{eq:proof:thm:estimationerror:2:22} satisfies
\begin{align}
\label{eq:proof:thm:estimationerror:2:32}
    \inf_{\substack{0 \leq \alpha \leq \sigma_{\theta_0} - \delta \\ \tau_1,\tau_2>0}}\; \max_{B < \beta \leq K_\beta}\; \mathcal O_{d,1}(\alpha, \tau_1, \tau_2, \beta) \overset{P}{\to} \inf_{\substack{0 \leq \alpha \leq \sigma_{\theta_0} - \delta \\ \tau_1,\tau_2>0}}\; \max_{B < \beta \leq K_\beta}\; \mathcal O_1(\alpha, \tau_1, \tau_2, \beta),
\end{align}
and that, when additionally restricted to $\alpha \in \mathcal S_\eta^c$, for arbitrary $\eta>0$, it satisfies
\begin{align}
\label{eq:proof:thm:estimationerror:2:33}
    \inf_{\substack{\alpha \in \mathcal S_\eta^c \\ \tau_1,\tau_2>0}}\; \max_{B < \beta \leq K_\beta}\; \mathcal O_{d,1}(\alpha, \tau_1, \tau_2, \beta) \overset{P}{\to} \inf_{\substack{\alpha \in \mathcal S_\eta^c \\ \tau_1,\tau_2>0}}\; \max_{B < \beta \leq K_\beta}\; \mathcal O_1(\alpha, \tau_1, \tau_2, \beta),
\end{align}
we can directly conclude the desired result \eqref{eq:est:2:error} for the first case \eqref{eq:proof:thm:estimationerror:2:9:2} from Fact~\ref{prop:cgmt}. Therefore, in order to finish the first case \eqref{eq:proof:thm:estimationerror:2:9:2}, we only need to prove the two convergences in probability \eqref{eq:proof:thm:estimationerror:2:32} and \eqref{eq:proof:thm:estimationerror:2:33}. It will be easier to prove the two convergences in probability in the following equivalent form (by Sion's minimax theorem)
\begin{align*}
    \min_{\alpha}\; \max_{B < \beta \leq K_\beta}\;\inf_{\tau_1,\tau_2>0}\; \mathcal O_{d,1}(\alpha, \tau_1, \tau_2, \beta) \overset{P}{\to} \min_{\alpha}\; \max_{B < \beta \leq K_\beta}\;\inf_{\tau_1,\tau_2>0}\; \mathcal O_1(\alpha, \tau_1, \tau_2, \beta).
\end{align*}

The two convergences in probability \eqref{eq:proof:thm:estimationerror:2:32} and \eqref{eq:proof:thm:estimationerror:2:33} are a consequence of Lemma~\ref{lemma:conv:opt}, as explained in what follows. %Before proceeding, we would like to highlight the (slight) difference between the assumptions of Lemmas~\ref{lemma:convexity:lemma}-\ref{lemma:conv:opt}, which require the random objective functions $\mathcal O_{d,1}$ to be convex, and the objective functions $\mathcal O_{d,1}$ under consideration here, whose convexity holds w.p.a.\ $1$ as $d \to \infty$. For the asymptotic (in probability) analysis conducted here, this difference is negligible, and Lemmas~\ref{lemma:convexity:lemma}-\ref{lemma:conv:opt} can be applied without loss of generality. 
For convenience, we will drop the \say{w.p.a.\ $1$ as $d \to \infty$} whenever we refer to the convexity of the functions $\mathcal O_{d,1}$ in the convergence analysis that follows, but the reader should remember that this is implicit. Moreover, we will drop the \say{under the condition $\alpha \leq \sigma_{\theta_0} - \delta$} whenever we will say that $\mathcal O_{d,1}$, $\mathcal O_{1}$, and their partial minimizations are convex.

%------------------------------------------------------------------------------------
\paragraph{Step 3.7: We prove the two convergences in probability \eqref{eq:proof:thm:estimationerror:2:32} and \eqref{eq:proof:thm:estimationerror:2:33}, and conclude the proof for the first modified (AO) problem \eqref{eq:proof:thm:estimationerror:2:9:2}.}

First, for fixed $(\alpha,\beta,\tau_1)$, $\{\mathcal O_{d,1}(\alpha,\tau_1,\cdot,\beta)\}_{d \in \mathbb N}$ is a sequence of random real-valued convex functions, converging in probability (pointwise, for every $\tau_2>0$) to the function $\mathcal O_1(\alpha,\tau_1,\cdot,\beta)$. Moreover, 
\begin{align*}
    \lim_{\tau_2 \to \infty} \mathcal O_1(\alpha,\tau_1,\tau_2,\beta) &=  \infty ,
\end{align*}
as shown in \eqref{eq:proof:thm:estimationerror:2:tau2infty}. As a consequence, we can apply statement (iii) of Lemma~\ref{lemma:conv:opt} to conclude that
\begin{align*}
    \inf_{\tau_2>0}\; \mathcal O_{d,1}(\alpha, \tau_1, \tau_2, \beta) \overset{P}{\to} \inf_{\tau_2>0}\; \mathcal O_1(\alpha, \tau_1, \tau_2, \beta).
\end{align*}

We now define $\mathcal O_1^{\tau_2}(\alpha, \tau_1, \beta):= \inf_{\tau_2>0}\; \mathcal O_1(\alpha, \tau_1, \tau_2, \beta)$ and $\mathcal O_{d,1}^{\tau_2}(\alpha, \tau_1, \beta):= \inf_{\tau_2>0}\; \mathcal O_{d,1}(\alpha, \tau_1, \tau_2, \beta)$. For fixed $(\alpha,\beta)$, $\{\mathcal O_{d,1}^{\tau_2}(\alpha, \cdot, \beta)\}_{d \in \mathbb N}$ is a sequence of random, real-valued convex functions (since they are defined as the partial minimization of jointly convex functions), converging in probability (pointwise, for every $\tau_1 > 0$) to the function $\mathcal O_1^{\tau_2}(\alpha, \cdot, \beta)$. Moreover, since $\tau_1$ and $\tau_2$ are decoupled, we have
\begin{align*}
    \lim_{\tau_1 \to \infty} \mathcal O_1^{\tau_2}(\alpha, \tau_1, \beta) = \infty \; \iff \; \lim_{\tau_1 \to \infty} \frac{\beta \tau_1}{2}  + \mathcal F(\alpha,\frac{\tau_1}{\beta})= \infty.
\end{align*}
Notice now that
\begin{align*}
    \lim_{\tau_1 \to \infty} \frac{\beta \tau_1}{2}  + \mathcal F(\alpha,\frac{\tau_1}{\beta}) \geq \lim_{\tau_1 \to \infty} \frac{\beta \tau_1}{2} = \infty,
\end{align*}
where the inequality follows from the fact that $\mathcal F(\alpha,\tau_1/\beta) \geq 0$ (recall that $f(v) \geq 0$ for all $v \in \mathbb R$). As a consequence, we can apply again statement (iii) of Lemma~\ref{lemma:conv:opt} to conclude that
\begin{align*}
    \inf_{\tau_1>0}\; \mathcal O_{d,1}^{\tau_2}(\alpha, \tau_1, \beta) \overset{P}{\to} \inf_{\tau_1>0}\; \mathcal O_1^{\tau_2}(\alpha, \tau_1, \beta).
\end{align*}

We now define $\mathcal O_1^{\tau_1,\tau_2}(\alpha, \beta):= \inf_{\tau_1>0}\; \mathcal O_1^{\tau_2}(\alpha, \tau_1, \beta)$ and $\mathcal O_{d,1}^{\tau_1,\tau_2}(\alpha, \beta):= \inf_{\tau_1>0}\; \mathcal O_{d,1}^{\tau_2}(\alpha, \tau_1, \beta)$. For fixed $\alpha$, $\{\mathcal O_{d,1}^{\tau_1,\tau_2}(\alpha, \cdot)\}_{d \in \mathbb N}$ is a sequence of random, real-valued concave functions (since they are defined as a minimization of concave functions), converging in probability (pointwise, for every $\beta > B$) to the function $\mathcal O_1^{\tau_1, \tau_2}(\alpha, \cdot)$. Therefore, we can apply statement (ii) of Lemma~\ref{lemma:conv:opt} (for $0 < \beta - B \leq K_\beta - B$) to conclude that
\begin{align}
\label{eq:proof:thm:estimationerror:2:alpha:0}
    \sup_{B < \beta \leq K_\beta}\; \mathcal O_{d,1}^{\tau_1,\tau_2}(\alpha, \beta) \overset{P}{\to} \sup_{B < \beta \leq K_\beta}\; \mathcal O_1^{\tau_1,\tau_2}(\alpha, \beta).
\end{align}

We will now show that the constraint $B < \beta \leq K_\beta$ can be relaxed to $\beta > B$ without loss of generality. Recall that $K_\beta$ was inserted for convenience, in equation \eqref{eq:proof:thm:estimationerror:2:1:1}, to be able to use the CGMT (which required $u$ to live in a compact set). Now, notice that
\begin{align*}
    \lim_{\beta \to \infty}\; \mathcal O_{1}^{\tau_1,\tau_2}(\alpha, \beta) &\leq \lim_{\beta \to \infty}\; \mathcal O_{1}^{\bar{\tau}_1,\bar{\tau}_2}(\alpha, \beta) 
    \\ &= \lim_{\beta \to \infty}\; \frac{\beta \bar\tau_1}{2} + \frac{\varepsilon_0 \beta \bar\tau_2}{2} - \frac{\beta^2}{2M} + \mathcal F(\alpha,\bar\tau_1/\beta) - \alpha \beta \sqrt{\rho} \sqrt{\frac{\rho \varepsilon_0}{\bar\tau_2^2}\sigma_{\theta_0}^2 + 1} + \frac{ \sqrt{\varepsilon_0}\beta \rho}{2\bar\tau_2} \left( \sigma_{\theta_0}^2 + \alpha^2 \right)
    \\ &= \lim_{\beta \to \infty}\; - \frac{\beta^2}{2M} + O(\beta) = -\infty,
\end{align*}
for some constants $\bar{\tau}_1, \bar{\tau}_1 > 0$ (let's say, for example, equal to $1$). Here, $O(\beta)$ encapsulates the remaining terms, which grow (at most) linearly in $\beta$. The inequality above follows from the fact that $\mathcal O_{1}^{\tau_1,\tau_2}$ is defined as a minimization over $\tau_1,\tau_2>0$. The assumption $\mathbb E_{\mathcal N(0,1) \otimes \mathbb P_Z}\left[ f(\alpha G + Z)\right] < \infty$ is very important here, since it guarantees that
\begin{align*}
    \lim_{\beta \to \infty}\; \mathcal F(\alpha,\frac{\bar\tau_1}{\beta}) = \lim_{\beta \to \infty}\; \mathbb E_{\mathcal N(0,1) \otimes \mathbb P_Z}\left[ e_{f}(\alpha G + Z;\frac{\bar\tau_1}{\beta})\right] &= \mathbb E_{\mathcal N(0,1) \otimes \mathbb P_Z}\left[\lim_{\beta \to \infty}\; e_{f}(\alpha G + Z;\frac{\bar\tau_1}{\beta})\right] \\ &= \mathbb E_{\mathcal N(0,1) \otimes \mathbb P_Z}\left[f(\alpha G + Z)\right] < \infty,
\end{align*}
where the second equality follows from the monotone convergence theorem, and the last equality follows from \cite[Theorem~1.25]{rockafellar2009variational}.

Therefore, there exists some sufficiently large $K_{\beta}$ such that
\begin{align*}
    \sup_{B < \beta \leq K_{\beta}}\; \mathcal O_{1}^{\tau_1,\tau_2}(\alpha, \beta) =  \sup_{\beta > B}\; \mathcal O_{1}^{\tau_1,\tau_2}(\alpha, \beta),
\end{align*}
%Consequently, since $-\beta^2/(2M)$ is independent of the dimension $d$ and the random terms in $\mathcal O_{d,1}^{\bar{\tau}_1,\bar{\tau}_2}$ are bounded w.p.a.\ $1$, there exists some sufficiently large $K_{\beta,1}$ such that
%\begin{align*}
%    \sup_{\beta > B}\; \mathcal O_{d,1}^{\tau_1,\tau_2}(\alpha, \beta) = \sup_{B < \beta \leq K_{\beta,1}}\; \mathcal O_{d,1}^{\tau_1,\tau_2}(\alpha, \beta),
%\end{align*}
%for all $d \in \mathbb N$. The same reasoning can be used for $\mathcal O_1^{\tau_1,\tau_2}(\alpha, \beta)$, leading to the existence of some sufficiently large $K_{\beta,2}$ for which
%\begin{align*}
%    \sup_{\beta > B}\; \mathcal O_{1}^{\tau_1,\tau_2}(\alpha, \beta) = \sup_{B < \beta \leq K_{\beta,2}}\; \mathcal O_{1}^{\tau_1,\tau_2}(\alpha, \beta).
%\end{align*}
%Therefore, by considering $K_\beta = \max\{K_{\beta,1},K_{\beta,2}\}$, we can replace (without loss of generality) the constraint $B < \beta \leq K_\beta$ on the left and right-hand side in \eqref{eq:proof:thm:estimationerror:2:alpha:0} with $\beta > B$, 
leading to
\begin{align}
\label{eq:proof:thm:estimationerror:2:alpha:0:0}
    \sup_{B < \beta \leq K_{\beta}}\; \mathcal O_{d,1}^{\tau_1,\tau_2}(\alpha, \beta) \overset{P}{\to} \sup_{\beta > B}\; \mathcal O_1^{\tau_1,\tau_2}(\alpha, \beta).
\end{align}
Notice that the convergence \eqref{eq:proof:thm:estimationerror:2:alpha:0:0} can also be proven using statement (iii) of Lemma~\ref{lemma:conv:opt}, using the fact that $\lim_{\beta \to \infty}\; \mathcal O_{1}^{\tau_1,\tau_2}(\alpha, \beta) = -\infty$.

We now define $\mathcal O_1^{\beta,\tau_1,\tau_2}(\alpha):= \sup_{\beta > B}\; \mathcal O_1^{\tau_1,\tau_2}(\alpha, \beta)$ and $\mathcal O_{d,1}^{\beta,\tau_1,\tau_2}(\alpha):= \sup_{B < \beta \leq K_{\beta}}\; \mathcal O_{d,1}^{\tau_1,\tau_2}(\alpha, \beta)$. Each of these functions is convex in $\alpha$, since they were obtained by first minimizing over $\tau_1,\tau_2$ a jointly convex function in $(\alpha,\tau_1,\tau_2)$, and then maximizing over $\beta$ a convex function in $\alpha$. 

For the final step of the proof, we will distinguish between the two cases $\alpha_{\star,1}=0$ and $\alpha_{\star,1}>0$. This is due to the Convexity Lemma~\ref{lemma:convexity:lemma}, which stands at the core of the convergence analysis, and which requires the domain of the functions of interest to be open. We first consider the case $\alpha_{\star,1}=0$. Since \eqref{eq:proof:thm:estimationerror:2:alpha:0} is valid for any $\alpha \geq 0$, it is in particular valid for $\alpha = 0$, and thus we have that
\begin{align}
\label{eq:proof:thm:estimationerror:2:alpha:1}
    \mathcal O_{d,1}^{\beta,\tau_1,\tau_2}(0) \overset{P}{\to} \mathcal O_1^{\beta,\tau_1,\tau_2}(0).
\end{align}
Moreover, using statement (i) of Lemma~\ref{lemma:conv:opt}, for any $K>0$, we have that 
\begin{align}
\label{eq:proof:thm:estimationerror:2:alpha:2}
    \min_{K \leq\alpha\leq \sigma_{\theta_0} - \delta}\; \mathcal O_{d,1}^{\beta,\tau_1,\tau_2}(\alpha) \overset{P}{\to} \min_{K \leq \alpha\leq \sigma_{\theta_0} - \delta}\; \mathcal O_1^{\beta,\tau_1,\tau_2}(\alpha).
\end{align}
The result for $\alpha_{\star,1}=0$ now follows from \eqref{eq:proof:thm:estimationerror:2:alpha:1} and \eqref{eq:proof:thm:estimationerror:2:alpha:2} and the fact that 
\begin{align*}
    \mathcal O_1^{\beta,\tau_1,\tau_2}(0) < \min_{K \leq \alpha\leq \sigma_{\theta_0} - \delta}\; \mathcal O_1^{\beta,\tau_1,\tau_2}(\alpha),
\end{align*}
for any $K > 0$, due to the uniqueness of $\alpha_{\star,1}$.

We now focus on the case $\alpha_{\star,1} > 0$. We need to prove \eqref{eq:proof:thm:estimationerror:2:32} and \eqref{eq:proof:thm:estimationerror:2:33}, which are equivalent to the following two convergences in probability
\begin{align}
\label{eq:proof:thm:estimationerror:2:34}
    \inf_{0<\alpha\leq \sigma_{\theta_0} - \delta}\; \mathcal O_{d,1}^{\beta,\tau_1,\tau_2}(\alpha) \overset{P}{\to} \inf_{0<\alpha\leq \sigma_{\theta_0} - \delta}\; \mathcal O_1^{\beta,\tau_1,\tau_2}(\alpha),
\end{align}
and
\begin{align}
\label{eq:proof:thm:estimationerror:2:35}
    \inf_{\alpha \in \mathcal S_\eta^c}\; \mathcal O_{d,1}^{\beta,\tau_1,\tau_2}(\alpha) \overset{P}{\to} \inf_{\alpha \in \mathcal S_\eta^c}\; \mathcal O_1^{\beta,\tau_1,\tau_2}(\alpha),
\end{align}
respectively. Since $\{\mathcal O_{d,1}^{\beta,\tau_1,\tau_2}(\cdot)\}_{d \in \mathbb N}$ is a sequence of random, real-valued convex functions, converging in probability (pointwise, for every $\alpha > 0$) to the function $\mathcal O_1^{\beta,\tau_1,\tau_2}(\cdot)$, \eqref{eq:proof:thm:estimationerror:2:34} follows immediately from  statement (ii) of Lemma~\ref{lemma:conv:opt}. Moreover, since $\mathcal S_\eta^c = (0,\alpha_{\star,1}-\eta]\cup[\alpha_{\star,1}+\eta,\sigma_{\theta_0} - \delta]$, \eqref{eq:proof:thm:estimationerror:2:35} follows from statement (i) of Lemma~\ref{lemma:conv:opt}, for the interval $[\alpha_{\star,1}+\eta,\sigma_{\theta_0} - \delta]$, and statement (ii) of the same lemma, for the interval $(0,\alpha_{\star,1}-\eta]$. 

This concludes the proof of the two convergences in probability  \eqref{eq:proof:thm:estimationerror:2:32} and \eqref{eq:proof:thm:estimationerror:2:33}, and with it, the analysis of the first case \eqref{eq:proof:thm:estimationerror:2:9:2}. We will now consider the second case \eqref{eq:proof:thm:estimationerror:2:9:3}, and proceed by first scalarizing the problem over the magnitudes of $u$ and $w$, and subsequently studying its asymptotic optimal value, similarly to what we have done for the first case. Before doing this, however, we would like to anticipate the results that will be obtained in the second case, and briefly explain how everything can be used to conclude the proof of this theorem.

%------------------------------------------------------------------------------------
\paragraph{Step 4: We anticipate the results that will be obtained for the second modified (AO) problem \eqref{eq:proof:thm:estimationerror:2:9:3}, and briefly explain how everything can be used to conclude the proof of this theorem.}

Recall from the sequence of arguments \eqref{eq:proof:thm:estimationerror:2:8}-\eqref{eq:proof:thm:estimationerror:2:9:3} that the optimal value of the modified (AO) problem \eqref{eq:proof:thm:estimationerror:2:8} is equal to $\max\{\mathcal V_{d,1}, \mathcal V_{d,2}\}$, with $\mathcal V_{d,1}$ the optimal value of problem \eqref{eq:proof:thm:estimationerror:2:9:2}, and $\mathcal V_{d,2}$ the optimal value of problem \eqref{eq:proof:thm:estimationerror:2:9:3}. So far, we have shown that $\mathcal V_{d,1} \overset{P}{\to} \mathcal V_1$, with
\begin{align*}
    \mathcal V_1 := \inf_{\substack{0 \leq \alpha \leq \sigma_{\theta_0} - \delta \\ \tau_1,\tau_2>0}}\; \max_{B < \beta \leq K_\beta}\; \mathcal O_1(\alpha, \tau_1, \tau_2, \beta),
\end{align*}
and that the optimal solution $\alpha_{\star,1}$, which attains $\mathcal V_1$, is unique. In what follows, we will show that similar things happen in the second case, i.e., $\mathcal V_{d,2} \overset{P}{\to} \mathcal V_2$, for some $\mathcal V_2$ that will be defined later, and that the solution $\alpha_{\star,2}$, which attains $\mathcal V_2$, is unique. The two converges $\mathcal V_{d,1} \overset{P}{\to} \mathcal V_1$ and $\mathcal V_{d,2} \overset{P}{\to} \mathcal V_2$ are very important because they allow us to conclude what is the asymptotic optimal value of the modified (AO) problem. Indeed, if $\mathcal V_1 > \mathcal V_2$, then we know that $\mathcal V_{d,1} > \mathcal V_{d,2}$ w.p.a.\ $1$ as $d \to \infty$. In that case, we have that $\max\{\mathcal V_{d,1}, \mathcal V_{d,2}\} = \mathcal V_{d,1}$ asymptotically, and from Fact~\ref{prop:cgmt} and the analysis performed in \eqref{eq:proof:thm:estimationerror:2:31}-\eqref{eq:proof:thm:estimationerror:2:35}, we can conclude the desired result, i.e., $\|\hat{\theta}_{\mathrm{DRE}}-\theta_0\|/\sqrt{d} \overset{P}{\to} \alpha_{\star,1}$. Similarly, if $\mathcal V_{2} > \mathcal V_{1}$, we can conclude that $\|\hat{\theta}_{\mathrm{DRE}}-\theta_0\|/\sqrt{d} \overset{P}{\to} \alpha_{\star,2}$. Finally, if $\mathcal V_{2} = \mathcal V_{1}$, we cannot assess which of the two values $\mathcal V_{d,1}$ or $\mathcal V_{d,2}$ is larger, and therefore we can only conclude that $\|\hat{\theta}_{\mathrm{DRE}}-\theta_0\|/\sqrt{d} \leq \max\{\alpha_{\star,1},\alpha_{\star,2}\}$ w.p.a.\ $1$ as $d \to \infty$. 

Therefore, in order to finish the proof, we still need to prove the convergence $\mathcal V_{d,2} \overset{P}{\to} \mathcal V_2$, together with the uniqueness of $\alpha_{\star,2}$. We will do this in what follows, at a slightly faster pace when the steps are similar to the ones done for $\mathcal V_{d,1}$.

%------------------------------------------------------------------------------------
\paragraph{Step 5: We re-do all the steps for the second modified (AO) problem \eqref{eq:proof:thm:estimationerror:2:9:3}, following similar reasoning as above.}

Recall that $\mathcal V_{d,2}$ is the optimal value of the following optimization problem
\begin{align*}
    \max_{0 \leq \beta \leq B}\; \min_{\substack{w \in \mathcal S_w\\ s \in \mathbb R^n}}\; \max_{\|u\|= \beta}\;  \frac{1}{\sqrt{n}} (\|w\|g + z - s)^\top u - \frac{1}{\sqrt{n}} \|u\|h^\top w -\frac{\|u\|^2}{2M} + \frac{1}{n} F(s) + \frac{\|\theta_0 + \sqrt{d}w\|}{\sqrt{n}} \left( p + \frac{\|u\|^2}{q} \right),
\end{align*}
with $p := (\varepsilon_0\sqrt{\rho}M R_\theta)/2$ and $q := 2 \sqrt{\rho} M R_\theta$.

We start by scalarizing over the magnitude of $u$, leading to
\begin{align*}
    \max_{0 \leq \beta \leq B}\; \min_{\substack{w \in \mathcal S_w\\ s \in \mathbb R^n}}\;  \frac{\beta}{\sqrt{n}} \|\|w\|g + z - s\| - \frac{1}{\sqrt{n}} \beta h^\top w -\frac{\beta^2}{2M} + \frac{1}{n} F(s) + \frac{\|\theta_0 + \sqrt{d}w\|}{\sqrt{n}} \left( p + \frac{\beta^2}{q} \right),
\end{align*}
and by rewriting the first term using the square-root trick, and introducing the Moreau envelope of $F$, we obtain
\begin{align}
\label{eq:proof:thm:estimationerror:2:36}
    \max_{0 \leq \beta \leq B}\; \inf_{\substack{w \in \mathcal S_w\\ \tau_1>0}}\;   \frac{\beta \tau_1}{2}  -\frac{\beta^2}{2M} + \frac{1}{n} e_{F}(\|w\| g + z,\frac{\tau_1}{\beta}) - \frac{\beta}{\sqrt{n}} h^\top w + \frac{\|\theta_0 + \sqrt{d}w\|}{\sqrt{n}} \left( p + \frac{\beta^2}{q} \right).
\end{align}
From
\begin{align*}
    -\frac{\beta^2}{2M} + \frac{\|\theta_0 + \sqrt{d}w\|}{\sqrt{n}} \frac{\beta^2}{2 \sqrt{\rho} M R_\theta} = -\frac{\beta^2}{2M} \left( 1 - \frac{\|\theta_0 + \sqrt{d}w\|}{R_\theta\sqrt{d}} \right),
\end{align*}
and $\|\theta_0 + \sqrt{d}w\| \leq R_\theta\sqrt{d}$, we have that the objective function in \eqref{eq:proof:thm:estimationerror:2:36} is concave in $\beta$ (recall that $e_{F}(\|w\| g + z,\tau_1/\beta)$ is concave in $\beta$).

We now employ again the square-root trick on the last term in \eqref{eq:proof:thm:estimationerror:2:36}, leading to
\begin{align}
\label{eq:proof:thm:estimationerror:2:37}
    \max_{0 \leq \beta \leq B}\; \inf_{\substack{w \in \mathcal S_w\\ \tau_1,\tau_2>0}}\;   \frac{\beta \tau_1}{2} -\frac{\beta^2}{2M} + \frac{1}{n} e_{F}(\|w\| g + z,\frac{\tau_1}{\beta}) - \frac{\beta}{\sqrt{n}} h^\top w + \left(\frac{\tau_2}{2} + \frac{\|\theta_0 + \sqrt{d}w\|^2}{2n\tau_2}\right) \left( p + \frac{\beta^2}{q} \right).
\end{align}

Notice that after introducing the square-root trick, the objective function in \eqref{eq:proof:thm:estimationerror:2:37} is not concave in $\beta$ anymore. However, since the objective function in \eqref{eq:proof:thm:estimationerror:2:36} was concave in $\beta$, we have that the function
\begin{align}
\label{eq:proof:thm:estimationerror:2:38}
    \inf_{\tau_2>0}\;   \frac{\beta \tau_1}{2} -\frac{\beta^2}{2M} + \frac{1}{n} e_{F}(\|w\| g + z,\frac{\tau_1}{\beta}) - \frac{\beta}{\sqrt{n}} h^\top w + \left(\frac{\tau_2}{2} + \frac{\|\theta_0 + \sqrt{d}w\|^2}{2n\tau_2}\right) \left( p + \frac{\beta^2}{q} \right)
\end{align}
is concave in $\beta$.

We now separate the constraint $w \in \mathcal S_w$ into the two constraints $0 \leq \alpha \leq K_\alpha$ and $\|w\| = \alpha$, and by expressing $\|\theta_0 + \sqrt{d}w\|^2$ as $\|\theta_0\|^2 + d\|w\|^2 + 2\sqrt{d}\theta_0^\top w$, we can now scalarize the problem over the magnitude of $w$ and obtain
\begin{align}
\label{eq:proof:thm:estimationerror:2:39}
\begin{split}
    \max_{0 \leq \beta \leq B}\; \inf_{\substack{0 \leq \alpha \leq K_\alpha\\ \tau_1,\tau_2>0}}\; \frac{\beta \tau_1}{2} + \frac{p\tau_2}{2}+\frac{\beta^2 \tau_2}{2q} -\frac{\beta^2}{2M} + \frac{1}{n} e_{F}(\alpha g + z,\frac{\tau_1}{\beta}) - \frac{\alpha}{\sqrt{n}}\left\| \sqrt{\rho}\left( p + \frac{\beta^2}{q} \right) \frac{\theta_0}{\tau_2} - \beta h \right\| + \\ + \frac{\|\theta_0\|^2+d \alpha^2}{2 n \tau_2}\left( p + \frac{\beta^2}{q} \right).
\end{split}
\end{align}

We denote by $\mathcal O_{d,2}(\alpha,\tau_1,\tau_2,\beta)$ the objective function in \eqref{eq:proof:thm:estimationerror:2:39}, where the index $2$ recalls the fact that we are in the second case. Before proceeding, we would like to highlight that the function 
\begin{align}
\label{eq:proof:thm:estimationerror:2:40}
\begin{split}
    \inf_{\tau_2>0}\; \frac{\beta \tau_1}{2} + \frac{p\tau_2}{2}+\frac{\beta^2 \tau_2}{2q} -\frac{\beta^2}{2M} + \frac{1}{n} e_{F}(\alpha g + z,\frac{\tau_1}{\beta}) - \frac{\alpha}{\sqrt{n}}\left\| \sqrt{\rho}\left( p + \frac{\beta^2}{q} \right) \frac{\theta_0}{\tau_2} - \beta h \right\| + \\ + \frac{\|\theta_0\|^2+d \alpha^2}{2 n \tau_2}\left( p + \frac{\beta^2}{q} \right)
\end{split}
\end{align}
is concave in $\beta$, since it is the result of a minimization (over $\|w\| = \alpha$) of concave functions (recall that \eqref{eq:proof:thm:estimationerror:2:38} is concave in $\beta$). Since we have concavity in $\beta$ only after the minimization over $\tau_2$, and ultimately we are interested in exchanging the minimization over $\alpha$ with the maximization over $\beta$ (in order to recover the optimal solution $\alpha_{\star,2}$), we define by $\mathcal O_{d,2}^{\tau_2}(\alpha,\tau_1,\beta)$ the function in \eqref{eq:proof:thm:estimationerror:2:40}, and therefore, in this second case, we will be working with the optimization problem 
\begin{align}
\label{eq:proof:thm:estimationerror:2:41}
\begin{split}
    \max_{0 \leq \beta \leq B}\; \inf_{\substack{0 \leq \alpha \leq K_\alpha\\ \tau_1>0}}\;  \mathcal O_{d,2}^{\tau_2}(\alpha,\tau_1,\beta).
\end{split}
\end{align}

Therefore, we have that $\mathcal O_{d,2}^{\tau_2}(\alpha,\tau_1,\beta)$ is concave in $\beta$. Furthermore, we will show that $\mathcal O_{d,2}^{\tau_2}(\alpha,\tau_1,\beta)$ is jointly convex in $(\alpha,\tau_1)$, and continuous on its domain. Let's first concentrate on the joint convexity in $(\alpha,\tau_1)$. For this, we will first show that $\mathcal O_{d,2}(\alpha,\tau_1,\tau_2,\beta)$ is jointly convex in $(\alpha,\tau_1,\tau_2)$, and therefore after the partial minimization over $\tau_2$, we will have that $\mathcal O_{d,2}^{\tau_2}(\alpha,\tau_1,\beta)$ is jointly convex in $(\alpha,\tau_1)$. Since $\tau_1$ and $\tau_2$ are decoupled, it will be enough to prove the joint convexity in $(\alpha,\tau_1)$ and $(\alpha,\tau_2)$. The joint convexity in $(\alpha,\tau_1)$ follows from the joint convexity in $(\alpha,\tau_1)$ of the Moreau envelope $e_{F}(\alpha g + z,\tau_1/\beta)$, which we have proven previously, in the first case. Moreover, the joint convexity in $(\alpha, \tau_2)$ will also be shown to follow from the analysis performed in the first case. In particular, we will now show that the last two terms in $\mathcal O_{d,2}(\alpha,\tau_1,\tau_2,\beta)$ can be rewritten in the form \eqref{eq:proof:thm:estimationerror:2:17}, and therefore the reasoning used in the first case can be re-used here. Since $\beta,p,q,\rho,\|\theta_0\|$ have no influence on the convexity, and $(p+\beta/q) > 0$, we can simplify the last two terms in $\mathcal O_{d,2}(\alpha,\tau_1,\tau_2,\beta)$, and obtain
\begin{align*}
    -\alpha \left\|\frac{1}{\tau_2} \frac{\theta_0}{\|\theta_0\|}-\frac{\beta}{\sqrt{\rho}(p+\beta^2/q)} \frac{h}{\|\theta_0\|} \right\| + \frac{1}{\tau_2}\left( \frac{\|\theta_0\|^2/d + \alpha^2}{2\|\theta_0\|/\sqrt{d}}\right)
\end{align*}
or, equivalently,
\begin{align}
\label{eq:proof:thm:estimationerror:2:42}
    -\alpha \sqrt{\frac{1}{\tau_2^2} +\frac{\beta^2}{\rho(p+\beta^2/q)^2} \frac{\|h\|^2}{\|\theta_0\|^2} - \frac{1}{\tau_2} \frac{2\beta}{\sqrt{\rho}(p+\beta^2/q)}\frac{\theta_0^\top h}{\|\theta_0\|^2}} + \frac{1}{\tau_2}\left( \frac{\|\theta_0\|^2/d + \alpha^2}{2\|\theta_0\|/\sqrt{d}}\right).
\end{align}

Notice now that the only differences between \eqref{eq:proof:thm:estimationerror:2:42} and \eqref{eq:proof:thm:estimationerror:2:17} are the two coefficients $\beta^2/(\rho(p+\beta^2/q)^2)$ instead of $1/(\rho\varepsilon_0)$, and $2\beta/(\sqrt{\rho}(p+\beta^2/q))$ instead of $1/\sqrt{\rho\varepsilon_0}$. However, since $0 \leq \beta \leq B$, we have that the new coefficients are bounded, and therefore they do not affect in any way the convexity analysis performed for \eqref{eq:proof:thm:estimationerror:2:17}. As a consequence, following the same exact steps as in \eqref{eq:proof:thm:estimationerror:2:17}-\eqref{eq:proof:thm:estimationerror:2:21}, we have that \eqref{eq:proof:thm:estimationerror:2:42} is jointly convex in $(\alpha,\tau_2)$ w.p.a.\ $1$ as $d \to \infty$ if and only if the following condition holds
\begin{align}
\label{eq:proof:thm:estimationerror:2:43}
    \alpha \leq \frac{\|\theta_0\|}{\sqrt{d}}\sqrt{1+\tau_2^2 \frac{\beta^2}{\rho(p+\beta^2/q)^2}\frac{\|h\|^2}{\|\theta_0\|^2} - \tau_2 \frac{2\beta}{\sqrt{\rho}(p+\beta^2/q)}\frac{\theta_0^\top h}{\|\theta_0\|^2}}.
\end{align}

From \eqref{eq:proof:thm:estimationerror:2:43}, we recover the same sufficient condition as in the first case, i.e.,
\begin{align}
\label{eq:proof:thm:estimationerror:2:44}
    \alpha \leq \sigma_{\theta_0} - \delta,
\end{align}
for some arbitrarily small $\delta>0$, which guarantees that both $\mathcal O_{d,1}$ and $\mathcal O_{d,2}$ are jointly convex in $(\alpha,\tau_2)$ w.p.a.\ $1$ as $d \to \infty$. As a consequence, from now on, we will work with $K_\alpha := \sigma_{\theta_0} - \delta$. Similarly to the first case, in what follows we will drop the \say{w.p.a.\ $1$ as $d \to \infty$} whenever we refer to the convexity of the functions $\mathcal O_{d,2}$. Moreover, we will drop the \say{under the condition $\alpha \leq \sigma_{\theta_0} - \delta$} whenever we will say that $\mathcal O_{d,2}$, and its partial minimizations, are convex. 

We have just shown that $\mathcal O_{d,2}(\alpha,\tau_1,\tau_2,\beta)$ is jointly convex in $(\alpha,\tau_1,\tau_2)$, and therefore, after the partial minimization over $\tau_2$, we have that $\mathcal O_{d,2}^{\tau_2}(\alpha,\tau_1,\beta)$ is jointly convex in $(\alpha,\tau_1)$. 

We will now prove that $\mathcal O_{d,2}^{\tau_2}(\alpha,\tau_1,\beta)$ is continuous on its domain. First notice that $\mathcal O_{d,2}(\alpha,\tau_1,\tau_2,\beta)$ is continuous on its domain (the proof of this follows exactly the same lines as the proof of continuity for $\mathcal O_{d,1}(\alpha,\tau_1,\tau_2,\beta)$). Then, we will show that the minimization over $\tau_2 > 0$ can be equivalently rewritten as a minimization over $\tau_2 \in \mathcal T_2$, with $\mathcal T_2$ a compact set. Consequently, $\mathcal O_{d,2}^{\tau_2}(\alpha,\tau_1,\beta)$ is a continuous function. Now, in order to show that we can restrict $\tau_2>0$ to $\tau_2 \in \mathcal T_2$ without loss of generality, we will show that the optimal solution $\tau_{2\star}$ satisfies $0<\tau_{2\star}<\infty$. The fact that $\tau_{2\star}>0$ follows from
\begin{align*}
    \lim_{\tau_2 \to 0^+}\; \mathcal O_{d,2}(\alpha,\tau_1,\tau_2,\beta) &= \lim_{\tau_2 \to 0^+}\; \frac{p\tau_2}{2}+\frac{\beta^2 \tau_2}{2q} - \frac{\alpha}{\sqrt{n}}\left\| \sqrt{\rho}\left( p + \frac{\beta^2}{q} \right) \frac{\theta_0}{\tau_2} - \beta h \right\| + \frac{\|\theta_0\|^2+d \alpha^2}{2 n \tau_2}\left( p + \frac{\beta^2}{q} \right) \\ &= \lim_{\tau_2 \to 0^+}\; - \frac{\alpha \sqrt{d}}{n} \left( p + \frac{\beta^2}{q} \right) \frac{\|\theta_0\|}{\tau_2} + \frac{\|\theta_0\|^2+d \alpha^2}{2 n \tau_2}\left( p + \frac{\beta^2}{q} \right) \\ &= \lim_{\tau_2 \to 0^+}\; \frac{\rho}{2 \tau_2} \left( p + \frac{\beta^2}{q} \right) \left( \frac{\|\theta_0\|}{\sqrt{d}} - \alpha \right)^2,
\end{align*}
which is equal to $\infty$, w.p.a.\ $1$ as $d \to \infty$, due to the condition \eqref{eq:proof:thm:estimationerror:2:44}. Moreover, $\tau_{2\star}< \infty$ follows from
\begin{align*}
    \lim_{\tau_2 \to \infty}\; \mathcal O_{d,2}(\alpha,\tau_1,\tau_2,\beta) &= \lim_{\tau_2 \to \infty}\; \frac{p\tau_2}{2}+\frac{\beta^2 \tau_2}{2q} - \frac{\alpha}{\sqrt{n}}\left\| \sqrt{\rho}\left( p + \frac{\beta^2}{q} \right) \frac{\theta_0}{\tau_2} - \beta h \right\| + \frac{\|\theta_0\|^2+d \alpha^2}{2 n \tau_2}\left( p + \frac{\beta^2}{q} \right) \\ &= \lim_{\tau_2 \to \infty}\; \frac{p\tau_2}{2}+\frac{\beta^2 \tau_2}{2q} - \alpha \beta\frac{\|h\|}{\sqrt{n}},
\end{align*}
which is equal to $\infty$, w.p.a.\ $1$ as $d \to \infty$, since $\|h\|/\sqrt{n} \overset{P}{\to} \sqrt{\rho}$.

As a consequence of the continuity and convexity-concavity, we can apply Sion's minimax principle to exchange the minimization and the maximization in \eqref{eq:proof:thm:estimationerror:2:41} and obtain 
\begin{align}
\label{eq:proof:thm:estimationerror:2:45}
    \inf_{\substack{0 \leq \alpha \leq \sigma_{\theta_0} - \delta\\ \tau_1>0}}\; \max_{0 \leq \beta \leq B}\; \mathcal O_{d,2}^{\tau_2}(\alpha,\tau_1,\beta).
\end{align}

We have now arrived at the scalar formulation \eqref{eq:proof:thm:estimationerror:2:45}, which has the same optimal value as the modified (AO) problem \eqref{eq:proof:thm:estimationerror:2:9:3}, w.p.a.\ $1$ as $d \to \infty$. Therefore, following the same reasoning as in the first case, the next step is to study the convergence in probability of its optimal value, and to show that its optimal solution $\alpha_{\star,2}$ is unique. The convergence in probability of its objective function $\mathcal O_{d,2}^{\tau_2}(\alpha,\tau_1,\beta)$ will be done in two steps: first, we will study the convergence in probability of the function $\mathcal O_{d,2}(\alpha,\tau_1,\tau_2,\beta)$, and secondly, we will employ statement (iii) of Lemma~\ref{lemma:conv:opt} to study the convergence in probability of $\mathcal O_{d,2}^{\tau_2}(\alpha,\tau_1,\beta)$. Following this, we will prove that $\alpha_{\star,2}$ is unique.

Following similar arguments to the convergence in probability $\mathcal O_{d,1}(\alpha,\tau_1,\tau_2,\beta) \overset{P}{\to} \mathcal O_1(\alpha,\tau_1,\tau_2,\beta)$ previously treated, we have that for fixed $(\alpha,\tau_1,\tau_2,\beta)$ satisfying the constraints in \eqref{eq:proof:thm:estimationerror:2:45}, $\mathcal O_{d,2}(\alpha,\tau_1,\tau_2,\beta)$ converges in probability to
\begin{align}
\label{eq:proof:thm:estimationerror:2:46}
\begin{split}
    \mathcal O_2(\alpha,\tau_1,\tau_2,\beta) := \frac{\beta \tau_1}{2} + \frac{p\tau_2}{2}+\frac{\beta^2 \tau_2}{2q} -\frac{\beta^2}{2M} + \mathcal F(\alpha,\tau_1/\beta) - \alpha\sqrt{\rho} \sqrt{\left( p + \frac{\beta^2}{q} \right)^2 \frac{\rho \sigma_{\theta_0}^2}{\tau_2^2} + \beta^2} + \\ + \rho \left( p + \frac{\beta^2}{q} \right) \frac{\sigma_{\theta_0}^2+\alpha^2}{2 \tau_2}.
\end{split}
\end{align}

Now, for fixed $(\alpha,\tau_1,\beta)$ satisfying the constraints in \eqref{eq:proof:thm:estimationerror:2:45}, $\{\mathcal O_{d,2}(\alpha,\tau_1,\cdot,\beta)\}_{d \in \mathbb N}$ is a sequence of random real-valued convex functions, converging in probability (pointwise, for every $\tau_2>0$) to the function $\mathcal O_{2}(\alpha,\tau_1,\cdot,\beta)$. Moreover,
\begin{align}
\label{eq:proof:thm:estimationerror:2:47}
\begin{split}
    \lim_{\tau_2 \to \infty}\; \mathcal O_{2}(\alpha,\tau_1,\tau_2,\beta) &= \lim_{\tau_2 \to \infty}\; \frac{p\tau_2}{2}+\frac{\beta^2 \tau_2}{2q} - \alpha\sqrt{\rho} \sqrt{\left( p + \frac{\beta^2}{q} \right)^2 \frac{\rho \sigma_{\theta_0}^2}{\tau_2^2} + \beta^2} + \rho \left( p + \frac{\beta^2}{q} \right) \frac{\sigma_{\theta_0}^2+\alpha^2}{2 \tau_2} \\ &= \lim_{\tau_2 \to \infty}\; \frac{p\tau_2}{2}+\frac{\beta^2 \tau_2}{2q} - \alpha\beta \sqrt{\rho} = \infty.
    \end{split}
\end{align}

As a consequence, using statement (iii) of Lemma~\ref{lemma:conv:opt}, we have that
\begin{align}
\label{eq:proof:thm:estimationerror:2:48}
    \mathcal O_{d,2}^{\tau_2}(\alpha,\tau_1,\beta) \overset{P}{\to} \mathcal O_{2}^{\tau_2}(\alpha,\tau_1,\beta) : = \inf_{\tau_2 > 0}\; \mathcal O_2(\alpha,\tau_1,\tau_2,\beta)
\end{align}

Notice that $\mathcal O_{2}^{\tau_2}$ is precisely the second objective function of the minimax problem \eqref{eq:est:2:minimax} from the statement of Theorem~\ref{thm:estimationerror:2}. Since $\mathcal O_{2}^{\tau_2}$ is the pointwise limit (in probability, for each ($\alpha, \tau_1, \beta$)) of the sequence of objective functions $\mathcal O_{d,2}^{\tau_2}$ which are convex-concave, and convexity is preserved by pointwise limits, we have that $\mathcal O_2^{\tau_2}$ is jointly convex in $(\alpha,\tau_1)$ and concave in $\beta$. 

We now consider the asymptotic minimax optimization problem
\begin{align}
\label{eq:proof:thm:estimationerror:2:49}
    \inf_{\substack{0 \leq \alpha \leq \sigma_{\theta_0} - \delta\\ \tau_1>0}}\; \max_{0 \leq \beta \leq B}\; \mathcal O_{2}^{\tau_2}(\alpha,\tau_1,\beta),
\end{align}
and prove that its optimal solution $\alpha_{\star,2}$ is unique. To prove this, we will show that the function $\inf_{\tau_1>0}\; \max_{0 \leq \beta \leq B}\; \mathcal O_{2}^{\tau_2}(\alpha,\tau_1,\beta)$ is strictly convex in $\alpha$. As in the case of $\mathcal O_1(\alpha,\tau_1,\tau_2,\beta)$, and following the same reasoning as in \eqref{eq:proof:thm:estimationerror:2:27}-\eqref{eq:proof:thm:estimationerror:2:28}, we have that $\mathcal O_2(\alpha,\tau_1,\tau_2,\beta)$ is jointly strictly convex in $(\alpha,\tau_2)$ if the condition \eqref{eq:proof:thm:estimationerror:2:44} (i.e $\alpha \leq \sigma_{\theta_0} - \delta$) is satisfied. Indeed, the joint strict convexity in $(\alpha,\tau_2)$ depends only on the last two terms in \eqref{eq:proof:thm:estimationerror:2:46}, which after a simplification (over $\beta,p,q,\rho,\sigma_{\theta_0}$ which have no influence on the convexity), can be rewritten as
\begin{align*}
    -\alpha \sqrt{\frac{1}{\tau_2^2}+\frac{\beta^2}{\rho (p+\beta^2/q)^2 \sigma_{\theta_0}^2}} + \frac{1}{\tau_2}\left( \frac{\sigma_{\theta_0}^2 + \alpha^2}{2 \sigma_{\theta_0}} \right),
\end{align*}
from which we recover (similarly as before) the following necessary and sufficient condition for the joint strict convexity in $(\alpha,\tau_2)$,
\begin{align*}
    \alpha < \sigma_{\theta_0} \sqrt{1+\tau_2^2 \frac{\beta^2}{\rho (p+\beta^2/q)^2 \sigma_{\theta_0}^2}},
\end{align*}
which holds whenever condition \eqref{eq:proof:thm:estimationerror:2:44} holds. 

This can be used to prove that the optimal solution $\alpha_{\star,2}$ of problem \eqref{eq:proof:thm:estimationerror:2:49} is unique, as we have previously seen for $\mathcal O_1$. For completeness, however, we will briefly remind the main steps of the proof of uniqueness. First, from
\begin{align*}
    \lim_{\tau_2 \to 0^+}\; \mathcal O_{2}(\alpha,\tau_1,\tau_2,\beta) &= \lim_{\tau_2 \to 0^+}\; \frac{p\tau_2}{2}+\frac{\beta^2 \tau_2}{2q} - \alpha\sqrt{\rho} \sqrt{\left( p + \frac{\beta^2}{q} \right)^2 \frac{\rho \sigma_{\theta_0}^2}{\tau_2^2} + \beta^2} + \rho \left( p + \frac{\beta^2}{q} \right) \frac{\sigma_{\theta_0}^2+\alpha^2}{2 \tau_2} \\ &= \lim_{\tau_2 \to 0^+}\; \frac{\rho}{2} \left( p + \frac{\beta^2}{q} \right) \frac{(\sigma_{\theta_0}-\alpha)^2}{\tau_2} = \infty,
\end{align*}
and
\begin{align*}
    \lim_{\tau_2 \to \infty}\; \mathcal O_{2}(\alpha,\tau_1,\tau_2,\beta) = \infty,
\end{align*}
(which was shown in \eqref{eq:proof:thm:estimationerror:2:47}), we have that the minimization of $\mathcal O_2(\alpha,\tau_1,\tau_2,\beta)$ over $\tau_2$ is attained by some $0<\tau_{2\star}<\infty$. From this and the joint strict convexity in $(\alpha,\tau_2)$ of $\mathcal O_2(\alpha,\tau_1,\tau_2,\beta)$, we have that the function $\mathcal O_2^{\tau_2}(\alpha,\tau_1,\beta)$ is strictly convex in $\alpha$. In particular, $\mathcal O_2^{\tau_2}(\alpha,\tau_1,\beta)$ can be written as the sum of a strictly convex function in $\alpha$ and a jointly convex function in $(\alpha,\tau_1)$. Then, after minimizing over $\tau_1>0$ (notice that, by Sion's minimax principle, we can exchange the maximization over $\beta$ with the minimization over $\tau_1$ in \eqref{eq:proof:thm:estimationerror:2:49}) we are left with a function that is strictly convex in $\alpha$ (since it is the sum of a strictly convex and a convex function). Finally, since the maximization over $\beta$ is attained, we have that $ \max_{0 \leq \beta \leq B}\; \inf_{\tau_1>0}\; \mathcal O_{2}^{\tau_2}(\alpha,\tau_1,\beta)$ is strictly convex in $\alpha$, and therefore $\alpha_{\star,2}$ is unique.

Similarly to the first case, the uniqueness of $\alpha_{\star,1}$ is the key element in the convergence analysis required by Fact~\ref{prop:cgmt}. Again, we define, for arbitrary $\eta>0$, the sets $\mathcal S_\eta := \{\alpha \in [0,\sigma_{\theta_0}-\delta]:\, |\alpha-\alpha_{\star,2}|<\eta\}$, with $\alpha_{\star,2}$ the unique solution of \eqref{eq:proof:thm:estimationerror:2:49}, and $\mathcal S_\eta^c := [0,\sigma_{\theta_0}-\delta]\setminus \mathcal S_\eta$. Since $\alpha_{\star,2}$ is the unique solution of \eqref{eq:proof:thm:estimationerror:2:49}, we have that
\begin{align}
\label{eq:proof:thm:estimationerror:2:50}
    \inf_{\substack{0 \leq \alpha \leq \sigma_{\theta_0} - \delta \\ \tau_1>0}}\; \max_{0 \leq \beta \leq B}\; \mathcal O_2^{\tau_2}(\alpha, \tau_1, \beta) < \inf_{\substack{\alpha \in \mathcal S_\eta^c \\ \tau_1>0}}\; \max_{0 \leq \beta \leq B}\; \mathcal O_2^{\tau_2}(\alpha, \tau_1, \beta)
\end{align}

Now, in order to finish the proof for the second case, we still have to prove the following two convergences in probability
\begin{align}
\label{eq:proof:thm:estimationerror:2:51}
    \inf_{\substack{0 \leq \alpha \leq \sigma_{\theta_0} - \delta \\ \tau_1>0}}\; \max_{0 \leq \beta \leq B}\; \mathcal O_{d,2}^{\tau_2}(\alpha, \tau_1, \beta) \overset{P}{\to} \inf_{\substack{0 \leq \alpha \leq \sigma_{\theta_0} - \delta \\ \tau_1>0}}\; \max_{0 \leq \beta \leq B}\; \mathcal O_2^{\tau_2}(\alpha, \tau_1,  \beta),
\end{align}
and
\begin{align}
\label{eq:proof:thm:estimationerror:2:52}
    \inf_{\substack{\alpha \in \mathcal S_\eta^c \\ \tau_1>0}}\; \max_{0 \leq \beta \leq B}\; \mathcal O_{d,2}^{\tau_2}(\alpha, \tau_1, \beta) \overset{P}{\to} \inf_{\substack{\alpha \in \mathcal S_\eta^c \\ \tau_1>0}}\; \max_{0 \leq \beta \leq B}\; \mathcal O_2^{\tau_2}(\alpha, \tau_1, \beta).
\end{align}
Then, the result will simply follow from Fact~\ref{prop:cgmt}. Again, it will be easier to prove the two convergences in probability in the following equivalent form (by Sion's minimax theorem)
\begin{align*}
    \min_{\alpha}\; \max_{0 \leq \beta \leq B}\;\inf_{\tau_1>0}\; \mathcal O_{d,2}^{\tau_2}(\alpha, \tau_1, \beta) \overset{P}{\to} \min_{\alpha}\; \max_{0 \leq \beta \leq B}\;\inf_{\tau_1>0}\; \mathcal O_2^{\tau_2}(\alpha, \tau_1, \beta).
\end{align*}

For fixed $(\alpha,\beta)$, with $\beta>0$, $\{\mathcal O_{d,2}^{\tau_2}(\alpha, \cdot, \beta)\}_{d \in \mathbb N}$ is a sequence of random, real-valued convex functions (since they are defined as the partial minimization of jointly convex functions), converging in probability (pointwise, for every $\tau_1 > 0$) to the function $\mathcal O_2^{\tau_2}(\alpha, \cdot, \beta)$.  Moreover, we have that
\begin{align*}
    \lim_{\tau_1 \to \infty} \mathcal O_2^{\tau_2}(\alpha, \tau_1, \beta) = \infty.
\end{align*}
This follows, as in the first case, from
\begin{align*}
    \lim_{\tau_1 \to \infty} \frac{\beta \tau_1}{2}  + \mathcal F(\alpha,\frac{\tau_1}{\beta})  \geq \lim_{\tau_1 \to \infty} \frac{\beta \tau_1}{2} = \infty.
\end{align*}
As a consequence, when $\beta>0$, we can apply again statement (iii) of Lemma~\ref{lemma:conv:opt} to conclude that
\begin{align*}
    \inf_{\tau_1>0}\; \mathcal O_{d,2}^{\tau_2}(\alpha, \tau_1, \beta) \overset{P}{\to} \inf_{\tau_1>0}\; \mathcal O_2^{\tau_2}(\alpha, \tau_1, \beta).
\end{align*}
Moreover, for $\beta=0$ the objective functions $O_{d,2}^{\tau_2}$ are reduced to
\begin{align*}
    \mathcal O_{d,2}^{\tau_2}(\alpha,\tau_1,0) = \frac{p\tau_{2\star}}{2} + \frac{p \rho}{2 \tau_{2\star}} \left( \frac{\|\theta_0\|}{\sqrt{d}} - \alpha \right)^2,
\end{align*}
and therefore the terms involving $\tau_1$ are cancelled. Here, $\tau_{2\star}$ is the optimal solution attaining the infimum (which exists as previously explained).

We now define $\mathcal O_2^{\tau_1,\tau_2}(\alpha, \beta):= \inf_{\tau_1>0}\; \mathcal O_2^{\tau_2}(\alpha, \tau_1, \beta)$ and $\mathcal O_{d,2}^{\tau_1,\tau_2}(\alpha, \beta):= \inf_{\tau_1>0}\; \mathcal O_{d,2}^{\tau_2}(\alpha, \tau_1, \beta)$. For fixed $\alpha$, $\{\mathcal O_{d,2}^{\tau_1,\tau_2}(\alpha, \cdot)\}_{d \in \mathbb N}$ is a sequence of random, real-valued concave functions (since they are defined as a minimization of concave functions), converging in probability (pointwise, for every $0<\beta\leq B$) to the function $\mathcal O_2^{\tau_1, \tau_2}(\alpha, \cdot)$. Therefore, we can apply statement (ii) of Lemma~\ref{lemma:conv:opt} to conclude that
\begin{align*}
    \sup_{0 < \beta \leq B}\; \mathcal O_{d,2}^{\tau_1,\tau_2}(\alpha, \beta) \overset{P}{\to} \sup_{0 < \beta \leq B}\; \mathcal O_2^{\tau_1,\tau_2}(\alpha, \beta),
\end{align*}
Moreover, for $\beta = 0$, it can be easily seen that
\begin{align*}
    \mathcal O_{d,2}^{\tau_2}(\alpha,\tau_1,0) \overset{P}{\to} \mathcal O_{2}^{\tau_2}(\alpha,\tau_1,0) = \frac{p\tau_{2\star}}{2} + \frac{p \rho}{2 \tau_{2\star}} \left( \sigma_{\theta_0} - \alpha \right)^2,
\end{align*}
where $\tau_{2\star}$ is the optimal solution attaining the infimum (which exists as previously explained). As a consequence, we have that 
\begin{align}
\label{eq:proof:thm:estimationerror:2:alpha2:0}
    \sup_{0 \leq \beta \leq B}\; \mathcal O_{d,2}^{\tau_1,\tau_2}(\alpha, \beta) \overset{P}{\to} \sup_{0 \leq \beta \leq B}\; \mathcal O_2^{\tau_1,\tau_2}(\alpha, \beta),
\end{align}

We have now arrived at the final step of the convergence proof, where again, due to the Convexity Lemma~\ref{lemma:convexity:lemma}, we will distinguish between the case $\alpha_{\star,2}=0$ and $\alpha_{\star,2}>0$. The proof follows exactly the same lines as in \eqref{eq:proof:thm:estimationerror:2:alpha:1}-\eqref{eq:proof:thm:estimationerror:2:35}, which we report in what follows for completeness.

We first consider the case $\alpha_{\star,2}=0$. Since \eqref{eq:proof:thm:estimationerror:2:alpha2:0} is valid for any $\alpha \geq 0$, it is in particular valid for $\alpha = 0$, and thus we have that
\begin{align}
\label{eq:proof:thm:estimationerror:2:alpha2:1}
    \mathcal O_{d,2}^{\beta,\tau_1,\tau_2}(0) \overset{P}{\to} \mathcal O_2^{\beta,\tau_1,\tau_2}(0).
\end{align}
Moreover, using statement (i) of Lemma~\ref{lemma:conv:opt}, for any $K>0$, we have that 
\begin{align}
\label{eq:proof:thm:estimationerror:2:alpha2:2}
    \min_{K \leq\alpha\leq \sigma_{\theta_0} - \delta}\; \mathcal O_{d,2}^{\beta,\tau_1,\tau_2}(\alpha) \overset{P}{\to} \min_{K \leq \alpha\leq \sigma_{\theta_0} - \delta}\; \mathcal O_2^{\beta,\tau_1,\tau_2}(\alpha).
\end{align}
The result for $\alpha_{\star,1}=0$ now follows from \eqref{eq:proof:thm:estimationerror:2:alpha2:1} and \eqref{eq:proof:thm:estimationerror:2:alpha2:2} and the fact that 
\begin{align*}
    \mathcal O_2^{\beta,\tau_1,\tau_2}(0) < \min_{K \leq \alpha\leq \sigma_{\theta_0} - \delta}\; \mathcal O_2^{\beta,\tau_1,\tau_2}(\alpha),
\end{align*}
for any $K > 0$, due to the uniqueness of $\alpha_{\star,2}$.

We now focus on the case $\alpha_{\star,2} > 0$. We need to prove \eqref{eq:proof:thm:estimationerror:2:51} and \eqref{eq:proof:thm:estimationerror:2:52}, which are equivalent to the following two convergences in probability
\begin{align}
\label{eq:proof:thm:estimationerror:2:53}
    \inf_{0<\alpha\leq \sigma_{\theta_0} - \delta}\; \mathcal O_{d,2}^{\beta,\tau_1,\tau_2}(\alpha) \overset{P}{\to} \inf_{0<\alpha\leq \sigma_{\theta_0} - \delta}\; \mathcal O_2^{\beta,\tau_1,\tau_2}(\alpha),
\end{align}
and
\begin{align}
\label{eq:proof:thm:estimationerror:2:54}
    \inf_{\alpha \in \mathcal S_\eta^c}\; \mathcal O_{d,2}^{\beta,\tau_1,\tau_2}(\alpha) \overset{P}{\to} \inf_{\alpha \in \mathcal S_\eta^c}\; \mathcal O_2^{\beta,\tau_1,\tau_2}(\alpha),
\end{align}
respectively. Since $\{\mathcal O_{d,2}^{\beta,\tau_1,\tau_2}(\cdot)\}_{d \in \mathbb N}$ is a sequence of random, real-valued convex functions, converging in probability (pointwise, for every $\alpha > 0$) to the function $\mathcal O_2^{\beta,\tau_1,\tau_2}(\cdot)$, \eqref{eq:proof:thm:estimationerror:2:53} follows immediately from  statement (ii) of Lemma~\ref{lemma:conv:opt}. Moreover, since $\mathcal S_\eta^c = (0,\alpha_{\star,2}-\eta]\cup[\alpha_{\star,2}+\eta,\sigma_{\theta_0} - \delta]$, \eqref{eq:proof:thm:estimationerror:2:54} follows from statement (i) of Lemma~\ref{lemma:conv:opt}, for the interval $[\alpha_{\star,2}+\eta,\sigma_{\theta_0} - \delta]$, and statement (ii) of the same lemma, for the interval $(0,\alpha_{\star,2}-\eta]$. 
This concludes the proof of the two convergences in probability  \eqref{eq:proof:thm:estimationerror:2:51} and \eqref{eq:proof:thm:estimationerror:2:52}, and with it, the analysis of the second case \eqref{eq:proof:thm:estimationerror:2:9:3}, and the proof.
\end{proof}

\begin{proof}[Proof of Proposition~\ref{prop:est:error:squared:loss}]
The proof follows similar lines as the proof of Theorem~\ref{thm:estimationerror:1}, and therefore we only provide a sketch. We start by rewriting the minimization on right-hand side of \eqref{eq:reg:for:squared:loss} in the following form
\begin{align}
\label{proof:est:error:squared:loss:1}
    \min_{\theta \in \mathbb R^d}\;  \frac{\rho}{d} \left( \sqrt{n}\|y-Ax\|  +  \sqrt{\varepsilon_0} \sqrt{n}\|\theta\|\right),
\end{align}
which enables us to quantify the high-dimensional estimation error using Theorem~$1$ in \cite{thrampoulidis2018precise}. Here, the loss function and the regularization function are both equal to $\sqrt{n} \|\cdot\|$. Therefore, similarly to the proof of Theorem~\ref{thm:estimationerror:1}, we can compute the Moreau envelope of $\sqrt{n} \|\cdot\|$ as
\begin{align*}
    e_{\sqrt{n}\|\cdot\|}(c;\tau) = \begin{cases} \sqrt{n}\|c\|-{n\tau}/{2} \quad\quad &\mbox{if }\; \|c\|>\sqrt{n}\tau, \\ \|c\|^2/{(2 \tau)}
 \quad\quad &\mbox{if }\; \|c\|\leq\sqrt{n}\tau, \end{cases}
\end{align*}
from which the two limits in probability from Theorem~$1$ in \cite{thrampoulidis2018precise} can be computed as follows
\begin{align*}
    \frac{1}{n} \left( e_{\sqrt{n}\|\cdot\|}(c g + z;\tau) - \sqrt{n}\|z\| \right) \overset{P}{\to} \begin{cases} \sqrt{(c^2 + \sigma_{Z}^2)}  - {\tau}/{2} - {\sigma_{Z}}/{\sqrt{\rho}} \quad\quad &\mbox{if }\; \sqrt{c^2 + \sigma_{Z}^2} > \tau, \\  \left( c^2 + \sigma_{Z}^2 \right)/(2 \tau)- {\sigma_{Z}}
 \quad\quad &\mbox{if }\; \sqrt{c^2 + \sigma_{Z}^2} \leq \tau, \end{cases}
\end{align*}
which is equal to the function $\mathcal E(c,\tau)$ defined above, and
\begin{align*}
    \frac{1}{d} \left( e_{\sqrt{n}\|\cdot\|}(c h + \theta_0;\tau) - \sqrt{n}\|\theta_0\| \right) \overset{P}{\to} \begin{cases} \sqrt{(c^2 + \sigma_{\theta_0}^2)/\rho}  - {\tau}/{(2 \rho)} - {\sigma_{\theta_0}}/{\sqrt{\rho}} \quad\quad &\mbox{if }\; \sqrt{\rho}\sqrt{c^2 + \sigma_{\theta_0}^2} > \tau, \\  \left( c^2 + \sigma_{\theta_0}^2 \right)/(2 \tau)- {\sigma_{\theta_0}}/{\sqrt{\rho}} 
 \quad\quad &\mbox{if }\; \sqrt{\rho}\sqrt{c^2 + \sigma_{\theta_0}^2} \leq \tau. \end{cases}
\end{align*}
which is equal to the function $\mathcal G(c,\tau)$ defined in \eqref{eq:R}. The proofs of these convergences follow the same arguments as in Theorems~\ref{thm:estimationerror:1} and \ref{thm:estimationerror:2}, thus are omitted. Moreover, all the assumptions required by Theorem~$1$ in \cite{thrampoulidis2018precise} can be easily shown to hold (as in the proof of Theorems~\ref{thm:estimationerror:1}, or as in Section~$\text{V.E}$ in \cite{thrampoulidis2018precise}). This concludes the sketch of the proof.
\end{proof}

\begin{proof}[Proof of Theorem~\ref{thm:estimationerror:3}]
Similarly to Theorem~\ref{thm:estimationerror:2}, since the proof is quite long and intricate, we divide it into steps, which we briefly explain before jumping into the technical details.

\paragraph{Step 1: We show that problem \eqref{eq:regdro} can be equivalently rewritten as a (PO) problem.} The proof builds upon the dual formulation \eqref{eq:dro:dual:uni:3} presented in Fact~\ref{lemma:nonconvexduality:3}. After introducing the change of variable $w = (\theta-\theta_0)/\sqrt{d}$ and expressing it in vector form, \eqref{eq:dro:dual:uni:3} becomes
\begin{align}
\label{eq:proof:thm:estimationerror:3:0}
    \min_{w \in \mathcal W}\; \max_{u \in \mathbb R^n} \; - \frac{1}{n} u^\top (\sqrt{d}A)w +\frac{1}{n} u^\top z + \frac{1}{4\lambda n}\|\theta_0 + \sqrt{d}w\|^2\|u\|^2-\frac{1}{n}\sum_{i=1}^{n}L^*(u_i)
\end{align}
where $\mathcal W$ is the feasible set of $w$ obtained from $\Theta$ after the change of variable, $A \in \mathbb R^{n \times d}$ denotes the matrix whose rows are the vectors $x_i$, for $i=1,\ldots,n$, and $z \in \mathbb R^n$ is the measurement noise vector with entries i.i.d.\ distributed according to $\mathbb P_Z$. Notice that, due to Assumption~\ref{assump:Theta}, the set $\mathcal W$ is convex and compact. Moreover, following the assumptions~\ref{assump:Theta} and \ref{assump:cgmt}\ref{assump:theta0}, we have that $\mathcal W \subset \{w \in \mathbb R^d: \,  \|w\| \leq R_\theta + \sigma_{\theta_0}\}$ w.p.a.\ 1 as $d\to \infty$.

From Lemma~\ref{lemma:growthrates:3} we know that the optimal solution $u_\star$ is in the order of $\sqrt{n}$. As a consequence, by rescaling the variable $u$ as $u \to u \sqrt{n}$, and introducing the convex compact set $\mathcal S_u := \{u \in \mathbb R^n:\, \|u\|\leq K_\beta\}$, for some sufficiently large $K_\beta > 0$, we can equivalently rewrite \eqref{eq:proof:thm:estimationerror:3:0} as
\begin{align}
\label{eq:proof:thm:estimationerror:3:1}
    \min_{w \in \mathcal W}\; \max_{u \in \mathcal S_u} \; - \frac{1}{\sqrt{n}} u^\top (\sqrt{d}A)w +\frac{1}{\sqrt{n}} u^\top z + \frac{1}{4\lambda}\|\theta_0 + \sqrt{d}w\|^2\|u\|^2-\frac{1}{n}\sum_{i=1}^{n}L^*(u_i \sqrt{n})
\end{align}

We will now show that the assumption that $L$ can be written as the Moreau envelope with parameter $M$ of a convex function $f$ is natural, and always satisfied by $M$-smooth functions (which is precisely our case, due to Assumption~\ref{assump:dro}\ref{assump:ell:diff}). Since $L$ is $M$-smooth, we have that its convex conjugate $L^*$ is $1/M$-strongly convex, and therefore $L^*$ can be written as
\begin{align}
    \label{eq:proof:thm:estimationerror:3:2}
    L^*(\cdot) = \frac{1}{2M}(\cdot)^2 + f^*(\cdot),
\end{align}
where $f^*$ is the conjugate of a convex function $f$. Notice that $f$ can be easily obtained (as a function of $L$) from \eqref{eq:proof:thm:estimationerror:3:2} as $f = (L^* - 1/(2M)(\cdot)^2)^*$. Conversely, the loss function $L$ is nothing but the Moreau envelope of $f$, as shown in what follows
\begin{align*}
    L(\cdot) = \left(\frac{1}{2M}(\cdot)^2 + f^*(\cdot)\right)^*(\cdot) = \left(\frac{M}{2}(\cdot)^2 \star_{\text{inf}} f(\cdot)\right)(\cdot) = e_f\left(\cdot,\frac{1}{M}\right),
\end{align*}
where $\star_{\text{inf}}$ denotes the infimal convolution operation. Since $L$ is continuous and convex (notice that $L$ is trivially proper due to Assumption~\ref{assump:cgmt}\ref{assump:ell:0}), with $\min_{v \in \mathbb R} L(v) = 0$, we have that $L^*$ is lower semicontinuous and convex, and therefore both $f$ and $f^*$ are lower semicontinuous and convex.

Using the decomposition \eqref{eq:proof:thm:estimationerror:3:2}, problem \eqref{eq:proof:thm:estimationerror:3:1} can be rewritten as
\begin{align}
\label{eq:proof:thm:estimationerror:3:3}
    \min_{w \in \mathcal W}\; \max_{u \in \mathcal S_u} \; - \frac{1}{\sqrt{n}} u^\top (\sqrt{d}A)w +\frac{1}{\sqrt{n}} u^\top z + \frac{1}{4\lambda}\|\theta_0 + \sqrt{d}w\|^2\|u\|^2- \frac{1}{2M}\|u\|^2- \frac{1}{n}\sum_{i=1}^{n}f^*(u_i \sqrt{n}).
\end{align}
It is important now to highlight that the expression
\begin{align*}
    \frac{1}{4\lambda}\|\theta_0 + \sqrt{d}w\|^2\|u\|^2- \frac{1}{2M}\|u\|^2
\end{align*}
is concave is $u$. This follows from the assumption that $\lambda > M R_\theta^2 d/2 \geq  M\|\theta_0 + \sqrt{d}w\|^2/2$. As will be shown later, the decomposition \eqref{eq:proof:thm:estimationerror:3:2} and the concavity in $u$ of the above-mentioned expression will allow us to use the CGMT.

For ease of notation, we introduce the function $F^*:\mathbb R^n \to \mathbb R$, which is defined as
\begin{align*}
    F^*(u):=\sum_{i=1}^{n}f^*(u_i).
\end{align*}
Since $f,f^*$ are lower semicontinuous and convex, we have that $F^*$ is lower semicontinuous and convex, and therefore its convex conjugate $F = F^{**}$ can be easily recovered as follows
\begin{align*}
    F(u) = \sup_{v \in \mathbb R^n} v^\top u - F^*(v) = \sum_{i=1}^n\; \sup_{v_i \in \mathbb R}\; v_i u_i - f^*(v_i) = \sum_{i=1}^n f(u_i).
\end{align*}
We now rewrite $F^*(u \sqrt{n})$ using its convex conjugate as follows
\begin{align}
\label{eq:proof:thm:estimationerror:3:5}
    \min_{\substack{w \in \mathcal W \\ s \in \mathbb R^n}}\; \max_{u \in \mathcal S_u}\;  - \frac{1}{\sqrt{n}} u^\top (\sqrt{d}A)w +\frac{1}{\sqrt{n}} u^\top z + \frac{1}{4\lambda}\|\theta_0 + \sqrt{d}w\|^2\|u\|^2 -\frac{1}{2 M}\|u\|^2-\frac{1}{\sqrt{n}}s^\top u + \frac{1}{n} F(s)
\end{align}
where we have used Sion's minimax principle to exchange the minimization over $s$ with the maximization over $u$. In particular, this last step is possible only because the objective function in \eqref{eq:proof:thm:estimationerror:3:5} is concave in $u$, due to the decomposition \eqref{eq:proof:thm:estimationerror:3:2}. 

Notice that the new objective function in \eqref{eq:proof:thm:estimationerror:3:5} is convex in $(w,s)$ and concave in $u$. Indeed, the concavity in $u$ holds, as explained above,  from the assumption that $\lambda  > M R_\theta^2 d/2 \geq  M\|\theta_0 + \sqrt{d}w\|^2/2$. Moreover, the convexity in $w$ can be easily seen from the convexity in $w$ of the term $1/(4\lambda)\|\theta_0 + \sqrt{d}w\|^2\|u\|^2$ and the linearity of the other terms. Finally, the convexity in $s$ can be concluded from the convexity of $F$, and the joint convexity in $(w,s)$ follows easily since $w$ and $s$ are decoupled.

Using Assumption~\ref{assump:0}, we can see that $\sqrt{d}A$ has entries i.i.d.\ $\mathcal N(0,1)$. Moreover, the objective function \eqref{eq:proof:thm:estimationerror:3:5} is convex-concave in $(w,u)$, with $\mathcal W, \mathcal S_u$ convex compact sets. As a consequence, problem \eqref{eq:proof:thm:estimationerror:3:5} is a (PO) problem, in the form \eqref{eq:cgmt:po}.

%------------------------------------------------------------------------------------
\paragraph{Step 2: We apply CGMT and obtain the modified (AO) problem.}

Since \eqref{eq:proof:thm:estimationerror:3:5} is a (PO) problem, we can associate to it the following (AO) problem
\begin{align}
\label{eq:proof:thm:estimationerror:3:6}
\begin{split}
    \min_{\substack{w \in \mathcal W \\ s \in \mathbb R^n}}\; \max_{u \in \mathcal S_u}\;   \frac{1}{\sqrt{n}} \|w\|g^\top u - \frac{1}{\sqrt{n}} \|u\|h^\top w +\frac{1}{\sqrt{n}} u^\top z + \frac{1}{4\lambda}\|\theta_0 + \sqrt{d}w &\|^2\|u\|^2 -\frac{1}{2M}\|u\|^2- \\& -\frac{1}{\sqrt{n}}s^\top u + \frac{1}{n} F(s),
\end{split}
\end{align}
which is tightly related to the (PO) problem in the high-dimensional regime, as explained in Section~\ref{sec:cgmt}. At this point, following the discussion presented in Section~\ref{sec:cgmt}, we consider the following modified (AO) problem,
\begin{align}
\label{eq:proof:thm:estimationerror:3:7}
\begin{split}
    \max_{0 \leq \beta \leq K_\beta}\; \min_{\substack{w \in \mathcal W\\ s \in \mathbb R^n}}\; \max_{\|u\|= \beta}\;  \frac{1}{\sqrt{n}} \|w\|g^\top u - \frac{1}{\sqrt{n}} \|u\|h^\top w +\frac{1}{\sqrt{n}} u^\top z + \frac{1}{4\lambda}\|\theta_0 + \sqrt{d}w\|^2\|u\|^2 -\frac{1}{2M}\|u\|^2- \\ -\frac{1}{\sqrt{n}}s^\top u + \frac{1}{n} F(s),
\end{split}
\end{align}
where we have separated the maximization over $u \in \mathcal S_u = \{u \in \mathbb R^n:\, \|u\|\leq K_\beta\}$ into the maximization over the magnitude of $u$, i.e., $0 \leq \beta \leq K_\beta$, and the maximization over $\|u\| = \beta$. Recall that, since the (AO) problem is not convex-concave (due to the random vectors $g$ and $h$), the exchange of the minimization over $(w,s)$ and the maximization over $\beta$ is not justified. Therefore the two problems \eqref{eq:proof:thm:estimationerror:3:6} and \eqref{eq:proof:thm:estimationerror:3:7} might not be equivalent. However, as seen in Fact~\ref{prop:cgmt}, the modified (AO) problem \eqref{eq:proof:thm:estimationerror:3:7} can be used to study the optimal solution of the (PO) problem \eqref{eq:proof:thm:estimationerror:3:5} in the asymptotic regime, based solely on the optimal value of the modified (AO) problem. We will do so after simplifying the problem, by reducing it to a scalar one, as shown in what follows.

%------------------------------------------------------------------------------------
\paragraph{Step 3: We consider the modified (AO) problem \eqref{eq:proof:thm:estimationerror:3:7}, and show that it can be reduced to a scalar problem, which involves only three scalar variables.}

We start by scalarizing problem \eqref{eq:proof:thm:estimationerror:3:7} over the magnitude of $u$, leading to the following problem
\begin{align}
\label{eq:proof:thm:estimationerror:3:8}
    \max_{0 \leq \beta \leq K_\beta}\; \min_{\substack{w \in \mathcal W\\ s \in \mathbb R^n}}\;   \frac{\beta}{\sqrt{n}}\| \|w\|g + z - s \| - \frac{\beta}{\sqrt{n}} h^\top w + \frac{\beta^2}{4 d \lambda_0}\|\theta_0 + \sqrt{d}w\|^2 -\frac{\beta^2}{2M} + \frac{1}{n} F(s),
\end{align}
where we have also used the fact that $\lambda = d \lambda_0$.

We now employ the square-root trick to rewrite the first term in the objective function of \eqref{eq:proof:thm:estimationerror:3:8} as follows
\begin{align*}
    \frac{1}{\sqrt{n}}\|\|w\|g + z - s \| = \inf_{\tau_1>0} \frac{\tau_1}{2} + \frac{1}{2n\tau_1}\|\|w\|g + z - s \|^2.
\end{align*}

Introducing this in \eqref{eq:proof:thm:estimationerror:3:8} and re-organizing the terms gives
\begin{align*}
    \max_{0 \leq \beta \leq K_\beta}\; \min_{\substack{w \in \mathcal W\\ s \in \mathbb R^n \\ \tau_1>0}}\;   \frac{\beta \tau_1}{2} -\frac{\beta^2}{2M} + \frac{\beta}{2n\tau_1}\| \|w\|g + z - s \|^2  + \frac{1}{n} F(s) - \frac{\beta}{\sqrt{n}} h^\top w + \frac{\beta^2}{4 d \lambda_0}\|\theta_0 + \sqrt{d}w\|^2,
\end{align*}
which allows us to introduce the Moreau envelope of $F$, i.e.,
\begin{align}
\label{eq:proof:thm:estimationerror:3:9}
    e_{F}(\alpha g + z,\frac{\tau_1}{\beta}) := \min_{s \in \mathbb R^n}\; \frac{\beta}{2\tau_1}\|\alpha g + z - s \|^2 + F(s),
\end{align}
giving rise to
\begin{align}
\label{eq:proof:thm:estimationerror:3:10}
    \max_{0 \leq \beta \leq K_\beta}\; \min_{\substack{w \in \mathcal W \\ \tau_1>0}}\;   \frac{\beta \tau_1}{2}  -\frac{\beta^2}{2M} + \frac{1}{n} e_{F}(\|w\| g + z,\frac{\tau_1}{\beta}) - \frac{\beta}{\sqrt{n}} h^\top w + \frac{\beta^2}{4 d \lambda_0}\|\theta_0 + \sqrt{d}w\|^2.
\end{align}

We now proceed with the aim of scalarizing the problem over the magnitude of $w$ (similarly to what we have done for $u$). In order to scalarize the problem over the magnitude of $w$,  instead of the constraint set $\mathcal W$, we would need to work with a set of the form $\mathcal S_w := \{w \in \mathbb R^d: \, \|w\| \leq K_\alpha\}$. Notice that if we choose $K_\alpha := R_\theta + \sigma_{\theta_0}$, then we have $\mathcal W \subset \mathcal S_w$ w.p.a.\ $1$ as $d \to \infty$. Therefore, from now on we impose such a constraint on $w$ (i.e., $w \in \mathcal S_w$), and we postpone the discussion about the implications that this constraint has on the error $\|\hat{\theta}_{\mathrm{DRE}}-\theta_0\|$ to the end of the proof.

Similarly to what was done for $u$, we can now separate the constraint $w \in \mathcal S_w$ into the two constraints $0 \leq \alpha \leq K_\alpha$ and $\|w\| = \alpha$, and by expressing $\|\theta_0 + \sqrt{d}w\|^2$ as $\|\theta_0\|^2 + d\|w\|^2 + 2\sqrt{d}\theta_0^\top w$, we can now scalarize the problem over the magnitude of $w$ and obtain
\begin{align}
\label{eq:proof:thm:estimationerror:3:11}
    \max_{0 \leq \beta \leq K_\beta}\; \inf_{\substack{0 \leq \alpha \leq K_\alpha \\ \tau_1>0}}\; \frac{\beta \tau_1}{2} -\frac{\beta^2}{2M} + \frac{1}{n} e_{F}(\alpha g + z,\frac{\tau_1}{\beta}) - \alpha \beta\left\|\frac{\beta}{2 \lambda_0} \frac{\theta_0}{\sqrt{d}} - \frac{h}{\sqrt{n}} \right\| + \frac{\beta^2}{4 \lambda_0}\left(\frac{\|\theta_0\|^2}{d} + \alpha^2\right).
\end{align}

We denote by $\mathcal O_{d}(\alpha,\tau_1,\beta)$ the objective function in problem \eqref{eq:proof:thm:estimationerror:3:11}, where the index $d$ recalls that $\mathcal O_{d}$ is parametrized by the dimension $d$.

%------------------------------------------------------------------------------------
\paragraph{Step 4: We show that $\mathcal O_{d}$ is continuous on its domain, jointly convex in $(\alpha,\tau_1)$, and concave in $\beta$.} 

Let's first concentrate on the continuity property. Notice that we only need to show that the Moreau envelope $e_{F}(\alpha g + z,\tau_1/\beta)$ is continuous, since all the other terms in the objective function are trivially continuous. Since $F$ is lower semicontinuous and convex (recall that $F(u) = \sum_{i=1}^n f(u_i)$), the continuity of the Moreau envelope $e_{F}(\alpha g + z,\tau_1/\beta)$ follows from \cite[Theorem~2.26(b)]{rockafellar2009variational}.

Let's now focus on the joint convexity in $(\alpha,\tau_1)$. Notice that the last two terms in \eqref{eq:proof:thm:estimationerror:3:11} are linear and quadratic in $\alpha$, and therefore trivially convex. Moreover, if we prove that the objective function on the right-hand side of \eqref{eq:proof:thm:estimationerror:3:9} is jointly convex in $(\alpha,\tau_1,s)$, then, after minimizing over $s \in \mathbb R^n$, the Moreau envelope $e_{F}(\alpha g + z,\tau_1/\beta)$ remains jointly convex in $(\alpha,\tau_1)$. But this is certainly the case since $\|\alpha g - s \|^2$ is jointly convex in $(\alpha,s)$, $1/\tau_1\|\alpha g - s \|^2$ is the perspective function of $\|\alpha g - s \|^2$, and therefore jointly convex in $(\alpha,\tau_1,s)$, and the shifted function $1/\tau_1\|\alpha g + z - s \|^2$ remains jointly convex. 

Finally, we will show that $\mathcal O_{d}(\alpha,\tau_1,\beta)$ is concave in $\beta$. Notice first that the objective function in \eqref{eq:proof:thm:estimationerror:3:10} is concave in $\beta$. Indeed, this follows from the concavity of the Moreau envelope $e_{F}(\alpha g + z,\tau_1/\beta)$ (since it is obtained from the minimization of affine functions in $\beta$), and the concavity of
\begin{align*}
    -\frac{\beta^2}{2M} + \frac{\beta^2}{4 d \lambda_0}\|\theta_0 + \sqrt{d}w\|^2,
\end{align*}
(due to the assumption that $\lambda_0 > M R_{\theta}^2/2$). Now, since $\mathcal O_{d}(\alpha,\tau_1,\beta)$ is the result of a minimization over these concave functions, it is concave in $\beta$.

As a consequence of the continuity and convexity-concavity, we can apply Sion's minimax principle to exchange the minimization and the maximization in \eqref{eq:proof:thm:estimationerror:3:11} to obtain 
\begin{align}
\label{eq:proof:thm:estimationerror:3:12}
    \inf_{\substack{0 \leq \alpha \leq K_\alpha \\ \tau_1>0}}\;\max_{0 \leq \beta \leq K_\beta}\; \frac{\beta \tau_1}{2} -\frac{\beta^2}{2M} + \frac{1}{n} e_{F}(\alpha g + z,\frac{\tau_1}{\beta}) - \alpha \beta\left\|\frac{\beta}{2 \lambda_0} \frac{\theta_0}{\sqrt{d}} - \frac{h}{\sqrt{n}} \right\| + \frac{\beta^2}{4 \lambda_0}\left(\frac{\|\theta_0\|^2}{d} + \alpha^2\right).
\end{align}

%------------------------------------------------------------------------------------
\paragraph{Step 5: We study the convergence in probability of $\mathcal O_{d}$.}

We have now arrived at the scalar formulation \eqref{eq:proof:thm:estimationerror:3:12}, which has the same optimal value as the modified (AO) problem \eqref{eq:proof:thm:estimationerror:3:7}. Therefore, following the reasoning presented in Fact~\ref{prop:cgmt}, the next step is to study the convergence in probability of its optimal value. To do so, we start by studying the convergence in probability of its objective function $\mathcal O_{d}$.

Notice first that the last two terms of $\mathcal O_{d}$ converge in probability as follows.
\begin{align}
    \label{eq:proof:thm:estimationerror:3:13}    
    \alpha \beta\left\|\frac{\beta}{2 \lambda_0} \frac{\theta_0}{\sqrt{d}} - \frac{h}{\sqrt{n}} \right\| &\overset{P}{\to} \alpha \beta \sqrt{\frac{\beta^2 \sigma_{\theta_0}^2}{4 \lambda_0^2} + \rho}, \\
    \label{eq:proof:thm:estimationerror:3:14} 
    \frac{\beta^2}{4 \lambda_0}\left(\frac{\|\theta_0\|^2}{d} + \alpha^2\right) &\overset{P}{\to} \frac{\beta^2}{4 \lambda_0} \left( \sigma_{\theta_0}^2 + \alpha^2 \right),
\end{align}
where \eqref{eq:proof:thm:estimationerror:3:13} can be recovered using
\begin{align*}
    \frac{\theta_0^\top h}{d} \overset{P}{\to} 0,
\end{align*}
(see Lemma~\ref{lemma:triangular:array}) and the continuous mapping theorem (which states that continuous functions preserve limits in probability, when their arguments are sequences of random variables). 

We will now study the convergence in probability of the normalized Moreau envelope $e_{F}(\alpha g + z,{\tau_1}/{\beta})/n$. Since $F(u):=\sum_{i=1}^{n}f(u_i)$, \eqref{eq:proof:thm:estimationerror:3:9} reduces to
\begin{align*}
    \frac{1}{n} e_{F}(\alpha g + z,\frac{\tau_1}{\beta}) = \frac{1}{n} \sum_{i=1}^{n} \;\min_{s_i \in \mathbb R}\; \frac{\beta}{2 \tau_1} (\alpha G + Z - s_i)^2 + f(s_i) = \frac{1}{n} \sum_{i=1}^{n} \; e_{f}(\alpha G + Z;\frac{\tau_1}{\beta}),
\end{align*}
and its limit (in probability) can be immediately recovered from the weak law of large numbers as
\begin{align}
    \label{eq:proof:thm:estimationerror:3:15}
    \frac{1}{n}  e_{F}(\alpha g + z,\frac{\tau_1}{\beta}) \overset{P}{\to} \mathbb E_{\mathcal N(0,1) \otimes \mathbb P_Z}\left[ e_{f}(\alpha G + Z;\frac{\tau_1}{\beta})\right],
\end{align}
which is precisely the expected Moreau envelope $\mathcal F(\alpha,\tau_1/\beta)$ defined in \eqref{eq:L_*}. The quantity under the expectation in \eqref{eq:proof:thm:estimationerror:3:15} is absolutely integrable, and therefore the expected Moreau envelope $\mathcal F$ is well defined, as explained in what follows. For every $\beta \geq 0$ and $\tau_1>0$ we have
\begin{align*}
    \left|e_{f}(\alpha G + Z;\frac{\tau_1}{\beta})\right| = \min_{v \in \mathbb R}\; \frac{\beta}{2 \tau_1} \left( \alpha G + Z - v \right)^2 + f(v) &\leq \frac{\beta}{2 \tau_1} \left( \alpha G + Z \right)^2 + f(0) \\ &= \frac{\beta}{2 \tau_1} (\alpha G + Z)^2
\end{align*}
which is integrable due to the fact that both $G$ and $Z$ have finite second moment. In the first and last equality we have used the fact that $f(0) = \min_{v \in \mathbb R} f(v) = 0$, which can be shown to hold true as follows. Since $L(0) = \min_{v \in \mathbb R} L(v) = 0$ (from Assumption~\ref{assump:cgmt}\ref{assump:ell:0}), we have that $L^*(0) = \min_{v \in \mathbb R} L^*(v) = 0$, and from $f^*(\cdot) = L^*(\cdot) - 1/(2M)(\cdot)^2$, with $L^*$ that is $1/M$-strongly convex, we have that $f^*(0) = \min_{v \in \mathbb R} f^*(v) = 0$, which results in the desired $f(0) = \min_{v \in \mathbb R} f(v) = 0$. Moreover, for $\tau_1 \to 0^+$,  
\begin{align*}
    \lim_{\tau_1 \to 0^+}\; \left|e_{f}(\alpha G + Z;\frac{\tau_1}{\beta})\right| = \lim_{\tau_1 \to 0^+}\; \min_{v \in \mathbb R}\; \frac{\beta}{2 \tau_1} \left( \alpha G + Z - v \right)^2 + f(v) =  f(\alpha G + Z),
\end{align*}
whose expectation is finite by assumption. In particular, the second equality follows from \cite[Theorem~1.25]{rockafellar2009variational}.

Therefore, from \eqref{eq:proof:thm:estimationerror:3:13}-\eqref{eq:proof:thm:estimationerror:3:15}, we have that $\mathcal O_{d}(\alpha,\tau_1,\beta)$ converges in probability to the function
\begin{align}
\label{eq:proof:thm:estimationerror:3:16:0}
    \mathcal O(\alpha,\tau_1,\beta) := \frac{\beta \tau_1}{2} - \frac{\beta^2}{2M} + \mathcal F(\alpha,\tau_1/\beta) - \alpha \beta \sqrt{\frac{\beta^2 \sigma_{\theta_0}^2}{4 \lambda_0^2} + \rho} + \frac{\beta^2}{4 \lambda_0} \left( \sigma_{\theta_0}^2 + \alpha^2 \right).
\end{align}

Notice that $\mathcal O$ is precisely the objective function of the minimax problem \eqref{eq:est:3:minimax} from the statement of Theorem~\ref{thm:estimationerror:3}. Since $\mathcal O$ is the pointwise limit (in probability, for each ($\alpha, \tau_1, \beta$)) of the sequence of objective functions $\{\mathcal O_{d}\}_{d \in \mathbb N}$, which are convex-concave, and convexity is preserved by pointwise limits (see Lemma~\ref{lemma:convexity:lemma}), we have that $\mathcal O$ is jointly convex in $(\alpha,\tau_1)$ and concave in $\beta$. 

Consider the asymptotic minimax optimization problem
\begin{align}
\label{eq:proof:thm:estimationerror:3:16}
    \inf_{\substack{0 \leq \alpha \leq K_\alpha \\ \tau_1>0}}\; \max_{0 \leq \beta \leq K_\beta}\; \mathcal O(\alpha, \tau_1, \beta).
\end{align}

%------------------------------------------------------------------------------------
\paragraph{Step 6: We show that the asymptotic problem \eqref{eq:proof:thm:estimationerror:3:16} has an unique minimizer $\alpha_{\star}$.}

As explained in Section~\ref{sec:cgmt}, the uniqueness of the optimal solution $\alpha_{\star}$ is fundamental in the convergence analysis required by Fact~\ref{prop:cgmt}. In order to prove this, we will show that the function $\max_{0 \leq \beta \leq K_\beta} \inf_{\tau_1>0} \mathcal O(\alpha, \tau_1, \beta)$ is strictly convex in $\alpha$ (this function is equal to $\inf_{\tau_1>0} \max_{0 \leq \beta \leq K_\beta} \mathcal O(\alpha, \tau_1, \beta)$ by Sion's minimax theorem). First, notice that $\mathcal O(\alpha, \tau_1, \beta)$ can be written as the sum of a function which is jointly convex in $(\alpha,\tau_1)$ and a function which is strictly convex in $\alpha$. Therefore, after minimizing over $\tau_1$, we have that $\mathcal O^{\tau_1}(\alpha, \beta) := \inf_{\tau_1 > 0} \mathcal O(\alpha, \tau_1, \beta)$ is the sum of a convex function in $\alpha$ (since partial minimization of jointly convex functions is convex) and the (same as before) strictly convex function in $\alpha$. Now, from the strict convexity of $\mathcal O^{\tau_1}$ in $\alpha$ it follows that for any $\lambda \in (0,1)$, $\alpha^{(1)} \neq \alpha^{(2)}$, and $\beta$ we have that
\begin{align*}
    \mathcal O^{\tau_1}(\lambda\alpha^{(1)} + (1-\lambda)\alpha^{(2)},\beta) &< \lambda\mathcal O^{\tau_1}(\alpha^{(1)},\beta) + (1-\lambda) \mathcal O^{\tau_1}(\alpha^{(2)},\beta)\\ &\leq \lambda \max_{0 \leq \beta \leq K_\beta} \;\mathcal O^{\tau_1}(\alpha^{(1)}) + (1-\lambda)\max_{0 \leq \beta \leq K_\beta} \; \mathcal O^{\tau_1} (\alpha^{(2)}).
\end{align*}

We can now take the maximimum over $\beta \in [0,K_\beta]$ on the left-hand side, and since this will be attained (due to the upper-semicontinuity in $\beta$ of $\mathcal O^{\tau_1}$, and the compactness of $[0,K_\beta]$), we can conclude that
\begin{align*}
    \max_{0 \leq \beta \leq K_\beta}\; \mathcal O^{\tau_1}(\lambda\alpha^{(1)} + (1-\lambda)\alpha^{(2)},\beta) < \lambda \max_{0 \leq \beta \leq K_\beta} \;\mathcal O^{\tau_1}(\alpha^{(1)}) + (1-\lambda)\max_{0 \leq \beta \leq K_\beta} \; \mathcal O^{\tau_1} (\alpha^{(2)}).
\end{align*}

This shows that $\max_{0 \leq \beta \leq K_\beta} \inf_{\tau_1 > 0} \mathcal O(\alpha, \tau_1, \beta)$ is strictly convex in $\alpha$, and therefore the optimal solution $\alpha_\star$ is unique.

%------------------------------------------------------------------------------------
\paragraph{Step 7: We anticipate how the uniqueness of $\alpha_{\star}$, together with Fact~\ref{prop:cgmt}, can be used to conclude the proof.}

The uniqueness of $\alpha_{\star}$ is a key element in the convergence analysis required by Fact~\ref{prop:cgmt}, as explained in what follows. We first define, for arbitrary $\eta>0$, the sets $\mathcal S_\eta := \{\alpha \in [0,K_\alpha]:\, |\alpha-\alpha_{\star}|<\eta\}$, with $\alpha_{\star}$ the unique solution of \eqref{eq:proof:thm:estimationerror:3:16} and $K_\alpha := R_\theta + \sigma_{\theta_0}$, and $\mathcal S_\eta^c := [0,K_\alpha]\setminus \mathcal S_\eta$. Since $\alpha_{\star}$ is the unique solution of \eqref{eq:proof:thm:estimationerror:3:16}, we have that
\begin{align}
\label{eq:proof:thm:estimationerror:3:17}
    \inf_{\substack{0 \leq \alpha \leq K_\alpha\\ \tau_1>0}}\; \max_{0 \leq \beta \leq K_\beta}\; \mathcal O(\alpha, \tau_1, \beta) < \inf_{\substack{\alpha \in \mathcal S_\eta^c \\ \tau_1>0}}\; \max_{0 \leq \beta \leq K_\beta}\; \mathcal O(\alpha, \tau_1, \beta)
\end{align}

Now, due to \eqref{eq:proof:thm:estimationerror:3:17}, if we prove that the optimal value of the (scalarized version of the) modified (AO) problem \eqref{eq:proof:thm:estimationerror:3:7} satisfies
\begin{align}
\label{eq:proof:thm:estimationerror:3:18}
    \inf_{\substack{0 \leq \alpha \leq K_\alpha \\ \tau_1>0}}\; \max_{0 \leq \beta \leq K_\beta}\; \mathcal O_{d}(\alpha, \tau_1, \beta) \overset{P}{\to} \inf_{\substack{0 \leq \alpha \leq K_\alpha \\ \tau_1>0}}\; \max_{0 \leq \beta \leq K_\beta}\; \mathcal O(\alpha, \tau_1, \beta),
\end{align}
and that, when additionally restricted to $\alpha \in \mathcal S_\eta^c$, for arbitrary $\eta>0$, it satisfies
\begin{align}
\label{eq:proof:thm:estimationerror:3:19}
    \inf_{\substack{\alpha \in \mathcal S_\eta^c \\ \tau_1>0}}\; \max_{0 \leq \beta \leq K_\beta}\; \mathcal O_{d}(\alpha, \tau_1, \beta) \overset{P}{\to} \inf_{\substack{\alpha \in \mathcal S_\eta^c \\ \tau_1>0}}\; \max_{0 \leq \beta \leq K_\beta}\; \mathcal O(\alpha, \tau_1, \beta),
\end{align}
we can directly conclude the desired result \eqref{eq:est:3:error} from Fact~\ref{prop:cgmt}. Therefore, in order to finish the proof, we only need to prove the two convergences in probability \eqref{eq:proof:thm:estimationerror:3:18} and \eqref{eq:proof:thm:estimationerror:3:19}. It will be easier to prove the two convergences in probability in the following equivalent form (by Sion's minimax theorem)
\begin{align*}
    \min_{\alpha}\; \max_{0 \leq \beta \leq K_\beta}\;\inf_{\tau_1>0}\; \mathcal O_{d}(\alpha, \tau_1, \beta) \overset{P}{\to} \min_{\alpha}\; \max_{0 \leq \beta \leq K_\beta}\;\inf_{\tau_1>0}\; \mathcal O(\alpha, \tau_1, \beta).
\end{align*}

Moreover, since for $\beta = 0$ we have $\mathcal O_d \equiv 0$, for all $d \in \mathbb N$, and $\mathcal O \equiv 0$ (this follows easily from $f(0) = \min_{v \in \mathbb R} f(v) = 0$), in the remaining of the proof we restrict our attention, without loss of generality, to the nontrivial case $\beta>0$.

The two convergences in probability \eqref{eq:proof:thm:estimationerror:3:18} and \eqref{eq:proof:thm:estimationerror:3:19} are a consequence of Lemma~\ref{lemma:conv:opt}, as explained in what follows.

%------------------------------------------------------------------------------------
\paragraph{Step 8: We prove the two convergences in probability \eqref{eq:proof:thm:estimationerror:3:18} and \eqref{eq:proof:thm:estimationerror:3:19}, and conclude the proof.}

First, for fixed $(\alpha,\beta)$, $\{\mathcal O_{d}(\alpha,\cdot,\beta)\}_{d \in \mathbb N}$ is a sequence of random real-valued convex functions, converging in probability (pointwise, for every $\tau_1>0$) to the function $\mathcal O(\alpha,\cdot,\beta)$. Moreover, 
\begin{align*}
    \lim_{\tau_1 \to \infty}\; \mathcal O(\alpha,\tau_1,\beta) = \lim_{\tau_1 \to \infty}\; \frac{\beta \tau_1}{2} + \mathcal F(\alpha,\tau_1/\beta) + \text{const.} \geq \lim_{\tau_1 \to \infty}\; \frac{\beta \tau_1}{2},
\end{align*}
which is equal to $+\infty$. In particular, the inequality follows from the fact that $\mathcal F(\alpha,\tau_1/\beta) \geq 0$.

As a consequence, we can apply statement (iii) of Lemma~\ref{lemma:conv:opt} to conclude that
\begin{align*}
    \inf_{\tau_1>0}\; \mathcal O_{d}(\alpha, \tau_1, \beta) \overset{P}{\to} \inf_{\tau_1>0}\; \mathcal O(\alpha, \tau_1, \beta).
\end{align*}

We now define $\mathcal O^{\tau_1}(\alpha, \beta):= \inf_{\tau_1>0}\; \mathcal O(\alpha, \tau_1, \beta)$ and $\mathcal O_{d}^{\tau_1}(\alpha, \beta):= \inf_{\tau_1>0}\; \mathcal O_{d}(\alpha, \tau_1, \beta)$. For fixed $\alpha$, $\{\mathcal O_{d}^{\tau_1}(\alpha, \cdot)\}_{d \in \mathbb N}$ is a sequence of random, real-valued concave functions (since they are defined as a minimization of concave functions), converging in probability (pointwise, for every $\beta \geq 0$) to the function $\mathcal O{\tau_1}(\alpha, \cdot)$. Therefore, we can apply statement (ii) of Lemma~\ref{lemma:conv:opt} to conclude that
\begin{align}
\label{eq:proof:thm:estimationerror:3:alpha:0}
    \sup_{0 < \beta \leq K_\beta}\; \mathcal O_{d}^{\tau_1}(\alpha, \beta) \overset{P}{\to} \sup_{0 < \beta \leq K_\beta}\; \mathcal O^{\tau_1}(\alpha, \beta).
\end{align}

We will now show that the constraint $0 < \beta \leq K_\beta$ can be relaxed to $\beta > 0$, without loss of generality. Recall that $K_\beta$ was inserted for convenience, in equation \eqref{eq:proof:thm:estimationerror:3:1}, to be able to use the CGMT (which required $u$ to live in a compact set). Now, notice that
\begin{align*}
    \lim_{\beta \to \infty}\; \mathcal O^{\tau_1}(\alpha, \beta) &\leq \lim_{\beta \to \infty}\; \mathcal O^{\bar{\tau}_1}(\alpha, \beta) \\ &= \lim_{\beta \to \infty}\; \frac{\beta \bar{\tau}_1}{2} - \frac{\beta^2}{2M} + \mathcal F(\alpha,\bar{\tau}_1/\beta) - \alpha \beta \sqrt{\frac{\beta^2 \sigma_{\theta_0}^2}{4 \lambda_0^2} + \rho} + \frac{\beta^2}{4 \lambda_0} \left( \sigma_{\theta_0}^2 + \alpha^2 \right) \\ &= \lim_{\beta \to \infty}\; - \frac{\beta^2}{2M} - \frac{\alpha\sigma_{\theta_0}\beta^2 }{2 \lambda_0} + \frac{\beta^2}{4 \lambda_0} \left( \sigma_{\theta_0}^2 + \alpha^2 \right) + O(\beta) \\ &= \lim_{\beta \to \infty}\; \beta^2 \left(- \frac{1}{2M} + \frac{(\sigma_{\theta_0} - \alpha)^2}{4 \lambda_0}\right) + O(\beta) = -\infty
\end{align*}
for some constant $\bar{\tau}_1$ (let's say, for example, equal to $1$). Here, $O(\beta)$ encapsulates the remaining terms, which grow (at most) linearly in $\beta$. The inequality above follows from the fact that $\mathcal O^{\tau_1}$ is defined as a minimization over $\tau_1>0$, and the last equality follows from the fact that
\begin{align*}
    -\frac{1}{2M} + \frac{(\sigma_{\theta_0} - \alpha)^2}{4 \lambda_0} < 0,
\end{align*}
since $\lambda_0 > M R_\theta^2/2$, and $\alpha \leq K_\alpha = R_\theta + \sigma_{\theta_0}$. The assumption $\mathbb E_{\mathcal N(0,1) \otimes \mathbb P_Z}\left[ f(\alpha G + Z)\right] < \infty$ is very important here, since it guarantees that
\begin{align*}
    \lim_{\beta \to \infty}\; \mathcal F(\alpha,\frac{\bar\tau_1}{\beta}) = \lim_{\beta \to \infty}\; \mathbb E_{\mathcal N(0,1) \otimes \mathbb P_Z}\left[ e_{f}(\alpha G + Z;\frac{\bar\tau_1}{\beta})\right] &= \mathbb E_{\mathcal N(0,1) \otimes \mathbb P_Z}\left[\lim_{\beta \to \infty}\; e_{f}(\alpha G + Z;\frac{\bar\tau_1}{\beta})\right] \\ &= \mathbb E_{\mathcal N(0,1) \otimes \mathbb P_Z}\left[f(\alpha G + Z)\right] < \infty.
\end{align*}
Here, the second equality follows from the monotone convergence theorem, and the last equality follows from \cite[Theorem~1.25]{rockafellar2009variational}.

Therefore, there exists some sufficiently large $K_{\beta}$ such that
\begin{align*}
    \sup_{\beta > 0}\; \mathcal O^{\tau_1}(\alpha, \beta) = \sup_{0 < \beta \leq K_{\beta}}\; \mathcal O^{\tau_1}(\alpha, \beta),
\end{align*}
%The same reasoning can be used for $\mathcal O_d^{\tau_1}(\alpha, \beta)$, leading to the existence of some sufficiently large $K_{\beta,2}$ for which
%\begin{align*}
%    \sup_{\beta > 0}\; \mathcal O_{d}^{\tau_1}(\alpha, \beta) = \sup_{0 < \beta \leq K_{\beta,2}}\; \mathcal O_{d}^{\tau_1}(\alpha, \beta).
%\end{align*}
%In particular, $K_{\beta,2}$ can be chosen uniformly over $d$
%Therefore, by considering $K_\beta = \max\{K_{\beta,1},K_{\beta,2}\}$, we can replace (without loss of generality) the constraint $0 < \beta \leq K_\beta$ on the left and right-hand side in \eqref{eq:proof:thm:estimationerror:3:alpha:0} with $\beta > 0$, 
leading to
\begin{align}
\label{eq:proof:thm:estimationerror:3:alpha:0:0}
    \sup_{0 < \beta \leq K_\beta}\; \mathcal O_{d}^{\tau_1}(\alpha, \beta) \overset{P}{\to} \sup_{\beta > 0}\; \mathcal O^{\tau_1}(\alpha, \beta).
\end{align}

Notice that the convergence \eqref{eq:proof:thm:estimationerror:3:alpha:0:0} can also be proven using statement (iii) of Lemma~\ref{lemma:conv:opt}, since $\lim_{\beta \to \infty}\; \mathcal O^{\tau_1}(\alpha, \beta) = -\infty$.

We now define $\mathcal O^{\beta,\tau_1}(\alpha):= \sup_{\beta > 0}\; \mathcal O^{\tau_1}(\alpha, \beta)$ and $\mathcal O_{d}^{\beta,\tau_1}(\alpha):= \sup_{0 < \beta \leq K_\beta}\; \mathcal O_{d}^{\tau_1}(\alpha, \beta)$. Each of these functions is convex in $\alpha$, since they were obtained by first minimizing over $\tau_1$ a jointly convex function in $(\alpha,\tau_1)$, and then maximizing over $\beta$ a convex function in $\alpha$. 

For the final step of the proof, we will distinguish between the two cases $\alpha_{\star}=0$ and $\alpha_{\star}>0$. This is due to the Convexity Lemma~\ref{lemma:convexity:lemma}, which stands at the core of the convergence analysis, and which requires the domain of the functions of interest to be open. We first consider the case $\alpha_{\star}=0$. Since \eqref{eq:proof:thm:estimationerror:3:alpha:0} is valid for any $\alpha \geq 0$, it is in particular valid for $\alpha = 0$, and thus we have that
\begin{align}
\label{eq:proof:thm:estimationerror:3:alpha:1}
    \mathcal O_{d}^{\beta,\tau_1}(0) \overset{P}{\to} \mathcal O^{\beta,\tau_1}(0).
\end{align}
Moreover, using statement (i) of Lemma~\ref{lemma:conv:opt}, for any $K>0$, we have that 
\begin{align}
\label{eq:proof:thm:estimationerror:3:alpha:2}
    \min_{K \leq\alpha\leq K_\alpha}\; \mathcal O_{d}^{\beta,\tau_1}(\alpha) \overset{P}{\to} \min_{K \leq \alpha\leq K_\alpha}\; \mathcal O^{\beta,\tau_1}(\alpha).
\end{align}
The result for $\alpha_{\star}=0$ now follows from \eqref{eq:proof:thm:estimationerror:3:alpha:1} and \eqref{eq:proof:thm:estimationerror:3:alpha:2} and the fact that 
\begin{align*}
    \mathcal O^{\beta,\tau_1}(0) < \min_{K \leq \alpha\leq K_\alpha}\; \mathcal O^{\beta,\tau_1}(\alpha),
\end{align*}
for any $K > 0$, due to the uniqueness of $\alpha_{\star}$.

We now focus on the case $\alpha_{\star} > 0$. We need to prove \eqref{eq:proof:thm:estimationerror:3:18} and \eqref{eq:proof:thm:estimationerror:3:19}, which are equivalent to the following two convergences in probability
\begin{align}
\label{eq:proof:thm:estimationerror:3:20}
    \inf_{0<\alpha\leq K_\alpha}\; \mathcal O_{d}^{\beta,\tau_1}(\alpha) \overset{P}{\to} \inf_{0<\alpha\leq K_\alpha}\; \mathcal O^{\beta,\tau_1}(\alpha),
\end{align}
and
\begin{align}
\label{eq:proof:thm:estimationerror:3:21}
    \inf_{\alpha \in \mathcal S_\eta^c}\; \mathcal O_{d}^{\beta,\tau_1}(\alpha) \overset{P}{\to} \inf_{\alpha \in \mathcal S_\eta^c}\; \mathcal O^{\beta,\tau_1}(\alpha),
\end{align}
respectively. Since $\{\mathcal O_{d}^{\beta,\tau_1}(\cdot)\}_{d \in \mathbb N}$ is a sequence of random, real-valued convex functions, converging in probability (pointwise, for every $\alpha > 0$) to the function $\mathcal O^{\beta,\tau_1}(\cdot)$, \eqref{eq:proof:thm:estimationerror:3:20} follows immediately from  statement (ii) of Lemma~\ref{lemma:conv:opt}. Moreover, since $\mathcal S_\eta^c = (0,\alpha_{\star}-\eta]\cup[\alpha_{\star}+\eta,K_\alpha]$, \eqref{eq:proof:thm:estimationerror:3:21} follows from statement (i) of Lemma~\ref{lemma:conv:opt}, for the interval $[\alpha_{\star}+\eta,K_\alpha]$, and statement (ii) of the same lemma, for the interval $(0,\alpha_{\star}-\eta]$. This concludes the proof of the two convergences in probability  \eqref{eq:proof:thm:estimationerror:3:18} and \eqref{eq:proof:thm:estimationerror:3:19}.

The last step of the proof is to understand the implications of relaxing the constraint $w \in \mathcal W$ to $w \in \mathcal S_w:=\{w \in \mathbb R^d:\;\|w\|\leq K_\alpha:= R_\theta + \sigma_{\theta_0}\}$. Recall that this was a necessary step for scalarizing the problem over the magnitude of $w$. Since $0 \in \mathcal W$ and $\mathcal W \subset \mathcal S_w$ w.p.a.\ $1$, we have that 
\begin{align*}
    \lim_{d \to \infty} \frac{\|\hat{\theta}_{\mathrm{DRE}}-\theta_0\|^2}{d} \leq \alpha_\star,
\end{align*}
where $\alpha_\star$ is the optimal solution of the relaxed problem (i.e., with constraint $w \in \mathcal S_w$). However, since $\{w \in \mathbb R^d:\;\|w\|\leq R_\theta - \sigma_{\theta_0}\} \subset \mathcal W \cap \mathcal S_w$, we have that if $\alpha_\star \leq R_\theta - \sigma_{\theta_0}$, then
\begin{align*}
    \lim_{d \to \infty} \frac{\|\hat{\theta}_{\mathrm{DRE}}-\theta_0\|^2}{d} = \alpha_\star.
\end{align*}
This concludes the proof.
\end{proof}

\begin{proof}[Proof of Corollary~\ref{cor:regdro:squared:loss}]
The proof can be easily obtained following the same lines as Theorem~\ref{thm:estimationerror:3}, and therefore we only provide a sketch, highlighting only the main differences from the proof of Theorem~\ref{thm:estimationerror:3}.

After expressing the dual formulation \eqref{eq:dro:dual:uni:3} in vector form, introducing the change of variable $w=(\theta-\theta_0)/\sqrt{d}$, and motivating the restriction of the variable $u$ to the compact set $\mathcal S_u$, we arrive at the minimax formulation \eqref{eq:proof:thm:estimationerror:3:1}, 
\begin{align*}
    \min_{w \in \mathcal S_w}\; \max_{u \in \mathcal S_u} \; - \frac{1}{\sqrt{n}} u^\top (\sqrt{d}A)w +\frac{1}{\sqrt{n}} u^\top z + \frac{1}{4\lambda}\|\theta_0 + \sqrt{d}w\|^2\|u\|^2-\frac{1}{n}\sum_{i=1}^{n}L^*(u_i \sqrt{n})
\end{align*}

Now, since $L(u)=u^2$, we have that $L^*(u) = u^2/4$, and therefore $M = 2$. Consequently, this formulation can be rewritten as
\begin{align*}
    \min_{w \in \mathcal S_w}\; \max_{u \in \mathcal S_u} \; - \frac{1}{\sqrt{n}} u^\top (\sqrt{d}A)w +\frac{1}{\sqrt{n}} u^\top z + \frac{1}{4\lambda}\|\theta_0 + \sqrt{d}w\|^2\|u\|^2-\frac{\|u\|^2}{4}.
\end{align*}

Since this problem is convex-concave, and both $w$ and $u$ live in convex, compact sets, this is in the form of a (PO) problem, and we can now apply the CGMT, and arrive at the modified (AO) problem,
\begin{align*}
    \max_{0 \leq \beta \leq K_\beta}\; \min_{w \in \mathcal S_w}\; \max_{\|u\|= \beta}\;  \frac{1}{\sqrt{n}} \|w\|g^\top u - \frac{1}{\sqrt{n}} \|u\|h^\top w +\frac{1}{\sqrt{n}} u^\top z + \frac{1}{4\lambda}\|\theta_0 + \sqrt{d}w\|^2\|u\|^2 -\frac{1}{2M}\|u\|^2.
\end{align*}

Now, as in the proof of Theorem~\ref{thm:estimationerror:3}, we can first scalarize over $u$ and subsequently over $w$,
\begin{align*}
    \max_{0 \leq \beta \leq K_\beta}\; \min_{0 \leq \alpha \leq K_\alpha}\; \frac{\beta}{\sqrt{n}}\|\alpha g +z\|-\frac{\beta^2}{2M} - \alpha \beta\left\|\frac{\beta}{2 \lambda_0} \frac{\theta_0}{\sqrt{d}} - \frac{h}{\sqrt{n}} \right\| + \frac{\beta^2}{4 \lambda_0}\left(\frac{\|\theta_0\|^2}{d} + \alpha^2\right),
\end{align*}
and subsequently use Sion's minimax principle to exchange the minimization and the maximization, and obtain
\begin{align*}
    \min_{0 \leq \alpha \leq K_\alpha}\;\max_{0 \leq \beta \leq K_\beta}\;  \frac{\beta}{\sqrt{n}}\|\alpha g +z\|-\frac{\beta^2}{2M} - \alpha \beta\left\|\frac{\beta}{2 \lambda_0} \frac{\theta_0}{\sqrt{d}} - \frac{h}{\sqrt{n}} \right\| + \frac{\beta^2}{4 \lambda_0}\left(\frac{\|\theta_0\|^2}{d} + \alpha^2\right),
\end{align*}

The optimal value of this minimax problem can be shown to converge in probability, as $d \to \infty$, to
\begin{align*}
    \beta\sqrt{\alpha^2 + \sigma_Z}-\frac{\beta^2}{2M} - \alpha \beta \sqrt{\frac{\beta^2 \sigma_{\theta_0}^2}{4 \lambda_0^2} + \rho} + \frac{\beta^2}{4 \lambda_0} \left( \sigma_{\theta_0}^2 + \alpha^2 \right).
\end{align*}
Then, following the same lines as in the proof of Theorem~\ref{thm:estimationerror:3}, the result can be concluded.
\end{proof}

\section{Technical preliminary lemmas}
\label{appendix:preliminary:results}

\begin{lemma}
\label{lemma:triangular:array}
Let $h \in \mathbb R^d$ be a random vector with entries i.i.d.\ standard normal, and $\theta_0 \in \mathbb R^d$ be a random vector, independent from $h$, which satisfies $\|\theta_0\|^2/d \overset{P}{\to} \sigma_{\theta_0}^2$, for $\sigma_{\theta_0}>0$. Then,
\begin{align*}
    \frac{\theta_0^\top h}{d} \overset{P}{\to} 0.
\end{align*}
\end{lemma}
\begin{proof}
First notice that $\mathbb E[\theta_0^\top h] = 0$ due to the zero mean of $h$, and the fact that $\theta_0$ and $h$ are independent. Now let $\mathbb V(\theta_0^\top h)$ denote the variance of $\theta_0^\top h$. Using $\mathbb E[h_i^2] = 1$, $\mathbb E[h_i h_j] = 0$, for $i \neq j$, it can be easily checked that the variance of $\theta_0^\top h$ satisfies $\mathbb{V}(\theta_0^\top h) = \mathbb E [\|\theta_0\|^2]$. Consequently, using the assumption $\|\theta_0\|^2/d \overset{P}{\to} \sigma_{\theta_0}^2$, we have that $\mathbb E [\|\theta_0\|^2]/d \to \sigma_{\theta_0}^2$, and therefore $\mathbb{V}(\theta_0^\top h)/(d^2) \to 0$. The result now follows easily from \cite[Theorem~2.2.6.]{durrett2019probability}. 
\end{proof}

\begin{lemma}[Convexity lemma]
\label{lemma:convexity:lemma}
Let $\{\mathcal O_d:\mathcal X \to \mathbb R\}_{d \in \mathbb N}$ be a sequence of random functions, defined on an open and convex subset $\mathcal X$ of $\mathbb R^d$, which are convex w.p.a.\ $1$ as $d\to \infty$. Moreover, let $\mathcal O:\mathcal X \to \mathbb R$ be a deterministic function for which $\mathcal O_d(x) \overset{P}{\to} \mathcal O(x)$, for all $x \in \mathcal X$. Then, $\mathcal O$ is convex and the convergence is uniform over each compact subset $\mathcal K$ of $\mathcal X$, i.e.,
\begin{align*}
    \sup_{x \in \mathcal K}|\mathcal O_d(x) - \mathcal O(x)| \overset{P}{\to} 0.
\end{align*}
\end{lemma}
\begin{proof}
This is a slight generalization of the standard Convexity Lemma (see \cite[Section~6]{pollard1991asymptotics}), which assumes that the functions $\{\mathcal O_d\}_{d \in \mathbb N}$ are convex (instead of convex w.p.a.\ $1$). A similar proof, which we include for completeness, can be used to conclude the result in this slightly more general case, as shown in what follows.

We start by showing that the function $\mathcal O$ is convex. If $\mathcal O$ was not convex, then there would exist $x_1,x_2 \in \mathcal X$ and $\lambda \in (0,1)$ such that $\mathcal O(\lambda x_1 + (1-\lambda)x_2) - \lambda \mathcal O(x_1) - (1-\lambda)\mathcal O(x_2) = \epsilon$, for some $\epsilon > 0$. Now, since $\mathcal O_d$ converges in probability pointwise to $\mathcal O$, we have that 
\begin{align*}
    \left|\mathcal O_d(\lambda x_1 + (1-\lambda) x_2)- \lambda \mathcal O_d(x_1) - (1-\lambda)\mathcal O_d(x_2) - \mathcal O(\lambda x_1 + (1-\lambda) x_2) + \lambda \mathcal O(x_1) + (1-\lambda)\mathcal O(x_2)\right| \leq \frac{\epsilon}{2}
\end{align*}
holds w.p.a.\ $1$. Therefore, 
\begin{align*}
    \mathcal O_d(\lambda x_1 + (1-\lambda)x_2) - \lambda \mathcal O_d(x_1) - (1-\lambda)\mathcal O_d(x_2) \geq \frac{\epsilon}{2}
\end{align*}
holds w.p.a.\ $1$. However, this contradicts the fact that $\{\mathcal O_d\}_{d \in \mathbb N}$ are convex w.p.a.\ $1$. Consequently, $\mathcal O$ is a convex function.

We now proceed to proving the uniform convergence on compact subsets. For this, it is enough to consider the case where $\mathcal K$ is a cube with edges parallel to the coordinate directions $e_1,\ldots,e_d$, since every compact subset can be covered by finitely many such cubes. Now, fix $\epsilon > 0$. Since $\mathcal O$ is convex, we know that it is continuous on the interior of its domain. Moreover, from the Heine-Cantor theorem we know that every continuous function on a compact set is uniformly continuous. Therefore, we can partition $\mathcal K$ in cubes of side $\delta$ such that that $\mathcal O$ varies by less than $\epsilon$ on each cube of side $2 \delta$. Moreover, we expand $\mathcal K$ by adding another layer of these $\delta$-cubes around each side, and we call this new cube $\mathcal K^\delta$. Notice that we may pick $\delta$ small enough such that $K^\delta \subset \mathcal X$.

Now, let $\mathcal V$ be the set of all vertices of all $\delta$-cubes in $\mathcal K^\delta$. Since $\mathcal V$ is a finite set, the convergence in probability is uniform over $\mathcal V$, i.e.,
\begin{align}
\label{eq:lemma:convexity:lemma:1}
    \max_{x \in \mathcal V} |\mathcal O_d(x) - \mathcal O(x)| \overset{P}{\to} 0.
\end{align}
Each $x \in \mathcal K$ will belong to a $\delta$-cube with vertices $\{x_i\}_{i \in \mathcal I}$ in $\mathcal K^\delta$, where $\mathcal I$ is a finite set of indices. Therefore, $x$ can be written as a convex combination of these vertices, i.e., $x = \sum_{i \in \mathcal I} \lambda_i x_i$. Since $\mathcal O_d$ is convex w.p.a.\ $1$, we have that $\mathcal O_d(x) \leq \sum_{i \in \mathcal I} \lambda_i \mathcal O_d(x_i)$ holds w.p.a.\ $1$. Therefore,
\begin{align*}
    \mathcal O_d(x) \leq \max_{i \in \mathcal I} \mathcal O_d(x_i) 
    &\leq  \max_{i \in \mathcal I} |\mathcal O_d(x_i) - \mathcal O(x_i) + \mathcal O(x_i) - \mathcal O(x) + \mathcal O(x)|
    \\ &\leq \max_{i \in \mathcal I} |\mathcal O_d(x_i) - \mathcal O(x_i)| + \max_{i \in \mathcal I} |\mathcal O(x_i) - \mathcal O(x)| + \mathcal O(x)
\end{align*}
holds w.p.a.\ $1$. Now, from \eqref{eq:lemma:convexity:lemma:1} and the fact that $\mathcal O$ varies by less than $\epsilon$ on the $\delta$-cube with vertices $\{x_i\}_{i \in \mathcal I}$, we have that the following upper bound holds w.p.a.\ $1$
\begin{align*}
    \sup_{x \in \mathcal K} \mathcal O_d(x) - \mathcal O(x) \leq 2 \epsilon.
\end{align*}

We will now prove the lower bound. Each $x \in \mathcal K$ lies into a $\delta$-cube with a vertex $x_0$, and can be written as $x = x_0 + \sum_{i =1}^d \lambda_i e_i$, with $0 \leq \lambda_i < \delta$ for all $i=1,\ldots,d$. We now define $x_i:= x_0-\lambda_i e_i$, and notice that $x_i \in \mathcal V$. Then, $\theta_0$ can be written as the convex combination $x_0 = \beta x + \sum_{i=1}^d \beta_i x_i$, with $\beta = {\delta}/\left({\delta + \sum_{j=1}^d \lambda_j}\right)$ and $\beta_i = {\lambda_i}/\left({\delta + \sum_{j=1}^d \lambda_j}\right)$. Notice that $\beta \geq 1/(d+1)$.

Since $\mathcal O_d$ is convex w.p.a.\ $1$, we have that $\mathcal O_d(x_0) \leq \beta \mathcal O_d(x) + \sum_{i=1}^d \beta_i \mathcal O_d(x_i)$ holds w.p.a.\ $1$. Therefore,
\begin{align*}
    \beta \mathcal O_d(x) &\geq \mathcal O_d(x_0) - \sum_{i=1}^d \beta_i \mathcal O_d(x_i)
    \\ &=
    \left(\mathcal O_d(x_0) - \mathcal O(x_0)\right) + \left( \mathcal O(x_0) - \sum_{i=1}^d \beta_i \mathcal O(x_i) \right) - \sum_{i=1}^d \beta_i \left(\mathcal O_d(x_i) - \mathcal O(x_i) \right)
    \\ &\geq \left(\mathcal O_d(x_0) - \mathcal O(x_0)\right) + \left( \mathcal O(x) -\epsilon - \sum_{i=1}^d \beta_i (\mathcal O(x_i)+\epsilon) \right) - \sum_{i=1}^d \beta_i \left(\mathcal O_d(x_i) - \mathcal O(x_i) \right)
\end{align*}
holds w.p.a.\ $1$. Now, from \eqref{eq:lemma:convexity:lemma:1} and the fact that $\mathcal O$ is convex, we have w.p.a.\ $1$ that
\begin{align*}
    \beta \mathcal O_d(x) \geq - \epsilon + \beta \mathcal O(x) -\epsilon - 2 \sum_{i=1}^d \beta_i \epsilon,
\end{align*}
and therefore
\begin{align*}
    \inf_{x \in \mathcal K} \mathcal O_d(x) - \mathcal O(x) \geq - \frac{4 \epsilon}{\beta} \geq - 4 (d + 1)\epsilon
\end{align*}
holds w.p.a.\ $1$. Since $d$ is finite, this concludes the proof.
\end{proof}

\begin{lemma}[Convergence of convex w.p.a.\ $1$ optimization problems]
\label{lemma:conv:opt}
Let $\{\mathcal O_d:(0,\infty) \to \mathbb R\}_{d \in \mathbb N}$ be a sequence of random functions, which are convex w.p.a.\ $1$ as $d \to \infty$. Moreover, let $\mathcal O:(0,\infty) \to \mathbb R$ be a deterministic function for which $\mathcal O_d(x) \overset{P}{\to} \mathcal O(x)$, for all $x \in \mathcal (0,\infty)$. Then, 
\begin{enumerate}
\item[(i)] for any $K_1, K_2 > 0$,
\begin{align*}
    \min_{K_1 \leq x \leq K_2} \mathcal O_d(x) \overset{P}{\to} \min_{K_1 \leq x \leq K_2} \mathcal O(x).
\end{align*}
\item[(ii)] for any $K > 0$,
\begin{align*}
    \inf_{0 < x \leq K} \mathcal O_d(x) \overset{P}{\to} \inf_{0 < x \leq K} \mathcal O(x).
\end{align*}
\item[(iii)] Moreover, if $\lim_{x \to \infty} \mathcal O(x) = +\infty$,
\begin{align*}
    \inf_{x>0} \mathcal O_d(x) \overset{P}{\to} \inf_{x>0} \mathcal O(x).
\end{align*}
\end{enumerate}
\end{lemma}
\begin{proof}
This is a generalization of \cite[Lemma~10]{thrampoulidis2018precise}, which assumes that the functions $\{\mathcal O_d\}_{d \in \mathbb N}$ are convex (instead of convex w.p.a.\ $1$). This result builds upon the uniform convergence in probability presented in the Convexity Lemma~\ref{lemma:convexity:lemma}.

\begin{itemize}
\item[(i)] The proof of Assertion (i) is a simple application of the Convexity Lemma, as explained in what follows. Let $x_\star$ and $\{x_\star^{(d)}\}_{d \in \mathbb N}$ be the minimizers of $\mathcal O$ and  $\{\mathcal O_d\}_{d \in \mathbb N}$ on $[K_1,K_2]$, respectively. Since $\mathcal O_d \overset{P}{\to} \mathcal O$ pointwise on $(0,\infty)$, from the Convexity Lemma we know that $\mathcal O_d \overset{P}{\to} \mathcal O$ uniformly on the compact set $[K_1,K_2]$, i.e., $\sup_{x \in [K_1,K_2]} |\mathcal O_d(x) - \mathcal O(x)|\leq\epsilon$ w.p.a.\ $1$. In particular, this holds true for the points $x_\star, \{x_\star^{(d)}\}_{d \in \mathbb N}$. Therefore,
\begin{align*}
    -\epsilon \leq \mathcal O_d(x_\star^{(d)}) - \mathcal O(x_\star^{(d)}) \leq \min_{K_1 \leq x \leq K_2} \mathcal O_d(x) - \min_{K_1 \leq x \leq K_2} \mathcal O(x) \leq \mathcal O_d(x_\star) - \mathcal O(x_\star) \leq \epsilon,
\end{align*}
w.p.a.\ $1$, proving Assertion (i).

\item[(ii)] First notice that $\mathcal O$ is convex by Lemma~\ref{lemma:convexity:lemma}. Now, since $\mathcal O$ is real-valued (on $(0,\infty)$) and convex, we have that $\inf_{0 < x \leq K} \mathcal O(x) > -\infty$. Moreover, since $\{\mathcal O_d\}_{d \in \mathbb N}$ are real-valued (on $(0,\infty)$) and convex w.p.a.\ $1$, we have that $\inf_{0 < x \leq K} \mathcal O_d(x) > -\infty$ w.p.a.\ $1$. We first prove this statement for $\mathcal O$. Suppose that $\inf_{0 < x \leq K} \mathcal O(x) = -\infty$. Then, there exists a sequence $\{x_n\}_{n \in \mathbb N}>0$ such that $\mathcal O(x_n) \to -\infty$ as $x_n \to 0$. Consequently, for any $x \in (0,K)$ there exists a sequence $\{\lambda_n\}_{n \in \mathbb N} \in (0,1)$, with $\lambda_n \geq (K-x)/K>0$, satisfying $x = \lambda_n x_n + (1-\lambda_n) K$. From the convexity of $\mathcal O$, we have that
\begin{align*}
    \mathcal O(x) \leq \lambda_n \mathcal O(x_n) + (1-\lambda_n) \mathcal O(K),
\end{align*}
which, for $x_n \to 0$, implies that $\mathcal O(x) = -\infty$ for all $x \in (0,K)$, leading to a contradiction (since $\mathcal O$ is assumed real-valued). This reasoning readily extends to the sequence $\{\mathcal O_d\}_{d \in \mathbb N}$, for which we have that $\inf_{0 < x \leq K} \mathcal O_d(x) > -\infty$ w.p.a.\ $1$.

In order to prove the desired result we have to show that for all $\epsilon > 0$,
\begin{align*}
    |\inf_{0 < x \leq K} \mathcal O_d(x) - \inf_{0 < x \leq K} \mathcal O(x)| \leq \epsilon
\end{align*}
w.p.a.\ $1$ as $d \to \infty$. We start by proving the upper bound $\inf_{0 < x \leq K} \mathcal O_d(x) - \inf_{0 < x \leq K} \mathcal O(x) \leq \epsilon$. Let $\bar{x}>0$ be such that $\mathcal O(\bar{x})-\inf_{0 < x \leq K} \mathcal O(x) \leq \epsilon/2$. From the pointwise convergence of $\mathcal O_d$ to $\mathcal O$ we have that $|O_d(\bar{x}) - \mathcal O(\bar{x})|\leq\epsilon/2$ w.p.a.\ $1$. Consequently,
\begin{align*}
    \inf_{0 < x \leq K} \mathcal O_d(x) - \inf_{0 < x \leq K} \mathcal O(x) \leq \mathcal O_d(\bar{x}) - \mathcal O(\bar{x}) + \frac{\epsilon}{2} \leq \epsilon,
\end{align*}
w.p.a.\ $1$, showing that the upper bound holds.

Now, we focus on proving the lower bound $\inf_{0 < x \leq K} \mathcal O_d(x) - \inf_{0 < x \leq K} \mathcal O(x) > -\epsilon$, and we will do this by considering two cases. First, we consider the case where all the infimums $\inf_{0 < x \leq K} \mathcal O_d(x)$ are attained, and the minimizers are lower bounded by some $x_0>0$, and show that the result immediately follow from the Convexity Lemma~\ref{lemma:convexity:lemma}. Let $\{x_\star^{(d)}\}_{d \in \mathbb N}>0$ be the minimizers of $\{\mathcal O_d\}_{d \in \mathbb N}$ on $(0,K]$, and suppose that they belong to $[x_0,K]$. Now, since $\mathcal O_d \overset{P}{\to} \mathcal O$ pointwise on $(0,\infty)$, from the Convexity Lemma we know that $\mathcal O_d \overset{P}{\to} \mathcal O$ uniformly on the compact set $[x_0,K]$, i.e., $\sup_{x \in [x_0,K]} |\mathcal O_d(x) - \mathcal O(x)|\leq\epsilon$ w.p.a.\ $1$. In particular, this holds true for the points $\{x_\star^{(d)}\}_{d \in \mathbb N}$. Therefore,
\begin{align*}
    \inf_{0 < x \leq K} \mathcal O_d(x) - \inf_{0 < x \leq K} \mathcal O(x) \geq \mathcal O_d(x_\star^{(d)})-\mathcal O(x_\star^{(d)}) \geq -\epsilon
\end{align*}
w.p.a.\ $1$, which concludes the proof in this case.

Secondly, we consider the case where the infimums are attained in the limit of $x \to 0$. Let $\bar{x}^{(d)}>0$ be such that $\mathcal O_d(\bar{x}^{(d)})- \inf_{0 < x \leq K} \mathcal O_d(x) < \epsilon/3$. Moreover, let $\bar{x}_2>\bar{x}_1>\bar{x}^{(d)}$ be such that $\mathcal O(\bar{x}_1) - \inf_{0 < x \leq K} \mathcal O(x)<\epsilon/3$ and $\mathcal O(\bar{x}_2) - \inf_{0 < x \leq K} \mathcal O(x)<\epsilon/3$. Then, w.p.a.\ $1$, we have
\begin{align*}
    \inf_{0 < x \leq K} \mathcal O_d(x) > \mathcal O_d(\bar{x}^{(d)}) - \frac{\epsilon}{3} &\geq \frac{1}{\lambda_d} \mathcal O_d(\bar{x}_1) - \frac{1-\lambda_d}{\lambda_d} \mathcal O_d(\bar{x}_2)- \frac{\epsilon}{3}\\ &> \frac{1}{\lambda_d} \left(\mathcal O(\bar{x}_1) -\frac{\epsilon}{3}\right) - \frac{1-\lambda_d}{\lambda_d} \left(\mathcal O(\bar{x}_2) + \frac{\epsilon}{3}\right) - \frac{\epsilon}{3} \\ &> \frac{1}{\lambda_d} \left(\inf_{0 < x \leq K} \mathcal O(x) -\frac{2\epsilon}{3}\right) - \frac{1-\lambda_d}{\lambda_d} \left(\inf_{0 < x \leq K} \mathcal O(x) + \frac{2\epsilon}{3}\right) - \frac{\epsilon}{3} \\ &= \inf_{0 < x \leq K} \mathcal O(x) - \epsilon,
\end{align*}
proving the desired lower bound. In particular, the second inequality holds w.p.a.\ $1$, and follows from the convexity w.p.a.\ $1$ of $\mathcal O_d$, with $\lambda_d > (\bar{x}_2-\bar{x}_1)/\bar{x}_2$, while the third inequality follows from the pointwise convergence of $\mathcal O_d$ to $\mathcal O$, and it holds w.p.a.\ $1$. This concludes the proof.

\item[(iii)] Since $\mathcal O$ is real-valued, there exists $C>0$ such that $\inf_{x>0} \mathcal O(x) \leq C$. Moreover, since $\lim_{x \to \infty} \mathcal O(x) = +\infty$, there exists $K_1 >0$ such that $\mathcal O(x) > C$ for all $x > K_1$. Therefore, 
\begin{align*}
    \inf_{x>0} \mathcal O(x) = \inf_{0<x \leq K_1} \mathcal O(x).
\end{align*}
Now, fix $\epsilon > 0$, and let $K_2>K_1$ be such that $\mathcal O(K_2) > \mathcal O(K_1) + 3\epsilon$. Notice that this is possible since $\lim_{x \to \infty} \mathcal O(x) = +\infty$. From the pointwise convergence in probability, we have that both $|\mathcal O_d(K_1)-\mathcal O(K_1)| \leq \epsilon$ and $|\mathcal O_d(K_2)-\mathcal O(K_2)| \leq \epsilon$ hold w.p.a.\ $1$. Therefore, $\mathcal O_d(K_1) < \mathcal O_d(K_2)$ holds w.p.a.\ $1$. Now, using this and the fact that $\mathcal O_d$ is convex w.p.a\ $1$, we can conclude that $\mathcal O_d(x) > \mathcal O_d(K_2)$ w.p.a\ $1$, for all $x > K_2$. Consequently, 
\begin{align*}
    \inf_{x>0} \mathcal O_d(x) = \inf_{0<x \leq K_2} \mathcal O_d(x).
\end{align*}
holds w.p.a.\ $1$. Finally, letting $K = \max\{K_1,K_2\}$, we have that
\begin{align*}
    \inf_{x>0} \mathcal O_d(x) = \inf_{0<x \leq K} \mathcal O_d(x) \overset{P}{\to} \inf_{0<x \leq K} \mathcal O(x) = \inf_{x>0} \mathcal O(x),
\end{align*}
where the convergence in probability follows from Assertion~(ii). This concludes the proof.
\end{itemize}
\end{proof}

\end{document}